%% file: main.tex
\documentclass{article}

\PassOptionsToPackage{square, numbers, compress}{natbib}

\usepackage[final]{neurips_2022}

\usepackage[utf8]{inputenc} 
\usepackage[T1]{fontenc}    
\usepackage{url}            
\usepackage{booktabs}       
\usepackage{amsfonts}       
\usepackage{nicefrac}       
\usepackage{microtype}      
\usepackage{xcolor}         

\usepackage[colorlinks=true, allcolors=blue, pdfencoding=auto, psdextra, linktocpage]{hyperref} 

\usepackage{amsmath}
\usepackage{amssymb}
\usepackage{mathtools}
\usepackage{amsthm}

\usepackage[capitalize,noabbrev]{cleveref}

\theoremstyle{plain}
\newtheorem{theorem}{Theorem}[section]

\newtheorem{lemma}[theorem]{Lemma}
\newtheorem{corollary}[theorem]{Corollary}
\theoremstyle{definition}
\newtheorem{definition}[theorem]{Definition}
\newtheorem{assumption}[theorem]{Assumption}

\newtheorem{remark}[theorem]{Remark}

\usepackage{mycustom}
\usepackage[utf8]{inputenc} 
\usepackage[T1]{fontenc}    

\usepackage{alphabeta}

\usepackage{xspace}         

\usepackage{mathrsfs}

\usepackage{xfrac}

\usepackage{graphicx}
\usepackage{comment}
\usepackage{bbm}
\usepackage{textcomp}
\usepackage{wrapfig}
\usepackage{caption}
\usepackage{subcaption}
\usepackage{float}

\usepackage{pgfplots}
\pgfplotsset{compat=1.8}
\tikzset{elegant/.style={smooth,thick,samples=500,magenta}}

\usepackage[shortlabels]{enumitem}
\setlist[enumerate,1]{leftmargin=0.6cm}

\theoremstyle{definition}

\crefname{assumption}{Assumption}{Assumptions}

\input{math_commands}
\input{symbols}

\title{
 Understanding the Generalization Benefit of Normalization Layers: 
Sharpness Reduction
}

\begin{document}

\author{%
	Kaifeng Lyu \qquad Zhiyuan Li \qquad Sanjeev Arora\\
	Department of Computer Science\\
	Princeton University\\
	\texttt{\{klyu,zhiyuanli,arora\}@cs.princeton.edu}
}

\maketitle

\begin{abstract}
Normalization layers (e.g., Batch Normalization, Layer Normalization) were introduced
to help with optimization difficulties in very deep nets,
but they clearly also help generalization,
even in not-so-deep nets.
Motivated by the long-held belief that 
flatter minima lead to better generalization,
this paper gives mathematical analysis
and supporting experiments suggesting that normalization
(together with accompanying weight-decay)
encourages GD to reduce the sharpness of loss surface.
Here ``sharpness'' is carefully defined given that the loss is scale-invariant, a known consequence of normalization.
Specifically, for a fairly broad class of neural nets with normalization,
our theory explains how GD with a finite learning rate enters the so-called Edge of Stability (EoS) regime,
and characterizes the trajectory of GD in this regime via a continuous sharpness-reduction flow.
\end{abstract}

\section{Introduction}
Training modern deep neural nets crucially relies on normalization layers to
make the training process less sensitive to hyperparameters and initialization.
The two of the most popular normalization layers are Batch Normalization
(BN)~\citep{ioffe2015batch} for vision tasks and Layer Normalization
(LN)~\citep{ba2016layer} for language tasks. Recent works also proposed other
normalization layers aiming for better performance, most notably including Group
Normalization (GN)~\citep{wu2018group}, Weight Normalization
(WN)~\citep{salimans2016weight}, Scaled Weight
Standardization (SWS)~\citep{qiao2019micro,huang2017centered,brock2021characterizing},~etc.
Most normalization layers amount to a reparametrization of the neural net so
that the loss becomes invariant to the scale of most parameters, and with a
minor change, to {\em all} parameters: $\Loss(c\vw) = \Loss(\vw)$ for all
scalings $c>0$~\citep{ioffe2015batch,arora2018theoretical,li2020exp}. The
current paper assumes this scale-invariance for all parameters and analyzes the
trajectory of gradient descent with \textit{weight decay}~(WD):
\begin{equation} \label{eq:upd-gdwd}
  \vw_{t+1} \gets (1 - \heta \hlam)\vw_t - \heta \nabla \Loss(\vw_t).
\end{equation}
The use of WD is a common practice that has been adopted in training
state-of-the-art neural nets, such as
ResNets~\citep{he2016deepres,he2016identityres} and
Transformers~\citep{devlin2019bert,brown2020gpt3}. Previous ablation studies
showed that adding WD to normalized nets indeed leads to better
generalization~\citep{zhang2018three,lewkowycz2020training,zhang2017rethinking}.
More notably, \citet{liu2020bad} conducted experiments of training ResNets
initialized from global minima with poor test accuracy, and showed that SGD with WD escapes from those bad global minima and attains good test
accuracy. In contrast, training with vanilla SGD yields significant generalization
degradation.

In the traditional view, WD
regularizes the model by penalizing the parameter norm, but this may appear
nonsensical for scale-invariant loss because one can scale down the norm
arbitrarily without changing the loss value. However, the scale of the parameter
\textit{does} matter in backward propagation, and thus WD can affect the
training dynamics. In particular, simple calculus shows $\nabla \Loss(\vw) =
\frac{1}{\normtwosm{\vw}} \nabla \Loss(\frac{\vw}{\normtwosm{\vw}}) \propto
\frac{1}{\normtwosm{\vw}}$ and $\nabla^2 \Loss(\vw) =
\frac{1}{\normtwosm{\vw}^2} \nabla^2 \Loss(\frac{\vw}{\normtwosm{\vw}}) \propto
\frac{1}{\normtwosm{\vw}^2}$, so WD is in effect trying to enlarge the gradient
and Hessian in training. This makes the training dynamics very different from
unnormalized nets and requires revisiting classical convergence
analyses~\citep{li2020exp,li2020reconciling,lobacheva2021on,li2022robust}.


The current paper aims to improve mathematical understanding of how
normalization improves generalization. While this may arise from many places, we focus on
studying the dynamics of (full-batch) GD~\eqref{eq:upd-gdwd},
which is a
necessary first step towards understanding SGD. We show that the interplay
between normalization and WD provably induces an implicit bias to persistently
reduce the \emph{sharpness} of the local loss landscape during the training
process, which we call the \textit{sharpness-reduction bias}.

\begin{figure}[t] 
  \includegraphics[width=\textwidth]{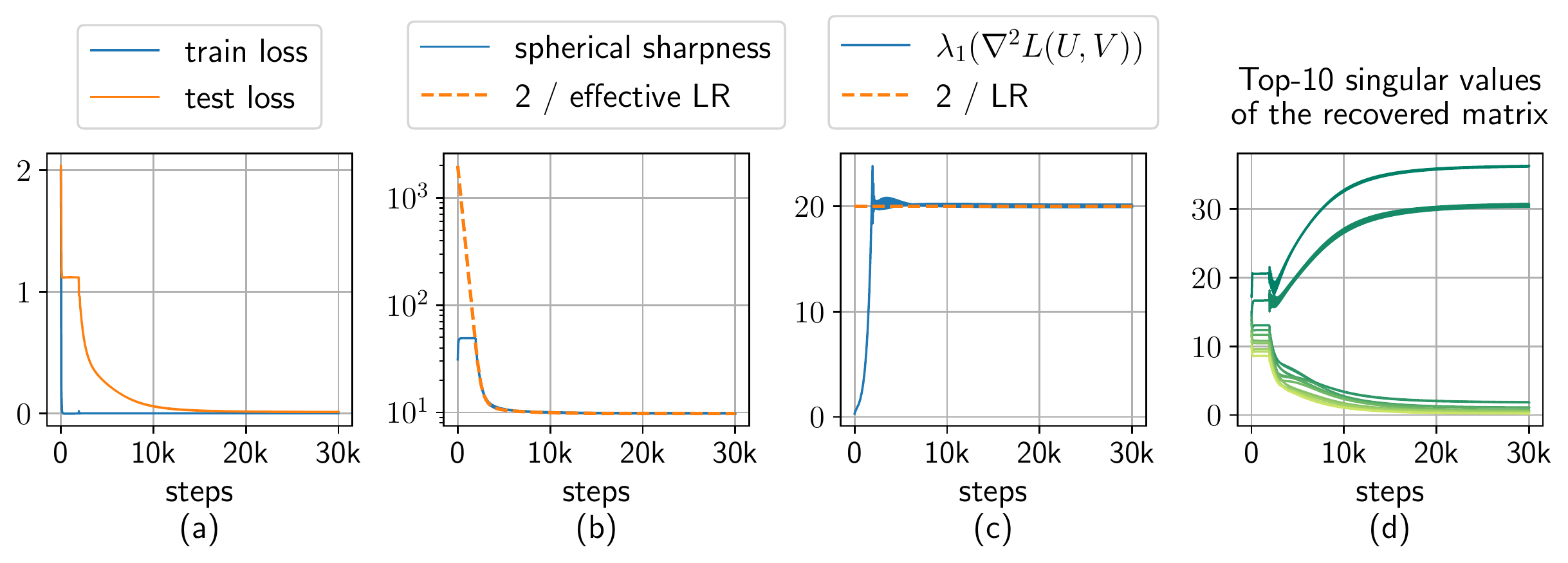}
  \vspace{-0.25in}
  \caption{
    \small
    Experiment on overparameterized matrix completion with Batch Normalization.
    Given $800$ ($32\%$) entries $\Omega$ of a rank-$2$ matrix $\mM \in \R^{50 \times 50}$,
    use \GDWD to optimize the loss $\Loss(\mU, \mV) := \frac{1}{\abssm{\Omega}} \sum_{(i, j) \in \Omega} (\BN([\mU \mV^\top]_{i, j}) - M_{i,j})^2$,
    where $\mU, \mV \in \R^{50 \times 50}$ (thus no explicit constraint on rank).
    Starting from step $\sim$ 2k, spherical sharpness drops significantly (\textbf{b}), which
    encourages low-rank~(\textbf{d}) and
    causes the test loss (MSE of all entries)
    to decrease from $1.12$ to $0.013$ (\textbf{a}). See also \Cref{sec:exp-matlin}. 
  }
  \vspace{-0.15in}
  \label{fig:matcom-intro}
\end{figure}

\begin{figure}[t] 
  \includegraphics[width=\textwidth]{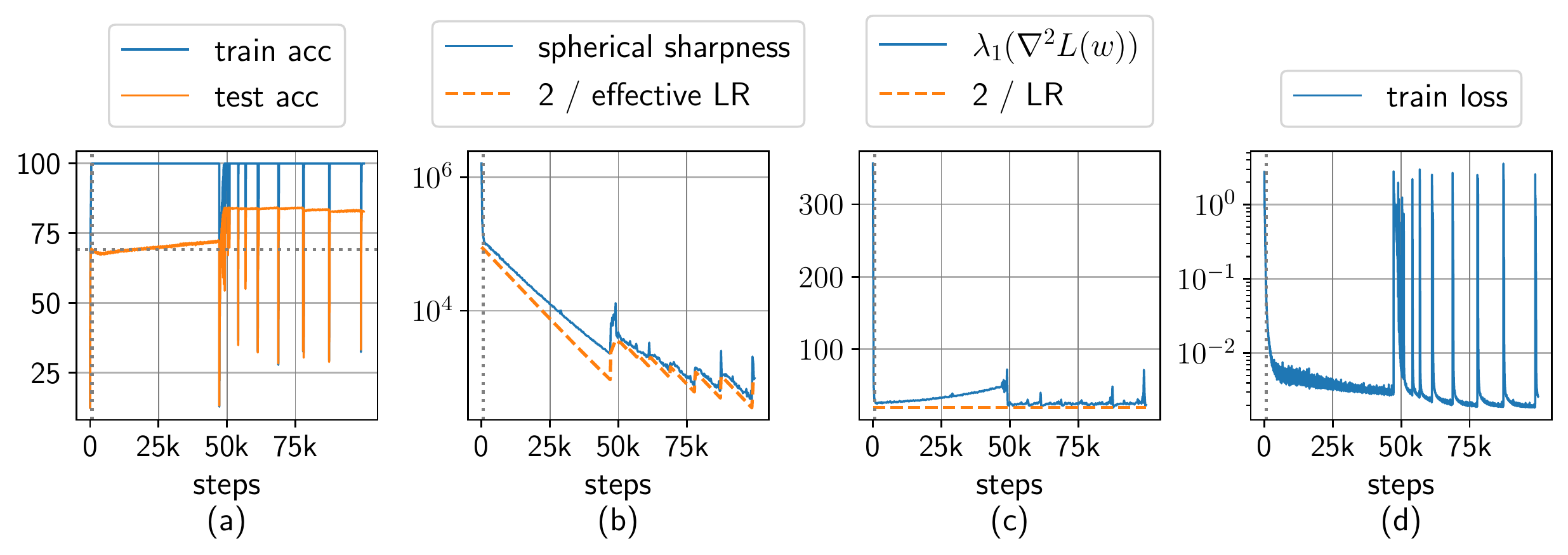}
  \vspace{-0.25in}
  \caption{
    \small
    In training a smooth and scale-invariant VGG-11 on CIFAR-10 with (full-batch) \GDWD,
    the spherical sharpness keeps decreasing and the test accuracy keeps increasing.
    BN is added after every linear layer to ensure scale-invariance.
    $100\%$ training accuracy is achieved after $\sim 680$ steps (dotted line),
    but as the training continues for $\sim$ 47k steps,
    the spherical sharpness keeps decreasing~(\textbf{b})
    and the test accuracy increases from $69.1\%$ to $72.0\%$~(\textbf{a}).
    Then the training exhibits destabilization but the test accuracy is further boosted to $84.3\%$.
    Removing either of BN or WD eliminates this phenomenon;
    see \Cref{sec:exp-ab-norm,sec:exp-ab-wd}.
  }
  \label{fig:full-cifar}
  \vspace{-0.2in}
\end{figure}

It is long believed that flatter minima generalize better~\citep{hochreiter1997flat,keskar2017large,neyshabur2017exploring},
but the notion of sharpness/flatness makes sense only if it is carefully defined
in consideration of various symmetries in neural nets.
One of the most straightforward measures of sharpness is
the maximum eigenvalue of Hessian, namely $\lambda_1(\nabla^2 \Loss(\vw_t))$.
But for normalized nets,
this sharpness measure
is vulnerable to weight rescaling, because one can scale the weight norm
to make a minimizer arbitrarily flat~\citep{dinh2017sharp}.
Also, this sharpness measure may not decrease with the number of training steps:
an empirical study by \citet{cohen2021gradient}
shows that for various neural nets (including normalized nets),
GD has an overwhelming tendency to persistently increase $\lambda_1(\nabla^2 \Loss(\vw_t))$
until it reaches the \textit{Edge of Stability (EoS) regime},
a regime where $\lambda_1(\nabla^2 \Loss(\vw_t))$ stays around $2 / \heta$ ($\heta$ is the learning rate).
See also \Cref{sec:discussion} and Figure \hyperref[fig:full-cifar]{\ref*{fig:full-cifar}c}.

\vspace{-0.08in}
\subsection{Our Contributions}

\vspace{-0.06in}
The sharpness measure we use in this paper takes care of the scale-invariance in
normalized nets. We are motivated by our experiments on matrix completion (with
BN) and CIFAR-10, where our sharpness measure decreases as the training
proceeds, and the generalization improves accordingly; see
\Cref{fig:matcom-intro,fig:full-cifar}. We note that techniques from previous
works~\citep{mcallester2003simplified,neyshabur2017exploring,foret2021sam} can
be easily adopted here to establish a PAC-Bayes bound on the test error, where
our sharpness measure appears as an additive term (see \Cref{sec:pac-bayes}).

\begin{definition}[Spherical Sharpness] \label{def:sph-sha}
  For a scale-invariant loss $\Loss(\vw)$ (i.e., $\Loss(c\vw) = \Loss(\vw)$ for all $c > 0$),
  the spherical sharpness at $\vw \in \R^D$ is defined by $\lambda_1(\nabla^2 \Loss(\frac{\vw}{\normtwosm{\vw}}))$,
  the maximum eigenvalue of the Hessian matrix after projecting $\vw$ onto the unit sphere.
\end{definition}
Based on \Cref{def:sph-sha}, we study the aforementioned sharpness-reduction
bias in training normalized nets with \GDWD (defined in~\eqref{eq:upd-gdwd}).
For constant learning rate $\heta$ and weight decay $\hlam$, we can rewrite this
rule equivalently as Projected Gradient Descent (PGD) on the unit sphere with
\emph{adaptive} learning rates, $\vtheta_{t+1} \gets \Pi(\vtheta_t - \eeta_t
\nabla \Loss(\vtheta_t))$, where $\vtheta_t := \frac{\vw_t}{\normtwosm{\vw_t}}$
is the direction of $\vw_t$, and $\eeta_t$ is the ``effective'' learning rate at
step $t$ (see \Cref{lm:scale-invariant-on-sphere}). We call $\eeta_t$ adaptive
because it can be shown to resemble 
the behaviors
of adaptive gradient methods (e.g., RMSprop~\citep{hinton2012rmsprop}):
$\eeta_t$ increases when gradient is small and
decreases when gradient is large (\Cref{fig:ex}). Our main
contributions are as follows:
\begin{enumerate}
  \item After $\vtheta_t$ reaches a point near the manifold of minimizers of $\Loss$, 
  we theoretically show that the
  effective learning rate $\eeta_t$ increases until
  GD enters a regime where 
  $2 / \eeta_t$ roughly equals to the spherical sharpness (or equivalently $2/\heta \approx \lambda_1(\nabla^2 \Loss(\vw_t))$),
  namely the EoS regime (\Cref{sec:main-result-1}).
  \item In the EoS regime, we show that for GD with a small (but finite) learning rate,
  $\vtheta_t$ oscillates around the manifold and moves approximately along a sharpness-reduction flow,
  which is a gradient flow for minimizing spherical sharpness on the manifold (with gradient-dependent learning rate) (\Cref{sec:main-result-2}).
  \item As an application of our theory,
  we show that for linear regression with BN, \GDWD finds the minimizer
  that corresponds to the linear model with minimum weight norm,
  which
  looks surprisingly the same as the conventional effect of WD
  but is achieved through the completely different sharpness-reduction mechanism
  (\Cref{sec:lin}).
  \item We experimentally verified the sharpness-reduction phenomenon predicted
  by our theorem and its benefits to generalization on CIFAR-10 with VGG-11 and
  ResNet-20, as well as matrix completion with BN (\Cref{sec:add-exp}).
  \item We generalize our theoretical results of sharpness-reduction bias to a broader class of adaptive gradient methods,
  most notably a variant of RMSprop with scalar learning rate (\Cref{sec:qrms}).
\end{enumerate}

\myparagraph{Technical Contribution.}
Our proof technique is novel and may have independent interest to the ML
community. The main challenge is that we need to analyze the implicit bias of GD
in the EoS regime which crucially relies on step size being finite --- this is
in sharp contrast to many previous works on implicit bias of
GD~\citep{soudry2018iclrImplicit,soudry2018implicit,lyu2020gradient,ji2020directional,gunasekar2018implicitconv,
gunasekar2018implicit,li2018algorithmic,razin2020implicit,
arora2019implicit,chizat2020implicit,li2021towards,
lyu2021gradient,razin2022implicit,stoger2021small,ge2021understanding} where the
same bias exists at infinitesimal LR. Our analysis is inspired by a previous
line of works~\citep{blanc2020ou,damian2021label,li2022what} showing that label
noise can drive SGD to move on the minimizer manifold along the direction of
minimizing the trace of Hessian. We borrow a few lemmas from those analyses, but
the overall proof strategy is very different because our setting does not even
have any stochastic gradient noise. Instead, we connect the dynamics in the EoS
regime to power methods and show that GD oscillates around the minimizer
manifold. This oscillation then becomes a driving power that pushes the
parameter to move on the manifold. Finally, we analyze the speed of this
movement by modeling two key parameters of the dynamics as a 1-dimensional
Hamiltonian system (\Cref{fig:qrms-vis}). To the best of our knowledge, we are
the first to provide theoretical proof for a sharpness measure to decrease during
the standard GD training, without any additional regularization (e.g., label noise~\citep{blanc2020ou,damian2021label,li2022what})
and without involving uncommon variants of GD (e.g., normalized GD or non-smooth
wrappings on the loss function~\citep{arora2022understanding}).

\input{related}

\input{prelim}

\input{scale-invariant}

\input{applications}

\input{discussion}

\vspace{-0.04in}
\section{Conclusions and Future Work} \label{sec:con}
We exhibited settings  where 
gradient descent has an implicit bias
to reduce spherical sharpness
in training neural nets with normalization layers and weight decay,
and we verified experimentally this sharpness-reduction bias predicted by our theorem
as well as its generalization benefit  on CIFAR-10.

Our theoretical analysis applies to dynamics around a minimizer manifold
and requires a small (but finite) learning rate
so that we can show that the parameter oscillates locally
and
approximately tracks a sharpness-reduction flow.
We note that in practice a decrease in 
 spherical sharpness is observed even with
 moderate LR and even before getting close to a minimizer manifold.
Explaining these phenomena is left for future work.
Now we list some other future directions. The first is to
generalize our results to SGD,
where the sharpness measure may not be the spherical sharpness
and could depend on the structure of gradient noise.
Second, to understand the benefit of reducing spherical sharpness on specific tasks,
e.g., why does reducing spherical sharpness encourage low-rank
on matrix completion with BN (\Cref{fig:matcom-intro})?
Third, to study sharpness-reduction bias for neural net architectures that are not scale-invariant on all parameters (e.g., with certain unnormalized layers).

\vspace{-0.04in}
\section*{Acknowledgements}
This work is funded by NSF, ONR, Simons Foundation, DARPA and SRC. 
ZL is also supported by Microsoft Research
PhD Fellowship.

\newpage
\bibliography{main}



\newpage
\appendix

\tableofcontents

\newpage

\input{add-related}

\newpage

\input{connection}

\newpage

\input{pac-bayes}

\newpage

\input{add-prelim}

\input{simple-ex}

\input{app-enter-eos}

\newpage

\input{proof-sketch}

\input{working-zone}

\input{gd-lemma}

\input{lr-lemma}

\input{full-space-proof}

\input{full-space-alignment-proof}

\input{full-space-drifting-proof}

\input{spherical-proof}

\input{rmsdrift-proof}

\newpage

\input{app-applications}

\newpage

\input{add-exp}

\end{document}

%% file: math_commands.tex
\usepackage{amsmath,amsfonts,bm}

\def\1{\bm{1}}

\def\vzero{{\bm{0}}}

\def\vmu{{\bm{\mu}}}
\def\vepsilon{{\bm{\epsilon}}}

\def\vtheta{{\bm{\theta}}}

\def\vphi{{\bm{\phi}}}

\def\vpsi{{\bm{\psi}}}
\def\vxi{{\bm{\xi}}}
\def\vzeta{{\bm{\zeta}}}

\makeatletter
\newcommand{\ve}{\@ifnextchar\bgroup{\velong}{{\bm{e}}}}
\newcommand{\velong}[1]{{\bm{#1}}}
\makeatother

\def\vg{{\bm{g}}}

\def\vu{{\bm{u}}}
\def\vv{{\bm{v}}}
\def\vw{{\bm{w}}}
\def\vx{{\bm{x}}}

\def\vz{{\bm{z}}}

\def\mH{{\bm{H}}}
\def\mI{{\bm{I}}}

\def\mM{{\bm{M}}}

\def\mP{{\bm{P}}}

\def\mU{{\bm{U}}}
\def\mV{{\bm{V}}}
\def\mW{{\bm{W}}}

\def\mSigma{{\bm{\Sigma}}}

\DeclareMathAlphabet{\mathsfit}{\encodingdefault}{\sfdefault}{m}{sl}
\SetMathAlphabet{\mathsfit}{bold}{\encodingdefault}{\sfdefault}{bx}{n}

\def\gM{{\mathcal{M}}}
\def\gN{{\mathcal{N}}}

\def\gP{{\mathcal{P}}}

\def\gZ{{\mathcal{Z}}}

\newcommand{\E}{\mathbb{E}}
\newcommand{\R}{\mathbb{R}}

\newcommand{\KL}{D_{\mathrm{KL}}}

\newcommand{\normltwo}{L^2}

\DeclareMathOperator*{\argmin}{arg\,min}

\DeclareMathOperator{\Tr}{Tr}

%% file: symbols.tex
\newcommand{\poly}{\mathrm{poly}}
\newcommand{\polylog}{\mathrm{polylog}}

\newcommand{\diag}{\mathrm{diag}}
\newcommand{\rank}{\mathrm{rank}}

\newcommand{\BN}{\mathrm{BN}}

\newcommand{\GDWD}{GD+WD\xspace}
\newcommand{\SGDWD}{SGD+WD\xspace}

\newcommand{\lrS}{\mathcal{H}}
\newcommand{\rmsS}{\mathcal{H}_{\text{RMS}}}
\newcommand{\gwsiS}{\mathcal{H}_{\text{GWSI}}}
\newcommand{\qrmsS}{\mathcal{H}_{\text{QRMS}}}

\newcommand{\dotp}[2]{\left<#1, #2\right>}
\newcommand{\dotpsm}[2]{\langle #1, #2\rangle}

\newcommand{\barg}{\bar{g}}

\newcommand{\hath}{\hat{h}}

\newcommand{\hatu}{\hat{u}}
\newcommand{\hatg}{\hat{g}}

\newcommand{\hatvx}{\hat{\vx}}

\newcommand{\tilvx}{\tilde{\vx}}

\newcommand{\tilv}{\tilde{v}}

\newcommand{\tilvw}{\tilde{\vw}}
\newcommand{\tilb}{\tilde{b}}

\newcommand{\vwgt}{\vw_{\mathrm{GT}}}
\newcommand{\bgt}{b_{\mathrm{GT}}}

\newcommand{\Normal}{\mathcal{N}}

\newcommand{\Loss}{\mathcal{L}}

\newcommand{\DatS}{\mathcal{S}} 
\newcommand{\DatZ}{\mathcal{Z}} 

\newcommand{\lamH}{\lambda^{\mathrm{H}}}
\newcommand{\vvH}{\vv^{\mathrm{H}}}

\newcommand{\mPHnzt}{\mP^{\mathrm{H}}_{\mathtt{NZT}}} 
\newcommand{\mPHz}{\mP^{\mathrm{H}}_{0}} 

\newcommand{\gamin}{\gamma_{\min}}

\newcommand{\aht}{\abssm{h_t}}
\newcommand{\ahtc}{\abssm{h_t}^3}
\newcommand{\ahte}{\abssm{h_t}\eta}

\newcommand{\lammax}{\lambda_{\max}}

\newcommand{\abs}[1]{\left\lvert #1 \right\rvert}
\newcommand{\abssm}[1]{\lvert #1 \rvert}

\newcommand{\norm}[1]{\left\| #1 \right\|}
\newcommand{\normsm}[1]{\| #1 \|}
\newcommand{\normtwo}[1]{\norm{#1}_2}
\newcommand{\normtwosm}[1]{\normsm{#1}_2}

\newcommand{\normFsm}[1]{\| #1 \|_{\mathrm{F}}}

\newcommand{\Obig}[1]{O\!\left(#1\right)}
\newcommand{\Osm}[1]{O(#1)}

\newcommand{\contC}{\mathcal{C}}
\newcommand{\sphS}{\mathbb{S}}

\newcommand{\Sta}{\mathcal{S}}

\newcommand{\manGa}{\mathit{\Gamma}}

\newcommand{\DGa}{D_\Gamma}
\newcommand{\gradGa}{\nabla_{\Gamma}}

\newcommand{\TGa}[2][\manGa]{\mathsf{T}_{#2}(#1)}
\newcommand{\NGa}[2][\manGa]{\mathsf{N}_{#2}(#1)}

\newcommand{\onec}[1]{\mathbbm{1}_{[#1]}}

\newcommand{\dd}{\textup{\textrm{d}}}

\newcommand{\alphamax}{\alpha_{\max}}

\newcommand{\vmux}{\vmu_{\mathrm{x}}}

\newcommand{\mSigmax}{\mSigma_{\mathrm{x}}}

\newcommand{\muy}{\mu_{\mathrm{y}}}
\newcommand{\sigmay}{\sigma_{\mathrm{y}}}

\newcommand{\vwopt}{\vw^*}

\newcommand{\vthetaopt}{\vtheta^*}

\newcommand{\hatvtheta}{\hat{\vtheta}}

\newcommand{\eeta}{\tilde{\eta}}
\newcommand{\heta}{\hat{\eta}}
\newcommand{\ieta}{\eta_{\mathrm{in}}}
\newcommand{\hlam}{\hat{\lambda}}

\newcommand{\muPL}{\mu_{\mathtt{PL}}}
\newcommand{\Cb}{C_{\mathrm{b}}}

\newcommand{\tildeLoss}{\tilde{\Loss}}

\newcommand{\Ball}[2]{B^{#1}(#2)}

\newcommand{\Zeis}{\gZ^{\epsilon_0} \cap \sphS^{D-1}}
\newcommand{\Zeig}{\gZ^{\epsilon_0} \cap \manGa}
\newcommand{\Zei}{\gZ^{\epsilon_0}}
\newcommand{\Zeos}{\gZ^{\epsilon_1} \cap \sphS^{D-1}}
\newcommand{\Zeog}{\gZ^{\epsilon_1} \cap \manGa}
\newcommand{\Zeo}{\gZ^{\epsilon_1}}

\newcommand{\cl}{\mathrm{cl}}

\newcommand{\PolyakLojasiewicz}{Polyak-\L{}ojasiewicz\xspace}

\newcommand{\thirdL}[1]{\rho_3(#1)}

%% file: related.tex
\vspace{-0.08in}
\section{Related Works} \label{sec:related}

\vspace{-0.04in}
\myparagraph{Sharpness and Generalization.}  It has been long believed that
flat minima generalize better~\citep{hochreiter1997flat}.
Several empirical studies~\citep{keskar2017large,li2018visualizing,wu2017towards,jastrzkebski2017three} verified
the positive correlation between flatness and generalization.
\citet{neyshabur2017exploring} justified this via PAC-Bayes
theory~\citep{mcallester2003simplified}. 
Several other theoretical papers explored the
generalization properties of flat minima specifically for two-layer
nets~\citep{blanc2020ou,mulayoff2021minimastability,haochen2021shape,li2022what,ding2022flat}
and deep linear nets~\citep{mulayoff2020unique}.
\citet{jiang2020fantastic}
conducted extensive experiments for all existing generalization measures to
evaluate their correlation and causal relationships with generalization error,
concluding that sharpness-based measures perform the best overall. In light of
this, \citet{foret2021sam} proposed SAM algorithm to improve the generalization
by minimizing the sharpness. Despite so many positive results on sharpness-based measures, 
a common issue of many works is that the measures may suffer from sensitivity to rescaling
of parameters in deep nets~\citep{dinh2017sharp}. Another issue is that the
minima could lie in asymmetric valleys that are flat on one side and sharp on
the other~\citep{he2019asymmetric}.

\myparagraph{Understanding Normalization Layers.} 
The benefits of normalization layers can be shown in various
aspects. A series of works studied the forward propagation of deep nets at
random initialization, showing that normalization layers stabilize the growth of
intermediate layer outputs with
depth~\citep{brock2021characterizing,balduzzi2017shattered,de2020identity},
provably avoid rank collapse~\citep{daneshmand2020rankcollapse} and
orthogonalize representations~\citep{daneshmand2021batch}. Although these works
mainly focused on BN~\citep{ioffe2015batch},
\citet{lubana2021beyond,labatie2021proxy} provided thorough discussions on the
applicability of these arguments to other normalization layers. It is also
believed that BN has a unique regularization effect through the noise in batch
statistics~\citep{luo2019towards,teye2018bayesian,shekhovtsov2019stochastic}.
Several other works argued that normalization layers lead to a smoothening or
preconditioning effect of the loss
landscape~\citep{santurkar2018does,bjorck2018understanding,chaudhuri2019investigation,karakida2019normalization,lin2021spectral,lange2022preconditioning},
which may help optimization. By analyzing the training dynamics,
\citet{arora2018theoretical} rigorously 
proved that normalization yields an
auto-tuning effect of the effective learning rate $\eeta_t$, which makes the
asymptotic speed of optimization much less sensitive to the learning rate and initialization.
In linear regression settings, 
\citet{cai2019quantitative,kohler2019exp} 
showed that training with BN leads to a faster convergence rate; \citet{wu2020implicit} studied the implicit regularization effect of
WN~\citep{salimans2016weight}. For two-layer nets with normalization, \citet{ma2021riemannian}
derived a mean-field formulation of the training dynamics; \citet{dukler2020optimization} proved a convergence rate via NTK-based analysis.
The current paper focuses on the interplay
between normalization and WD during training,
whereas all the above works either do not analyze the dynamics or assume no WD.

\myparagraph{Interplay Between Normalization and WD.} 
A common feature of normalization layers (including but not limited to BN, WN,
LN, GN, SWS) is that they make the loss invariant to the scale of layer weights.
In presence of both scale-invariance and WD, training dynamics can go out of the
scope of the classical optimization theory, e.g., one can train the net to small
loss even with learning rates exponentially increasing~\citep{li2020exp}. A series of works
investigated into the interplay between normalization and WD and argued that the
training dynamic with SGD eventually reaches an ``equilibrium'' state, where the
parameter norm~\citep{li2020reconciling,van2017l2,chiley2019online} and the
size of angular update~\citep{wan2021spherical} become stable.
\citet{li2020reconciling,wang2022three}
provided empirical and theoretical evidence that the function represented by the
net also equilibrates to a stationary distribution that is independent of
initialization. This could be related to \citet{liu2020bad}'s experiments on the
ability of SGD with WD to escape from bad initialization, but it remains 
unclear why the generalization should be good at the equilibrium state. In
this paper, we focus on (full-batch) GD, which is the most basic
and important special case of SGD.

%% file: prelim.tex
\section{Preliminaries} \label{sec:prelims}

Let $\sphS^{D-1} := \{ \vtheta \in \R^D : \normtwosm{\vtheta} = 1\}$ be the unit
sphere equipped with subspace topology. We say a loss function $\Loss(\vw)$
defined on $\R^{D} \setminus \{\vzero\}$ is \textit{scale-invariant} if
$\Loss(c\vw) = \Loss(\vw)$ for all $c > 0$. In other words, the loss value does
not change with the parameter norm. For a differentiable scale-invariant
function $\Loss(\vw)$, the gradient is $(-1)$-homogeneous and it is always
perpendicular to $\vw$, i.e., $\nabla \Loss(c\vw) = c^{-1} \nabla \Loss(\vw)$
for all $c > 0$ and $\dotp{\nabla \Loss(\vw)}{\vw} = 0$ (see \Cref{lm:scale-invariant-grad-hess}).

The focus of this paper is the dynamics of \GDWD on scale-invariant loss.
\eqref{eq:upd-gdwd} gives the update rule for learning rate (LR) $\heta$ and
weight decay (WD) $\hlam$. We use $\vtheta_t := \frac{\vw_t}{\normtwosm{\vw_t}}$
to denote the projection of $\vw_t$ onto $\sphS^{D-1}$ at step $t$. We write
\GDWD on scale-invariant loss as a specific kind of Projected Gradient Descent
(PGD) and define the \emph{effective learning rate} to be the
LR $\eeta_t := \frac{\heta}{(1-\heta\hlam)\normtwosm{\vw_t}^2}$ that appears in
the update rule of PGD. This notion is slightly different from the effective
learning rate $\frac{\heta}{\normtwosm{\vw_t}^2}$ defined in previous
works~\citep{van2017l2,hoffer2018norm,arora2018theoretical},
but ours is more convenient for our analysis.
\begin{lemma} \label{lm:scale-invariant-on-sphere}
    When the parameters $\vw_t$ are updated as \eqref{eq:upd-gdwd}, $\vtheta_{t}$ satisfies the following equation:
    \begin{equation} \label{eq:theta-update-simple}
        \vtheta_{t+1} = \Pi(\vtheta_t - \eeta_t \nabla \Loss(\vtheta_t)),
    \end{equation}
    where $\eeta_t := \frac{\heta}{(1-\heta\hlam)\normtwosm{\vw_t}^2}$ is called the \textbf{effective learning rate} at step $t$,
    and $\Pi: \vw \mapsto \frac{\vw}{\normtwosm{\vw}}$ is
    the projection operator that projects any vector onto the unit sphere.
\end{lemma}

%% file: scale-invariant.tex
\section{\GDWD on Scale-Invariant Loss Functions} \label{sec:main-result}

This section  analyzes \GDWD \eqref{eq:upd-gdwd} on a scale-invariant loss
$\Loss(\vw)$, in particular what happens after approaching a manifold of local
minimizers.
\Cref{sec:main-result-1} analyzes the dynamics in
the stable regime, where loss is guaranteed to decrease monotonically, and
\Cref{thm:li-gdwd} suggests $\vw_t$ can get close to a local minimizer at some
time $t_0$. We show that the effective LR keeps increasing after $t_0$, causing
\GDWD to eventually leave this stable regime and enter a new regime which we
call the Edge of Stability (EoS). In \Cref{sec:main-result-2}, we establish our
main theorem, which connects the dynamics of $\vw_t$ in the EoS regime to a
sharpness-reduction flow.

\subsection{\GDWD Eventually Leaves the Stable Regime} \label{sec:main-result-1}

A standard step of analyzing optimization methods is to do Taylor expansion locally for the loss function,
and show that how the optimization method decreases the loss using a \textit{descent lemma}.
In our case of scale-invariant loss functions,
we use $\mH(\vw) := \nabla^2 \Loss(\vw) \in \R^{D \times D}$ to denote the Hessian matrix of $\Loss$ at $\vw \in \R^D$,
and $\lamH_1(\vw) := \lambda_1(\mH(\vw))$ to denote the top eigenvalue of $\mH(\vw)$.
\begin{lemma}[Descent Lemma] \label{lm:descent}
  For scale-invariant loss $\Loss(\vw)$, 
  at step $t$ of \GDWD we have
  \begin{align*}
    \Loss(\vtheta_{t+1}) \le \Loss(\vtheta_t) - \eeta_t (1 - \eeta_t \lambda_{\max}^{(t)} / 2) \normtwosm{\nabla \Loss(\vtheta_t)}^2.
  \end{align*}
  where $\lambda_{\max}^{(t)} := \sup_{\alpha \in [0, \eeta_t]} \left\{ \lamH_1(\vtheta_t - \alpha \nabla \Loss(\vtheta_t)) \right\}$ is an upper bound of spherical sharpness locally.
\end{lemma}
This descent lemma shows that the training loss $\Loss(\vtheta_t)$ keeps
decreasing as long as the effective LR $\eeta_t$ is smaller than
$2/\lambda_{\max}^{(t)}$, We call the regime of $\eeta_t < 2 /
\lambda_{\max}^{(t)}$ as the \textit{stable regime} of \GDWD. If $\eeta_t
\approx 2 / \lambda_{\max}^{(t)}$ with a small difference, then we call it as
the \textit{Edge of Stability (EoS) regime}.
We remark that this condition for EoS regime is essentially the
\emph{same} as $\heta \approx 2 / \lamH_1(\vw)$ in \citet{cohen2021gradient}'s definition
because $\eeta_t \cdot \lambda_{\max}^{(t)} \approx \heta \cdot \lamH_1(\vw)$; see
\Cref{sec:eos-cohen}.

Fix an initial point $\vw_0 \in \R^D \setminus \{\vzero\}$.
Now we aim to characterize the dynamics of \GDWD when LR $\heta$ and WD $\hlam$ are small enough.
The convergence rate of \GDWD has been analyzed by \citet{li2022robust}.
Here we present a variant of their theorem 
that bounds both the gradient and effective LR.
\begin{theorem}[Variant of Theorem D.2, \citet{li2022robust}] \label{thm:li-gdwd}
    Let $\Loss(\vw)$ be a scale-invariant loss function and 
    $\rho_2 := \sup\{ \normtwosm{\nabla^2 \Loss(\vw)} : \vw \in \sphS^{D-1} \}$ be the smoothness constant of $\Loss$ restricted on the unit sphere.
    For \GDWD \eqref{eq:upd-gdwd} with $\heta \hlam \le 1/2$ and $\eeta_0 \le \frac{1}{\pi^2 \rho_2(1-\heta\hlam)}$, let
    $T_0 := \left\lceil \frac{1}{2 \heta \hlam}  \ln \frac{\normtwosm{\vw_0}^2}{\rho_2 \pi^2 \heta}  \right\rceil$
    steps,
    there must exist $0 \le t \le T_0$ such that
    $\normtwosm{\nabla \Loss(\vtheta_t)}^2 \le 8\pi^4\rho_2^2\hlam\heta$ and $\eeta_t \le \frac{2}{\pi^2 \rho_2 (1-\heta\hlam)}$.
\end{theorem}
\Cref{thm:li-gdwd} shows that
for some $t_0 \le T_0$,
 $\normtwosm{\nabla \Loss(\vtheta_{t_0})}^2 \le O(\hlam \heta)$ and $\eeta_{t_0} \le \frac{1}{\pi^2 \rho_2} < \frac{2}{\rho_2}$,
which means $\vtheta_{t_0}$ is  an approximate first-order stationary point of $\Loss$ on the unit sphere.
This does not guarantee that $\vtheta_{t_0}$ is close to any global minimizer,
but in practice the training loss rarely gets stuck at a non-optimal value
when the model is overparameterized~\citep{lee2016minimizers,panageas2017only,lee2017first,zhang2017rethinking}.
We are thus motivated to study the case where 
$\vtheta_{t_0}$ not only has small gradient 
$\normtwosm{\nabla \Loss(\vtheta_{t_0})}^2 \le O(\hlam \heta)$ 
but also is close to a local minimizer $\vthetaopt \in \sphS^{D-1}$ of $\Loss$
in the sense that $\normtwosm{\vtheta_{t_0} - \vthetaopt} \le O((\hlam \heta)^{1/2})$
(assuming smoothness, the latter implies the former).

\begin{wrapfigure}{r}{0.4\textwidth}
  \vspace{-0.2in}
  \centering
    \begin{subfigure}[b]{0.19\textwidth}
      \centering
      \includegraphics[width=\textwidth]{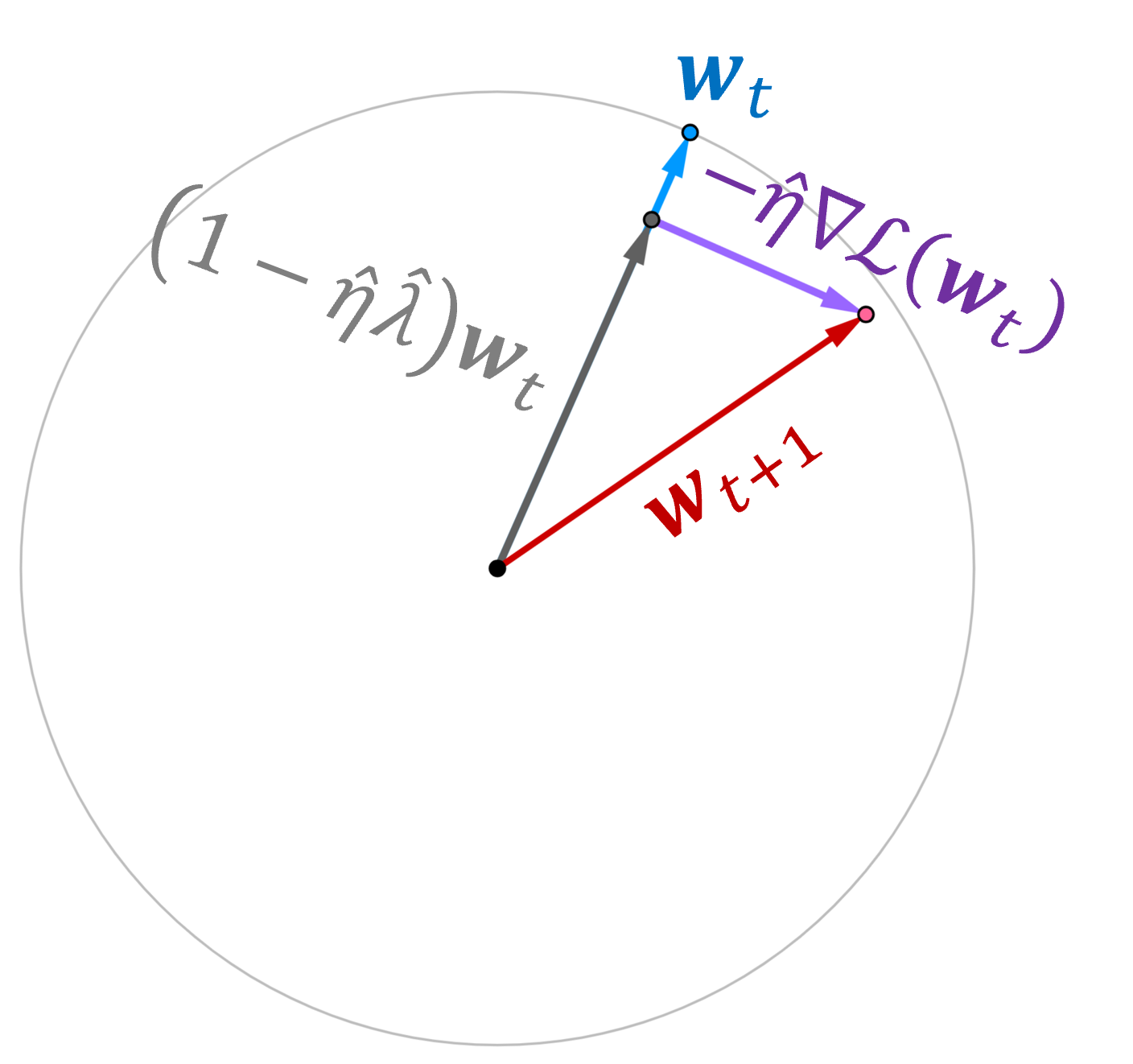}
      \caption{}
      \label{fig:norm-smaller}
    \end{subfigure}%
    \begin{subfigure}[b]{0.19\textwidth}
      \centering
      \includegraphics[width=\textwidth]{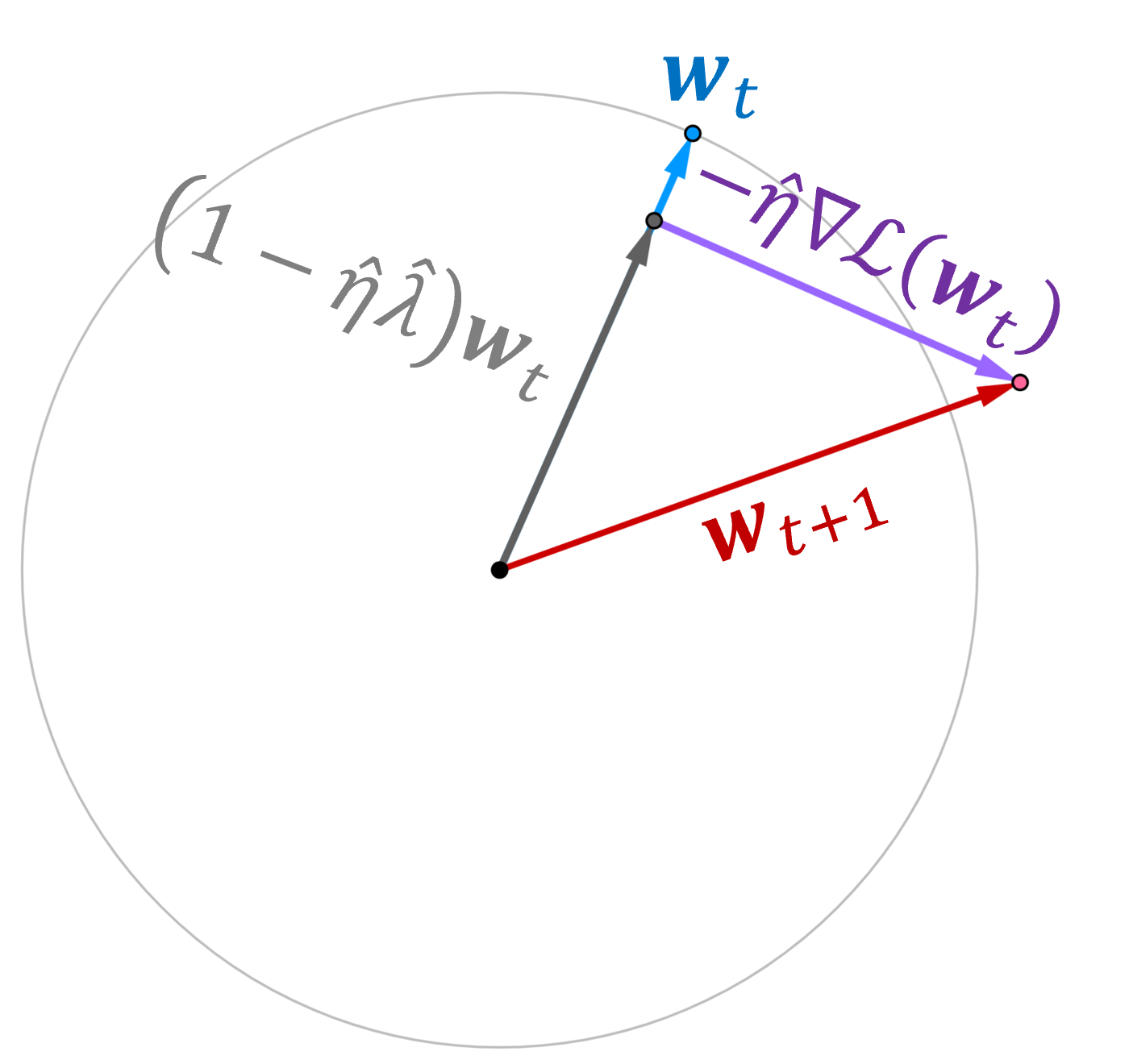}
      \caption{}
      \label{fig:norm-bigger}
    \end{subfigure}
  \caption{\small The norm of $\vw_t$ decreases when gradient is small and increases when gradient is large.}
  \label{fig:ex}
  \vspace{-0.3in}
\end{wrapfigure}

As the gradient is small 
near the local minimizer $\vthetaopt$,
starting from step $t_0$,
the norm of $\vw_t$ decreases due to the effect of WD. See \Cref{fig:norm-smaller}.
Since the effective LR
is inversely proportional
to $\normtwosm{\vw_t}^2$,
this leads to the effective LR to increase.
Then \Cref{thm:eos-near-opt} will
show that the \GDWD dynamic eventually leaves the stable regime at some time $t_1 > t_0$,
and enters the EoS regime where $\eeta_t \approx 2/\lambda_{\max}^{(t)}$.

To establish \Cref{thm:eos-near-opt}, we need to assume that 
$\Loss$ satisfies \PolyakLojasiewicz (PL) condition locally,
which is a standard regularity condition in the optimization literature to ease theoretical analysis
around a minimizer.
Intuitively, PL condition guarantees that the gradient grows faster than a quadratic function as we move a parameter $\vtheta$ away from $\vthetaopt$.
Note that PL condition is strictly weaker than convexity as the function can still be non-convex under PL condition~(see, e.g., \citep{karimi2016pl}).
\begin{definition}[\PolyakLojasiewicz Condition] \label{def:pl}
  For a scale-invariant loss $\Loss(\vw)$ and $\mu > 0$, we say that $\Loss$ satisfies $\mu$-\PolyakLojasiewicz condition
  (or $\mu$-PL)
  locally around a local minimizer $\vthetaopt$ on $\sphS^{D-1}$ if
  for some neighborhood $U \subseteq \sphS^{D-1}$ of $\vthetaopt$,
  $\forall \vtheta \in U: \frac{1}{2}\normtwosm{\nabla \Loss(\vtheta)}^2 \ge \mu \cdot (\Loss(\vtheta) - \Loss(\vthetaopt))$.
\end{definition}
\begin{theorem} \label{thm:eos-near-opt}
  Let $\Loss(\vw)$ be a $\contC^2$-smooth scale-invariant loss that satisfies $\mu$-PL around a local minimizer $\vthetaopt$ on the unit sphere,
  and $\rho_2 := \sup\{ \normtwosm{\nabla^2 \Loss(\vw)} : \vw \in \sphS^{D-1}\}$.
  For \GDWD on $\Loss(\vw)$ with learning rate $\heta$ and weight decay $\hlam$,
  if at some step $t_0$, $\normtwosm{\vtheta_{t_0} - \vthetaopt} \le O((\hlam \heta)^{1/2})$
  and $\eeta_{t_0} \le \frac{2}{\rho_2} < \frac{2}{\lamH_1(\vthetaopt)}$,
  and if $\hlam \heta$ is small enough,
  then there exists a time $t_1 > t_0$ such that $\normtwosm{\vtheta_{t_1} - \vthetaopt} = O((\hlam \heta)^{1/2})$
  and $\eeta_{t_1} = \frac{2}{\lamH_1(\vthetaopt)} + O((\hlam \heta)^{1/2})$.
\end{theorem}

\subsection{Dynamics at the Edge of Stability} \label{sec:main-result-2}

\begin{wrapfigure}{r}{0.4\textwidth}
  \vspace{-0.1in}
  \includegraphics[width=0.38\textwidth]{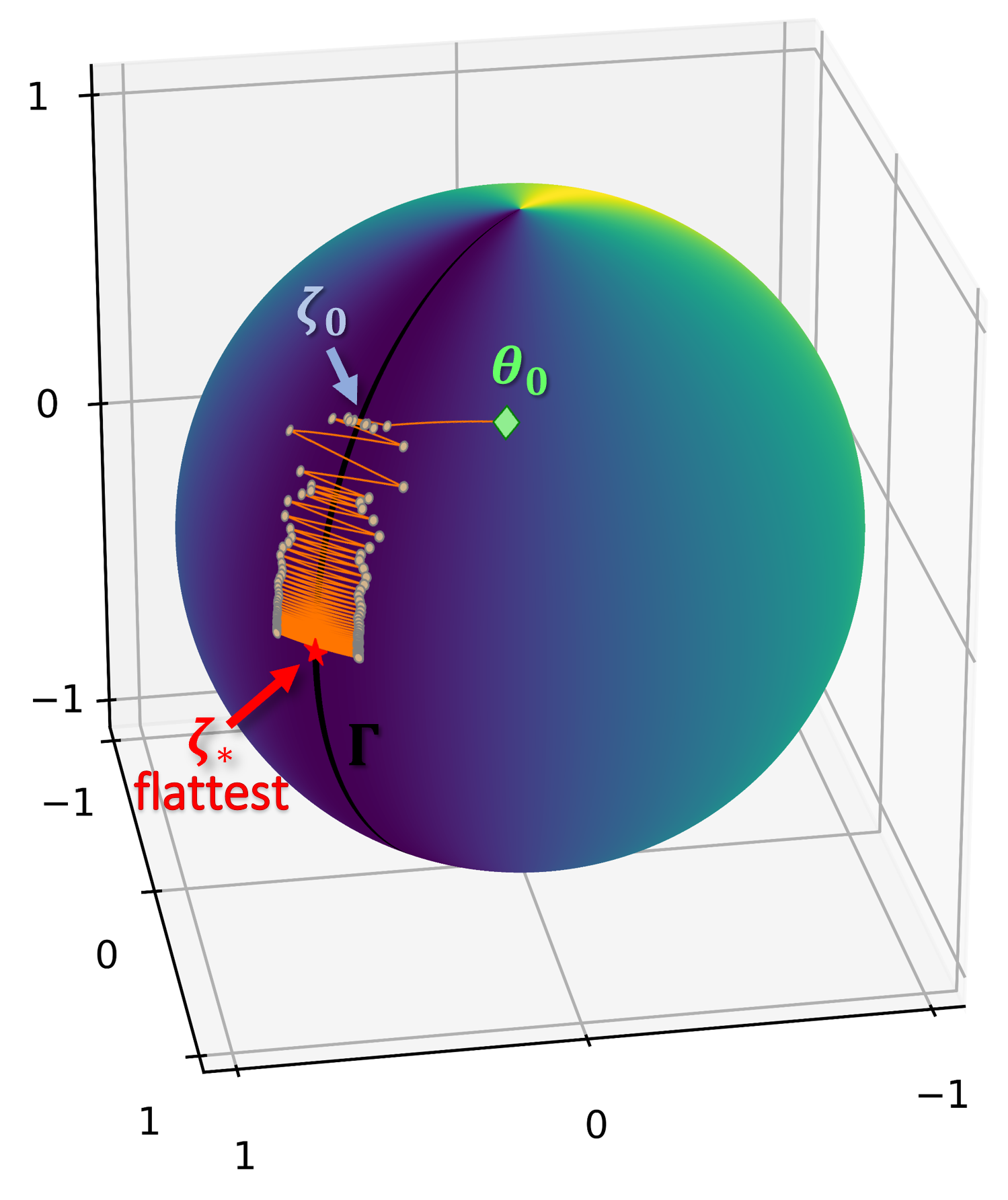}
  \caption{\small
    The trajectory of $\vtheta_t$ on a 3D scale-invariant loss function.  
  Darker color means lower loss on the unit sphere,
  and points in the black line are minimizers
  (see \Cref{sec:3d}).
  In the end, $\vtheta_t$ approaches the flattest one (red star).}
  \label{fig:ex3d}
  \vspace{-0.2in}
\end{wrapfigure}
From the analysis in the previous subsection, we know that $\vtheta_t$ can get close to a local minimizer $\vthetaopt$
and enter the EoS regime at some step $t_1$.
But what happens after $t_1$?

\Cref{fig:ex3d} gives a warm-up example on a 3D scale-invariant loss $\Loss: \R^3 \setminus \{\vzero\}\to \R$,
where the black line is a manifold $\manGa$ consisting of all the minimizers. 
In training with \GDWD,
$\vtheta_t$ first goes close to a local minimizer $\vzeta_0$,
then \Cref{thm:eos-near-opt} suggests
that WD causes the effective LR to steadily increase until the dynamic enters the EoS regime.
Now something interesting happens --- 
$\vtheta_t$ moves a bit away from $\vzeta_0$ and starts to oscillate
around the manifold $\manGa$.
This oscillation
is not completely perpendicular to $\manGa$ but actually forms a small angle that pushes
$\vtheta_t$ to move downward persistently until $\vtheta_t$ approaches the minimizer $\vzeta_*$ denoted in the plot.

For a general scale-invariant loss $\Loss: \R^D\setminus \{\vzero\}\to \R$,
which minimizer does $\vtheta_t$ move towards?
In this work, we consider the setting where there is a manifold $\manGa$ consisting only of local minimizers (but not necessarily all of them).
We show that
$\vtheta_t$ always oscillates around the manifold once it approaches the manifold and enters the EoS regime,
and meanwhile $\vtheta_t$ keeps moving in a direction of reducing spherical sharpness.

\subsubsection{Assumptions} \label{sec:main-result-2-assumption}

Now we formally introduce our main assumption on the local minimizer manifold $\manGa$.
\begin{assumption} \label{ass:man}
  The loss function $\Loss: \R^D \setminus \{\vzero\} \to \R$ is $\contC^4$-smooth and scale-invariant.
  $\manGa$ is a $\contC^2$-smooth, $(\DGa - 1)$-dimensional submanifold of $\sphS^{D-1}$
  for some $0 \le \DGa < D$, where every $\vtheta \in \manGa$ is a local minimizer of $\Loss$ on $\sphS^{D-1}$
  and $\rank(\mH(\vtheta)) = D - \DGa$.
\end{assumption}
Scale-invariance has become a standard assumption
in studying neural nets with normalization layers~\citep{li2020exp,li2020reconciling,lobacheva2021on}.
For VGG and ResNet, the scale-invariance can be ensured
after making minor changes to the architectures (see \Cref{sec:exp-details-cifar10}).
The training loss $\Loss$ may not be smooth if the activation is ReLU,
but lately it has become clear that differentiable activations
such as Swish~\citep{ramachandran2017swish}, GeLU~\citep{hendrycks2016gelu}
can perform equally well.
Swish is indeed used in our VGG-11 experiments (\Cref{fig:full-cifar}),
but ResNet with ReLU activation also exhibits a sharpness-reduction bias empirically (see \Cref{sec:exp-cifar10}).

For any local minimizer $\vtheta \in \manGa$,
the eigenvalues $\lamH_k(\vtheta)$ must be non-negative.
And $\lamH_k(\vtheta) = 0$ for all $D - \DGa < k \le D$,
since $\manGa$ is of dimension $\DGa - 1$.
The condition $\rank(\mH(\vtheta)) = D - \DGa$ ensures that
the Hessian is maximally non-degenerate on $\manGa$,
which also appears as a key assumption 
in previous works \citep{li2022what,arora2022understanding,fehrman2020convergence}.
This condition simplifies the calculus on $\manGa$ in our analysis as it ensures that the null space of the matrix $\mH(\vtheta)$ equals to the tangent space of $\manGa$ at $\vtheta \in \manGa$.
It is also closely related to PL condition (\Cref{def:pl}) as
\Cref{ass:man}
implies that $\Loss(\vtheta)$ satisfies $\mu$-PL (for some $\mu > 0$) locally around every $\vtheta \in \manGa$ on the unit sphere (\citet{arora2022understanding}, Lemma B.3).

To ease our analysis, we also need the following regularity condition to ensure that the largest eigenvalue is unique.
In our experiments, sharpness reduction happens even when the multiplicity of the top eigenvalue is more than $1$, but we leave the analysis of that case to future work.
\begin{assumption} \label{ass:topeigen}
  For all $\vtheta \in \manGa$, $\lamH_1(\vtheta) > \lamH_2(\vtheta)$. That is, the top eigenvalue of $\mH(\vtheta)$ is unique.
\end{assumption}

\subsubsection{Main Theorem} \label{sec:main-thm-stat}
First, we define $\ieta := \heta \hlam$ as the intrinsic learning rate (name from \citet{li2020reconciling}) for convenience.
As suggested in \Cref{thm:li-gdwd,thm:eos-near-opt},
$\vtheta_t$ can get close to a local minimizer and be in the EoS regime at some step $t_1$:
if $\vzeta_0$ is the local minimizer, then $\normtwosm{\vtheta_{t_1} - \vzeta_0} = O(\ieta^{1/2})$
and $\eeta_{t_1} = \frac{2}{\lamH_1(\vzeta_0)} + O(\ieta^{1/2})$.
In our main theorem, we start our analysis from step $t_1$ while setting $t_1 = 0$ WLOG (otherwise we can shift the step numbers).
We connect \GDWD in the EoS regime
to the following gradient flow~\eqref{eq:main-zeta} on the manifold $\manGa$ minimizing spherical sharpness (with gradient-dependent learning rate),
and show that one step of \GDWD tracks a time interval of length $\ieta$ in the gradient flow.
\begin{equation}
    \vzeta(0) = \vzeta_0 \in \manGa, \qquad \frac{\dd}{\dd \tau} \vzeta(\tau) = -\frac{2\gradGa \log \lamH_1(\vzeta(\tau))}{4 + \normtwosm{\gradGa \log \lamH_1(\vzeta(\tau))}^2}.
    \label{eq:main-zeta}
\end{equation}
Here we use the notation $\gradGa R(\vtheta)$
for any $R: \R^D \to \R$ 
to denote the projection of $\nabla R(\vtheta)$ onto the tangent space $\TGa{\vtheta}$ at $\vtheta \in \manGa$.
$\vzeta(\tau)$ reduces sharpness as it moves in direction of 
the negative gradient of $\log \lamH_1(\vzeta(\tau))$ on $\manGa$.
A simple chain rule shows how fast the spherical sharpness decreases:
\begin{equation*}
  \frac{\dd}{\dd t} \log \lamH_1(\vzeta(\tau)) = -\frac{2 \normtwosm{\gradGa \log \lamH_1(\vzeta(\tau))}^2}{4 + \normtwosm{\gradGa \log \lamH_1(\vzeta(\tau))}^2}
  \approx \begin{cases}
    -\frac{1}{2} \normtwosm{\gradGa \log \lamH_1(\vzeta(\tau))}^2 &~~~\text{for small gradient}; \\
    -2 &~~~\text{for large gradient}.
  \end{cases}
\end{equation*}
Note that it is not enough to just assume that $\vtheta_0$ is close to $\vzeta_0$.
If $\vtheta_0 = \vzeta_0$ holds exactly, 
then the subsequent dynamic of $\vw_t$ is described by $\vw_t = (1 - \heta \hlam)^t \vw_0$
with direction unchanged.
There are also some other bad initial directions of $\vw_0$ that may not lead to the sharpness-reduction bias.
This motivates us to do a smoothed analysis for the initial direction:
the initial direction is $\vzeta$ with tiny random perturbation,
where the perturbation scale is allowed to vary from $\exp(-\ieta^{-o(1)})$ to $\ieta^{1/2-o(1)}$,
and we show that a good initial direction is met
with high probability as $\ieta \to 0$.%
\footnote{Here $\ieta^{-o(1)}$ can be constant, $O(\log(1/\ieta))$, or $O(\polylog(1/\ieta))$, but not $\ieta^{-\epsilon}$ if $\epsilon > 0$ is a constant. As mentioned later, this need for random initialization is very similar to the one needed in power method for computing eigenvalues.}
Alternatively, one can regard it as a modeling 
of the tiny random noise in \GDWD due to the precision errors in floating-point operations.
See Figure \hyperref[fig:lin-idea]{\ref{fig:lin-idea}b};
the training loss can never be exactly zero in practice.

\myparagraph{Initialization Scheme.}
Given a local minimizer $\vzeta_0 \in \manGa$,
we initialize $\vw_0 \in \R^D \setminus \{\vzero\}$ as follows:
draw 
$\vxi \sim \Normal(\vzero, \sigma_0^2 \mI / D)$
from Gaussian and set the direction of $\vw_0$ to $\frac{\vzeta_0 + \vxi}{\normtwosm{\vzeta_0 + \vxi}}$,
where $\sigma_0$ can take any value in $[\exp(-\ieta^{-o(1)}), \ieta^{1/2-o(1)}]$;
then set the parameter norm $\normtwosm{\vw_0}$
to be any value that satisfies $\abs{\eeta_0 - \frac{2}{\lamH_1(\vzeta_0)}} \le \ieta^{1/2-o(1)}$,
where $\eeta_0 := \frac{\heta}{(1-\heta\hlam) \normtwosm{\vw_0}^2}$ is the effective LR for the first step.

\begin{theorem} \label{thm:gdwd-main}
  Under \Cref{ass:man,ass:topeigen},
  for \GDWD \eqref{eq:upd-gdwd}
  with sufficiently small 
  intrinsic learning rate $\ieta := \heta \hlam$,
  if we follow the above initialization scheme for some $\vzeta_0 \in \manGa$,
  then with probability $1 - O(\ieta^{1/2-o(1)})$,
  the trajectory of $\vtheta_t := \frac{\vw_t}{\normtwosm{\vw_t}}$ approximately tracks a sharpness-reduction flow
  $\vzeta: [0, T] \to \manGa$ that starts from $\vzeta_0$ and evolves as the ODE \eqref{eq:main-zeta} up to time $T$ (if solution exists),
  in the sense that $\normtwosm{\vtheta_t - \vzeta(t \ieta)} = O(\ieta^{1/4 - o(1)})$ for all $0 \le t \le T / \ieta$.
\end{theorem}

\begin{remark}[Magnitude of Oscillation]
  As suggested by \Cref{fig:ex3d}, $\vtheta_t$ actually oscillates
  around the manifold.
  But according to our analysis,
  the magnitude of oscillation is as small as $O(\ieta^{1/2-o(1)})$,
  so it is absorbed into our final bound $O(\ieta^{1/4-o(1)})$
  for the distance between $\vtheta_t$ and $\vzeta(t \ieta)$.
\end{remark}

\subsubsection{Proof Idea} \label{sec:proof-idea}
Throughout our proof, we view \GDWD for $\vw_t$ as a PGD for $\vtheta_t$
with effective LR $\eeta_t$ (\Cref{lm:scale-invariant-on-sphere}).
To track $\vtheta_t$ with $\vzeta(t \ieta)$,
for each step $t$, we construct a local minimizer $\vphi_t \in \manGa$
that serves as the ``projection'' of $\vtheta_t$ onto the manifold $\manGa$,
in the sense that the displacement $\vx_t := \vtheta_t - \vphi_t$
is approximately perpendicular to the tangent space of $\manGa$ at $\vphi_t$.
Our entire proof works through induction.
According to the initial conditions,
the dynamic is initially in the EoS regime: $\normtwosm{\vx_t} \le \ieta^{1/2-o(1)}$
and $\abssm{\eeta_t - 2/\lamH_1(\vphi_t)} \le \ieta^{1/2-o(1)}$ at $t = 0$.
In our induction, we maintain 
the induction hypothesis that these two EoS conditions continue to hold for all $t \ge 0$. 

\myparagraph{Period-Two Oscillation.}
A key insight in our proof is that
after a few initial steps,
$\vtheta_t$ is oscillating around $\vphi_t$
along the $\pm \vvH_1(\vtheta)$ directions,
where 
$\vvH_1(\vtheta)$
is a unit top eigenvector of $\mH(\vtheta)$
and is chosen in a way 
that $\vvH_1(\vtheta)$ is continuous on $\manGa$.
More specifically,
$\vx_t = 
h_t\vvH_1(\vphi_t) + O(\normtwosm{\vx_t}^2)$
for $h_t := \dotpsm{\vx_t}{\vvH_1(\vphi_t)}$.
The oscillation is of period 2: $h_t > 0$ when $t$ is even and $h_t < 0$ when $t$ is odd.
See Figure \hyperref[fig:lin-idea]{\ref*{fig:lin-idea}d} for an example.

This oscillation can be connected to a power method for the matrix $\mI - \eeta_t \mH(\vphi_t)$.
In the EoS regime,
we can approximate $\vtheta_{t+1}$ (when $\vx_t$ is small)
as
$\vtheta_{t+1}
  = \Pi(\vtheta_t - \eeta_t\nabla \Loss(\vtheta_t))
  \approx \Pi(\vtheta_t - \eeta_t \mH(\vphi_t) \vx_t)
  \approx \vtheta_t - \eeta_t \mH(\vphi_t) \vx_t$
by Taylor expansions
of $\nabla \Loss$
and $\Pi: \R^D \setminus \{\bm{0}\} \to \sphS^{D-1}$.
We can further show that $\vphi_{t+1} \approx \vphi_t$
due to our choice of projections.
Then the connection to power method is shown below:
\[
  \vx_{t+1} \approx \vtheta_{t+1} - \vphi_t \approx (\mI - \eeta_t \mH(\vphi_t)) \vx_t.
\]
By simple linear algebra,
$\vvH_1(\vphi_t)$
is an eigenvector of $\mI - \eeta_t \mH(\vphi_t)$,
associated with eigenvalue $1 - \eeta_t \lamH_1(\vphi_t) \approx -1$.
The remaining eigenvalues are $\{ 1 - \eeta_t \lamH_i(\vphi_t) \}_{i=2}^{D}$,
where $\lamH_i(\vphi_t)$ is the $i$-th largest eigenvalue of $\mH(\vtheta_t)$,
and they lie in the range $(-1, 1]$ since $\lamH_i(\vphi_t) \in [0, \lamH_1(\vphi_t))$.
Using a similar analysis to power method, we show that $\vx_t$ quickly aligns to the direction of $\pm \vvH_1(\vphi_t)$ after a few initial steps,
as the corresponding eigenvalue
has approximately the largest absolute value.\footnote{
  Our construction of $\vphi_t$
  ensures that $\vx_t$ only
  has a small overlap with the $1$-eigenspace of $\mI - \eeta_t \mH(\vphi_t)$,
  so $\vx_t$ can only align to $\pm \vvH_1(\vphi_t)$.
}

To formally establish the above result, we need
a tiny initial alignment
between $\vx_0$ and $\vvH_1(\vphi_0)$,
just as the initial condition in power method.
This is where we need the initial random perturbation.

\begin{figure}[t]
  \centering
  \includegraphics[width=\textwidth]{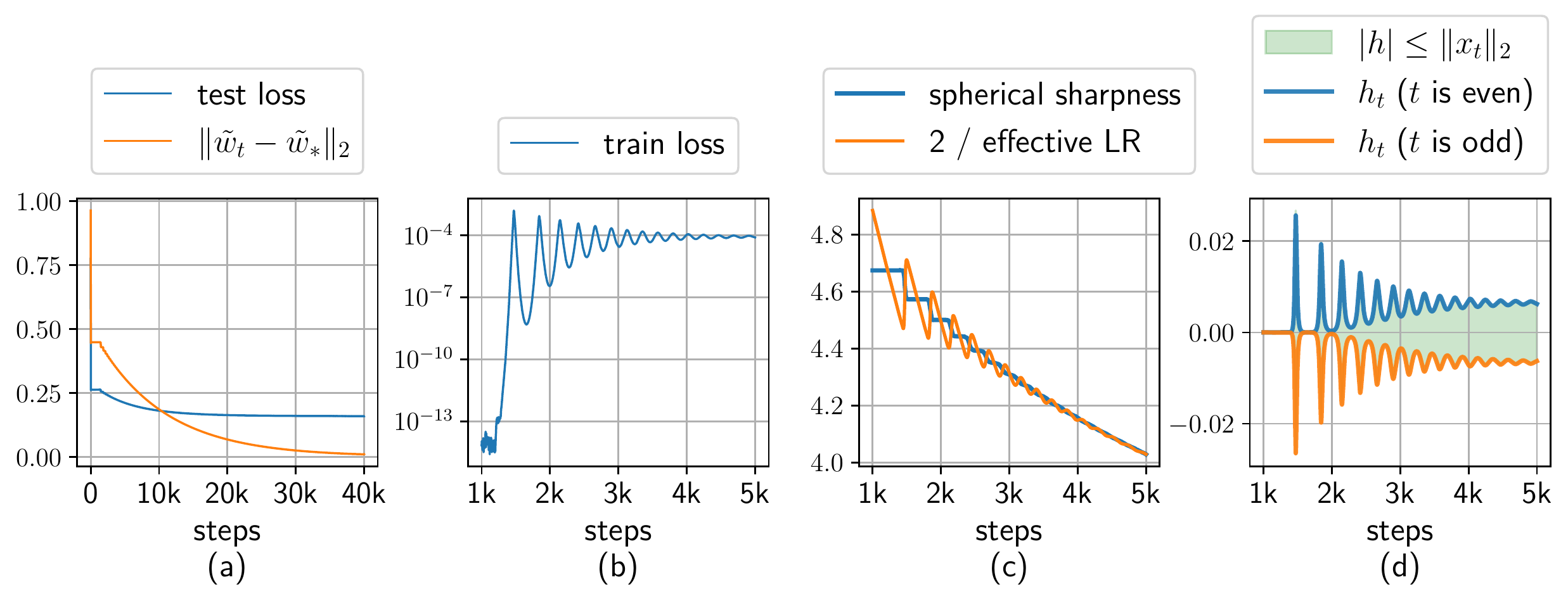}
  \caption{
    \small
    Illustration of the oscillation and periodic behaviors of \GDWD on linear regression with BN
    (see \Cref{sec:proof-idea,sec:lin}).
    The training loss decreases to $\approx 10^{-14}$ in the first 1k steps
    and achieves test loss $0.26$.
    Starting from step $\sim$ 1k, the dynamic enters the EoS regime.
    \textbf{(a).}~The test loss decreases to $0.16$ as a distance measure to the flattest solution \eqref{eq:M} decreases towards~$0$;
    \textbf{(b).}~The training loss oscillates around $\sim 10^{-4}$ in the EoS regime;
    \textbf{(c).}~$2/\eeta_t$
    switches back and forth between being smaller and larger than $\lamH_1(\vphi_t)$;
    \textbf{(d).}~The parameter oscillates around the minimizer manifold along the top eigenvector direction,
    and the magnitude of oscillation $\abssm{h_t}$ rises and falls periodically.
  }
  \label{fig:lin-idea}
  \vspace{-0.15in}
\end{figure}

\myparagraph{Oscillation Drives $\vphi_t$ to Move.}
This period-two oscillation is the driving power to push $\vphi_t$
to move on the manifold.
The main idea here is to 
realize that the oscillation direction deviates slightly from the direction of $\pm \vvH_1(\vphi_t)$
by using a higher-order approximation.
We specifically use the Taylor approximation to show that
this deviation leads 
$\vphi_t$ to move slightly on $\manGa$:
after each cycle of oscillation, $\vphi_{t+2} \approx \vphi_t - 4 h_t^2 \gradGa \log \lamH_1(\vphi_t) + O(\ieta^{1.5-o(1)})$,
which resembles two steps of gradient descent on $\manGa$
to minimize the logarithm of spherical sharpness with learning rate $2h_t^2$,

\myparagraph{Periodic Behavior of $h_t$ and $\eeta_t$.}
It remains to analyze the dynamics of $h_t$
so that we can know how fast the sharpness reduction is.
Our analysis is inspired by an empirical study from~\citet{lobacheva2021on},
which reveals a periodic behavior of gradients and effective learning rates in training normalized nets
with weight decay.
In our theoretical setting, we capture
this periodic behavior by showing that $h_t$ and $\eeta_t$ \textit{do} evolve periodically.
See Figures \hyperref[fig:lin-idea]{\ref*{fig:lin-idea}c} and \hyperref[fig:lin-idea]{\ref*{fig:lin-idea}d} for an example.

The key is that $\eeta_t$ changes as an adaptive gradient method:
$\eeta_t$ increases when gradient is small and decreases when gradient is large 
(due to the effect of WD; see \Cref{fig:norm-smaller,fig:norm-bigger}),
and in our case the gradient norm scales as $\abssm{h_t}$ since $\nabla \Loss(\vtheta_t) \approx h_t \lamH_1(\vphi_t) \vvH_1(\vphi_t)$.
By the power method approximation, $h_{t+2} \approx (1 - \eeta_t \lamH_1(\vphi_t))^2 h_t$,
so $\abssm{h_t}$ decreases when $\eeta_t < 2 / \lamH_1(\vphi_t)$.
But $\abssm{h_t}$ cannot decrease forever, since $\eeta_t$ increases when $\abssm{h_t}$ is sufficiently small.
When $\eeta_t$ rises to over $2 / \lamH_1(\vphi_t)$,
$\abssm{h_t}$ changes from decreasing to increasing according to our approximation.
But $h_t$ cannot increase indefinitely either, since
$\eeta_t$ decreases when $\abssm{h_t}$ is sufficiently large.
A period finishes when $\eeta_t < 2 / \lamH_1(\vphi_t)$ holds again.

In our theoretical analysis, we connect this periodic behavior with a 1-dimensional Hamiltonian system (see \Cref{sec:outline-rmsdrift-def}),
and show that 
$2 h_t^2$ in each step can be approximated by its average value in the period
without incurring a large error.
Further calculations show that this average value
is approximately 
$\frac{2 \ieta}{4 + \normtwosm{\gradGa \log \lamH_1(\vzeta(t\ieta))}}$,
the learning rate
in the flow~\eqref{eq:main-zeta} multiplied with $\ieta$.
We can therefore conclude that each step of $\vphi_t$ (or $\vtheta_t$)
tracks a time interval of $\ieta$ in the flow.

\myparagraph{Extensions.} We note that this periodic behavior is not limited to \GDWD on scale-invariant loss,
since the above intuitive argument holds as long as 
the effective LR changes adaptively with respect to gradient change.
Based on this intuition,
an important notion called \textit{Quasi-RMSprop scheduler}
is proposed.
For a PGD method, a learning rate scheduler is a rule for 
changing the effective LR in each step,
and Quasi-RMSprop is a specific class of schedulers we define,
including the way that the effective LR changes
in \GDWD on scale-invariant loss (if viewed as PGD).
Our proof is done in a unified way that 
works as long as the effective LR changes in each step according to a Quasi-RMSprop scheduler.
As a by-product, a similar theorem can be proved for GD (without projection) on non-scale-invariant loss
if the LR changes as a Quasi-RMSprop in each step.
For example, we can extend our analysis to RMSprop with a scalar learning rate.
See \Cref{sec:qrms}.

%% file: applications.tex
\section{Case Study: Linear Regression with Batch Normalization} \label{sec:lin}
In this section, we analyze the \GDWD dynamics on linear regression with Batch Normalization (BN),
as a simple application of our theory.
Let $\{(\vx_i, y_i)\}_{i=1}^{n}$ be a dataset, where $\vx_i \in \R^d$ and $y_i \in \R$ are inputs and regression targets.
We study the over-parameterized case where $d \gg n$,
and we assume that the regression targets are generated by an unknown linear model.

A classic linear model is parameterized by $(\vw, b) \in \R^d \times \R$ and outputs $\vw^\top \vx + b$ given input $\vx$,
but now we add a BN to the output.
More specifically, we consider a batch-normalized linear model $\Phi(\vx; \vw, \gamma, \beta) := \gamma \cdot \frac{\vw^\top \vx - \mu_1}{\sigma_1} + \beta$,
where $\mu_1, \sigma_1$ are the mean and standard deviation of $\{\vw^\top \vx_i\}_{i=1}^{n}$ over the whole dataset\footnote{Note that the batch size is $n$ here as we are running full-batch GD},
and the bias term $b$ is cancelled out due to BN.
Note that $\Phi(\vx; \vw, \gamma, \beta)$ is still a linear function with respect to $\vx$.
Let $\vmux \in \R^d$ and $\mSigmax \in \R^{d \times d}$
be the mean and covariance of the input data $\{\vx_i\}_{i=1}^{n}$.
Then $\Phi(\vx; \vw, \gamma, \beta)$ can be rewritten as:
\begin{equation} \label{eq:lin-bn-Phi}
    \Phi(\vx; \vw, \gamma, \beta) = \tilvw^\top \vx + \tilb,
    \qquad \text{where} \quad \tilvw := \sfrac{\gamma \vw}{\normsm{\vw}_{\mSigmax}}, \quad \tilb := \beta - \tilvw^\top \vmux.
\end{equation}
No matter how $\vw$ is set, the output mean and variance of $\Phi$ are always $\beta$ and $\gamma^2$.
To simplify our analysis, we fix $\beta, \gamma$
to be non-trainable constants so that the mean and variance of $\Phi$'s output match with those of $\{y_i\}_{i=1}^{n}$,
that is, we set $\beta = \muy$ and $\gamma = \sigmay$ to be the mean and standard deviation
of $y_i$ over the whole dataset.
Then the training loss is $\Loss(\vw) := \frac{1}{n} \sum_{i \in [n]} (\Phi(\vx_i; \vw, \gamma, \beta) - y_i)^2$.

\begin{theorem} \label{thm:bn-lin-main}
    In our setting of linear regression with BN,
    the sharpness-reduction flow $\vzeta$ defined in \eqref{eq:main-zeta}
    converges to the solution $\vwopt \in \sphS^{d-1}$
    that minimizes sharpness $\lamH_1(\vwopt)$ on $\manGa$, regardless of the initialization.
    Moreover, the coefficients $(\tilvw, \tilb)$ associated with $\vwopt$ (defined in \eqref{eq:lin-bn-Phi})
    are the optimal solution of the following constrained optimization problem \eqref{eq:M}: 
    \begin{equation}
        \min \quad \normtwosm{\vw}^2 \quad \text{s.t.} \quad \vw^\top \vx_i + b = y_i, \quad \forall i \in [n]. \tag{M} \label{eq:M}
    \end{equation}
\end{theorem}

\vspace{-0.05in}
At first sight the result may appear trivial because the intent of WD is to
regularize $\normltwo$-norm. But this is deceptive because in scale-invariant
nets the regularization effect of WD is not explicit. This result also
challenges conventional view of optimization. GD is usually viewed as a
discretization of its continuous counterpart, gradient flow (GF), and
theoretical insight for the discrete update including convergence rate and
implicit bias is achieved by analyzing the continuous counterpart (See
\Cref{sec:add-related} for a list). However, GF does not have the same
sharpness-reduction bias as GD. As discussed in \citep{li2020exp}, adding WD
only performs a time-rescaling on the GF trajectory on scale-invariant loss, but
does not change the point that GF converge to if we project the trajectory onto
the unit sphere. One can easily show that GF may converge to any zero-loss
solution, but no matter how small LR is, GD exhibits the sharpness-reduction
bias towards the optimal solution of \eqref{eq:M}. To our best knowledge, this
result is the  first concrete example where even with arbitrarily small LR, GD
can still generalize better than GF under natural settings.

%% file: discussion.tex
\vspace{-0.04in}
\section{Discussion} \label{sec:discussion}

\myparagraph{Experimental Verification of Sharpness Reduction.}
Besides \Cref{fig:matcom-intro,fig:full-cifar}, \Cref{sec:exp-matlin} provides
additional matrix completion experiments with different data size, and
\Cref{sec:exp-cifar10} provides CIFAR-10 experiments with ResNet-20. In all
these experiments, we observed that GD continues to improve the test accuracy
even after fitting the training set, and this phenomenon is correlated with the
decreasing trend of spherical sharpness. See also \Cref{sec:exp-pb} for the
validation for the periodic behavior we analyze in theory.

\myparagraph{Ablation Studies on Normalization and Weight Decay.}
Our theoretical analysis crucially relies on the interplay between normalization
and WD to establish the sharpness-reduction flow. We also conducted ablation
studies on normalization and WD to highlight the importance of this interplay.
First, if normalization is removed, the spherical sharpness becomes undefined,
and we do not know if GD implicitly minimizes any sharpness measure. But even if
a similar measure does exist, it cannot be strongly related to generalization,
because we can verify that the test accuracy becomes very bad without
normalization ($56.8\%$ on CIFAR-10, \Cref{fig:full-cifar-vgg-ab-norm}), and
continuing training after fitting the training set no longer improves test
accuracy. Second, if WD is removed, the analysis in \citet{arora2018theoretical}
guarantees convergence in the stable regime, and we can verify that the
spherical sharpness and test accuracy stop changing when the loss is small.
The final test accuracy is stuck at $66.4\%$ (\Cref{fig:full-cifar-vgg-wd-vs-nowd}),
whereas training with WD leads to $84.3\%$.

\myparagraph{Explaining the Progressive Sharpening and EoS Phenomena.}
\citet{cohen2021gradient} conducted extensive empirical studies on the dynamics of GD in deep learning (without weight decay),
formally $\vw_{t+1} \gets \vw_t - \heta \nabla \tildeLoss(\vw_t)$.
They observed the \emph{progressive sharpening} phenomenon:
$\lambda_1(\nabla^2 \tildeLoss(\vw_t))$ tends to increase so long as it is less than $2 / \heta$.
Then they observed that the training typically enters the EoS regime, which they define as a regime
that (1)~$\lambda_1(\nabla^2\tildeLoss(\vw_t))$ hovers right at, or just above
$2 / \heta$; and (2)~the training loss $\tildeLoss(\vw_t)$ goes up and down over
short timescales, yet still decreases in the long-term run.
A recent research trend focuses on explaining
the progressive sharpening and EoS phenomena~\citep{ahn2022understanding,ma2022multiscale,arora2022understanding,chen2022gradient}.
Our work corresponds
to an important special case where $\tildeLoss(\vw)$ is a scale-invariant loss
with $\normltwo$-regularization, namely $\Loss(\vw) + \frac{\hlam}{2}
\normtwosm{\vw}^2$. By analyzing the interplay between normalization and WD, the
first part of our results (\Cref{sec:main-result-1}) attributes
progressive sharpening to norm change, and the second part
(\Cref{sec:main-result-2}) justifies in theory that the training can make progress in the EoS
regime. See \Cref{sec:eos-cohen} for more discussion.

%% file: add-related.tex
\section{Additional Related Works} \label{sec:add-related}

\myparagraph{Sharpness Measures and Parameter Rescaling.}
To best capture the generalization performance, the measure of sharpness should
give the same value whenever the function represented by the neural net is the
same. As mentioned in the introduction, $\lambda_1(\nabla^2 \Loss(\vw))$ does
not satisfy this property for normalized nets because it is sensitive to weight
rescaling. The spherical sharpness takes care of scale-invariance w.r.t.~all
parameters, but it is certainly not the only rescaling symmetry in deep nets.
However, to the best of our knowledge, the spherical sharpness is
the only measure that provably decreases in training normalized nets.
Although many other sharpness measures may take care of more
symmetries~\citep{yi2019positively,yi2019bninvariant,tsuzuku20normalized,rangamani2021invariant,foret2021sam},
it is unclear whether GD/SGD is implicitly reducing
them.

\myparagraph{SGD Noise Helps to Escape Sharp Minima.}
It has been a folklore that noise in stochastic gradient helps escapes sharp local minima. With the simplification of
assuming the loss is quadratic and treating SGD as its canonical continuous SDE approximation,
\citet{zhu2019anisotropic,xie2021diffusion} showed that anisotropic noise (\emph{e.g.}, the noise covariance is equal to
Hessian) has a better escape efficiency of SGD out of a sharp minimizer,
in comparison with isotropic noise. Under the same
assumptions, \citet{ibayashi2022expescape} proved the exponential escape efficiency without assuming SDE reaches the
stationary distribution. \citet{kleinberg2018alternative} showed that SGD can escape sharp local minima assuming one-point convexity. While all previous escaping analysis of SGD are based on the continuous approximation,
another approach called stability analysis is able to show that SGD cannot converge to sharp local minima when learning
rate is larger than some threshold~\citep{wu2017towards,wu2018howsgd,ma2021linear}.

\myparagraph{Implicit Bias of GD.} There are mainly two types of implicit bias
result for GD, where the first type of result applies essentially to the
continuous limit of gradient descent, namely gradient flow, and tolerates error
discretization and stochasticity when the learning rate is sufficiently
small~\citep{soudry2018iclrImplicit,soudry2018implicit,lyu2020gradient,ji2020directional,gunasekar2018implicitconv,
gunasekar2018implicit,li2018algorithmic,razin2020implicit,
arora2019implicit,chizat2020implicit,li2021towards,
lyu2021gradient,razin2022implicit,stoger2021small,ge2021understanding}. The
analyses of GD/SGD based on Neural Tangent Kernel (NTK) also essentially belong to this
type, because though the analysis in NTK regime tolerates stochasticity and finite learning rate, GD/ SGD do not learn different solutions compared to gradient flow. Such works includes (but are not limited to)~\citep{jacot2018ntk,li2018learning,du2019gradient,arora2019exact,arora2019fine,allenzhu2019convergence,allenzhu2019gobeyond,zou2018stochastic,chizat2019lazy,yang2019scaling,cao2019generalization,chen2021how}.
This type of result typically relates the generalization quality to the
initialization of GD, and GD in such regimes cannot escape from bad local minima
once reaching there. In contrast, the second type of results, to which the
current paper belongs, fundamentally relies on the discrete nature of gradient
descent. For example, \citet{barrett2021implicit} showed that for small LR,
gradient descent is approximately equal to gradient flow minimizing a new
objective, \emph{i.e.}, the original objective plus $\eta$ times squared norm of
gradients.
\citet{kong2020multiscale}
studied a special class of ``multiscale'' loss functions,
and they showed that large learning rate
introduces chaos to the dynamics of GD
and provides a mechanism to escape local minima.
\citet{wang2022large} proved that for the matrix factorization problem, GD with a large learning rate has an implicit bias towards
a solution with balanced matrix norms.
Stability analysis~\citep{wu2017towards,wu2018howsgd,mulayoff2021minimastability,ma2021linear} also
belongs to this type.

\myparagraph{Comparison with \citet{arora2022understanding}.} The paper by \citet{arora2022understanding} is probably the most related work among the second type of the implicit bias results.
They assume that there is a smooth function $L$
which satisfies certain regularity conditions around the minimizer manifold (including that Hessian is maximally non-degenerate on the manifold),
and show that running normalized GD on $L$
or GD on $\sqrt{L}$
with sufficiently small LR tracks a deterministic flow
on the minimizer manifold and decreases the largest eigenvalue of $\nabla^2 L$, which has a similar flavor to our result.
However, our setting (scale invariant loss + WD) is more natural and their result and technique do not apply to our
setting because there is no minimizer under their definition, not to mention manifolds. To show the
spherical sharpness decreases, we have to develop new proof techniques,
including connecting the dynamics to a 1-dimensional Hamiltonian system.
Another difference is that our analysis applies to gradient descent directly,
but 
the analysis by
\citet{arora2022understanding}
requires injecting stochastic noise to gradients.

\myparagraph{Edge of Stability.}
\citet{cohen2021gradient} provided an extensive empirical study showing that
GD typically occurs at the Edge of Stability (EoS), where the top eigenvalue of Hessian is approximately $2$ / LR
and the descent lemma does not guarantee the loss to decrease.
\citet{ahn2022understanding}
explored the dynamics of GD in the EoS regime
through insightful experiments, and
they attribute the EoS phenomenon to 
the lack of flat stationary points near GD trajectory
and the existence of a subset near minima that is forward invariant under GD update.
\citet{ma2022multiscale} proved the EoS phenomenon for a class of loss functions that
are decomposable as a sum of 1-dimensional functions with subquadratic growth.
\citet{chen2022gradient} provided detailed analyses for the EoS phenomenon on two-layer single-neuron net
and matrix factorization. The aforementioned work by \citet{arora2022understanding}
analyzes normalized GD on $L$ or GD on $\sqrt{L}$ also in the EoS regime,
where the latter case corresponds to a class of loss functions that grow approximately linearly near minima (e.g., the absolute value function $\abssm{x}$).
The evolution of the eigenvalues of Hessian has also been studied for SGD \citep{gilmer2022curvature,jastrzebski2020breakeven,lewkowycz2020large}.
Our work can be seen as a theoretical explanation of the EoS phenomenon
for scale-invariant loss functions with $\normltwo$-regularization (see also \Cref{sec:eos-cohen}),
which are a more broad and natural class of training loss in deep learning compared with those being studied in previous theoretical analyses.

%% file: connection.tex
\section{A General Theory for a Broader Class of Adaptive Gradient Methods} \label{sec:qrms}

Our main theorem for \GDWD on scale-invariant loss, \Cref{thm:gdwd-main},
is actually a corollary of a more general theorem that holds
for any PGD on $\sphS^{D-1}$
with effective learning rates changing adaptively
according a specific kind of update rules.
We name this kind of update rules as \textit{Quasi-RMSprop Scheduler},
where the name is due to its similarity to RMSprop~\citep{hinton2012rmsprop}.
One of the main reasons that we define this concept is that proving a theorem
for quasi-RMSprop schedulers in general is no harder than proving that only for \GDWD on scale-invariant loss,
and sometimes the math involved is more simple and elegant when we analyze the dynamics at a higher level through quasi-RMSprop.

In this section, we introduce this key notion, Quasi-RMSprop scheduler.
To motivate it, we first recall the update
rule of RMSprop and present one of its variants with a single scalar learning rate.
Then we formally introduce the notion of quasi-RMSprop schedulers,
which is a class of rules for setting effective learning rates similarly as RMSprop.
Then we
prove a recursive formula for the effective learning rates when \GDWD optimizes a scale-invariant loss,
and we categorize it as an instance of quasi-RMSprop scheduler.

After introducing this key notion, we then present our general theorem
that holds for any PGD, or even GD, as long as the effective learning rate is set by a quasi-RMSprop scheduler in each step.
Examples include Scalar RMSprop (\Cref{cor:scalar-rmsprop})
and \GDWD on scale-invariant loss (\Cref{thm:gdwd-main}).

\subsection{Scalar RMSprop and Quasi-RMSprop Scheduler}
The usual RMSprop algorithm maintains a vector $\vv_t$ storing the moving
average of the squared gradients for every coordinate, i.e., $\vv_{t+1} \gets \beta \vv_t + (1-\beta) (\nabla \Loss(\vtheta_t))^{\odot 2}$,
where $\vg^{\odot 2}$ stands for the vector obtained by squaring $\vg$ coordinatewise.
When updating the training parameter, RMSprop divides the usual GD update by the square root of the moving average coordinatewise,
i.e., $\vtheta_{t+1} \gets \vtheta_t - \frac{\eta}{\sqrt{\vv_{t+1} + \epsilon}} \odot \nabla \Loss(\vtheta_t)$,
where 
$\epsilon$ is a small constant to avoid division by zero. 
Here all the addition, division, square root operations are coordinatewise,
and $\odot$ stands for coordinatewise multiplication.

Now we
consider a variant of RMSprop, which we call \textit{Scalar RMSprop}, where
the moving average $v_t$ is maintained as a scalar storing the moving average of the squared norm of gradients,
rather than a vector that stores the moving averages separately for each coordinate.

\begin{definition}[Scalar RMSprop, Standard Form] \label{def:scalar-rms-std}
\textit{Scalar RMSprop} is an iterative method with the following update rule:
\begin{align}
  \vtheta_{t+1} &\gets \vtheta_t - \frac{\eta}{\sqrt{v_t}} \nabla \Loss(\vtheta_t), \label{eq:scalar-rmsprop-theta} \\
  v_{t+1} &\gets \beta v_t + (1-\beta) \normtwosm{\nabla \Loss(\vtheta_t)}^2. \label{eq:scalar-rmsprop-v}
\end{align}
\end{definition}

Besides that $v_t$ is changed from a vector to a scalar, another difference is that the gradient is divided by $\sqrt{v_t}$ in \eqref{eq:scalar-rmsprop-theta},
while in the usual RMSprop it is $\sqrt{\vv_{t+1}}$. In fact,
our later analysis also applies if $\sqrt{v_t}$ is changed to $\sqrt{v_{t+1}}$, but the version with $\sqrt{v_t}$
leads to simpler math. 

An alternative view of Scalar RMSprop is to regard it as
GD with time-varying learning rate $\eeta_t$,
that is,
$\vtheta_{t+1} \gets \vtheta_{t} - \eeta_t \nabla \Loss(\vtheta_t)$,
where $\eeta_t$ is the learning rate being used at the $t$-th step, which we call the \textit{effective learning rate}
at step $t$.

We view the effective learning rate here as the output of a \textit{learning rate scheduler}.\footnote{The scheduler here is similar to \texttt{torch.optim.lr\_scheduler} in PyTorch, see \url{https://pytorch.org/docs/stable/optim.html\#how-to-adjust-learning-rate}}
In this view, we call the learning rate scheduler for Scalar RMSprop as \textit{RMSprop scheduler}.
\begin{definition}[Gradient-Based Learning Rate Scheduler]
  A \emph{gradient-based learning rate scheduler} is an algorithm $\lrS$
  that reads from a stream of vectors $\vg_0, \vg_1, \vg_2, \dots$,
  For all $t \ge 0$, $\lrS$ outputs a real number $\eeta_t := \lrS(\vg_0, \vg_1, \dots, \vg_t)$
  as soon as $\lrS$ finishes reading the first $t+1$ vectors.
  At each step of gradient descent equipped with a gradient-based learning rate scheduler $\lrS$,
  the gradients in training are revealed one by one as an input stream to $\lrS$,
  and $\eeta_t$ produced by $\lrS$ is used as the effective learning rate at step $t$.
\end{definition}
\begin{definition}[RMSprop Scheduler] \label{def:rms-sche}
  Given a constant $\tilv_0 > 0$, a base learning rate $\eta$ and a decay rate $\beta$ as hyperparameters,
  an RMSprop scheduler $\rmsS$ is a gradient-based learning rate scheduler that
  reads from a stream of vectors $\vg_0, \vg_1, \dots$,
  and generates effective learning rates $\eeta_t := \rmsS(\vg_0, \dots, \vg_t)$ for all $t \ge 0$
  according to the following recursion:
  \begin{align*}
    \eeta_t &\gets \frac{1}{\sqrt{\tilv_t}}, &
    \tilv_{t+1} &\gets \beta \tilv_t + (1-\beta) \barg_t^2, &
    &\text{where} \quad \barg_t := \normtwosm{\vg_t} / \eta.
  \end{align*}
\end{definition}
\begin{definition}[Scalar RMSprop, Alternative Form] \label{def:scalar-rms-alt}
  \textit{Scalar RMSprop} is gradient descent with a RMSprop scheduler $\rmsS$.
  \begin{align}
    \vtheta_{t+1} &\gets \vtheta_t - \eeta_t \nabla \Loss(\vtheta_t),  \qquad \text{where} \qquad \eeta_t := \rmsS(\nabla \Loss(\vtheta_0), \dots, \nabla \Loss(\vtheta_t)).
  \end{align}
\end{definition}
Note that we use $\tilv_t$ as an internal state of $\rmsS$ in \Cref{def:rms-sche}.
It is easy to verify that $\tilv_t$ is nothing but a reparameterization of $v_t$: setting $\tilv_t = v_t / \eta^2$ in \Cref{def:rms-sche}
recovers the update rule in \Cref{def:scalar-rms-std}.

Now we introduce the notion of quasi-RMSprop scheduler,
which is a class of gradient-based learning rate schedulers
that have update rules similar to RMSprop scheduler.
At first reading, one can just ignore the details and regard quasi-RMSprop scheduler
as an RMSprop scheduler with negligible perturbations when $\beta \to 1$.
\begin{definition}[Quasi-RMSprop Scheduler] \label{def:qrms-sche}
  A \textit{quasi-RMSprop scheduler} $\lrS$ is a gradient-based learning rate scheduler
  parameterized by a base learning rate $\eta$ and a decay rate $\beta$,
  satisfying the following properties:
  \begin{enumerate}
  \item $(\eta, \beta)$ is allowed to take value from a hyperparameter space $\gP_\lrS \subseteq (0, +\infty) \times (0, 1)$.
  \item There exist thresholds $\eta_{\max}, \beta_{\min}$,
  a continuous function $\delta: (0, +\infty) \to \R$,
  and a polynomial $P: \R \to \R$ such that the following holds.
  If $(\eta, \beta) \in \gP_{\lrS}$ and $\eta < \eta_{\max}, \beta > \beta_{\min}$,
  for any input stream $\vg_0, \vg_1, \vg_2, \dots$,
  there exists a sequence of positive real numbers $\tilde{v}_0, \tilde{v}_1, \tilde{v}_2, \dots$
  such that 
    $\eeta_t:= \lrS(\vg_0,\ldots, \vg_t)$
  satisfies the following two inequalities for all $t \ge 0$:
  \begin{equation*}
    \begin{aligned}
      \abs{\eeta_t - \frac{1}{\sqrt{\tilde{v}_t}}} &\le \delta(\tilde{v}_t) \cdot (1 - \beta) \cdot (1 + \bar{g}_t^2) \\
      \abs{\tilde{v}_{t+1} - \left(\beta \tilde{v}_t + (1-\beta) \bar{g}_t^2 \right)} &\le \delta(\tilde{v}_t) \cdot (1-\beta)^2 \cdot P(\bar{g}_t)
    \end{aligned}
    \qquad \text{where} \quad \bar{g}_t := \normtwosm{\vg_t} / \eta.
  \end{equation*}
  The real number $\tilv_t$ is called the \textit{moment estimate} at step $t$
  associated with the input stream and effective learning rates.
  \end{enumerate}
\end{definition}

It is clear that a RMSprop scheduler can be seen as a quasi-RMSprop scheduler with the same hyperparameters $\eta, \beta$,
and $\delta \equiv 0$.

\subsection{Reformulation of \GDWD on Scale-Invariant Loss via Quasi-RMSprop Scheduler} \label{sec:reform-gdwd}
Now we reformulate \GDWD on scale-invariant loss as
PGD with a quasi-RMSprop scheduler.
Recall that we say a loss function $\Loss(\vw)$ is  scale-invariant if $\Loss(c\vw) = \Loss(\vw)$ for all $c > 0$,
and \Cref{lm:scale-invariant-on-sphere} converts \GDWD on scale-invariant functions
to Projected Gradient Descent (PGD) on unit sphere,
i.e., $\vtheta_{t+1} \gets \Pi(\vtheta_t - \eeta_t \nabla \Loss(\vtheta_t))$,
where $\eeta_t := \frac{\heta}{(1 - \heta \hlam) \normtwosm{\vw_t}^2}$ 
is the effective learning rate
at step $t$.
However, the evolution of $\eeta_t$ over time is unclear unless we know how the parameter norm $\normtwosm{\vw_t}$ changes.
Similar as the above analysis for Scalar RMSprop, where we decompose Scalar RMSprop as the GD method and an RMSprop scheduler,
here we abstract the evolution of $\eeta_t$ as a learning rate scheduler,
which we call \textit{GWSI scheduler}
(name picked from the initials of \textbf{G}D+\textbf{W}D on \textbf{S}cale-\textbf{I}nvariant loss)
and view \GDWD on scale-invariant loss as
PGD on $\sphS^{D-1}$
with effective learning rates being set by a GWSI scheduler.
\begin{definition}[GWSI Scheduler] \label{def:gwsi-sche}
  Given a constant $\tilv_0 > 0$,
  a base learning rate $\eta$
  and a decay rate $\beta$ as hyperparameters,
  a GWSI scheduler $\gwsiS$ is a gradient-based learning rate scheduler that
  reads from a stream of vectors $\vg_0, \vg_1, \dots$,
  and generates effective learning rates $\eeta_t := \gwsiS(\vg_0, \dots, \vg_t)$ for all $t \ge 0$
  according to the following recursion:
  \begin{align*}
    \eeta_t &\gets \frac{1}{\sqrt{\tilv_t}}, &
    \tilv_{t+1} &\gets \beta \tilv_t + (1-\beta) \barg_t^2 + \frac{1}{4\beta \tilv_t} (1-\beta)^2 \barg_t^4, &
    &\text{where} \quad \barg_t := \normtwosm{\vg_t} / \eta.
  \end{align*}
\end{definition}
\begin{theorem} \label{thm:connect-siwd-rmsprop}
  For gradient descent \eqref{eq:upd-gdwd} with learning rate $\heta > 0$ and weight decay $\hlam > 0$
  on scale-invariant function $\Loss(\vw)$,
  let $\vtheta_t := \frac{\vw_t}{\normtwosm{\vw_t}}$ be the direction of $\vw_t$
  and $\eeta_t := \frac{\heta}{(1-\heta\hlam) \normtwosm{\vw_t}^2}$ be the effective learning rate at time $t$.
  Then $\eeta_t$ evolves exactly the same as
  the GWSI scheduler with hyperparameters $\tilv_0 := \frac{(1-\heta \hlam)^2\normtwosm{\vw_0}^4}{\heta^2}$,
  $\beta := (1-\heta \hlam)^4$,
  $\eta := \sqrt{(\beta^{-1} - 1)/2}$, and we can write the dynamics of $\vtheta_t$ as:
  \begin{align*}
    \vtheta_{t+1} &\gets \Pi(\vtheta_t - \eeta_t \nabla \Loss(\vtheta_t)),  \qquad \text{where} \qquad \eeta_t := \gwsiS(\nabla \Loss(\vtheta_0), \dots, \nabla \Loss(\vtheta_t)).
  \end{align*}
\end{theorem}
\begin{remark}
  To the best of our knowledge, this particular form of the recursion formula (\Cref{def:gwsi-sche}) for the effective learning rates
  of \GDWD on scale-invariant loss does not appear in prior works, but some variants have been studied before.
  \citet{arora2018theoretical} derived a similar formula when the weight decay is zero.
  \citet{li2020reconciling} obtained a Stochastic Differential Equation (SDE) for \SGDWD on scale-invariant functions, which is essentially a continuous approximation of our formulation
  when $\heta \hlam \to 0$.
  \citet{tanaka2021noether} studied the continuous-time approximation of the momentum method with WD on scale-invariant loss
  and establish a connection to adaptive gradient methods.
\end{remark}

We defer the proof for \Cref{thm:connect-siwd-rmsprop} to \Cref{sec:proof-connect-siwd-rmsprop}.
The formula of GWSI scheduler clearly resembles RMSprop scheduler:
the only difference is that GWSI scheduler has an extra $\frac{1}{4\beta \tilv_t} (1-\beta)^2 \barg_t^4$ term,
which is negligible when $(1-\beta)^2$ is small.
This intuition is formalized through the definition of quasi-RMSprop.
It can be easily seen that GWSI scheduler is a quasi-RMSprop scheduler with the same hyperparameters $\eta, \beta$.

In fact,
we can use \Cref{thm:connect-siwd-rmsprop} as a basis
to obtain a better way to 
write \GDWD on scale-invariant loss as PGD with a quasi-RMSprop scheduler,
in which $\eta, \beta$ are expressed more simply in terms of the intrinsic learning rate $\ieta := \heta \hlam$.
The main idea is that 
$\beta = 1 - 4 \ieta + O(\ieta^2)$ and $\eta = \sqrt{2 \ieta} \cdot (1 + O(\ieta))$
when $\ieta$ is small,
and we can absorb these approximation errors
into the error bounds in \Cref{def:qrms-sche}.
We defer the details to \Cref{sec:proof-gdwd-is-quasi-rmsprop}.
\begin{theorem} \label{thm:gdwd-is-quasi-rmsprop}
  There exists a quasi-RMSprop scheduler $\qrmsS$ with
  hyperparameter space $\{(\eta, \beta) : \beta = 1 - 2 \eta^2, \eta \in (0, \frac{1}{\sqrt{2}}) \}$
  such that the following holds
  for \GDWD on scale-invariant loss.
  If the intrinsic learning rate $\ieta := \heta \hlam$ lies in the range $(0, 1/4)$,
  then
  we can set the hyperparameters of~$\qrmsS$
  to be 
  $\eta = \sqrt{2\ieta}, \beta = 1 - 4\ieta$
  so that the dynamics of $\vtheta_t$ can be written as
  \begin{align*}
    \vtheta_{t+1} &\gets \Pi(\vtheta_t - \eeta_t \nabla \Loss(\vtheta_t)),  \qquad \text{where} \qquad \eeta_t := \qrmsS(\nabla \Loss(\vtheta_0), \dots, \nabla \Loss(\vtheta_t)).
  \end{align*} 
\end{theorem}

\subsection{Main Results for GD/PGD with Quasi-RMSprop Scheduler} \label{sec:quasi-main-thm}
As mentioned in \Cref{sec:proof-idea},
a key step in our analysis is to show the oscillation and periodic behaviors
of the effective learning rates.
But these behaviors can show up in other algorithms
besides \GDWD on scale-invariant loss.

In the following, we first present \Cref{thm:gdwd-main}, which holds for any PGD on $\sphS^{D-1}$
with effective learning rates being set by a quasi-RMSprop scheduler.
We focus on the case where $\eta > 0$ is small and $\beta = \Cb \eta^2 + O(\eta^4)$
for some constant $\Cb > 0$.
Our theorem shows that
if $\vtheta_0$ is initially near a local minimizer manifold $\manGa$
and is in the EoS regime,
then $\vtheta_t$ approximately tracks a sharpness-reduction flow defined as follows:
\begin{equation}
    \vzeta(0) = \vzeta_0 \in \manGa, \qquad \frac{\dd}{\dd \tau} \vzeta(\tau) = -\frac{\gradGa \log \lamH_1(\vzeta(\tau))}{4 + \frac{2}{\Cb}\normtwosm{\gradGa \log \lamH_1(\vzeta(\tau))}^2},
    \label{eq:zeta-general}
\end{equation}
Besides, we also generalize our theorem to any GD with quasi-RMSprop scheduler,
where Scalar RMSprop serves as an important example.

The proof outlines of these theorems are given in \Cref{sec:proof-outline-main},
and the proof details are spread over
\Cref{sec:working-zone-proof,sec:gd-proof,sec:lr-proof,sec:ful-proof,sec:sph-proof,sec:rmsdrift-proof}.

\subsubsection{Spherical Optimization}

Consider Projected Gradient Descent (PGD) on the unit sphere $\sphS^{D-1}$:
\begin{equation} \label{eq:sph-thm-pgd}
  \vtheta_{t+1} \gets \Pi(\vtheta_t - \eeta_t \nabla \Loss(\vtheta_t)),
\end{equation}
where $\Pi: \vw \mapsto \frac{\vw}{\normtwosm{\vw}}$ is the projection operator,
and the effective learning rate $\eeta_t$
is set by a quasi-RMSprop scheduler with base learning rate $\eta > 0$ and decay rate $\beta = 1 - \Cb \eta^2 + O(\eta^4)$.

Recall that the dynamic is in the EoS regime
if $\eeta_t \approx 2/\lammax^{(t)}$
(see \Cref{lm:descent}).
When
$\vtheta_0$ is around a local minimizer $\vzeta_0$
and $\eeta_0$ is generated by a quasi-RMSprop scheduler,
the condition of being in EoS is essentially
$\frac{1}{\sqrt{\tilv_0}} \approx \frac{2}{\lamH_1(\vzeta_0)}$.
In the following theorem,
we show that PGD with quasi-RMSprop scheduler evolves as \eqref{eq:zeta-general}
if it is initially in the EoS regime.

\myparagraph{Initialization Scheme.}
Given a local minimizer $\vzeta_0 \in \manGa$ and a hyperparameter $\alpha_0$,
we initialize the initial direction $\vtheta_0$ and initial moment estimate $\tilv_0$ as follows:
draw $\vxi \sim \Normal(\vzero, \sigma_0^2 \mI / D)$
from Gaussian and set the direction of $\vw_0$ to $\frac{\vzeta_0 + \vxi}{\normtwosm{\vzeta_0 + \vxi}}$,
where $\sigma_0$ can take any value in $[\exp(-\alpha_0^2) \eta, \alpha_0 \eta]$;
then set $\tilv_0$
to be any positive value that satisfies $\abs{\frac{1}{\sqrt{\tilv_0}} - \frac{2}{\lamH_1(\vzeta_0)}} \le \alpha_0 \eta$.

\begin{theorem} \label{thm:main-sph}
  Under \Cref{ass:man,ass:topeigen},
  for PGD described as \eqref{eq:sph-thm-pgd}
  and initialized as above scheme for some $\vzeta_0 \in \manGa$ and some $1 \le \alpha_0 \le \eta^{-o(1)}$,
  with probability $1 - O(\alpha_0 \eta \sqrt{\log(1/\eta)})$,
  the trajectory of $\vtheta_t$ approximately tracks a sharpness-reduction flow $\vzeta: [0, T] \to \manGa$
  that starts from $\vzeta_0$ and evolves as the ODE \eqref{eq:zeta-general} (if solution exists),
  in the sense that $\normtwosm{\vtheta_t - \vzeta(t \eta^2)} = O(\alpha_0^2 \eta^{1/2} \log(1/\eta))$ for all $0 \le t \le T / \eta^2$.
\end{theorem}

\Cref{thm:gdwd-main} is a direct corollary of \Cref{thm:main-sph}, as \GDWD on scale-invariant loss can be seen as PGD with GWSI scheduler.
\begin{proof}[Proof for \Cref{thm:gdwd-main}]
\Cref{thm:gdwd-is-quasi-rmsprop} implies that \GDWD
on scale-invariant loss with LR $\heta$
and WD $\hlam$
can be seen as a PGD on $\sphS^{D-1}$
with quasi-RMSprop scheduler,
where the base learning rate and decay rate
of this scheduler is $\eta = \sqrt{2\ieta}$ and $\beta = 1 - 4 \ieta$.
Therefore, $\beta = 1 - \Cb \eta^2$ for $\Cb = 2$.

Now we apply \Cref{thm:main-sph} to prove \Cref{thm:gdwd-main}.
It is easy to translate the initial conditions of \Cref{thm:gdwd-main} to \Cref{thm:main-sph} with $\alpha_0 = \eta^{-o(1)} = \ieta^{-o(1)}$.
Let $\hat{\vzeta}(t)$
be the sharpness-reduction flow defined as in \eqref{eq:zeta-general} with $\Cb=2$ and horizon $\hat{T} := 2T$.
\begin{equation}
    \hat{\vzeta}(0) = \vzeta_0 \in \manGa, \qquad \frac{\dd}{\dd \tau} \hat{\vzeta}(\tau) = -\frac{\gradGa \log \lamH_1(\hat{\vzeta}(\tau))}{4 + \normtwosm{\gradGa \log \lamH_1(\hat{\vzeta}(\tau))}^2}.
\end{equation}
Then
with probability $1 - O(\alpha_0 \eta \sqrt{1/\delta}) = 1 - O(\eta^{1-o(1)})$,
$\vtheta_t$ approximately tracks $\hat{\vzeta}(\tau)$ in the sense that $\normtwosm{\vtheta_t - \hat{\vzeta}(t \eta^2)} = O(\eta^{1/2-o(1)})$ for all $0 \le t \le \hat{T} /\eta^2$.
Replacing $\eta$ with $\sqrt{2\ieta}$ gives
$\normtwosm{\vtheta_t - \hat{\vzeta}(2t\ieta)} = O(\ieta^{1/4-o(1)})$ for all $0 \le t \le T / \ieta$,
and the success probability becomes $1 - O(\ieta^{1/2-o(1)})$.
We can finish the proof by noting that
$\hat{\vzeta}(2\tau)$ is just $\vzeta(\tau)$ defined in \eqref{eq:main-zeta}.
\end{proof}

\subsubsection{Full Space Optimization}
Now we present our general theorem for GD on $\R^D$ with quasi-RMSprop scheduler.
Here
GD with quasi-RMSprop scheduler can be written as:
\begin{equation} \label{eq:ful-thm-gd}
  \vtheta_{t+1} \gets \vtheta_t - \eeta_t \nabla \Loss(\vtheta_t),
\end{equation}
where the effective learning rate $\eeta_t$
is set by a quasi-RMSprop scheduler with base learning rate $\eta > 0$ and decay rate $\beta = 1 - \Cb \eta^2 + O(\eta^4)$.

Similar to the spherical case,
we assume that there is a manifold $\manGa$
consisting of local minimizers,
but now we are not assuming scale-invariance.
\begin{assumption} \label{ass:man-ful}
  The loss function $\Loss: \R^D \to \R$ is $\contC^4$-smooth.
  $\manGa$ is a $\contC^2$-smooth, $\DGa$-dimensional submanifold of $\sphS^{D-1}$
  for some $0 \le \DGa \le D$, where every $\vtheta \in \manGa$ is a local minimizer of $\Loss$ on $\R^{D}$
  and $\rank(\mH(\vtheta)) = D - \DGa$.
\end{assumption}

The following assumption is essentially the same as \Cref{ass:topeigen} except that now $\manGa$ is defined differently.
\begin{assumption} \label{ass:topeigen-ful}
  For all $\vtheta \in \manGa$, $\lamH_1(\vtheta) > \lamH_2(\vtheta)$. That is, the top eigenvalue of $\mH(\vtheta)$ is unique.
\end{assumption}

Based on \Cref{ass:man-ful,ass:topeigen-ful} above,
we study GD
starting in the EoS regime.
Similar to the spherical case,
we focus on the case where $\vtheta_0$ is around a local minimizer $\vzeta \in \manGa$,
and EoS is the regime in which $\eeta_t \approx \frac{2}{\lamH_1(\vzeta_0)}$,
which 
essentially means $\frac{1}{\sqrt{\tilv_0}}$
if the effective LR is set by a quasi-RMSprop scheduler.
In the theorem below, we show that GD with quasi-RMSprop scheduler
tracks the sharpness-reduction flow defined in \eqref{eq:zeta-general}.
Note that the ODE here is the same as the spherical case, but $\manGa$ is defined differently.

\myparagraph{Initialization Scheme.}
Given a local minimizer $\vzeta_0 \in \manGa$ and a hyperparameter $\alpha_0$,
we initialize the initial parameter $\vtheta_0$ and initial moment estimate $\tilv_0$ as follows:
draw $\vxi \sim \Normal(\vzero, \sigma_0^2 \mI / D)$
from Gaussian and set $\vw_0 \gets \vzeta_0 + \vxi$,
where $\sigma_0$ can take any value in $[\exp(-\alpha_0^2) \eta, \alpha_0 \eta]$;
then set $\tilv_0$
to be any positive value that satisfies $\abs{\frac{1}{\sqrt{\tilv_0}} - \frac{2}{\lamH_1(\vzeta_0)}} \le \alpha_0 \eta$.

\begin{theorem} \label{thm:main-ful}
  Under \Cref{ass:man-ful,ass:topeigen-ful},
  for GD described as \eqref{eq:ful-thm-gd}
  and initialized as above scheme for some $\vzeta_0 \in \manGa$ and some $1 \le \alpha_0 \le \eta^{-o(1)}$,
  with probability $1 - O(\alpha_0 \eta \sqrt{\log(1/\eta)})$,
  the trajectory of $\vtheta_t$ approximately tracks a sharpness-reduction flow $\vzeta: [0, T] \to \manGa$
  that starts from $\vzeta_0$ and evolves as the ODE \eqref{eq:zeta-general} (if solution exists),
  in the sense that $\normtwosm{\vtheta_t - \vzeta(t \eta^2)} = O(\alpha_0^2 \eta^{1/2} \log(1/\eta))$ for all $0 \le t \le T / \eta^2$.
\end{theorem}

A direct corollary is that Scalar RMSprop follows the sharpness-reduction flow~\eqref{eq:zeta-general},
since the RMSprop scheduler is a quasi-RMSprop scheduler.
\begin{corollary} \label{cor:scalar-rmsprop}
  The statement of \Cref{thm:main-ful} holds
  for Scalar RMSprop
  if the decay rate is set to $\beta = 1 - \Cb\eta^2 + O(\eta^2)$ for some constant $\Cb > 0$.
\end{corollary}

%% file: pac-bayes.tex
\section{PAC-Bayes Bounds Based on Spherical Sharpness} \label{sec:pac-bayes} 

In this section we give the PAC-Bayes bound for generalization error using spherical sharpness~(\Cref{def:sph-sha}). We will start with our setting and then recap the classic PAC-Bayes theorem in \citep{mcallester2003simplified}. The main result in this section is \Cref{thm:spherical_Sharpness_pac_bayes}.

\paragraph{Setting.} Let $\ell(\vw,\vz)$ be the loss of parameter $\vw$ on data point $\vz$ and assume  $\ell_{\max}=
\sup_{z\in\DatZ,\vw\in \sphS^{D-1}} \ell(\vw,\vz)<\infty$. Let $ \DatS := \{ \vz_i \}_{i=1}^{n}$ where $\vz_i$ are
sampled independently. Different to the previous notation, we use $\Loss_\DatS$ to denote the empirical loss on training
dataset $\DatS$ where we run optimization algorithms and $\Loss$ to denote the population loss. Let
$\thirdL{\Loss_\DatS}:= \sup_{\vtheta\in\sphS^{D-1}} \normtwosm{\nabla^3  \Loss_\DatS(\vtheta)}$. Since
$\sphS^{D-1}$ is compact and  $\normtwosm{\nabla^3 \Loss_\DatS(\cdot)}$ is continuous,  $\thirdL{\Loss_\DatS}$ is finite.

\begin{theorem}[PAC-Bayes theorem~\citep{mcallester2003simplified}]\label{thm:PAC_bayes_theorem}
Given any distribution $P$  on $\mathbb{R}^D$, with at least $1-\delta$ probability over the randomness of the dataset $\DatS$, for all distribution $Q$ on  $\mathbb{R}^D$, it holds that 
\begin{align*}\label{eq:pac_bayes}
\E_{\vw\sim Q} \Loss_{S}(\vw) - \E_{\vw\sim Q}\Loss(\vw)	 \le     \ell_{\max} \sqrt{\frac{\KL(Q \| P) + \ln(n/\delta)}{2(n-1)}}.
\end{align*} 
	
\end{theorem}

Now we are ready to state the main theorem in this section, \Cref{thm:spherical_Sharpness_pac_bayes}, which shows that small spherical sharpness $\lambda_1(\nabla^2 \Loss_S(\vtheta))$ leads to small generalization error.
\begin{theorem}\label{thm:spherical_Sharpness_pac_bayes} For any $\sigma \le \frac{1}{2 + 2\sqrt{(\ln n)/D}}$, with at
	least $1-\delta$ probability over the randomness of the dataset $\DatS$, where  $\DatS := \{ \vz_i \}_{i=1}^{n}$ and
	every $\vz_i$ is sampled independently, for any $\vtheta \in \sphS^{D-1}$,
\begin{align*}
&\E_{\vepsilon\sim \mathcal{N}(\bm{0}, \sigma^2\mI_D / D)} \Loss(\vtheta+\vepsilon) - \Loss_\DatS(\vtheta)\\
& \quad \le\frac{\sigma^2}{2} \lambda_1 (\nabla^2\Loss_\DatS(\vtheta)) 
+ \frac{16\sigma^3}{3}\thirdL{\Loss_\DatS} (1 + ((\ln n)/D)^{1.5}) 
+ \ell_{\max} \sqrt{\frac{D/\sigma^2 + 2\ln(n/\delta)}{n-1}},
\end{align*}
Thus with the standard assumption in \citep{foret2021sam} that $\E_{\vepsilon\sim \mathcal{N}(\bm{0}, \sigma^2\mI_D/D)} \Loss(\vtheta+\vepsilon) \ge \Loss(\vtheta)$, we have the same upper bound for $\Loss(\vtheta) - \Loss_{\DatS}(\vtheta)$.
\end{theorem}

To prove \Cref{thm:spherical_Sharpness_pac_bayes}, we will first need the following lemma. The proof is standard so it is omitted.
\begin{lemma} 
Let $Q = \mathcal{N}(\vmu_Q,\sigma^2_Q \mI_D)$ and $P = \mathcal{N}(\vmu_P,\sigma^2_P\mI_D)$, we have that 
	\begin{align*}
		D_{KL}(Q\|P)  = \frac{1}{2}\left[ \frac{D \sigma_Q^2 + \norm{\vmu_P-\vmu_Q}_2^2}{\sigma_P^2}-D+D\log\left( \frac{\sigma^2_P}{\sigma^2_Q}\right)\right]
	\end{align*}
\end{lemma}

\begin{proof}[Proof of \Cref{thm:spherical_Sharpness_pac_bayes}]
Let $Q := \Normal(\vtheta, \sigma^2 \mI_D / D)$, $P := \Normal(\vzero, \sigma^2 \mI_D / D)$.
Then $D_{KL}(Q\|P) = \frac{D}{2\sigma^2}$.
By \Cref{thm:PAC_bayes_theorem} we have
\begin{align*}
\E_{\vepsilon\sim \mathcal{N}(\bm{0}, \sigma^2\mI_D / D)} \Loss(\vtheta+\vepsilon) - \E_{\vepsilon\sim \mathcal{N}(\bm{0}, \sigma^2\mI_D / D)} \Loss_{S}(\vtheta+\vepsilon) \le \ell_{\max}\sqrt{\frac{\frac{D}{2\sigma^2 } + \ln(n/\delta)}{2(n-1)}}.
\end{align*}
Let $h(\sigma):= (1+\sqrt{(\ln n) / D})\sigma$. By assumption it holds that $h(\sigma)\le \frac{1}{2}$. Thus by Lemma 1 in \cite{laurent2000adaptive}, we have for any positive $t$:
\begin{align}
	\mathbb{P}[\norm{\vepsilon}_2^2-\sigma^2\ge 2\sigma^2 \sqrt{t/D} +2t\sigma^2 / D] \le \exp(-t).
\end{align}
Therefore, with probability $1-1/\sqrt{n}$, we have 
\begin{align*}
	\norm{\vepsilon}_2^2 \le \sigma^2\left(1 + 2 \sqrt{(\ln \sqrt{n}) / D} + 2 \cdot (\ln \sqrt{n}) / D\right)\le \sigma^2\left(1+\sqrt{(\ln n)/D}\right)^2\le h(\sigma)^2.
\end{align*}
Thus for any $\vtheta$ with $\norm{\vtheta}_2=1$, we have
\begin{align*}
	\E_{\vepsilon\sim \mathcal{N}(\bm{0}, \sigma^2\mI_D / D)} \Loss_{S}(\vtheta+\vepsilon) \le \ell_{\max}/\sqrt{n} + \E_{\vepsilon\sim \mathcal{N}(\bm{0}, \sigma^2\mI_D / D)} \left[ \Loss_{S}(\vtheta+\vepsilon) \onec{\norm{\vepsilon}_2\le h(\sigma)} \right].
\end{align*}

By Taylor expansion, $$\Loss_\DatS(\vtheta+\vepsilon)\le \Loss_\DatS(\vtheta) + \dotp{\vepsilon}{\nabla \Loss_\DatS(\vtheta)}+\frac{1}{2}\dotp{\vepsilon}{\nabla^2 \Loss_\DatS(\vtheta)\vepsilon} + \frac{1}{6}\sup_{\lambda\in [0,1]} \dotp{\nabla^3 \Loss_\DatS(\vtheta+\lambda\vepsilon)}{\vepsilon^{\otimes 3}}.$$

Note that $\norm{\vtheta+\lambda \vepsilon}_2 \ge \norm{\vtheta}_2 - \lambda\norm{\vepsilon}_2 \ge \frac{1}{2}$ for all $\lambda\in [0,1]$, 
it holds that $\norm{\nabla^3 \Loss_\DatS(\vtheta+\lambda\vepsilon)}= \norm{\vtheta+\lambda \vepsilon}_2^{-3}\norm{\nabla^3 \Loss_\DatS( \frac{\vtheta+\lambda\vepsilon}{\norm{\vtheta+\lambda\vepsilon}_2})}_2 \le 8 \thirdL{\Loss_\DatS} $, we have
\begin{align*}
	\E_{\vepsilon\sim \mathcal{N}(\bm{0}, \sigma^2\mI_D / D)} \left[\Loss_{S}(\vtheta+\vepsilon) \onec{\norm{\vepsilon}\le h(\sigma)}\right]
	&\le \Loss_\DatS(\vtheta) 
	+ \frac{\sigma^2}{2D} \Tr[\nabla^2 \Loss_\DatS(\vtheta)] + \frac{4}{3}\thirdL{\Loss_\DatS} h(\sigma)^3 \\
	&\le \Loss_\DatS(\vtheta) 
	+ \frac{\sigma^2}{2} \lambda_1(\nabla^2 \Loss_\DatS(\vtheta)) + \frac{4}{3}\thirdL{\Loss_\DatS} h(\sigma)^3.
\end{align*}
Thus we conclude that 
	\begin{align*}
&~\E_{\vepsilon\sim \mathcal{N}(\bm{0}, \sigma^2\mI_D / D)} \Loss(\vtheta+\vepsilon) - \Loss_\DatS(\vtheta)\\
\le &~\frac{\sigma^2}{2}\lambda_1(\nabla^2 \Loss_\DatS(\vtheta)) + \frac{4}{3}\thirdL{\Loss_\DatS} h(\sigma)^3 +
\ell_{\max} \left(\sqrt{\frac{\frac{D}{2\sigma^2 } + \ln(n/\delta)}{2(n-1)}} +1 \right)\\
\le &~\frac{\sigma^2}{2} \lambda_1 (\nabla^2\Loss_\DatS(\vtheta))  + \frac{16\sigma^3}{3}\thirdL{\Loss_\DatS} (1 + (\ln n / D)^{1.5}) +
\ell_{\max} \sqrt{\frac{D/\sigma^2 + 2\ln(n/\delta)}{n-1}},
	\end{align*}
which completes the proof.
\end{proof}

%% file: add-prelim.tex
\section{Additional Preliminaries} \label{sec:add-prelim}

\subsection{Additional Notations}
We use $\cl(\gM)$ to denote the closure of a set $\gM$.
For $\vtheta \in \R^D$ and $\gM \subseteq \R^D$,
we use $d_2(\vtheta, \gM) := \inf\{ \normtwosm{\vtheta - \vphi} : \vphi \in \gM \}$
to denote the $\normltwo$-distance from $\vtheta$ to $\gM$.
For $\vtheta \in \R^D$ and $\epsilon \ge 0$,
$\Ball{\epsilon}{\vtheta} := \{ \vtheta' \in \R^D : \normtwosm{\vtheta - \vtheta'} < \epsilon \}$
is the open $\epsilon$-ball centered at $\vtheta$.
For a set $\gM \subseteq \R^D$, $\gM^{\epsilon} := \bigcup_{\vtheta \in \gM} \Ball{\epsilon}{\vtheta}$
is the (open) $\epsilon$-neighborhood of $\gM$.
All the manifolds in our paper refer to manifolds without boundary.
For a manifold $\manGa$, we use $\TGa{p}$ and $\NGa{p}$
to denote the tangent and normal space of $\manGa$ at a point $p \in \manGa$.

Given a function $f: \R^D \to \R^n$
and vectors $\vx, \vv \in \R^D$,
we use $\partial f_{\vx}[\vv]$
to denote the directional derivative
$\partial f_{\vx}[\vv] := \lim_{t \to 0} \frac{1}{t} f(\vx + t\vv)$,
which also equals to the Jacobian of $f$ at $\vx$ multiplied with $\vv$.
We use $\partial^2 f_{\vx}[\vv, \vu]$ to denote the second-order derivative
$\partial (\partial f_{\vx}[\vv])_{\vx}[\vu]$.
For a real-valued function $\Loss: \R^D \to \R$,
we use $\nabla \Loss(\vtheta)$ for gradient,
$\nabla^2 \Loss(\vtheta)$ for Hessian.
For third-order derivatives of $\Loss$,
we define $\partial^3 \Loss_{\vtheta}[\vv, \vu] := \partial^2 (\nabla \Loss)_{\vtheta}[\vv, \vu] \in \R^D$.
If $\Loss$ is $\contC^4$-smooth, then the following Taylor expansion holds for $\nabla \Loss$:
\begin{align*}
  \nabla \Loss(\vtheta + \vx) = \nabla \Loss(\vtheta) + \nabla^2 \Loss(\vtheta) \vx + \frac{1}{2} \partial^3 \Loss_{\vtheta}[\vx, \vx] + O(\normtwosm{\vx}^3).
\end{align*}

\subsection{Scale-Invariant Functions}

The following lemma summarizes a few important properties of scale-invariant loss that have been 
exploited in previous works \citep{arora2018theoretical,li2020exp}.
For completeness, we include a proof here.
\begin{lemma} \label{lm:scale-invariant-grad-hess}
    The following hold for a twice-differentiable scale-invariant function $\Loss(\vw)$:
    \begin{enumerate}
      \item The gradient is $(-1)$-homogeneous and it is always perpendicular to $\vw$, i.e.,
        $\nabla \Loss(c\vw) = c^{-1} \nabla \Loss(\vw)$ for all $c > 0$ and $\dotp{\nabla \Loss(\vw)}{\vw} = 0$;
      \item The Hessian matrix is $(-2)$-homogeneous, i.e.,
        $\nabla^2 \Loss(c\vw) = c^{-2} \nabla^2 \Loss(\vw)$ for all $c > 0$.
      \item $\nabla^2 \Loss(\vw) \vw = -\nabla \Loss(\vw)$.
    \end{enumerate}
\end{lemma}
\begin{proof}
  Taking gradients with respect to $\vw$ on both sides of $\Loss(c\vw) = \Loss(\vw)$
  gives $\nabla \Loss(c\vw) = c^{-1} \nabla \Loss(\vw)$.
  Taking gradients again proves that $\nabla^2 \Loss(c\vw) = c^{-2} \nabla^2 \Loss(\vw)$.

  Taking derivative with respect to $c$ on both sides of $\Loss(c\vw) = \Loss(\vw)$
  gives $\dotpsm{\nabla \Loss(c\vw)}{\vw} = 0$.
  Taking $c=1$ gives $\dotpsm{\nabla \Loss(\vw)}{\vw} = 0$.
  Finally, we take gradients with respect to $\vw$,
  then $\nabla^2 \Loss(\vw) \vw + \nabla \Loss(\vw) = \vzero$.
\end{proof}

Now we provide proofs for \Cref{lm:scale-invariant-on-sphere} and \Cref{thm:li-gdwd} in \Cref{sec:prelims}.
\begin{proof}[Proof for \Cref{lm:scale-invariant-on-sphere}]
By \eqref{eq:upd-gdwd} and definition of $\vtheta_t$,
\begin{align*}
  \vtheta_{t+1} = \Pi(\vw_{t+1})
  &= \Pi\left((1-\heta\hlam)\vw_t - \heta \nabla \Loss(\vw_t)\right).
\end{align*}
Since $\Pi(\vx) = \Pi(c\vx)$ for all $c > 0$,
we can divide $\vw_{t+1}$ by $(1 - \heta \hlam) \normtwosm{\vw_t}$ and obtain
\begin{align*}
  \vtheta_{t+1}
  &= \Pi\left(\vtheta_t - \frac{\heta}{(1-\heta\hlam)\normtwosm{\vw_t}} \nabla \Loss(\vw_t)\right).
\end{align*}
Note that $\nabla \Loss(\vw_t) = \frac{1}{\normtwosm{\vw_t}} \Loss(\vtheta_t)$ by \Cref{lm:scale-invariant-grad-hess}.
So we can further rewrite the above formula:
\begin{align*}
  \vtheta_{t+1}
  &= \Pi\left(\vtheta_t - \frac{\heta}{(1-\heta\hlam)\normtwosm{\vw_t}^2} \nabla \Loss(\vtheta_t)\right),
\end{align*}
which proves the lemma by definition of $\eeta_t$.
\end{proof}

\begin{proof}[Proof of \Cref{thm:li-gdwd}]
	The first claim is directly from the Lemma D.2 of \citep{li2022robust}. The second claim about $\heta$ holds by scrutinizing their proof.
\end{proof}

\subsection{\PolyakLojasiewicz Condition}

\begin{definition}[\PolyakLojasiewicz]
  For a loss function $\Loss(\vtheta)$ and a constant $\mu > 0$,
  we say that $\Loss$ satisfies $\mu$-\PolyakLojasiewicz condition
  (or $\mu$-PL for brevity)
  on a set $U$ if
  \[
    \frac{1}{2}\normtwosm{\nabla \Loss(\vtheta)}^2 \ge \mu \cdot \left( \Loss(\vtheta) - \inf_{\vtheta' \in U} \Loss(\vtheta') \right),
  \]
  for all $\vtheta \in U$.
\end{definition}

\subsubsection{Full Space Optimization}

\begin{theorem} \label{thm:rank-to-pl-ful}
  Let $\Loss: \R^D \to \R$ be a $\contC^3$-smooth function,
  and $\manGa$ be a $\contC^1$-smooth, $\DGa$-dimensional submanifold of $\R^D$,
  where every $\vtheta \in \manGa$ is a local minimizer of $\Loss$ and $\rank(\nabla^2 \Loss(\vtheta)) = D - \DGa$.
  If $\gZ$ is a compact subset of $\manGa$,
  then there exist $\epsilon > 0, \mu > 0$ such that
  $\cl(\gZ^{\epsilon}) \cap \manGa$ is compact and
  $\Loss$ satisfies $\mu$-PL
  on
  $\gZ^{\epsilon}$.
\end{theorem}
\begin{proof}
  Since $\gZ$ is compact and $\manGa$ is a submanifold of $\R^D$,
  we can choose a small $\delta > 0$ such that 
  $\gN := \manGa \cap \cl(\gZ^{\delta})$ is compact. 

  It can be shown
  that 
  there exists an open neighborhood $U$ of the compact submanifold $\gN$ such that
  for every $\vtheta \in U$,
  the nearest point on $\gN$,
  $P(\vtheta) := \argmin\{ \normtwosm{\vtheta - \vphi} : \vphi \in \gN \}$,
  exists and is unique~\citep{foote1984distance}. 

  We choose $\epsilon < \delta / 2$
  to be small enough so that $\epsilon < \delta / 2$ and $\gZ^{\epsilon} \subseteq U$.
  For $\vtheta \in \gZ^{\epsilon}$, $P(\vtheta)$ lies in the interior of the manifold $\gN$,
  since $d_2(P(\vtheta), \gZ) \le \normtwosm{P(\vtheta) - \vtheta} + d(\vtheta, \gZ) < \delta$.
  Then it must hold that $\vtheta - P(\vtheta) \in \NGa[\gN]{P(\vtheta)}$;
  otherwise
  the differential of $\normtwosm{\vtheta - \vphi}^2$ on $\gN$ is non-zero
  at $\vphi = P(\vtheta)$,
  which contradicts to the fact that $P(\vtheta)$ is the nearest point to $\vtheta$ on $\gN$.

  Since $\Loss \in \contC^3$,
  and $\nabla^2 \Loss(\vtheta)$ is of constant rank $D-\DGa$
  on the compact manifold $\gN$,
  there exist
  $\lambda_{\min} > 0, \lambda_{\max} > 0$
  such that $\lambda_{D-\DGa}(\nabla^2 \Loss(\vtheta)) \ge \lambda_{\min}$
  and $\lambda_{1}(\nabla^2 \Loss(\vtheta)) \le \lambda_{\max}$
  for all $\vtheta \in \gN$.
  Also by $\Loss \in \contC^3$ and compactness of $\gN$,
  there exists $C_3 > 0$ such that
  the following Taylor expansions hold
  for all $\vtheta \in \gZ^{\epsilon}$,
  \begin{align*}
    \Loss(\vtheta) - \Loss(P(\vtheta)) &\le (\vtheta - P(\vtheta))^\top \nabla^2 \Loss(P(\vtheta)) (\vtheta - P(\vtheta)) + C_3 \normtwosm{\vtheta - P(\vtheta)}^3, \\
    &\le (\lambda_{\max} + C_3 \epsilon) \cdot \normtwosm{\vtheta - P(\vtheta)}^2. \\
    \normtwosm{\nabla \Loss(\vtheta)}^2
    &\ge (\vtheta - P(\vtheta))^\top \left(\nabla^2 \Loss(P(\vtheta))\right)^2 (\vtheta - P(\vtheta)) - C_3 \normtwosm{\vtheta - P(\vtheta)}^3 \\
    &\ge (\lambda_{\min}^2 - C_3 \epsilon) \cdot \normtwosm{\vtheta - P(\vtheta)}^2.
  \end{align*}
  Then $\Loss$ satisfies $\mu$-PL on $\gZ^{\epsilon}$
  for $\mu := \frac{\lambda_{\min}^2 - C_3 \epsilon}{2(\lambda_{\max} + C_3 \epsilon)}$,
  which is positive if we choose $\epsilon$ to be small enough in the beginning.
\end{proof}

\begin{theorem} \label{thm:pl-converge-near-ful}
  If $\vthetaopt$ is a local minimizer of a $\contC^2$-smooth function $\Loss: \R^D \to \R$ and
  $\Loss$ is $\mu$-PL on an open neighborhood $U$ of $\vthetaopt$,
  then for any $\vtheta_0$ sufficiently close to $\vthetaopt$,
  a gradient flow $\frac{\dd \vtheta}{\dd t} = -\nabla \Loss(\vtheta)$
  starting with $\vtheta_0$ converges to a point $\vtheta_{\infty}$ as
  $t \to +\infty$ and $\normtwosm{\vtheta_{\infty} - \vthetaopt} = O(\normtwosm{\vtheta_0 - \vthetaopt})$.
\end{theorem}
\begin{proof}
Let $T := \inf\{ t : \vtheta(t) \notin U \}$. For all $t < T$,
\begin{align*}
  \frac{\dd}{\dd t} (\Loss(\vtheta) - \Loss(\vthetaopt))^{1/2}
  &= \frac{1}{2} (\Loss(\vtheta) - \Loss(\vthetaopt))^{-1/2} \cdot \dotp{\nabla\Loss(\vtheta)}{\frac{\dd \vtheta}{\dd t}} \\
  &= -\frac{1}{2}(\Loss(\vtheta) - \Loss(\vthetaopt))^{-1/2} \cdot \normtwosm{\nabla\Loss(\vtheta)} \cdot \normtwo{\frac{\dd \vtheta}{\dd t}}.
\end{align*}
Since $(\Loss(\vtheta) - \Loss(\vthetaopt))^{1/2} \le \frac{1}{\sqrt{2\mu}} \normtwosm{\nabla \Loss(\vtheta)}$,
we have
\begin{align*}
  \frac{\dd}{\dd t} (\Loss(\vtheta) - \Loss(\vthetaopt))^{1/2}
  \le -\frac{\sqrt{2\mu}}{2} \normtwo{\frac{\dd \vtheta}{\dd t}}.
\end{align*}
Integrating on both sides proves the following
\[
  \frac{\sqrt{2\mu}}{2}\int_{0}^{T} \normtwo{\frac{\dd \vtheta(\tau)}{\dd \tau}} \dd \tau \le 
  \sqrt{\Loss(\vtheta_0) - \Loss(\vthetaopt)} = O(\normtwosm{\vtheta_0 - \vthetaopt}).
\]
So if $\normtwosm{\vtheta_0 - \vthetaopt}$ is small enough, then $T = +\infty$
and $\vtheta(t)$ converges to a point in $U$ as $t \to +\infty$.
Moreover, $\normtwosm{\vtheta_{\infty} - \vthetaopt} \le \normtwosm{\vtheta_{\infty} - \vtheta_0} + \normtwosm{\vtheta_0 - \vthetaopt} = O(\normtwosm{\vtheta_0 - \vthetaopt})$.
\end{proof}

\subsubsection{Spherical Optimization}

\begin{theorem} \label{thm:rank-to-pl-sph}
  Let $\Loss: \R^D \to \R$ be a $\contC^3$-smooth scale-invariant function,
  and $\manGa$ be a $\contC^1$-smooth, $(\DGa - 1)$-dimensional submanifold of $\sphS^{D-1}$,
  where every $\vtheta \in \manGa$ is a local minimizer of $\Loss$
  on $\sphS^{D-1}$
  and $\rank(\nabla^2 \Loss(\vtheta)) = D - \DGa$.
  If $\gZ$ is a compact subset of $\manGa$,
  then there exist $\epsilon > 0, \mu > 0$ such that
  $\cl(\gZ^{\epsilon}) \cap \manGa$ is compact
  and
  $\Loss$ satisfies $\mu$-PL
  on $\gZ^{\epsilon} \cap \sphS^{D-1}$.
\end{theorem}
\begin{proof}
  Let $\manGa' := \{ \nu \vtheta : \vtheta \in \manGa, \nu > 0\}$.
  Then $\manGa'$ is a $\contC^1$-smooth, $\DGa$-dimensional submanifold of $\R^D$,
  where every $\vtheta \in \manGa$ is a local minimizer of $\Loss$ on $\R^D$
  and $\rank(\nabla^2 \Loss(\vtheta)) = D - \DGa$.
  By \Cref{thm:rank-to-pl-sph},
  there exist $\epsilon > 0, \mu > 0$ such that $\Loss$ satisfies $\mu$-PL
  on $\gZ^{\epsilon}$, so it satisfies $\mu$-PL on $\gZ^{\epsilon} \cap \sphS^{D-1}$.
\end{proof}

\begin{theorem} \label{thm:pl-converge-near-sph}
  If $\vthetaopt$ is a local minimizer of a $\contC^2$-smooth and scale-invariant function $\Loss: \R^D \setminus \{\vzero\} \to \R$ and
  $\Loss$ is $\mu$-PL on an open neighborhood $U$ of $\vthetaopt$ on $\sphS^{D-1}$,
  then for any $\vtheta_0 \in \sphS^{D-1}$ sufficiently close to $\vthetaopt$,
  a gradient flow $\frac{\dd \vtheta}{\dd t} = -\nabla \Loss(\vtheta)$
  starting with $\vtheta_0$ converges to a point $\vtheta_{\infty}$ as
  $t \to +\infty$ and $\normtwosm{\vtheta_{\infty} - \vthetaopt} = O(\normtwosm{\vtheta_0 - \vthetaopt})$.
\end{theorem}
\begin{proof}
  Since $\Loss$ is $\mu$-PL on $U\subseteq \sphS^{D-1}$ and scale-invariant, we know $\Loss$ is $\frac{\mu}{2}$-PL on an open set in $\mathbb{R}^D$,
  $U' = \{\vw : \frac{\vw}{\norm{\vw}_2} \in U, \norm{\vw}_2\in [\frac{1}{\sqrt{2}}, \sqrt{2}]\}$.
  The proof is completed by  applying \Cref{thm:pl-converge-near-ful}.
\end{proof}

\section{Supplementary Material for Appendix~\ref{sec:reform-gdwd}}
\label{sec:connection-proof}
  
\subsection{Proof for Theorem \ref{thm:connect-siwd-rmsprop}} \label{sec:proof-connect-siwd-rmsprop}

\begin{lemma} \label{lm:connection-prelim}
  In the setting of \Cref{thm:connect-siwd-rmsprop},
  \begin{align}
    \normtwosm{\vw_{t+1}}^2 &= (1-\heta \hlam)^2 \normtwosm{\vw_t}^2 + \frac{\heta^2}{\normtwosm{\vw_t}^2} \normtwosm{\nabla \Loss(\vtheta_t)}^2. \label{eq:connection-w-norm2}
  \end{align}
\end{lemma}
\begin{proof}
Recall that $\vw_{t+1} = (1 - \heta \hlam) \vw_t - \heta \nabla \Loss(\vw_t)$.
By scale-invariance, $\dotpsm{\nabla \Loss(\vw_t)}{\vw_t} = 0$. Then by Pythagorean theorem (or Gougu Theorem),
\begin{align*}
  \normtwosm{\vw_{t+1}}^2 = (1 - \heta \hlam)^2 \normtwosm{\vw_t}^2 + \heta^2 \normtwosm{\nabla \Loss(\vw_t)}^2.
\end{align*}
Since $\nabla \Loss(\vw_t) = \frac{1}{\normtwosm{\vw_t}} \nabla \Loss(\vtheta_t)$ by scale-invariance, we can rewrite
the last term $\heta^2 \normtwosm{\nabla \Loss(\vw_t)}^2$ as $\frac{\heta^2}{\normtwosm{\vw_t}^2} \normtwosm{\nabla \Loss(\vtheta_t)}^2$,
which implies \eqref{eq:connection-w-norm2}.
\end{proof}

\begin{proof}[Proof for \Cref{thm:connect-siwd-rmsprop}]
Squaring both sides of \eqref{eq:connection-w-norm2}, we have
\begin{align*}
  \normtwosm{\vw_{t+1}}^4 &= (1-\heta \hlam)^4 \normtwosm{\vw_t}^4 + 2(1-\heta \hlam)^2 \heta^2 \normtwosm{\nabla \Loss(\vtheta_t)}^2 + \frac{\heta^4}{\normtwosm{\vw_t}^4} \normtwosm{\nabla \Loss(\vtheta_t)}^4.
\end{align*}
Let $\beta := (1-\heta\hlam)^4$, $\tilv_t := \frac{1}{\eeta_t^2} = \frac{(1-\heta \hlam)^2}{\heta^2} \normtwosm{\vw_t}^4$. Then
\begin{align*}
    \tilv_{t+1} &= \frac{(1-\heta \hlam)^2}{\heta^2} \left((1-\heta \hlam)^4 \normtwosm{\vw_t}^4 + 2(1-\heta \hlam)^2 \heta^2 \normtwosm{\nabla \Loss(\vtheta_t)}^2 + \frac{\heta^4}{\normtwosm{\vw_t}^4} \normtwosm{\nabla \Loss(\vtheta_t)}^4 \right) \\
    &= (1-\heta \hlam)^4 \tilv_t + 2(1-\heta\hlam)^4 \normtwosm{\nabla \Loss(\vtheta_t)}^2 + \frac{(1-\heta\hlam)^2\heta^2}{\normtwosm{\vw_t}^4} \normtwosm{\nabla \Loss(\vtheta_t)}^4 \\
    &= \beta \tilv_t + 2\beta \normtwosm{\nabla \Loss(\vtheta_t)}^2 + \frac{\beta}{\tilv_t} \normtwosm{\nabla \Loss(\vtheta_t)}^4,
\end{align*}
where the last equality uses the definition of $\beta$ and $\tilv_t$.

Let $\eta := \sqrt{\frac{1-\beta}{2\beta}}$ and $\barg_t := \normtwosm{\nabla \Loss(\vtheta_t)} / \eta$. Then
\begin{align*}
  \tilv_{t+1} &= \beta \tilv_t + 2\beta \cdot \eta^2 \barg_t^2 + \frac{\beta}{\tilv_t} \cdot \eta^4\barg_t^4 \\
  &= \beta \tilv_t + (1-\beta) \barg_t^2 + \frac{1}{4\beta \tilv_t} (1-\beta)^2 \barg_t^4,
\end{align*}
which is exactly the update rule of GWSI scheduler.
\end{proof}

\subsection{Proof for Theorem \ref{thm:gdwd-is-quasi-rmsprop}} \label{sec:proof-gdwd-is-quasi-rmsprop}

\begin{proof}
We specify the quasi-RMSprop scheduler $\qrmsS$ as follows.
Let $\gP_{\qrmsS} := \{(\eta, \beta) : \beta = 1- 2 \eta^2, \eta \in (0, \frac{1}{\sqrt{2}})\}$.
Given hyperparameters $\eta, \beta$,
we define $\beta' := (1-\frac{1}{4}(1 - \beta))^4$
and $\eta' := \sqrt{(\beta' - 1) / 2}$.
Then $\qrmsS$ produces the effective LRs as a GWSI scheduler with $(\eta', \beta')$:
\begin{align*}
  \eeta_t &\gets \frac{1}{\sqrt{\tilv_t}}, &
  \tilv_{t+1} &\gets \beta' \tilv_t + (1-\beta') \hatg_t^2 + \frac{1}{4\beta' \tilv_t} (1-\beta')^2 \hatg_t^4, &
  &\text{where} \quad \hatg_t := \normtwosm{\vg_t} / \eta'.
\end{align*}
When $\eta = \sqrt{2\ieta}$ and $\beta = 1 - 4\ieta$,
it is easy to see that $\qrmsS$ produces the same effective LRs as \GDWD on scale-invariant functions (\Cref{thm:connect-siwd-rmsprop}).
Now we only need to verify that
$\qrmsS$ is indeed a quasi-RMSprop scheduler.

When $\beta$ is close enough to $1$, we have $\beta' = \beta + O((1-\beta)^2)$, $\beta' \ge 1/2$, $\eta' = \eta \cdot (1 + O(1-\beta))$.
Let $C_0$ be a constant such that
$\eta / \eta' \le C_0, \abssm{1 - (\eta / \eta')^2} \le C_0(1 - \beta),
1-\beta' \le C_0(1-\beta), \abssm{\beta' - \beta} \le C_0 (1-\beta)^2$.
Let $\barg_t := \normtwosm{\vg_t} / \eta$.
Then $\hatg_t = (\eta / \eta')\barg_t$, and thus
\begin{align*}
  \hatg_t &\le C_0 \barg_t, &
  \abssm{\hatg_t^2 - \barg_t^2} &= \abssm{1 - (\eta/\eta')^2} \cdot \barg_t^2 \le C_0 (1-\beta) \barg_t^2.
\end{align*}
We only need to verify that 
$\abs{\tilv_{t+1} - \left(\beta \tilv_t + (1-\beta) \barg_t^2 \right)} \le \delta(\tilv_t) \cdot (1-\beta)^2 \cdot P(\barg_t)$
for some continuous function $\delta$ and some polynomial $P$.
\begin{align*}
    & \abs{\tilv_{t+1} - \left(\beta \tilv_t + (1-\beta) \barg_t^2 \right)} \\
    & \qquad \le \abs{\tilv_{t+1} - \left(\beta' \tilv_t + (1-\beta') \hatg_t^2 \right)}
    + (1-\beta')\abs{\hatg^2_t - \barg_t^2}
    + \abssm{\beta' - \beta} \cdot (\tilv_t + \barg_t^2).
\end{align*}
For the first term,
we have
\[
\abs{\tilv_{t+1} - \left(\beta' \tilv_t + (1-\beta') \hatg_t^2 \right)} = 
\frac{1}{4 \beta' \tilv_t} (1-\beta')^2 \hatg_t^4 \le \frac{C_0^6}{2\tilv_t} \cdot (1-\beta)^2 \cdot \barg_t^4.
\]
For the second and third terms, we have
\begin{align*}
  (1-\beta')\abs{\hatg^2_t - \barg_t^2}
  &\le C_0^2(1-\beta)^2 \barg_t^2, &
  \abssm{\beta' - \beta} \cdot (\tilv_t + \barg_t^2)
  &\le C_0(1-\beta)^2 \cdot (\tilv_t + \barg_t^2).
\end{align*}
Finally we can conclude
\begin{align*}
  \abs{\tilv_{t+1} - \left(\beta \tilv_t + (1-\beta) \barg_t^2 \right)} 
  &\le 
\frac{C_0^6}{2\tilv_t} \cdot (1-\beta)^2 \cdot \barg_t^4
+
  C_0^2(1-\beta)^2 \barg_t^2
+ C_0(1-\beta)^2 \cdot (\tilv_t + \barg_t^2) \\
&\le \left(\frac{C_0^6}{2\tilv_t} + C_0^2 + C_0 (1 + \tilv_t) \right) \cdot (1-\beta)^2 \cdot (\barg_t^4 + \barg_t^2),
\end{align*}
which verifies that $\qrmsS$ is indeed a quasi-RMSprop scheduler.
\end{proof}

%% file: simple-ex.tex
\section{Details of the 3D Example} \label{sec:3d}

In this section, we give more details for \Cref{fig:ex3d}.

The loss $\Loss(\vw)$ is constructed as follows.
First, we define the following scale-invariant function:
\begin{align*}    
    F(x, y, z) := 2 - \frac{x+y}{\sqrt{x^2-xy+y^2}}.
\end{align*}
By taking gradient on $\sphS^{2}$, one can easily see that 
the minimum is attained when $(x, y)$ points to $(1, 1)$ in direction.
In other words, the minimizer manifold of $F$
is $\manGa := \{ (x, y, z) \in \sphS^2 : x = y > 0 \}$.

Then we fix an orthogonal matrix $\mU$ (generated randomly)
and define $\Loss: \R^3 \setminus \{ \vzero \} \to \R, \vw \mapsto F(\mU\vw)$,
i.e., the function $F$ after an orthogonal transformation.
For plotting the figure,
we transform the coordinates back to the domain of $F$.

The initial point is $\vw_0 = (0.3, 1.3, 1.2)$ in the domain of $F$.
We run gradient descent on $\Loss$ with LR $\heta = 0.5$
and WD $\hlam = 0.08$.
It can be seen from the figure that 
$\vtheta_t$ does not stop moving after reaching $\vzeta_0$.
The point that $\vtheta_t$ eventually oscillate around is $\vzeta_* = (\frac{1}{\sqrt{2}}, \frac{1}{\sqrt{2}}, 0)$.

One can check that the Hessian matrix of $F$ at $(x, y, z) \in \manGa$ is
\begin{align*}
H(x, y, z) = \frac{3}{x^2 + y^2}
    \begin{bmatrix}
    1 & -1 & 0 \\
    -1 & 1 & 0 \\
    0 & 0 & 0
\end{bmatrix} = \frac{3}{1 - z^2}
    \begin{bmatrix}
    1 & -1 & 0 \\
    -1 & 1 & 0 \\
    0 & 0 & 0
\end{bmatrix}.
\end{align*}
Therefore, the spherical sharpness is controlled by $\abssm{z}$.
The smaller the absolute value of $z$,
the flatter the minimizer.
And the flattest one is 
$\vzeta_*$,
which has $z$-coordinate being zero.
This matches with our theory of sharpness-reduction bias
as \GDWD moves along $\manGa$ and oscillates near $\vzeta_*$ in the end.

%% file: app-enter-eos.tex
\section{Supplementary Material for Section~\ref{sec:main-result-1}}
\label{sec:app-enter-eos}

\subsection{Proof for Descent Lemma}

\begin{proof}[Proof for \Cref{lm:descent}]
By Taylor expansion,
\begin{align*}
    \Loss(\vtheta_{t+1}) = \Loss(\vtheta_t - \eeta_t \nabla \Loss(\vtheta_t))
    &\le \Loss(\vtheta_t) - \dotpsm{\nabla \Loss(\vtheta_t)}{\eeta_t \nabla \Loss(\vtheta_t)}
    + \frac{1}{2}\lammax^{(t)} \normtwosm{\eeta_t \nabla \Loss(\vtheta_t)}^2 \\
    &= \Loss(\vtheta_t) - \eeta_t (1 - \eeta_t \lammax^{(t)} / 2) \normtwosm{\nabla \Loss(\vtheta_t)}^2,
\end{align*}
which proves the lemma.
\end{proof}

\subsection{Proof for Theorem \ref{thm:eos-near-opt}: GD Eventually Enters the EoS Regime}

Let $R := C_0 (\hlam \heta)^{1/2}$ be a radius so that
$\vthetaopt$ is a minimizer of $\Loss$ on $U := \Ball{R}{\vtheta} \cap \sphS^{D-1}$
and
$\mu$-PL holds within $U$,
where $C_0$ is a large constant to be specified later.
Let $\lammax := \sup\{ \lamH_1(\vtheta) : \vtheta \in U \}$.
Then we know that $\lammax = \lamH_1(\vthetaopt) + O(\hlam \heta)$.
Let $T_0$ be the largest number so that $\vtheta_t \in U$ for all $t_0 \le t \le T_0$,
We define a potential function $\Psi(\vtheta) := \sqrt{\Loss(\vtheta) - \Loss(\vthetaopt)}$.
\begin{lemma} \label{lm:Psi-decay}
If $\vtheta_t \in U$ and $\eeta_t < \frac{2}{\lammax}$ for some $t_0 \le t < T_0$, then
\begin{align*}
    \Psi(\vtheta_t) - \Psi(\vtheta_{t+1}) 
    &\ge \frac{\sqrt{2\mu}}{2}(1 - \eeta_t \lammax / 2) \eeta_t \normtwosm{\nabla \Loss(\vtheta_t)}.
\end{align*}
\end{lemma}
\begin{proof}
By descent lemma,
\begin{align*}
    \Loss(\vtheta_{t+1})
    &\le \Loss(\vtheta_t) - \eeta_t (1 - \eeta_t \lammax / 2) \normtwosm{\nabla \Loss(\vtheta_t)}^2.
\end{align*}
Then
\begin{align*}
    \Psi(\vtheta_t) - \Psi(\vtheta_{t+1}) = \frac{\Loss(\vtheta_t) - \Loss(\vtheta_{t+1})}{\Psi(\vtheta_t) + \Psi(\vtheta_{t+1})}
    &\ge \frac{(1 - \eeta_t \lammax / 2) \eeta_t \normtwosm{\nabla \Loss(\vtheta_t)}^2}{2 \Psi(\vtheta_t)}.
\end{align*}
By $\mu$-PL, $\normtwosm{\nabla \Loss(\vtheta_t)} \ge \sqrt{2\mu} \cdot \Psi(\vtheta_t)$. Combining these together proves the lemma.
\end{proof}

\begin{lemma} \label{lm:anti-pl}
There exists $C_2 = O(1)$ such that $\normtwosm{\nabla \Loss(\vtheta)}^2 \le C_2 (\Loss(\vtheta) - \Loss(\vthetaopt))$ for all $\vtheta \in U$.
\end{lemma}
\begin{proof}
It is equivalent to give an upper bound for $\sup_{\vtheta \in U} \{ G(\vtheta) \}$, where $G(\vtheta) := \frac{\normtwosm{\nabla \Loss(\vtheta)}^2}{\Loss(\vtheta) - \Loss(\vthetaopt)}$.
Since $\Loss \in \contC^2$, $G$ is continuous in its domain.
So it suffices to upper bound $G(\vtheta)$ around every singular point, i.e., around every minimizer of $\Loss$ on $U$.
And for every minimizer $\vtheta' \in U$, we can do Taylor expansions for $\Loss$ and $\nabla \Loss$
to show that $G(\vtheta)$ is indeed bounded by $O(1)$ around $\vtheta'$.
\end{proof}

\begin{proof}[Proof for \Cref{thm:eos-near-opt}]
Recall that $\eeta_{t_0} \le \frac{2}{\rho_2} < \frac{2}{\lamH_1(\vthetaopt)}$.
Let $\delta := 2 - \eeta_{t_0} \lammax \in (0, 2)$
and $T_1$ be the largest number so that $\eeta_t \le \frac{2 - \delta / 4}{\lammax}$ for all $t_0 \le t \le T_1$.
By \Cref{lm:Psi-decay}, for all $t_0 \le t < \min\{T_0, T_1\}$,
\begin{align*}
    \Psi(\vtheta_t) - \Psi(\vtheta_{t+1}) 
    &\ge \frac{\sqrt{2\mu}}{16} \delta \eeta_t \normtwosm{\nabla \Loss(\vtheta_t)}.
\end{align*}
Telescoping the sum we have $\frac{\sqrt{2\mu}}{16} \delta \sum_{\tau={t_0}}^{t-1}\eeta_{\tau} \normtwosm{\nabla \Loss(\vtheta_\tau)} \le \Psi(\vtheta_0)$.
By smoothness of $\Loss$, $\Psi(\vtheta_{t_0}) = O(\normtwosm{\vtheta_{t_0} - \vthetaopt}) = O((\hlam \heta)^{1/2})$.
So for all $t_0 \le t \le \min\{T_0, T_1\}$,
\[
    \normtwosm{\vtheta_t -\vtheta_{t_0}} \le 
    \sum_{\tau={t_0}}^{t-1}\eeta_{\tau} \normtwosm{\nabla \Loss(\vtheta_{\tau})} = O((\hlam\heta)^{1/2}),
\]
which implies that $T_0 > T_1$ or $T_0 = T_1 = +\infty$ if we choose $C_0$ to be large enough.

By \Cref{thm:connect-siwd-rmsprop},
$\eeta_t$ can be seen as the output of a GWSI scheduler with $\beta = (1-\hlam\heta)^4 = 1 - \Theta(\hlam\heta)$ and $\eta = \sqrt{(\beta^{-1}-1)/2} = O((\hlam\heta)^{1/2})$.
Then by the update rule,
\begin{align} \label{eq:eeta-aaa}
    \eeta_{t+1}^{-2} = \beta \eeta_t^{-2} + (1-\beta) \barg_t^2 + \frac{\eeta_t^2}{4\beta} (1-\beta)^2 \barg_t^4,
\qquad \text{where} \qquad \barg_t := \normtwosm{\nabla \Loss(\vtheta_t)} / \eta.
\end{align}
For $\vtheta \in U$,
we have $\normtwosm{\nabla \Loss(\vtheta)}^2 \le C_2(\Loss(\vtheta) - \Loss(\vthetaopt))$
by
\Cref{lm:anti-pl}.
So for all $t_0 \le t \le T_0$, we have
\[
    \barg_t = \frac{1}{\eta} \normtwosm{\nabla \Loss(\vtheta_t)}
    \le \frac{C_2}{\eta} (\Loss(\vtheta_t) - \Loss(\vthetaopt))
    \le \frac{C_2}{\eta} (\Loss(\vtheta_0) - \Loss(\vthetaopt))
    =O(1).
\]
Then 
$\eeta_{t+1}^{-2} = \beta \eeta_t^{-2} + (1-\beta) \cdot O(1)$,
which implies $\eeta_t \ge \Omega(1)$ for some $t = t_0 + O( \frac{1}{1 - \beta} \log(\eeta_0^{-2}))$.
Therefore, 
we can infer that
it must hold for some steps $t$
that
$\eeta_t \in \left[ c_{\min}, \frac{2 - \delta / 4}{\lammax} \right]$,
where $c_{\min}$ is some constant.

As $\barg_t \ge 0$,
the update rule \eqref{eq:eeta-aaa} also implies
$\eeta_{t+1}^{-2} \ge \beta \eeta_t^{-2}$,
or equivalently
$\eeta_{t+1} \le (1-\heta\hlam)^{-2} \eeta_t$.
This suggests that the number of steps such that
$\eeta_t \in \left[ \frac{2 - c}{\lammax}, \frac{2 - \delta / 4}{\lammax} \right]$,
is at least $\Omega(\sfrac{1}{(\heta \hlam)})$.
When $\eeta_t$ does lie in this range, by \Cref{lm:descent} we have
\begin{align*}
    \Loss(\vtheta_{t+1}) \le \Loss(\vtheta_t) - \frac{1}{8} \delta \eeta_t \normtwosm{\nabla \Loss(\vtheta_t)}^2
\end{align*}
Combining with $\mu$-PL gives
\begin{align*}
    \Loss(\vtheta_{t+1}) - \Loss(\vthetaopt)
    &\le (1 - \mu \delta \eeta_t / 4) \cdot (\Loss(\vtheta_t) - \Loss(\vthetaopt)) \\
    &\le \left(1 - \mu \delta c_{\min} / 4\right) \cdot (\Loss(\vtheta_t) - \Loss(\vthetaopt)).
\end{align*}
Thus the loss decays by a constant factor in every step (the factor is in $(0, 1)$ as we can choose $c_{\min}$ as small as we want).
As this process lasts for at least $\Omega(\sfrac{1}{(\heta \hlam)})$ steps,
the loss first decreases to $\Loss(\vthetaopt) + O((\heta \hlam)^{10})$ 
after $O(\log \frac{1}{\heta \hlam})$ steps,
then it stays small until $T_1$.

Now we show that $T_1$ is finite.
By \Cref{lm:anti-pl}
and \eqref{eq:eeta-aaa},
when the loss is $\Loss(\vthetaopt) + O((\heta \hlam)^{10})$
the effective LR steadily grows as $\eeta_{t+1} = (1-\heta\hlam)^{-2} \eeta_t + o(1)$.
So at some step $t$, it must hold that $\eeta_t > \frac{2-\delta/4}{\lammax}$, which proves $T_1 < +\infty$.

Let $C_1$ be a large constant,
and $T_2$ be the largest number so that $\eeta_t < \frac{2 - 2 C_1 (\heta \hlam)^{1/2}}{\lammax}$ for all $t_0 \le t \le T_2$.
By \Cref{lm:Psi-decay}, for all $T_1 \le t < \min\{T_2, T_0\}$,
\[
    \Psi(\vtheta_t) - \Psi(\vtheta_{t+1}) \ge \frac{\sqrt{2\mu}}{2} C_1 (\heta \hlam)^{1/2} \eeta_t \normtwosm{\nabla \Loss(\vtheta_t)}.
\]
Telescoping the sum gives $
(\heta \hlam)^{1/2} \sum_{\tau = T_1}^{t-1} \eeta_{\tau} \normtwosm{\nabla \Loss(\vtheta_{\tau})} \le O(\Psi(\vtheta_{T_1})) \le O((\heta \hlam)^5)$,
where the last inequality
is due to $\Loss(\vtheta_{T_1}) = \Loss(\vthetaopt) + O((\heta \hlam)^{10})$.
This shows that $\normtwosm{\vtheta_t - \vtheta_{T_1}} = O((\heta\hlam)^{4.5})$,
and thus $\normtwosm{\vtheta_t - \vthetaopt} = O((\heta\hlam)^{1/2})$ by triangle inequality.
Now we have $T_0 > T_2$ or $T_0 = T_2 = +\infty$ when $C_0$ is chosen to be large enough.
We can finish the proof with a similar argument as for $T_1$ to show that $T_2$ cannot be infinite either.
\end{proof}

\subsection{Connection to the EoS Regime in Cohen et al.'s Definition} \label{sec:eos-cohen}

Now we elaborate how our definition of EoS $\eeta_t \approx 2/\lammax^{(t)}$ is related to the original
definition of EoS in \citet{cohen2021gradient}.
In their work, they studied the dynamics of GD (without weight decay).
When the loss is $\tildeLoss$,
the update rule is given by $\vw_{t+1} \gets \vw_t - \heta \nabla \tildeLoss(\vw_t)$.
They define the EoS regime as a regime in which (1) $\lambda_1(\nabla^2\tildeLoss(\vw_t))$ hovers right at,
or just above $2 / \heta$; and (2) the training loss $\tildeLoss(\vw_t)$ goes up and down over short timescales,
yet still decreases in the long-term run.

\myparagraph{View I: Rewriting as GD.}
We can write \GDWD on scale-invariant loss as GD on scale-invariant loss with $\normltwo$-regularization,
i.e., GD on 
$\tildeLoss(\vw) := \Loss(\vw) + \frac{\hlam}{2} \normtwosm{\vw}^2$.
Now we show that $\eeta_t \approx \frac{2}{\lammax^{(t)}}$ is essentially the same as
$\heta \approx \frac{2}{\lambda_1\left(\nabla^2 \tildeLoss(\vw_t)\right)}$.

It suffices to show
$\eeta_t \cdot \lammax^{(t)} \approx \heta \cdot \lambda_1(\nabla^2 \tildeLoss(\vw_t))$.
When the gradient is small, $\vtheta_t$ does not move far in one step, then $\lammax^{(t)} \approx \lambda_1(\nabla^2 \Loss(\vtheta_t))$.
By scale-invariance, $\lambda_1(\nabla^2 \Loss(\vtheta_t)) = \normtwosm{\vw_t}^2 \cdot \lambda_1(\nabla^2 \Loss(\vw_t))$ (\Cref{lm:scale-invariant-grad-hess}).
Recall that $\eeta_t := \frac{\heta}{(1 - \heta\hlam) \normtwosm{\vw_t}^2}$. Then we have
\[
    \eeta_t \cdot \lammax^{(t)} \approx
    \tfrac{\heta}{(1 - \heta\hlam) \normtwosm{\vw_t}^2} \cdot \normtwosm{\vw_t}^2 \cdot \lambda_1(\nabla^2 \Loss(\vw_t))
    \approx
    \tfrac{\heta}{(1 - \heta\hlam)} \cdot \lambda_1(\nabla^2 \Loss(\vw_t))
    \approx
    \heta \cdot \lambda_1(\nabla^2 \Loss(\vw_t)).
\]
Note that $\heta \cdot \lambda_1(\nabla^2\tildeLoss(\vw_t)) = \heta \cdot \lambda_1(\nabla^2 \Loss(\vw_t)) + \heta\hlam$.
When $\heta\hlam$ is small, we can then conclude that
$\eeta_t \cdot \lammax^{(t)} \approx \heta \cdot \lambda_1(\nabla^2 \tildeLoss(\vw_t))$.

Now we show below that our main theorem on sharpness-reduction bias
implies the second condition in \citet{cohen2021gradient}'s definition, that is, $\tildeLoss$ decreases in the long-term run.

For the regularizer $\frac{\hlam}{2}\normtwosm{\vw_t}^2$,
note that $\normtwosm{\vw_t}^2 \approx \heta/\eeta_t \approx
\frac{1}{2}\heta \lambda_1(\nabla^2\Loss(\vtheta_t))$ in the EoS regime. So 
$\normtwosm{\vw_t}^2$ as well as the
regularizer is decreasing due to the sharpness-reduction bias (\Cref{thm:gdwd-main}).

The scale-invariant part
$\Loss(\vtheta_t)$ is not always decreasing, but now we show that its time average can be upper
bounded by the time average of norm squared.
By \Cref{lm:connection-prelim} and \Cref{lm:scale-invariant-grad-hess}, we have
\[
\normtwosm{\vw_{t+1}}^2
- \normtwosm{\vw_{t}}^2
= (2-\heta\hlam)\heta\hlam\normtwosm{\vw_{t}}^2 + \heta^2 \normtwosm{\nabla \Loss(\vw_t)}^2.
\]
Since
$\normtwosm{\vw_t}^2$ decreases in the long run, we know that for any long
enough time window $T_0$ to $T_1-1$, $\sum_{t=T_0}^{T_1-1}\normtwosm{\nabla
\Loss(\vw_t)}^2 \lesssim \sum_{t=T_0}^{T_1-1}
\frac{2\hlam}{\heta}\normtwosm{\vw_t}^2$.
Further due to the alignment between
the gradient and the top eigenvalue of the Hessian in the EoS regime, we have
$\Loss(\vw_t)\approx \frac{\normtwosm{\nabla
\Loss(\vw_t)}^2}{2\lambda_1(\nabla^2\Loss(\vw_t))} \approx \frac{\heta}{4}\normtwosm{\nabla
\Loss(\vw_t)}^2$.
Therefore, we conclude that the
average loss over a long enough time window is always upper bounded by the
average of squared weight norm, that is,
\begin{align*}
\frac{1}{T_1 - T_0}\sum_{t=T_0}^{T_1-1}\Loss(\vw_t)
\lesssim \frac{1}{T_1 - T_0}\sum_{t=T_0}^{T_1-1} 
\frac{\heta}{4} \cdot \frac{2\hlam}{\heta}\normtwosm{\vw_t}^2
\approx
\frac{1}{T_1 - T_0}
\sum_{t=T_0}^{T_1-1}
\frac{\hlam}{2}\normtwosm{\vw_t}^2,
\end{align*}
where the last step uses the fact that GD
operates in EoS.
So $\Loss(\vw_t)$ decreases in the long-term run.

Combining the above two parts, we can conclude that 
the regularized loss $\tildeLoss(\vw_t)$ has a tendency to decrease in the long-term run.

\myparagraph{View II: Generalizing EoS to PGD.}
For a gradient-based method in general, $2 / \heta$ should be replaced to the
maximum sharpness bound that the loss function is guaranteed to decrease through
Taylor expansions, e.g., \citet{cohen2021gradient} derived the bounds exactly
for Polyak and Nesterov momentum in Appendix B of their paper. In our
definition, we view \GDWD on scale-invariant loss as PGD on $\sphS^{D-1}$, and
thus we define the EoS regime for PGD as the regime where $\eeta_t \approx 2 /
\lammax^{(t)}$, where $\lammax^{(t)}$ is the local upper bound of spherical
sharpness (\Cref{lm:descent}). This captures the first condition of
\citet{cohen2021gradient}'s definition. Repeating our argument in View I, we can
show the second condition, namely the condition that the loss $\Loss(\vtheta_t)$
decreases in the long-term run.

\myparagraph{Progressive Sharpening.}
All the above discussion is about the EoS phenomenon. Another phenomenon
identified by \citet{cohen2021gradient} is progressive sharpening, which is the
phenomenon that $\lambda_1(\nabla^2 \tildeLoss(\vw_t))$ tends to increase so
long as it is less than $2 / \heta$.
\Cref{thm:eos-near-opt} in our paper justifies this phenomenon in View II,
i.e., if $\eeta_t$ is less than $\lamH_1(\vthetaopt)$,
then $\eeta_t$ increases until it reaches $\frac{2}{\lamH_1(\vthetaopt)}$.
The key insight in our analysis is that WD decreases the norm
when gradient is small,
and smaller norm leads larger $\lambda_1(\nabla^2 \tildeLoss(\vw_t))$.
This shows that the progressive sharpening phenomenon in our case
can be well explained by the interplay between normalization and WD.

%% file: proof-sketch.tex
\section{Proof Outlines of Our Theorems on Sharpness Reduction} \label{sec:proof-outline-main}

In this section,
we give proof outlines of \Cref{thm:main-ful,thm:main-sph}.
The main proof idea
for the spherical case
is stated in \Cref{sec:proof-idea},
but technically it is easier
to state and prove the lemmas for full space optimization.
Therefore, we present the full details for the full space case
and omit certain details for the spherical case if they are similar
to the full space case.

As mentioned in \Cref{sec:proof-idea},
a key ingredient in our proof is to show that
the period-2 oscillation drives the parameter to move along the manifold.
For both full space and spherical cases,
we project $\vtheta_t$
onto the manifold $\manGa$
with a carefully-defined projection function, $\vphi_t := \Phi(\vtheta_t)$.
Then we show that 
after a period of oscillation ($2$ steps),
the projection drifts from $\vphi_t$ to a new position $\vphi_{t+2}$
approximately along the direction of $\gradGa \log \lamH_1(\vphi_t)$.
To analyze the speed of each drift,
we show that
the oscillation can be tracked with
two variables $h \in \R$, $u \in \R$,
where
$h$ is related to the displacement of $\vtheta_t$ from the manifold,
$u$ is related to the closeness of the current dynamic to the edge of stability.
We formally define 
a discrete process called \textit{RMS-drift process}
and show that 
the oscillations in both full space and spherical cases
can be regarded as RMS-drift processes.

The rest of the section is organized as follows.
In \Cref{sec:outline-notation} we introduce some additional notations.
In \Cref{sec:outline-rmsdrift-def} we formally introduce the concept
of RMS-drift process.
In \Cref{sec:outline-ful} we show how to
reduce our analysis in the full space case
into studying an RMS-drift process.
In \Cref{sec:outline-sph} we show how to
reduce our analysis in the spherical case
into studying an RMS-drift process.
Finally, in \Cref{sec:outline-rmsdrift-analysis}
we analyze the RMS-drift process
and show that the projections of parameters in training can be tracked with
the sharpness-reduction flow defined in \eqref{eq:zeta-general}.

\subsection{Additional Notations} \label{sec:outline-notation}

We follow the notations in \Cref{sec:add-prelim}.
We also need some additional notations in this section.
It is implied by
the uniqueness of the top eigenvalue
(\Cref{ass:topeigen-ful} or \Cref{ass:topeigen})
and $\Loss \in \contC^4$
that $\lamH_1(\vtheta)$ is $\contC^2$-smooth on $\manGa$,
and we can construct a $\contC^2$-smooth function $\vvH_1(\vtheta)$
such that
$\vvH_1(\vtheta)$ is a unit top eigenvector of $\mH(\vtheta)$ on $\manGa$.
Let $\mPHz(\vtheta)$ be the projection matrix onto the null space of $\mH_t$,
and $\mPHnzt(\vtheta)$ be the projection matrix onto the space spanned by the eigenvectors of $\mH_t$
corresponding to non-zero and non-top eigenvalues ($\lambda \ne 0, \lamH_1(\vphi_t)$).
For any local minimizer 
$\vtheta \in \manGa$,
we define $\gamma(\vtheta) := \frac{1}{\lamH_1(\vtheta)} \min\left\{\lamH_1(\vtheta) -
\lamH_2(\vtheta), \lamH_{D-\DGa}(\vtheta)\right\} $ to be the
relative eigenvalue gap (between $\lamH_1$ and $\lamH_2$, or between $\lamH_{D- \DGa}$ and $0$).
Then $\mPHz(\vtheta), \mPHnzt(\vtheta), \gamma(\vtheta)$ are all $\contC^2$-smooth on $\manGa$.
Finally, we define $\mu(\vtheta) := \frac{2}{\lamH_1(\vtheta)}$.

We define $\Phi(\vtheta)$ to be the convergence point of a gradient flow of $\Loss$ starting from $\vtheta$.
That is, the limit of $\tilde{\vtheta}(t)$ as $t \to +\infty$ when
$\tilde{\vtheta}(t)$ is described by the following ODE:
\[
  \frac{\dd \tilde{\vtheta}}{\dd t} = -\nabla \Loss(\tilde{\vtheta}), \quad \text{where} \quad \tilde{\vtheta}(0) = \vtheta.
\]
We leave $\Phi(\vtheta)$ undefined if the ODE does not converge to any point.

For GD/PGD with quasi-RMSprop scheduler (\Cref{def:qrms-sche}),
the state at step $t$ can be written as $(\vtheta_t, \tilv_t)$,
where $\vtheta_t$ is the trainable parameter in GD/PGD and $\tilv_t$ is the moment estimate.
Whenever $\Phi(\vtheta_t)$ exists and is in $\manGa$,
we define
\begin{align*}
  \vphi_t &:= \Phi(\vtheta_t) & \mH_t &:= \mH(\vphi_t) \\
  \mu_t &:= \mu(\vphi_t) & \mU_t &:= \mI - \mu_t \mH_t \\
  \vx_t &:= \vtheta_t - \vphi_t
\end{align*}
Most importantly, we define two hidden variables $(h_t, u_t)$ as follows:
\begin{align}
  h_t &:= \tfrac{1}{\eta}\dotpsm{\vvH_1(\vphi_t)}{\vx_t}, &
  u_t &:= \tfrac{1}{\eta}( \mu_t^2 \tilv_t- 1).
  \label{eq:hu-def}
\end{align}
Note that the $h_t$ defined here {\em differs} with that in \Cref{sec:proof-idea} by a factor of $1/\eta$.
We introduce this factor for the sake of convenience,
and we will use the definition with the factor $1/\eta$
only in our theoretical analysis.

\subsection{RMS-drift Process: Introduction} \label{sec:outline-rmsdrift-def}

\begin{definition}
  A \textit{drift state} is described by a tuple $S = (h, u, \vphi)$ in the drift state space $\Sta := \R \times \R \times \manGa$.
  We say that $S$ is \textit{$\alpha$-bounded} if $\max\{\abssm{h}, \abssm{u}\} \le \alpha$.
\end{definition}
\begin{definition}
  Given two drift states $S_t = (h_t, u_t, \vphi_t)$, $S_{t+2} = (h_{t+2}, u_{t+2},
  \vphi_{t+2})$ in the drift state space $\Sta$,
  for learning rate $\eta > 0$ and hyperparameter $\Cb > 0$,
  we say that the transition $S_t \to S_{t+2}$
  is a \textit{$C_0$-RMS-drift transition}
  if for all $\alpha \ge 1$,
  as long as $S_t$ is $\alpha$-bounded,
  $S_{t+2}$ is close to an auxiliary state $S'_{t+2} := (h'_{t+2}, u'_{t+2}, \vphi'_{t+2})$
  in the following sense:
  \begin{align*}
    h'_{t+2} &:= (1 - 2\eta u_t) h_t,  &
    \abssm{h_{t+2} - h'_{t+2}} &\le C_0 \alpha^2 \abssm{h_t} \eta^2, \\
    u'_{t+2} &:= u_t + 4 \eta h_t^2 (2 \Cb + \normtwosm{\gradGa \log \lamH_1(\vphi_t)}^2) - 2\eta \Cb, &
    \abssm{u_{t+2} - u'_{t+2}} &\le C_0 \alpha (1 + h_t^2) \eta^2,  \\
    \vphi'_{t+2} &:= \vphi_t - 2 \eta^2 h_t^2 \gradGa \log \lamH_1(\vphi_t), &
    \normtwosm{\vphi_{t+2} - \vphi'_{t+2}} &\le C_0 \alpha h_t^2 \eta^3.
  \end{align*}
  For a sequence of states $S_0, S_2, S_4, \dots, S_{2M}$, we say it is a $C_0$-\textit{RMS-drift process} if $S_t \to S_{t+2}$ is a
  $C_0$-RMS-drift transition for all even
  numbers $0 \le t < 2M$.
\end{definition}

In our analysis of GD/PGD with quasi-RMSprop scheduler,
we can rewrite the dynamics as RMS-drift processes (after a few warm-up steps),
where 
$\vphi_t := \Phi(\vtheta_t)$ is the gradient flow projection
of the parameter at step $t$ onto $\manGa$,
and $h_t,u_t$ are two hidden variables defined in \eqref{eq:hu-def}.
This RMS-drift process
serves as an abstraction of the original dynamics
that contains the minimal but sufficient amount of information
so that we can compute the continuous approximation for the trajectory of $\vphi_t$.

In RMS-drift process, $\vphi_t$ evolves as gradient descent on $\manGa$ for minimizing $\log \lamH_1(\vtheta)$,
and the corresponding learning rate is changing with $h_t$ per step.
To obtain the final flow approximation \eqref{eq:zeta-general},
we need to sum up the learning rates over time.

An intuitive way to understand RMS-drift process is to use the following first-order continuous approximation with time scaling $(h(\tau),u(\tau),\vphi(\tau)) \approx (h_{\tau/\eta}, u_{\tau/\eta}, \vphi_{\tau/\eta})$ and ignore all the second order terms $O(\eta^2)$:
\begin{align*}
  \frac{\dd h}{\dd \tau} &= - u h, & \frac{\dd u}{\dd \tau} &= 2h^2(2\Cb + \normtwosm{\gradGa \log \lamH_1(\vphi)}^2) - \Cb,
  & \frac{\dd \vphi}{\dd \tau} &= \vzero.
\end{align*}
This approximation gives an important insight:
$h_t$ and $u_t$ are changing much faster than $\vphi_t$
when $\eta$ is small.
Therefore, we can analyze this first-order approximation
to obtain an average value of $2\eta^2h_t^2$,
and use this average value as the ``effective'' learning rate
in the flow approximation.

\begin{figure}[htbp]
  \centering
  \includegraphics[width=\textwidth]{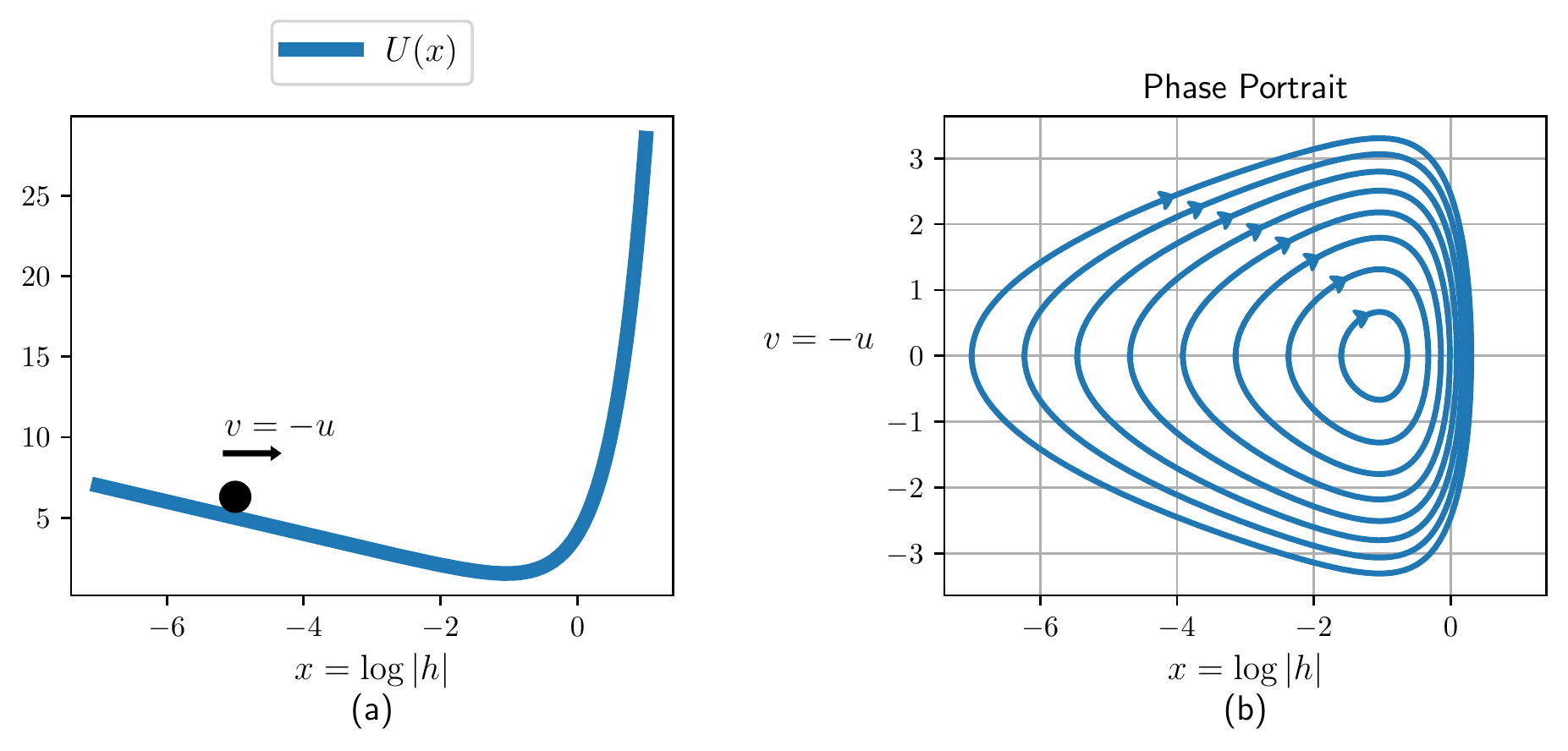}
  \vspace{-0.2in}
  \caption{
    A visualization of the dynamical system described by \eqref{eq:qrms-cont} with $K=2, \Cb=1$.
    This system can be associated with a physical system in which a unit-mass particle
    moves in a potential well $U(x) = K^2 e^{2x} - \Cb x$ without any energy loss; see (a).
    $\log \abssm{h}$ can be seen as the position of the particle,
    and $-u$ can be seen as the velocity.
    This system must be periodic because of the conservation of energy,
    and it can be seen clearly from its phase portrait (b).
  }
  \label{fig:qrms-vis}
\end{figure}

In fact, $(h, u)$ forms a $1$-dimensional Hamiltonian system after a coordinate transformation
and evolves periodically in the above ODE.
To see this, 
we can compute the time derivatives of $(\log \abssm{h}, -u)$ while letting $K := \sqrt{2\Cb + \normtwosm{\gradGa \log \lamH_1(\vphi)}^2}$:
\begin{align} \label{eq:qrms-cont}
  \frac{\dd \log \abssm{h}}{\dd \tau} &= -u, & \frac{\dd (-u)}{\dd \tau} &= -\left(2K^2 h^2 - \Cb\right).
\end{align}

Consider a unit-mass particle in the system with position $x(t)$, velocity $v(t) = x'(t)$ and acceleration $a(t) = v'(t)$.
Suppose that this system has the potential energy
$U(x) := K^2e^{2x} - \Cb x$.
Then we can see that the position and velocity of this particle evolve exactly the same as $(\log \abssm{h}, -u)$!
This also shows that $(\log \abssm{h}, -u)$
evolves periodically because $1$-dimensional Hamiltonian system with unimodal potential energy must be periodic.
See also \Cref{fig:qrms-vis}.

In \Cref{sec:outline-rmsdrift-analysis} below, we will use this observation
to obtain the flow approximation nicely.
But now 
we first outline how to reduce the original dynamics to an RMS-drift process.

\subsection{Reduction to RMS-drift Process: The Case of Full Space Optimization} \label{sec:outline-ful}

Now we outline how to reduce the dynamics to an RMS-drift process
in the setting of \Cref{thm:main-ful},
the main theorem for the case of full-space optimization.

\subsubsection{Construction of Working Zones}

Let $\gZ := \{ \vzeta(t) : t \in [0, T] \} \subseteq \manGa$
be the set of points passed by the sharpness-reduction flow~\eqref{eq:zeta-general}.
Inspired by \citet{arora2022understanding}, we construct a two-level nested working zone
$(\Zei, \Zeo)$,
where $\Zei, \Zeo$ are essentially the $\epsilon_0$- and $\epsilon_1$-neighborhoods of $\gZ$
with $\epsilon_0, \epsilon_1$ carefully chosen.
In our later analysis, we will ensure that $\vtheta_t$ is always in $\Zei$,
and its gradient flow projection $\vphi_t$ is always well-defined and lies in $\Zeog$.

The following lemma shows some important properties of the working zone in our construction.
We defer the proof to \Cref{sec:working-zone-proof-construction}.
\begin{lemma}[Working Zone Lemma] \label{lm:ful-working-zone}
  There exist $0 < \epsilon_0 < \epsilon_1$ such that $\Zei, \Zeo$ satisfy the following:
  \begin{enumerate}
    \item $\cl(\Zeo) \cap \manGa$ is compact;
    \item $\Loss$ satisfies $\muPL$-PL on $\Zeo$ for some $\muPL > 0$;
    \item $\Phi$ is well-defined on $\Zei$, and $\Phi(\vtheta) \in \Zeog$ for all $\vtheta \in \Zei$;
    \item $\Phi$ is $\contC^3$-smooth on $\Zei$;
    \item $\gamma(\vtheta) \ge \gamin$ holds uniformly on $\Zeog$ for some $\gamin > 0$.
  \end{enumerate}
\end{lemma}

It can be seen from \Cref{lm:ful-working-zone} that 
$\vphi_t, \mH_t, \mu_t, \mU_t, \vx_t, h_t, u_t$
are all well-defined
as long as $\vtheta_t \in \Zei$.
Now we define some useful notions in the EoS regime.
A state is $\alpha$-bounded
if $\vtheta_t$ is close to $\manGa$
and the dynamic is in the EoS regime due to $\tilv_t \approx \mu_t^{-2}$.
A state is $\alpha$-deviated if $\vtheta_t$
is not too close to $\manGa$.
\begin{definition}[$\alpha$-Bounded State]
  We say that the state $(\vtheta_t, \tilv_t)$ at some step $t$
  is $\alpha$-bounded if
  $\vtheta_t \in \Zei$,
  $\normtwosm{\vx_t} \le \alpha \eta$ and $\abssm{u_t} \le \alpha$.
\end{definition}
\begin{definition}[$\alpha$-Deviated State]
  We say that the state $(\vtheta_t, \tilv_t)$ at some step $t$
  is $\alpha$-deviated (from the manifold $\manGa$) if
  $\normtwosm{\vx_t} \ge \eta\exp(-\alpha^2)$ or
  $\vtheta_t \notin \Zei$.
\end{definition}

Next, we define
a quantitative measurement
for how much $\vx_t$ aligns to the top eigenvector of $\mH_t$.
Note that $\vx_t$ can be decomposed into the projections
onto the top eigenspace, the null space, and the space spanned by non-zero and non-top eigenvectors;
we can write this decomposition as
$\vx_t = h_t \eta \vvH_1(\vphi_t) + \mPHz(\vphi_t) \vx_t + \mPHnzt(\vphi_t) \vx_t$.
The following lemma shows that $\mPHnzt(\vphi_t)$ is always negligible,
so we characterize the alignment to the top eigenvector only through comparing
$\eta h_t \vvH_1(\vphi_t)$ and $\mPHnzt(\vphi_t) \vx_t$.
\begin{lemma} \label{lm:mPHz-x-small}
  At any step $t$, if $\vtheta_t \in \Zei$, then $\normtwosm{\mPHz(\vphi_t) \vx_t} = O(\normtwosm{\vx_t}^2)$.
\end{lemma}
\begin{proof}
  Direct consequence of \Cref{lm:phiphix-ful}.
\end{proof}

\begin{definition}[$p$-Misaligned State]
  We say that the state $(\vtheta_t, \tilv_t)$ at some step $t$
  is at most $p$-misaligned (to the top eigenvector) if
  $\vtheta_t \in \Zei$,
  $\normtwosm{\mPHnzt(\vphi_t) \vx_t} \le p \cdot h_t \eta$.
\end{definition}

\subsubsection{Good Initialization}

First, we show that the initialization satisfies some desirable properties
with high probability.
The proof is deferred to \Cref{sec:ful-good-init}.
\begin{lemma} \label{lm:ful-good-init}
  There exists $\delta = O(\alpha_0 \eta \sqrt{\log(1/\eta)})$ such that
  the following holds.
  With probability $1 - \delta$,
  the initial state is $\Osm{\alpha_0\sqrt{\log(1/\delta)}}$-bounded,
  $O(\alpha_0 + \sqrt{\log(1/\delta)})$-deviated,
  at most $O(1/\delta)$-misaligned,
  and satisfies $\normtwosm{\vphi_0 - \vzeta_0} \le O(\alpha_0\eta \sqrt{\log(1 / \delta)})$.
\end{lemma}

\subsubsection{Alignment Phase}
At initialization, the state is not well-aligned to the top eigenvector.
But in the following, we show that it becomes at most $O(\ahte)$-misaligned
after only $\eta^{-o(1)}$ steps.
We defer the proofs to \Cref{sec:ful-proof-alignment}.

The key lemma is the following, which gives good approximations for various important variables.
\begin{lemma} \label{lm:ful-align-1-step}
  For small enough base learning rate $\eta$,
  at any step $t$,
  if the state $(\vtheta_t, \tilv_t)$ is $\alpha$-bounded
  for some $1 \le \alpha \le \eta^{-o(1)}$,
  then $\vtheta_{t+1} \in \Zei$, and
  \begin{align}
    \vphi_{t+1} &= \vphi_t + O(\normtwosm{\vx_t}^2)  \label{eq:ful-align-1-step-phi} \\
    \vx_{t+1} &= (\mI - \eeta_t \mH_t) \vx_t + O(\normtwosm{\vx_t}^2) \label{eq:ful-align-1-step-x} \\
    h_{t+1} &= -h_t + O(\alpha \normtwosm{\vx_t}) \label{eq:ful-align-1-step-h} \\
    \normtwosm{\mPHnzt(\vphi_{t+1}) \vx_{t+1}} &\le (1 - 1.9\gamin) \normtwosm{\mPHnzt(\vphi_t) \vx_t} + O(\normtwosm{\vx_t}^2) \label{eq:ful-align-1-step-pnzt} \\
    u_{t+1} &= u_t + O(\alpha^2 \eta) \label{eq:ful-align-1-step-u}
  \end{align}
\end{lemma}

Applying the above lemma through an induction proves the following theorem.
\begin{theorem} \label{thm:ful-alignment-main}
  There exists $\delta = O(\alpha_0 \eta \sqrt{\log(1/\eta)})$
  and 
  $T_1 = O(\log \frac{1}{\eta} + \alpha_0 \sqrt{\log(1/\delta)})$
  such that the following holds.
  Let $\alphamax := \alpha_0 \sqrt{\log(1/\delta)}$.
  If the initial state is $O(\alphamax)$-bounded,
  $O(\alphamax)$-deviated,
  at most $O(1/\delta)$-misaligned,
  and satisfies $\normtwosm{\vphi_0 - \vzeta_0} \le O(\alphamax\eta)$,
  then
  at step $t = T_1$,
  the state is
  at most $O(\ahte)$-misaligned
  while still being
  $O(\alphamax)$-bounded,
  $O(\alphamax)$-deviated,
  and satisfying $\normtwosm{\vphi_t - \vzeta_0} \le O(\alphamax \eta)$.
\end{theorem}

\subsubsection{Drifting Phase}

After the alignment phase, the state is now at most $O(\ahte)$-misaligned.
Then we have the following lemma showing that $(h_t, u_t, \vphi_t)$ evolves
as an $O(1)$-RMS-drift process.
We defer the proof to \Cref{sec:ful-proof-drifting}.

\begin{lemma} \label{lm:ful-drift-2-step}
  For small enough base learning rate $\eta$,
  at any step $t$,
  if 
  for some $1 \le \alpha \le \eta^{-o(1)}$,
  the state $(\vtheta_t, \tilv_t)$ is $\alpha$-bounded
  and at most $O(\ahte)$-misaligned,
  then $\vtheta_{t+2} \in \Zei$, and
  \begin{align}
    h_{t+2} &= (1 - 2 \eta u_t) h_t + \Osm{\alpha^2 \aht\eta^2}, \label{eq:ful-drift-2-step-h} \\
    \normtwosm{\mPHnzt(\vphi_{t+2}) \vx_{t+2}} &\le (1- 1.9\gamin)^2 \normtwosm{\mPHnzt(\vphi_t) \vx_t} + O(h_t^2 \eta^2), \label{eq:ful-drift-2-step-x-w} \\
    u_{t+2} &= u_t + 4 \eta h_t^2 (2 \Cb + \normtwosm{\gradGa \log \lamH_1(\vphi_t)}^2) - 2 \eta \Cb + \Osm{\alpha(1+h_t^2)\eta^2},  \label{eq:ful-drift-2-step-u} \\
    \vphi_{t+2} &= \vphi_t
    - 2 \eta^2 h_t^2 \gradGa \log \lamH_1(\vphi_t) + \Osm{\alpha h_t^2 \eta^3}. \label{eq:ful-drift-2-step-phi}
  \end{align}
  In other words, $(h_t, u_t) \to (h_{t+2}, u_{t+2})$ is an $O(1)$-RMS-drift transition
  if the state at step $t$ is $\alpha$-bounded and at most $O(\ahte)$-misaligned.
\end{lemma}

\subsection{Reduction to RMS-drift Process: The Case of Spherical Optimization} \label{sec:outline-sph}

Now we outline how to reduce the dynamics to an RMS-drift process
in the setting of \Cref{thm:main-ful},
the main theorem for the case of spherical optimization.
The basic logic is the same as the case of full space optimization,
so we only list the new lemma and theorem statements here.
The proofs in this section are deferred to
\Cref{sec:working-zone-proof-construction,sec:sph-proof}.

\subsubsection{Construction of Working Zones}

We still define $\gZ := \{ \vzeta(t) : t \in [0, T] \} \subseteq \manGa$.
The construction of working zone becomes the following.
We defer the proof to \Cref{sec:working-zone-proof-construction}.
\begin{lemma}[Working Zone Lemma] \label{lm:sph-working-zone}
  There exist $0 < \epsilon_0 < \epsilon_1$ such that $\Zei, \Zeo$ satisfy the following:
  \begin{enumerate}
    \item $\cl(\Zeo) \cap \manGa$ is compact;
    \item $\Loss$ satisfies $\muPL$-PL on $\Zeo$ for some $\muPL > 0$;
    \item $\Phi$ is well-defined on $\Zeis$ and $\Phi(\vtheta) \in \Zeog$ for all $\vtheta \in \Zeis$;
    \item $\Phi$ is $\contC^3$-smooth on $\Zeos$;
    \item $\gamma(\vtheta) \ge \gamin$ holds uniformly on $\Zeog$ for some $\gamin > 0$.
  \end{enumerate}
\end{lemma}

In the working zone,
we continue to define $\alpha$-bounded states,
$\alpha$-deviated states,
at most $p$-misaligned states
following the same definitions
as the full space case except that the definition of working zone is changed.

\begin{definition}[$\alpha$-Bounded State]
  We say that the state $(\vtheta_t, \tilv_t)$ at some step $t$
  is $\alpha$-bounded if
  $\vtheta_t \in \Zeis$,
  $\normtwosm{\vx_t} \le \alpha \eta$ and $\abssm{u_t} \le \alpha$.
\end{definition}
\begin{definition}[$\alpha$-Deviated State]
  We say that the state $(\vtheta_t, \tilv_t)$ at some step $t$
  is $\alpha$-deviated (from the manifold $\manGa$) if
  $\normtwosm{\vx_t} \ge \eta\exp(-\alpha^2)$ or
  $\vtheta_t \notin \Zeis$.
\end{definition}

\begin{lemma} \label{lm:mPHz-x-small-sph}
  At any step $t$, if $\vtheta_t \in \Zeis$, then $\normtwosm{\mPHz(\vphi_t) \vx_t} = O(\normtwosm{\vx_t}^2)$.
\end{lemma}
\begin{proof}
  Direct consequence of \Cref{lm:phiphix-sph}.
\end{proof}
\begin{definition}[$p$-Misaligned State]
  We say that the state $(\vtheta_t, \tilv_t)$ at some step $t$
  is at most $p$-misaligned (to the top eigenvector) if
  $\vtheta_t \in \Zeis$,
  $\normtwosm{\mPHnzt(\vphi_t) \vx_t} \le p \cdot h_t \eta$.
\end{definition}

\subsubsection{Good Initialization}

\begin{lemma} \label{lm:sph-good-init}
  The same statement as \Cref{lm:ful-good-init} holds for the spherical case.
\end{lemma}

\subsubsection{Alignment Phase}

\begin{lemma} \label{lm:sph-align-1-step}
  For small enough base learning rate $\eta$,
  at any step $t$,
  if the state $(\vtheta_t, \tilv_t)$ is $\alpha$-bounded
  for some $1 \le \alpha \le \eta^{-o(1)}$,
  then $\vtheta_{t+1} \in \Zeis$, and
  \eqref{eq:ful-align-1-step-phi} to \eqref{eq:ful-align-1-step-u}
  in the full space case
  continue to hold in the spherical case.
\end{lemma}

\begin{theorem} \label{thm:sph-alignment-main}
    The same statement as \Cref{thm:ful-alignment-main} holds for the spherical case.
\end{theorem}

\subsubsection{Drifting Phase}

\begin{lemma}\label{lm:sph-drift-2-step}
  For small enough base learning rate $\eta$,
  at any step $t$,
  if 
  for some $1 \le \alpha \le \eta^{-o(1)}$,
  the state $(\vtheta_t, \tilv_t)$ is $\alpha$-bounded
  and at most $O(\ahte)$-misaligned,
  then $\vtheta_{t+2} \in \Zeis$, and
  \eqref{eq:ful-drift-2-step-h} to \eqref{eq:ful-drift-2-step-phi}
  in the full space case
  continue to hold in the spherical case.
\end{lemma}

\subsection{RMS-drift Process: Analysis} \label{sec:outline-rmsdrift-analysis}

Now we outline how to obtain the final flow approximation \eqref{eq:zeta-general}
from the RMS-drift process.
We say that an RMS-drift process $S_0, \dots, S_{2M}$ is \textit{in the working zone}
if $\vphi_t \in \Zeog$
for all even numbers $0 \le t \le 2M$.
We focus on RMS-drift processes in the working zone,
and later we will show that the RMS-drift processes of interest
are indeed in the working zone.

First, we define a potential function that resembles the total energy (or Hamiltonian)
in physics.
\begin{definition}[Energy]
For a drift state $S = (h, u, \vphi)$, we define the energy $E(S)$ as follows: 
\[
  E(S) := \frac{1}{2} u^2 + (2\Cb + \normtwosm{\gradGa \log \lamH_1(\vphi)}^2) h^2 + \Cb \log \frac{1}{\abssm{h}}.
\]
\end{definition}
If the energy is bounded by $\alpha^2$ at some step $t$,
then it is easy to see that the state at step $t$
is $O(\alpha)$-bounded and $O(\alpha)$-deviated.

The first key lemma is the conservation of energy in RMS-drift process,
which shows that the energy is preserved for $O(1/\eta^2)$ steps.
The proof is deferred to \Cref{sec:rmsdrift-proof-energy}.
\begin{theorem} \label{thm:rmsdrift-ec}
  For an $O(1)$-RMSdrift process $S_0, \dots, S_{2M}$ in the working zone,
  if $E(S_0) \le \alpha^2$ for some parameter $1 \le \alpha \le \eta^{-o(1)}$
  and $M = O(1/\eta^2)$,
  then $E(S_t) = O(\alpha^2)$ for all even numbers $0 \le t \le 2M$.
\end{theorem}

Next, we show that as long as the states are $\eta^{-o(1)}$-bounded,
the RMS-drift process tracks the sharpness-reduction flow nicely.
We defer the proof to \Cref{sec:rmsdrift-flow}.
\begin{theorem} \label{thm:rmsdrift-flow}
  Let $\vzeta: [0, T] \mapsto \Zeog$
  be a sharpness-reduction flow defined in \eqref{eq:zeta-general}.
  For an $O(1)$-RMS-drift process $S_0, \dots, S_{2M}$ in the working zone, where $M := \lfloor \frac{T}{2\eta^2} \rfloor$,
  if $\normtwosm{\vphi_0 - \vzeta(0)} \le O(\alpha^2 \eta^{1/2})$,
  and $S_t$ is $O(\alpha)$-bounded for all even numbers $0 \le t \le 2M$,
  then $\normtwosm{\vphi_t - \vzeta(t\eta^2)} \le O(\alpha^2 \eta^{1/2})$
  for all even numbers $0 \le t \le 2M$.
\end{theorem}

\subsection{Finalizing Proofs}

\begin{proof}[Proof for \Cref{thm:main-ful}]
First, we use \Cref{lm:ful-good-init}
to ensure a good initialization.
Then we apply \Cref{thm:ful-alignment-main}
to show that 
the state at $t = T_1$
is at most $O(\abssm{h_t}\eta)$-misaligned,
$O(\alphamax)$-bounded,
$O(\alphamax)$-deviated,
and satisfies $\normtwosm{\vphi_t - \vzeta_0} \le O(\alphamax \eta)$.
By smoothness of $\vzeta$,
we also have $\normtwosm{\vphi_t - \vzeta(t \eta^2)}$.
Then we
consider the even-indexed and odd-indexed steps separately
and do an induction for each of them to show that the later dynamics
(1)~has bounded energy $O(\alphamax^2)$ (\Cref{thm:rmsdrift-ec});
(2)~stays at most $O(\abssm{h_t}\eta)$-misaligned (\Cref{lm:ful-drift-2-step});
(3)~follows the flow (\Cref{thm:rmsdrift-flow}).
The proof is done when the induction proceeds to $T / \eta^2$.
\end{proof}

\begin{proof}[Proof for \Cref{thm:main-sph}]
Same as above but we invoke the spherical version of lemmas and theorems.  
\end{proof}

%% file: working-zone.tex
\section{Lemmas for Working Zones}
\label{sec:working-zone-proof}

\subsection{Construction of Working Zones} \label{sec:working-zone-proof-construction}

\begin{proof}[Proof for \Cref{lm:ful-working-zone}]
By \Cref{ass:man-ful} and \Cref{thm:rank-to-pl-ful},
there exists $\epsilon_1$ such that Items 1 and 2 hold.
By \Cref{thm:pl-converge-near-ful},
we can choose $\epsilon_0' > 0$ small enough so that
any gradient flow starting in $\gZ^{\epsilon_0'}$ 
moves at most $O(\epsilon_0')$ in distance
and converges in $\Zeo$.
We can further combine this with 
the results by \citet{falconer1983differentiation}
to show that $\Phi$ is $\contC^{3}$-smooth in a neighborhood of $\gZ$.
Then we can choose $\epsilon_0 > 0$ small enough
so that $\epsilon_0 < \epsilon_0'$ and $\Zei$ is a subset of that neighborhood,
which ensures Items 3 and 4.
Finally, Item 5 is directly implied by the compactness of $\cl(\Zeo) \cap \manGa$.
\end{proof}

\begin{proof}[Proof for \Cref{lm:sph-working-zone}]
We can use a similar argument as above but with
\Cref{thm:rank-to-pl-sph,thm:pl-converge-near-sph}.
\end{proof}

\subsection{Gradient Flow Projection}

In the working zone, $\Phi(\vtheta)$ is well-defined and $\contC^3$-smooth.
Now we highlight some useful properties.

\subsubsection{Full Space Optimization}

The following two lemmas are from \citet{li2022what}.
\begin{lemma} \label{lm:Phi-proj-ful}
  For $\vphi \in \Zeig$, $\partial \Phi_{\vphi}[\vx] = \mPHz(\vphi) \vx$,
  which also
  equals to the projection of $\vx$ onto the tangent space $\TGa{\vphi}$ at $\vphi$.
\end{lemma}

\begin{lemma} \label{lm:Phi-2nd-normal-ful}
  For $\vphi \in \Zeig$ and $\vx \in \NGa{\vphi}$, $\partial^2 \Phi_{\vphi}[\vx, \vx] = \vzero$.
\end{lemma}

The following lemma can be proved by Taylor expansion.
\begin{lemma} \label{lm:phiphix-ful}
  For $\vtheta \in \Zei$ and $\vphi = \Phi(\vtheta)$, then $\normtwosm{\partial \Phi_{\vphi}[\vtheta - \vphi]} \le O(\normtwosm{\vtheta - \vphi}^2)$.
\end{lemma}
\begin{proof}
  When $\normtwosm{\vtheta - \vphi}$ is small enough,
  the linear interpolation of $\vtheta$ and $\vphi$ lies in $\Zeo$.
  By Taylor expansion,
  $\Phi(\vtheta) = \Phi(\vphi) + \partial \Phi_{\vphi}[\vtheta - \vphi] + O(\normtwosm{\vtheta - \vphi}^2)$.
  In fact, $\Phi(\vtheta) = \vphi = \Phi(\vphi)$. So we can conclude $\partial \Phi_{\vphi}[\vtheta - \vphi] = O(\normtwosm{\vtheta - \vphi}^2)$.
\end{proof}

\subsubsection{Spherical Optimization}

In the spherical case,
we have the following lemmas that highly resemble those in the full space case.
All these lemmas can be proved in the same manner:
we can first define $\manGa' := \{ \nu \vtheta : \vtheta \in \manGa, \nu \in (1/2, 2)\}$
and apply the counterparts of these lemmas in the full space case;
then translate the results back to the spherical case.
\begin{lemma} \label{lm:Phi-proj-sph}
  For $\vphi \in \Zeig$, $(\mI - \vphi\vphi^\top)\partial \Phi_{\vphi}[\vx] = (\mI - \vphi\vphi^\top) \mPHz \vx$,
  which also equals to the projection of $\vx$ onto the tangent space $\TGa{\vphi}$ at $\vphi$.
\end{lemma}

\begin{lemma} \label{lm:Phi-2nd-normal-sph}
  For $\vphi \in \Zeig$ and $\vx \in \NGa{\vphi}$, if $\dotpsm{\vx}{\vphi} = 0$, then $\partial^2 \Phi_{\vphi}[\vx, \vx] = \vzero$.
\end{lemma}

\begin{lemma} \label{lm:phiphix-sph}
  For $\vtheta \in \Zeis$ and $\vphi = \Phi(\vtheta)$, then $\normtwosm{\partial \Phi_{\vphi}[\vtheta - \vphi]} \le O(\normtwosm{\vtheta - \vphi}^2)$.
\end{lemma}

%% file: gd-lemma.tex
\section{Lemmas for Gradient Descent} \label{sec:gd-proof}

\subsection{Full Space Optimization}

\begin{lemma} \label{lm:gd-1-step}
  Under \Cref{ass:man-ful},
  consider one step of gradient descent
  $\vtheta_{t+1} = \vtheta_t - \eeta_t \nabla \Loss(\vtheta_t)$
  with effective learning rate $\eeta_t$.
  Let $\vphi \in \Zeig$ be a local minimizer.
  If for some parameters $\alpha = \eta^{-o(1)}$ and $r = O(\alpha)$,
  $\normtwosm{\vtheta_t - \vphi} = O(r\eta)$
  and $\eeta_t = \frac{2}{\lamH_1(\vphi)} + O(\alpha \eta)$,
  then after one step, we have the following approximations for $\vtheta_{t+1}$:
  \begin{enumerate}
    \item Zeroth-order approximation:
      \begin{align*}
        \vtheta_{t+1} - \vphi = O(r \eta).
      \end{align*}
    \item First-order approximation:
      \begin{align*}
        \vtheta_{t+1} - \vphi
        &= \left(\mI - \eeta_t \mH(\vphi)\right) (\vtheta_t - \vphi) + O(r^2 \eta^2) \\
        &= \left(\mI - \tfrac{2}{\lamH_1(\vphi)} \mH(\vphi)\right) (\vtheta_t - \vphi) + O(\alpha r \eta^2).
      \end{align*}
    \item Second-order approximation:
      \begin{align*}
        \vtheta_{t+1} - \vphi
        &= \left(\mI - \eeta_t \mH(\vphi)\right) (\vtheta_t - \vphi) - \tfrac{1}{\lamH_1(\vphi)} \partial^3 \Loss_{\vphi}[\vtheta_t - \vphi, \vtheta_t - \vphi] + O(\alpha r^2 \eta^3).
      \end{align*}
  \end{enumerate}
\end{lemma}
\begin{proof}
  For the zeroth-order approximation, we expand $\Loss(\vtheta_t)$ around $\vphi$ using Taylor expansion:
  \begin{align*}
    \nabla \Loss(\vtheta_t) &= \nabla \Loss(\vphi) + O(r \eta) = O(r\eta).
  \end{align*}
  So $\vtheta_{t+1} - \vphi = \vtheta_t - \vphi - \eeta_t \nabla \Loss(\vtheta_t) = O(r\eta)$.
  
  For the first-order approximation, we expand $\Loss(\vtheta_t)$ around $\vphi$ using Taylor expansion:
  \begin{align*}
    \nabla \Loss(\vtheta_t) &= \nabla \Loss(\vphi) + \mH(\vphi) (\vtheta_t - \vphi) + O((r \eta)^2) \\
    &= \mH(\vphi) (\vtheta_t - \vphi) + O(r^2 \eta^2).
  \end{align*}
  So $\vtheta_{t+1} - \vphi = \vtheta_t - \vphi - \eeta_t \nabla \Loss(\vtheta_t) = (\mI - \eeta_t\mH(\vphi))(\vtheta_t - \vphi) + O(r^2\eta^2)$.

  For the second-order approximation, we again use Taylor expansion to expand $\Loss(\vtheta_t)$:
  \begin{align*}
    \nabla \Loss(\vtheta_t) &= \nabla \Loss(\vphi) + \mH(\vphi) (\vtheta_t - \vphi) + \frac{1}{2}\partial^3 \Loss_{\vphi}[\vtheta_t - \vphi, \vtheta_t - \vphi] + O((r \eta)^3) \\
    &= \mH(\vphi) (\vtheta_t - \vphi) + \frac{1}{2}\partial^3 \Loss_{\vphi}[\vtheta_t - \vphi, \vtheta_t - \vphi] + O(r^3 \eta^3).
  \end{align*}
  So we have
  \begin{align*}
    \vtheta_{t+1} - \vphi &= \vtheta_t - \vphi - \eeta_t \nabla \Loss(\vtheta_t) \\
    &= (\mI - \eeta_t\mH(\vphi))(\vtheta_t - \vphi) - \frac{\eeta_t}{2}\partial^3 \Loss_{\vphi}[\vtheta_t - \vphi, \vtheta_t - \vphi] + O(r^3 \eta^3) \\
    &= (\mI - \eeta_t\mH(\vphi))(\vtheta_t - \vphi) - \left(\tfrac{1}{\lamH_1(\vphi)} + O(\alpha \eta)\right) \partial^3 \Loss_{\vphi}[\vtheta_t - \vphi, \vtheta_t - \vphi] + O(r^3 \eta^3) \\
    &= (\mI - \eeta_t\mH(\vphi))(\vtheta_t - \vphi) - \tfrac{1}{\lamH_1(\vphi)} \partial^3 \Loss_{\vphi}[\vtheta_t - \vphi, \vtheta_t - \vphi] + O(\alpha r^2 \eta^3),
  \end{align*}
  where the last equality uses $r = O(\alpha)$.
\end{proof}

\begin{lemma} \label{lm:gd-2-step}
  Under \Cref{ass:man-ful},
  consider two steps of gradient descent
  with effective learning rates $\eeta_t$ and $\eeta_{t+1}$:
  \begin{align*}
    \vtheta_{t+1} &= \vtheta_t - \eeta_t \nabla \Loss(\vtheta_t), \\
    \vtheta_{t+2} &= \vtheta_{t+1} - \eeta_{t+1} \nabla \Loss(\vtheta_{t+1}).
  \end{align*}
  Let $\vphi \in \Zeig$ be a local minimizer.
  If for some parameters $\alpha = \eta^{-o(1)}$ and $r = O(\alpha)$,
  $\normtwosm{\vtheta_t - \vphi} = O(r\eta)$,
  $\eeta_t = \frac{2}{\lamH_1(\vphi)} + O(\alpha \eta)$,
  and $\eeta_{t+1} = \frac{2}{\lamH_1(\vphi)} + O(\alpha \eta)$,
  then after two steps,
  \[
    \vtheta_{t+2} - \vphi = (\mI - \eeta_{t+1} \mH(\vphi))(\mI - \eeta_t \mH(\vphi)) (\vtheta_t - \vphi) - \vpsi_{\vphi}(\vtheta_t - \vphi) + O(\alpha r^2 \eta^3),
  \]
  where $\vpsi_{\vphi}(\hatvx) := \frac{1}{\lamH_1(\vphi)}\left( \mU_{\vphi} \partial^3 \Loss_{\vphi}[\hatvx, \hatvx] + \partial^3 \Loss_{\vphi}[\mU_\vphi \hatvx, \mU_\vphi \hatvx] \right)$,
  $\mU_{\vphi} := \mI - \frac{2}{\lamH_1(\vphi)} \mH(\vphi)$.
\end{lemma}
\begin{proof}
Let $\hatvx_{\tau} := \vtheta_{\tau} - \vphi$ for all $\tau \in \{t, t+1, t+2\}$.
By \Cref{lm:gd-1-step}, we have 
$\hatvx_{t+1} = O(r \eta)$ and the following first-order and second-order approximations for $\hatvx_{t+1}$:
\begin{align*}
    \hatvx_{t+1} &= \mU_\vphi \hatvx_t + O(\alpha r \eta^2). \\
    \hatvx_{t+1} &= (\mI - \eeta_t \mH(\vphi)) \hatvx_t - \tfrac{1}{\lamH_1(\vphi)} \partial^3 \Loss_{\vphi}[\hatvx_t, \hatvx_t] + O(\alpha r^2 \eta^3).
\end{align*}
Note that $\eeta_{t+1} = \eeta_t + O((1+r^2) \eta^2) = \frac{2}{\lamH_1(\vphi)} + O(\alpha \eta)$.
By \Cref{lm:gd-1-step} again, we have the following second-order approximation for $\hatvx_{t+2}$:
\begin{align*}
    \hatvx_{t+2} &= (\mI - \eeta_{t+1}\mH(\vphi)) \hatvx_{t+1} - \tfrac{1}{\lamH_1(\vphi)} \partial^3 \Loss_{\vphi}[\hatvx_{t+1}, \hatvx_{t+1}] + O(\alpha r^2\eta^3).
\end{align*}
Now we combine the two steps together. 
\begin{align*}
    \hatvx_{t+2} &= (\mI - \eeta_{t+1}\mH(\vphi))\left((\mI - \eeta_t \mH(\vphi)) \hatvx_t - \tfrac{1}{\lamH_1(\vphi)} \partial^3 \Loss_{\vphi}[\hatvx_t, \hatvx_t] + O(\alpha r^2 \eta^3)\right) \\
    & \qquad - \tfrac{1}{\lamH_1(\vphi)} \partial^3 \Loss_{\vphi}[\mU_\vphi \hatvx_t + O(\alpha r \eta^2), \mU_\vphi \hatvx_t + O(\alpha r \eta^2)] + O(\alpha r^2 \eta^3) \\
    &= (\mI - \eeta_{t+1} \mH(\vphi))(\mI - \eeta_t \mH(\vphi)) \hatvx_t - \tfrac{1}{\lamH_1(\vphi)} \mU_\vphi \partial^3 \Loss_{\vphi}[\hatvx_t, \hatvx_t] \\
    & \qquad - \tfrac{1}{\lamH_1(\vphi)} \partial^3 \Loss_{\vphi}[\mU_\vphi \hatvx_t, \mU_\vphi \hatvx_t] + O(\alpha r^2 \eta^3) \\
    &= (\mI - \eeta_{t+1} \mH(\vphi))(\mI - \eeta_t \mH(\vphi)) \hatvx_t - \vpsi_{\vphi}(\hatvx_t) + O(\alpha r^2 \eta^3),
\end{align*}
where the second equality uses $\mI - \eeta_{t+1}\mH(\vphi) = \mU_\vphi + O(\alpha \eta)$.
\end{proof}

The following lemma characterizes the function $\vpsi_{\vphi}(\hatvx)$
in \Cref{lm:gd-2-step} when $\hatvx$ is the top eigenvector of $\mH(\vphi)$.
We will use this property later in \Cref{sec:ful-proof-drifting}.
\begin{lemma} \label{lm:ful-psi-formula}
  Under \Cref{ass:man-ful,ass:topeigen-ful}, for $\vphi \in \manGa$,
\begin{align*}
  \vpsi_{\vphi}(\vvH_1(\vphi)) = (2\mI - \tfrac{2}{\lamH_1(\vphi)} \mH(\vphi)) \nabla \log \lamH_1(\vphi),
\end{align*}
where $\vpsi_{\vphi}$ is defined as in \Cref{lm:gd-2-step}.
Moreover, 
\begin{align*}
  \dotpsm{\vpsi_{\vphi}(\vvH_1(\vphi))}{\vvH_1(\vphi)} &= 0, \\
  \mPHz(\vphi) \vpsi_{\vphi}(\vvH_1(\vphi)) &= 2 \gradGa \log \lamH_1(\vphi).
\end{align*}
\end{lemma}
\begin{proof}
Note that $\mU_{\vphi} \vvH_1(\vphi) = - \vvH_1(\vphi)$. Then we can rewrite 
$\vpsi_{\vphi}(\vvH_1(\vphi))$ as follows:
\begin{align*}
  \vpsi_{\vphi}(\vvH_1(\vphi))
  &= \tfrac{1}{\lamH_1(\vphi)} \left( \mU_{\vphi} \partial^3\Loss_{\vphi}[\vvH_1(\vphi), \vvH_1(\vphi)] + \partial^3 \Loss_{\vphi}[-\vvH_1(\vphi), -\vvH_1(\vphi)] \right) \\
  &= \tfrac{1}{\lamH_1(\vphi)} (2\mI - \tfrac{2}{\lamH_1(\vphi)} \mH(\vphi)) \partial^3\Loss_{\vphi}[\vvH_1(\vphi), \vvH_1(\vphi)] \\
  &= \tfrac{1}{\lamH_1(\vphi)} (2\mI - \tfrac{2}{\lamH_1(\vphi)} \mH(\vphi)) \nabla \lamH_1(\vphi) \\
  &= (2\mI - \tfrac{2}{\lamH_1(\vphi)} \mH(\vphi)) \nabla \log \lamH_1(\vphi).
\end{align*}
Since $(2\mI - \tfrac{2}{\lamH_1(\vphi)} \mH(\vphi)) \vvH_1(\vphi) = 2\vvH_1(\vphi) - 2\vvH_1(\vphi) = \vzero$,
we have $\dotpsm{\vpsi_{\vphi}(\vvH_1(\vphi))}{\vvH_1(\vphi)} = 0$.

The projection matrix
onto the tangent space of $\manGa$ at $\vphi$
equals to $\mPHz(\vphi)$
by \Cref{lm:Phi-proj-ful}.
Then
\begin{align*}
  \mPHz(\vphi) \vpsi_{\vphi}(\vvH_1(\vphi))
  = \mPHz(\vphi)
    (2\mI - \tfrac{2}{\lamH_1(\vphi)} \mH(\vphi)) \nabla \log \lamH_1(\vphi)
  &= (2\mPHz(\vphi) - \vzero) \nabla \log \lamH_1(\vphi)
  \\
  &= 2 \gradGa \log \lamH_1(\vphi),
\end{align*}
which proves the lemma.
\end{proof}

\subsection{Spherical Optimization}

\begin{lemma} \label{lm:pgd-1-step}
  Under \Cref{ass:man},
  consider one step of projected gradient descent on $\sphS^{D-1}$,
  $\vtheta_{t+1} = \Pi(\vtheta_t - \eeta_t \nabla \Loss(\vtheta_t))$
  with effective learning rate $\eeta_t$.
  Let $\vphi \in \Zeig$ be a local minimizer.
  If for some parameters $\alpha = \eta^{-o(1)}$ and $r = O(\alpha)$,
  $\normtwosm{\vtheta_t - \vphi} = O(r\eta)$
  and $\eeta_t = \frac{2}{\lamH_1(\vphi)} + O(\alpha \eta)$,
  then after one step, we have the following approximations for $\vtheta_{t+1}$:
  \begin{enumerate}
    \item Zeroth-order approximation:
      \begin{align*}
        \vtheta_{t+1} - \vphi = O(r \eta).
      \end{align*}
    \item First-order approximation:
      \begin{align*}
        \vtheta_{t+1} - \vphi
        &= \left(\mI - \eeta_t \mH(\vphi)\right) (\vtheta_t - \vphi) + O(r^2 \eta^2) \\
        &= \left(\mI - \tfrac{2}{\lamH_1(\vphi)} \mH(\vphi)\right) (\vtheta_t - \vphi) + O(\alpha r \eta^2).
      \end{align*}
    \item Second-order approximation:
      \begin{align*}
        \vtheta_{t+1} - \vphi
        &= \left(\mI - \eeta_t \mH(\vphi)\right) (\vtheta_t - \vphi) - \tfrac{1}{\lamH_1(\vphi)} \partial^3 \Loss_{\vphi}[\vtheta_t - \vphi, \vtheta_t - \vphi] \\
        & \qquad - \tfrac{2}{\lamH_1(\vphi)^2}  \normtwosm{\mH(\vphi) (\vtheta_t - \vphi)}^2 \vphi + O(\alpha r^2 \eta^3).
      \end{align*}
  \end{enumerate}
\end{lemma}
\begin{proof}
  Let $\hatvtheta_{t+1} = \vtheta_t - \eeta_t \nabla \Loss(\vtheta_t)$.
  Then $\vtheta_{t+1} = \frac{\hatvtheta_{t+1}}{\normtwosm{\hatvtheta_{t+1}}}$.

  Since $\nabla \Loss(\vtheta_t)$ is perpendicular to $\vtheta_t$,
  $\normtwosm{\hatvtheta_{t+1}} = \sqrt{\normtwosm{\vtheta_t}^2 + \eeta_t^2 \normtwosm{\nabla \Loss(\vtheta_t)}^2 }$.
  By Taylor expansion, $\nabla \Loss(\vtheta_t) = O(r \eta)$.
  Then the norm of $\hatvtheta_{t+1}$ can be estimated by
  \begin{align*}
    \normtwosm{\hatvtheta_{t+1}} = \sqrt{\normtwosm{\vtheta_t}^2 + \eeta_t^2 \normtwosm{\nabla \Loss(\vtheta_t)}^2 }
    = \sqrt{1 + O(r^2\eta^2) } = 1 + O(r^2 \eta^2).
  \end{align*}
  Then $\vtheta_{t+1} = \frac{\hatvtheta_{t+1}}{\normtwosm{\hatvtheta_{t+1}}} = \hatvtheta_{t+1} + O(r^2 \eta^2)$.
  Combining with \Cref{lm:gd-1-step} proves the zeroth- and first-order approximations.

  To prove the second-order approximation, we need a tighter estimate for the norm of $\hatvtheta_{t+1}$.
  By Taylor expansion, we have $\nabla \Loss(\vtheta_t) = \mH(\vphi) (\vtheta_t - \vphi) + O(r^2 \eta^2)$.
  Squaring both sides gives $\normtwosm{\nabla \Loss(\vtheta_t)}^2 = \normtwosm{\mH(\vphi) (\vtheta_t - \vphi)}^2 + O(r^3 \eta^3)$.
  Then we have
  \begin{align*}
    \normtwosm{\hatvtheta_{t+1}} &= \sqrt{\normtwosm{\vtheta_t}^2 + \eeta_t^2 \normtwosm{\nabla \Loss(\vtheta_t)}^2 } \\
    &= \sqrt{1 + \left(\tfrac{2}{\lamH_1(\vphi)} + O(\alpha \eta)\right)^2 (\normtwosm{\mH(\vphi) (\vtheta_t - \vphi)}^2 + O(r^3 \eta^3)) } \\
    &= \sqrt{1 + \tfrac{4}{\lamH_1(\vphi)^2}  \normtwosm{\mH(\vphi) (\vtheta_t - \vphi)}^2 + O(\alpha r^2 \eta^3) } \\
    &= \sqrt{1 + \tfrac{4}{\lamH_1(\vphi)^2}  \normtwosm{\mH(\vphi) (\vtheta_t - \vphi)}^2} + O(\alpha r^2 \eta^3) \\
    &= 1 + \tfrac{2}{\lamH_1(\vphi)^2}  \normtwosm{\mH(\vphi) (\vtheta_t - \vphi)}^2 + O(\alpha r^2 \eta^3).
  \end{align*}
  So $\vtheta_{t+1}$ can be estimated by
  \begin{align*}
    \vtheta_{t+1} = \frac{\hatvtheta_{t+1}}{\normtwosm{\hatvtheta_{t+1}}}
    &= \frac{\hatvtheta_t}{1 + \tfrac{2}{\lamH_1(\vphi)^2}  \normtwosm{\mH(\vphi) (\vtheta_t - \vphi)}^2 + O(\alpha r^2 \eta^3)} \\
    &= \left(1 - \tfrac{2}{\lamH_1(\vphi)^2}  \normtwosm{\mH(\vphi) (\vtheta_t - \vphi)}^2\right) \hatvtheta_t + O(\alpha r^2 \eta^3).
  \end{align*}
  Then $\vtheta_{t+1} - \vphi$ can be estimated by
  \begin{align*}
    \vtheta_{t+1} - \vphi
    &= \left(1 - \tfrac{2}{\lamH_1(\vphi)^2}  \normtwosm{\mH(\vphi) (\vtheta_t - \vphi)}^2\right)
    \hatvtheta_t - \vphi + O(\alpha r^2 \eta^3) \\
    &= \left(1 - \tfrac{2}{\lamH_1(\vphi)^2}  \normtwosm{\mH(\vphi) (\vtheta_t - \vphi)}^2\right) (\hatvtheta_t - \vphi) \\
    & \qquad - \tfrac{2}{\lamH_1(\vphi)^2}  \normtwosm{\mH(\vphi) (\vtheta_t - \vphi)}^2 \vphi + O(\alpha r^2 \eta^3) \\
    &= (\hatvtheta_t - \vphi) + O(r^3 \eta^3) \\
    & \qquad - \tfrac{2}{\lamH_1(\vphi)^2}  \normtwosm{\mH(\vphi) (\vtheta_t - \vphi)}^2 \vphi + O(\alpha r^2 \eta^3) \\
    &= \left(\mI - \eeta_t \mH(\vphi)\right) (\vtheta_t - \vphi) - \tfrac{1}{\lamH_1(\vphi)} \partial^3 \Loss_{\vphi}[\vtheta_t - \vphi, \vtheta_t - \vphi] \\
    & \qquad - \tfrac{2}{\lamH_1(\vphi)^2}  \normtwosm{\mH(\vphi) (\vtheta_t - \vphi)}^2 \vphi + O(\alpha r^2 \eta^3),
  \end{align*}
  where the last equality uses the second-order approximation in \Cref{lm:gd-1-step}.
\end{proof}

\begin{lemma} \label{lm:pgd-2-step}
  Under \Cref{ass:man},
  consider two steps of projectioned gradient descent on $\sphS^{D-1}$
  with effective learning rates $\eeta_t$ and $\eeta_{t+1}$:
  \begin{align*}
    \vtheta_{t+1} &= \Pi(\vtheta_t - \eeta_t \nabla \Loss(\vtheta_t)), \\
    \vtheta_{t+2} &= \Pi(\vtheta_{t+1} - \eeta_{t+1} \nabla \Loss(\vtheta_{t+1})).
  \end{align*}
  Let $\vphi \in \Zeig$ be a local minimizer.
  If for some parameters $\alpha = \eta^{-o(1)}$ and $r = O(\alpha)$,
  $\normtwosm{\vtheta_t - \vphi} = O(r\eta)$,
  $\eeta_t = \frac{2}{\lamH_1(\vphi)} + O(\alpha \eta)$,
  and $\eeta_{t+1} = \frac{2}{\lamH_1(\vphi)} + O(\alpha \eta)$,
  then after two steps,
  \[
    \vtheta_{t+2} - \vphi = (\mI - \eeta_{t+1} \mH(\vphi))(\mI - \eeta_t \mH(\vphi)) (\vtheta_t - \vphi) - \vpsi_{\vphi}(\vtheta_t - \vphi) + O(\alpha r^2 \eta^3),
  \]
  where
  \begin{align*}
    \vpsi_{\vphi}(\hatvx) &:=
    \tfrac{1}{\lamH_1(\vphi)}\left( \mU_{\vphi} \partial^3 \Loss_{\vphi}[\hatvx, \hatvx] + \partial^3 \Loss_{\vphi}[\mU_\vphi \hatvx, \mU_\vphi \hatvx] \right) \\
    & \qquad\qquad + \tfrac{2}{\lamH_1(\vphi)^2}\left( \normtwosm{\mH(\vphi) \hatvx}^2 + \normtwosm{\mH(\vphi) \mU_\vphi \hatvx}^2 \right) \vphi, \\
    \mU_{\vphi} &:= \mI - \tfrac{2}{\lamH_1(\vphi)} \mH(\vphi).
  \end{align*}
\end{lemma}
\begin{proof}
Let $\hatvx_{\tau} := \vtheta_{\tau} - \vphi$ for all $\tau \in \{t, t+1, t+2\}$.
Let $\widehat{\vpsi}_{\vphi}(\vx)$ be the following function:
\begin{align*}
  \widehat{\vpsi}_{\vphi}(\hatvx) := \tfrac{1}{\lamH_1(\vphi)} \partial^3 \Loss_{\vphi}[\hatvx, \hatvx] + \tfrac{2}{\lamH_1(\vphi)^2}  \normtwosm{\mH(\vphi) \hatvx}^2 \vphi.
\end{align*}
By \Cref{lm:pgd-1-step}, we have 
$\hatvx_{t+1} = O(r \eta)$ and the following first-order and second-order approximations for $\hatvx_{t+1}$:
\begin{align*}
    \hatvx_{t+1} &= \mU_\vphi \hatvx_t + O(\alpha r \eta^2). \\
    \hatvx_{t+1} &= (\mI - \eeta_t \mH(\vphi)) \hatvx_t - \widehat{\vpsi}_{\vphi}(\hatvx_t) + O(\alpha r^2 \eta^3).
\end{align*}
Note that $\eeta_{t+1} = \eeta_t + O((1+r^2) \eta^2) = \frac{2}{\lamH_1(\vphi)} + O(\alpha \eta)$.
By \Cref{lm:pgd-1-step} again, we have the following second-order approximation for $\hatvx_{t+2}$:
\begin{align*}
    \hatvx_{t+2} &= (\mI - \eeta_{t+1}\mH(\vphi)) \hatvx_{t+1} - \widehat{\vpsi}_{\vphi}(\hatvx_{t+1}) + O(\alpha r^2\eta^3).
\end{align*}
Now we combine the two steps together.
\begin{align*}
    \hatvx_{t+2} &= (\mI - \eeta_{t+1}\mH(\vphi))\left((\mI - \eeta_t \mH(\vphi)) \hatvx_t - \widehat{\vpsi}_{\vphi}(\hatvx_{t}) + O(\alpha r^2 \eta^3)\right) \\
    & \qquad - \widehat{\vpsi}_{\vphi}( \mU_\vphi \hatvx_t + O(\alpha r \eta^2) ) + O(\alpha r^2 \eta^3) \\
    &= (\mI - \eeta_{t+1}\mH(\vphi))(\mI - \eeta_t \mH(\vphi) \hatvx_t - (\mU_\vphi + O(\alpha \eta))\widehat{\vpsi}_{\vphi}(\hatvx_{t}) + O(\alpha r^2 \eta^3) \\
    & \qquad - \left(\widehat{\vpsi}_{\vphi}( \mU_\vphi \hatvx_t) + O(\alpha r^2 \eta^3)\right) + O(\alpha r^2 \eta^3) \\
    &= (\mI - \eeta_{t+1}\mH(\vphi))(\mI - \eeta_t \mH(\vphi) \hatvx_t
    - \left(\mU_\vphi \widehat{\vpsi}_{\vphi}(\hatvx_{t}) + \widehat{\vpsi}_{\vphi}( \mU_\vphi \hatvx_t) \right)
    + O(\alpha r^2 \eta^3).
\end{align*}
where the second equality uses $\mI - \eeta_{t+1}\mH(\vphi) = \mU_\vphi + O(\alpha \eta)$.
Finally, we note that
\begin{align*}
  \mU_\vphi \widehat{\vpsi}_{\vphi}(\hatvx) + \widehat{\vpsi}_{\vphi}( \mU_\vphi \hatvx) &= 
  \tfrac{1}{\lamH_1(\vphi)} \mU_\vphi \partial^3 \Loss_{\vphi}[\hatvx, \hatvx] + \tfrac{2}{\lamH_1(\vphi)^2}  \normtwosm{\mH(\vphi) \hatvx}^2 \mU_\vphi \vphi \\
  & \qquad + \tfrac{1}{\lamH_1(\vphi)} \partial^3 \Loss_{\vphi}[\mU_\vphi \hatvx, \mU_\vphi \hatvx] + \tfrac{2}{\lamH_1(\vphi)^2}  \normtwosm{\mH(\vphi) \mU_\vphi \hatvx}^2 \vphi \\
  &= \tfrac{1}{\lamH_1(\vphi)}\left( \mU_{\vphi} \partial^3 \Loss_{\vphi}[\hatvx, \hatvx] + \partial^3 \Loss_{\vphi}[\mU_\vphi \hatvx, \mU_\vphi \hatvx] \right) \\
  & \qquad + \tfrac{2}{\lamH_1(\vphi)^2}\left( \normtwosm{\mH(\vphi) \hatvx}^2 + \normtwosm{\mH(\vphi) \mU_\vphi \hatvx}^2 \right) \vphi \\
  &= \vpsi_{\vphi}(\hatvx),
\end{align*}
where the second equality uses the fact that $\mU_{\vphi} = (\mI - \frac{2}{\lamH_1(\vphi)} \mH(\vphi)) \vphi = \vphi + \frac{2}{\lamH_1(\vphi)} \nabla \Loss(\vphi) = \vphi$ (\Cref{lm:scale-invariant-grad-hess}).
\end{proof}

\begin{lemma} \label{lm:sph-psi-formula}
  Under \Cref{ass:man,ass:topeigen},
  for $\vphi \in \manGa$,
  \begin{align*}
    \vpsi_{\vphi}(\vvH_1(\vphi)) = (2\mI - \tfrac{2}{\lamH_1(\vphi)} \mH_t) \nabla \log \lamH_1(\vphi) + 4\vphi,
  \end{align*}
  where $\vpsi_{\vphi}$ is defined as in \Cref{lm:pgd-2-step}. Moreover,
\begin{align}
  \dotpsm{\vpsi_{\vphi}(\vvH_1(\vphi))}{\vvH_1(\vphi)} &= 0, \label{eq:sph-psi-proj-v1} \\
  \mPHz(\vphi) \vpsi_{\vphi}(\vvH_1(\vphi)) &= 2 \gradGa \log \lamH_1(\vphi). \label{eq:sph-psi-proj-H0}
\end{align}
\end{lemma}
\begin{proof}
  Let $\mV_{\vphi} := 2\mI - \frac{2}{\lamH_1(\vphi)} \mH(\vphi)$.
    Using a similar argument as in \Cref{lm:ful-psi-formula},
    \[
      \tfrac{1}{\lamH_1(\vphi)} \left( \mU_{\vphi} \partial^3\Loss_{\vphi}[\vvH_1(\vphi), \vvH_1(\vphi)] + \partial^3 \Loss_{\vphi}[-\vvH_1(\vphi), -\vvH_1(\vphi)] \right)
      = \mV_{\vphi} \nabla \log \lamH_1(\vphi).
    \]
    Also notice that
    \begin{align*}
    \tfrac{2}{\lamH_1(\vphi)^2} \left( \normtwosm{\mH(\vphi) \vvH_1(\vphi)}^2 + \normtwosm{\mH(\vphi) \mU_{\vphi} \vvH_1(\vphi)}^2 \right) \vphi
    = \tfrac{2}{\lamH_1(\vphi)} ( \lamH_1(\vphi)^2 + \lamH_1(\vphi)^2 ) \vphi = 4 \vphi.
    \end{align*}
    Combining these together proves
    $\vpsi_{\vphi}(\vvH_1(\vphi)) = \mV_{\vphi} \nabla \log \lamH_1(\vphi) + 4\vphi$.

    To obtain the last two equations \eqref{eq:sph-psi-proj-v1}, \eqref{eq:sph-psi-proj-H0},
    we first note that 
    $\vphi$ is a $0$-eigenvector of $\mH(\vphi)$
    since 
    $\mH(\vphi) \vphi = -\nabla \Loss(\vphi) = \vzero$
    by \Cref{lm:scale-invariant-grad-hess}.
    So we have $\dotpsm{\vphi}{\vvH_1(\vphi)} = 0$ and $\mPHz(\vphi) \vphi = \vphi$.

    To prove \eqref{eq:sph-psi-proj-v1},
    we note that 
    $\mV_{\vphi} \vvH_1(\vphi) = \vzero$.
    Then 
    \begin{align*}
      \dotpsm{\vpsi_{\vphi}(\vvH_1(\vphi))}{\vvH_1(\vphi)} = \dotpsm{\nabla \log \lamH_1(\vphi)}{\mV_{\vphi}\vvH_1(\vphi)} + 4\dotpsm{\vphi}{\vvH_1(\vphi)} = 0.
    \end{align*}

    To prove \eqref{eq:sph-psi-proj-H0},
    first we note that $\mPHz(\vphi)\mV_{\vphi} = 2\mI$ and $\mPHz(\vphi) \vphi = \vphi$, which implies
    \begin{align*}
      \mPHz(\vphi) \vpsi_{\vphi}(\vvH_1(\vphi)) = 2\mPHz(\vphi) \nabla \log \lamH_1(\vphi) + 4 \vphi.
    \end{align*}
    Then 
    by \Cref{lm:Phi-proj-sph},
    we can decompose 
    $\mPHz(\vphi)\nabla \log \lamH_1(\vphi)$
    as a component parallel to $\vphi$ and a component perpendicular to $\vphi$:
    \begin{align*}
      \mPHz(\vphi)\nabla \log \lamH_1(\vphi) &= 
      (\mI - \vphi \vphi^\top)\mPHz(\vphi)\nabla \log \lamH_1(\vphi)
      + \dotpsm{\nabla \log \lamH_1(\vphi)}{\vphi} \vphi \\
      &=
      \gradGa \log \lamH_1(\vphi)
      + \left.\frac{\partial}{\partial c} \log \lamH_1(c\vphi)\right|_{c=1}.
    \end{align*}
    For the second term, we note that 
    $\lamH_1(c \vphi) = c^{-2} \lamH_1(\vphi)$ by scale-invariance,
    and thus we have $\frac{\partial}{\partial c}\log \lamH_1(c \vphi) = -2 / c$
    and $\left.\frac{\partial}{\partial c} \log \lamH_1(c\vphi)\right|_{c=1}= -2$. 
    Combining all these together gives
    \begin{align*}
      \mPHz(\vphi) \vpsi_{\vphi}(\vvH_1(\vphi)) =
      2\left(\gradGa \log \lamH_1(\vphi)
      - 2 \vphi\right)
      + 4 \vphi = 2\gradGa \log \lamH_1(\vphi),
    \end{align*}
    which proves \eqref{eq:sph-psi-proj-H0}.
\end{proof}

%% file: lr-lemma.tex
\section{Lemmas for Quasi-RMSprop Schedulers} \label{sec:lr-proof}

\begin{lemma} \label{lm:lr-1-step}
    Given gradients $\{\vg_t\}_{t \ge 0}$,
    let $\eeta_0, \eeta_1, \eeta_2, \dots$ be the effective learning rates produced by
    a quasi-RMSprop scheduler with base learning rate $\eta$
    and decay rate $\beta$,
    and let $\{\tilde{v}_t\}_{t \ge 0}$ be the corresponding moment estimating sequence.
    Consider the case of $\eta = o(1)$, $\beta = 1 - \Cb \eta^2 + O(\eta^4)$ for some $\Cb = \Theta(1)$.
    For $\mu = \Theta(1)$ and some $t \ge 0$, define $\hatu_{\tau} := \frac{1}{\eta}(\mu^2\tilde{v}_{\tau} - 1)$ for $\tau \in \{t, t+1\}$.
    If $\normtwosm{\vg_t} \le \alpha \eta$ and $\abssm{\mu^2 \tilde{v}_t - 1} \le \alpha \eta$
    for some $\alpha = (\tfrac{1}{\eta})^{o(1)}$ at step $t$,
    then the following holds
    \begin{align}
        \eeta_t &= \mu \cdot \left(1 - \tfrac{1}{2}\eta \hatu_t\right) + O(\alpha^2 \eta^2). \label{eq:lr-1-step-eta} \\
        \hatu_{t+1} &= \hatu_t + \Cb \eta (  \mu^2 \bar{g}_t^2 - 1) + O(\alpha \eta^2).\label{eq:lr-1-step-u}
    \end{align}
\end{lemma}
\begin{proof}
    By definition of RMSprop-like learning rate scheduler,
    \begin{align}
        \eeta_t         &= \frac{1}{\sqrt{\tilde{v}_t}} + O(\alpha^2 \eta^2) \label{eq:lr-1-step-a} \\
        \tilde{v}_{t+1} &= \beta \tilde{v}_t + (1-\beta) \bar{g}_t^2  + O(\poly(\alpha) \eta^4). \label{eq:lr-1-step-b}
    \end{align}
    By definition of $\hatu_t$,
    we can express $\tilde{v}_t$ as
    $\tilde{v}_t = \frac{1}{\mu^2}(1 + \eta \hatu_t)$.
    Combining this with \eqref{eq:lr-1-step-a} proves \eqref{eq:lr-1-step-eta}:
    \begin{align*}
        \eeta_t
        = \frac{1}{\sqrt{\frac{1}{\mu^2}(1 + \eta \hatu_t)}}
        = \frac{\mu}{\sqrt{1 + \eta \hatu_t}}
        &= \mu \cdot \left(1 - \tfrac{1}{2}\eta \hatu_t + O(\alpha^2 \eta^2)\right) \\
        &= \mu \cdot \left(1 - \tfrac{1}{2}\eta \hatu_t\right) + O(\alpha^2 \eta^2).
    \end{align*}
    By substituting $\frac{1}{\mu^2}(1+\eta \hatu_t)$ for $\tilde{v}_t$ in \eqref{eq:lr-1-step-b} we have
    \begin{align*}
        \frac{1}{\mu^2}(1+\eta \hatu_{t+1}) = 
        \beta \cdot \frac{1}{\mu^2}(1+\eta \hatu_t) + (1-\beta) \bar{g}_t^2  + O(\poly(\alpha) \eta^4).
    \end{align*}
    Multiplying $\mu^2$ and subtracting $1$ on both sides gives
    \begin{align*}
        \eta \hatu_{t+1}
        &= -(1 - \beta) + \beta \eta \hatu_t + (1-\beta) \mu^2 \bar{g}_t^2  + O(\poly(\alpha) \eta^4) \\
        &= \eta \hatu_t + (1 - \beta)\left(  \mu^2 \bar{g}_t^2 - \eta \hatu_t - 1\right) + O(\poly(\alpha) \eta^4).
    \end{align*}
    Now we divide $\eta$ on both sides. Then we have
    \begin{align*}
        \hatu_{t+1}
        &= \hatu_t + \left(\Cb \eta + O(\eta^3)\right) \left(  \mu^2 \bar{g}_t^2 - \eta \hatu_t - 1\right) + O(\poly(\alpha) \eta^3) \\
        &= \hatu_t + \Cb \eta \left(  \mu^2 \bar{g}_t^2 - \eta \hatu_t - 1\right) + O(\poly(\alpha) \eta^3) \\
        &= \hatu_t + \Cb \eta (  \mu^2 \bar{g}_t^2 - 1) + O(\alpha \eta^2),
    \end{align*}
    which proves \eqref{eq:lr-1-step-u}.
\end{proof}

\begin{lemma} \label{lm:lr-2-step}
    In the setting of \Cref{lm:lr-1-step},
    for $\mu = \Theta(1)$ and some $t \ge 0$, define $\hatu_{\tau} := \frac{1}{\eta}(\mu^2\tilde{v}_{\tau} - 1)$ for $\tau \in \{t, t+1, t+2\}$.
    If $\normtwosm{\vg_t} \le \alpha \eta$, $\normtwosm{\vg_{t+1}} \le \alpha \eta$,
    and $\abssm{\mu^2 \tilde{v}_t - 1} \le \alpha \eta$
    for some $\alpha = (\tfrac{1}{\eta})^{o(1)}$ at step $t$,
    then the following holds
    \begin{align}
        \eeta_{t+1} &= \mu \cdot \left(1 - \tfrac{1}{2}\eta \hatu_t\right) + O(\alpha^2 \eta^2). \label{eq:lr-2-step-eta}\\
        \hatu_{t+2} &= \hatu_t + \Cb \eta (  \mu^2 \bar{g}_t^2 + \mu^2 \bar{g}_{t+1}^2 - 2) + O(\alpha \eta^2). \label{eq:lr-2-step-u}
    \end{align}
\end{lemma}
\begin{proof}
    Note that \eqref{eq:lr-1-step-u} implies $\hatu_{t+1} = O(\alpha)$. Applying the
    inequalities \eqref{eq:lr-1-step-eta} and \eqref{eq:lr-1-step-u}
    to step $t+1$ proves \eqref{eq:lr-2-step-eta} and \eqref{eq:lr-2-step-u}.
\end{proof}

%% file: full-space-proof.tex
\section{Reduction to RMS-drift Process: The Case of Full Space Optimization} \label{sec:ful-proof}
In this section, we let $\{(\vtheta_t, v_t)\}_{t \ge 0}$ be a trajectory of gradient descent
with quasi-RMSprop scheduler,
and let $\eta, \beta$ be the base learning rate and decay rate.
We follow the notations and terminologies in \Cref{sec:proof-outline-main}.

\subsection{Good Initialization} \label{sec:ful-good-init}

\begin{proof}
[Proof for \Cref{lm:ful-good-init}]
Let $\hatvx_0 := \vtheta_0 - \vzeta_0$
and $r = \normtwosm{\hatvx_0}$ for short.
By Gaussian concentration and anti-concentration,
with probability $1- \delta$, the following holds:
\begin{align*}
&\Omega(\delta \sigma_0) \le r \le O(\sigma_0 \sqrt{\log(1/\delta)}), \\
&\abssm{\dotpsm{\hatvx_0}{\vvH_1(\vzeta_0)}} \ge \Omega(\delta r).
\end{align*}
By Taylor expansion of $\Phi$ and \Cref{lm:Phi-proj-ful},
\begin{align*}
\vphi_0 &= \vzeta_0 + \partial \Phi_{\vzeta_0}[\hatvx_0] + O(r^2) \\
&= \vzeta_0 + \mPHz(\vzeta_0) \hatvx_0 + O(r^2).
\end{align*}
Thus we can approximate $\vx_0$ by
\begin{align*}
    \vx_0 = \hatvx_0 + \vzeta_0 - \vphi_0 = (\mI - \mPHz) \hatvx_0 + O(r^2).
\end{align*}
Now we give a lower bound for $\abssm{h_0} \eta$:
\begin{align*}
\abssm{h_0} \eta = \abssm{\dotpsm{\vx_0}{\vvH_1(\vphi_0)}}
&\ge 
\abssm{\dotpsm{\hatvx_0}{\vvH_1(\vzeta_0)}}
- \abssm{\dotpsm{\hatvx_0 - \vx_0}{\vvH_1(\vzeta_0)}}
- \normtwosm{\hatvx_0} \cdot \normtwosm{\vvH_1(\vzeta_0) - \vvH_1(\vphi_0)} \\
&\ge \Omega(\delta r) - O(r^2) - O(r^2) \\
&\ge \Omega(r) \cdot (\Omega(\delta) - O(r)).
\end{align*}
Since $r = O(\sigma_0 \sqrt{\log(1/\delta)})$,
we can choose $\delta := C_0 \alpha_0 \eta \sqrt{\log(1/\eta)}$
with a large enough $C_0$ such that the above inequality gives $\abssm{h_0} \eta \ge \Omega(\delta r)$.

Now we verify the conditions claimed in the lemma statement.
First,
we can see from the following that the initial state is $\Osm{\alpha_0 \sqrt{\log(1/\delta)}}$-bounded:
\begin{align*}
    \normtwosm{\vx_0} &\le \normtwosm{(\mI - \mPHz) \hatvx_0} + O(r^2) \le O(r) \le O(\alpha_0 \eta \sqrt{\log(1/\delta)}) \\
    \abssm{u_0} &= \tfrac{1}{\eta} \abssm{\mu_t^2 \tilv_0 - 1} \le O(\alpha_0).
\end{align*}
It is also $O(\alpha_0 + \sqrt{\log(1/\delta)})$-deviated
since 
\[
    \normtwosm{\vx_0} \ge \abssm{h_0} \eta \ge \Omega(\delta r) \ge \Omega(\delta^2 \exp(-\alpha_0^2)) \ge \eta \exp\left(-\Obig{\alpha_0 + \sqrt{\log(1/\delta)}}^2\right).
\]
Next, we verify that the initial state is $O(1/\delta)$-misaligned:
\[
    \normtwosm{\mPHnzt(\vphi_0) \vx_0} \le \normtwosm{\vx_0} \le O(r) \le O(1/\delta) \cdot \abssm{h_0} \eta.
\]
Finally, $\normtwosm{\vphi_0 - \vzeta_0} \le O(r) \le O(\alpha_0 \eta \sqrt{\log(1/\delta)})$.
\end{proof}

%% file: full-space-alignment-proof.tex
\subsection{Alignment Phase} \label{sec:ful-proof-alignment}

We define $r_t = \normtwosm{\vx_t} / \eta$ and the following notations
for this subsection.
\begin{align*}
  \hatvx_{t+1} &:= \vtheta_{t+1} - \vphi_t & \hatu_{t+1} &:= \tfrac{1}{\eta}( \mu_{t}^2 \tilv_{t+1} - 1)\\
  \hath_{t+1} &:= \tfrac{1}{\eta}\dotpsm{\vvH_1(\vphi_t)}{\hatvx_{t+1}}
\end{align*}
\begin{lemma} \label{lm:ful-gap-2gamin}
  If $\vtheta_t \in \Zei$ at step $t$,
  then $\normtwosm{\mPHnzt(\vphi_t) (\mI - \mu_t \mH_t)} \le 1 - 2 \gamin$.
\end{lemma}
\begin{proof}
If $\lamH_1(\vphi_t), \dots, \lamH_D(\vphi_t)$ are the eigenvalues of
    $\mH(\vphi_t)$, then $\left\{ 1 - 
    \sfrac{2\lamH_i(\vphi_t)}{\lamH_1(\vphi_t)} \right\}_{i=1}^{D}$ are the
    eigenvalues of $\mI - \mu_t \mH_t$.
By definition of $\gamin$, we have
$1 - \sfrac{2\lamH_i(\vphi_t)}{\lamH_1(\vphi_t)} \in [-1 + 2\gamin, 1 - 2\gamin]$
as long as $\lamH_i(\vphi_t) \ne 0, \lamH_1(\vphi_t)$.
Therefore we have $\normtwosm{\mPHnzt(\vphi_t) (\mI - \mu_t \mH_t)} \le 1 - 2 \gamin$.
\end{proof}

\begin{lemma} \label{lm:ful-align-1-step-hat}
  In the setting of \Cref{lm:ful-align-1-step},
  \begin{align}
    \hatvx_{t+1} &= (\mI - \eeta_t \mH_t) \vx_t + O(r_t^2 \eta^2) \label{eq:ful-align-1-step-hatx} \\
    \hath_{t+1} &= -h_t + O(\alpha r_t \eta) \label{eq:ful-align-1-step-hath} \\
    \normtwosm{\mPHnzt(\vphi_t) \hatvx_{t+1}} &= (1-1.9\gamin) \normtwosm{\mPHnzt(\vphi_t) \vx_{t}} + O(r_t^2 \eta^2) \label{eq:ful-align-1-step-hatpnzt} \\
    \hatu_{t+1} &= u_t + O(\alpha^2 \eta) \label{eq:ful-align-1-step-hatu}
  \end{align}
\end{lemma}
\begin{proof}[Proof for \Cref{lm:ful-align-1-step-hat}]
By \Cref{lm:gd-1-step}, we have the following first-order approximation for $\hatvx_{t+1}$:
\begin{align*}
\hatvx_{t+1} = \left(\mI - \eeta_t \mH(\vphi_t)\right) \vx_t + O(r_t^2 \eta^2),
\end{align*}
which proves \eqref{eq:ful-align-1-step-hatx}.
Then we can prove \eqref{eq:ful-align-1-step-hath} as follows:
\begin{align*}
  \hath_{t+1} = \dotpsm{\vvH_1(\vphi_t)}{\hatvx_{t+1}}
  &= (1 - \eeta_t \lamH_1(\vphi_t)) h_t + O(r_t^2 \eta^2) \\
  &= (1 - \mu_t \lamH_1(\vphi_t)) h_t + O(\alpha r_t \eta^2) \\
  &= -h_t + O(\alpha r_t \eta^2).
\end{align*}

To prove \eqref{eq:ful-align-1-step-hatpnzt} from \eqref{eq:ful-align-1-step-hatx},
we note that 
if $\eta$ is sufficiently small, 
then $\eeta_t$ is sufficiently close to $\mu_t$.
In this case, by \Cref{lm:ful-gap-2gamin} we have
$\normtwosm{(\mI - \eeta_t\mH_t) \vx_t} \le (1-1.9\gamin) \normtwosm{\mPHnzt(\vphi_t) \vx_{t}}$.
Combining this with \eqref{eq:ful-align-1-step-hatx} proves
\eqref{eq:ful-align-1-step-hatpnzt}.

Finally, we prove \eqref{eq:ful-align-1-step-hatu}
By Taylor expansion of $\Loss$ around $\vphi_t$, $\normtwosm{\nabla \Loss(\vtheta_t)} = O(\normtwosm{\vx_t}) = O(r_t \eta)$.
By \Cref{lm:lr-1-step}, we can approximate $\hatu_{t+1}$ by
\begin{align*}
    \hatu_{t+1}
    &= u_t + \Cb \eta (\mu_t^2 \normtwosm{\nabla \Loss(\vtheta_t) / \eta}^2 - 1) + O(\alpha \eta^2) \\
    &= u_t + \Cb \eta (\mu_t^2 \cdot O(r_t^2) - 1) + O(\alpha \eta^2) \\
    &= u_t + O((r_t^2 + 1) \eta) + O(\alpha \eta^2) \\
    &= u_t + O(\alpha^2 \eta),
\end{align*}
which proves \eqref{eq:ful-align-1-step-hatu}.
\end{proof}

\begin{proof}[Proof for \Cref{lm:ful-align-1-step}]
Combining \eqref{eq:ful-align-1-step-hatx} with a Taylor expansion of $\Phi$ gives
\begin{align*}
  \vphi_{t+1} = \Phi(\vphi_t + \hatvx_{t+1})
  &= \vphi_t + \partial \Phi_{\vphi_t}[\hatvx_{t+1}] + O(r_t^2\eta^2) \\
  &= \vphi_t + \partial \Phi_{\vphi_t}[(\mI - \eeta_t \mH_t)\vx_t] + O(r_t^2\eta^2).
\end{align*}
By 
\Cref{lm:Phi-proj-ful} and \Cref{lm:phiphix-ful},
$\partial \Phi_{\vphi_t}[(\mI - \mu_t \mH_t)\vx_t] = \partial \Phi_{\vphi_t}[\vx_t] = O(r_t^2 \eta^2)$.
So we have $\vphi_{t+1} = \vphi_t + O(r_t^2 \eta^2)$, 
which proves \eqref{eq:ful-align-1-step-phi}.
Then $\vx_{t+1} = \hatvx_{t+1} + \vphi_{t} - \vphi_{t+1} = \hatvx_{t+1} + O(r_t^2 \eta^2)$,
which proves \eqref{eq:ful-align-1-step-x}.
Finally, 
by $\contC^1$-smoothness of $\mu(\vphi), \vvH_1(\vphi), \mPHnzt(\vphi)$ on $\manGa$,
\eqref{eq:ful-align-1-step-hath}, \eqref{eq:ful-align-1-step-hatpnzt}, \eqref{eq:ful-align-1-step-hatu}
imply \eqref{eq:ful-align-1-step-h}, \eqref{eq:ful-align-1-step-pnzt}, \eqref{eq:ful-align-1-step-u} respectively.
\end{proof}

\begin{proof}[Proof for \Cref{thm:ful-alignment-main}]
Let $\delta := C_0 \alpha_0 \eta \sqrt{\log(1/\eta)}$,
where $C_0 = O(1)$ is a sufficiently large constant so that our proof can work.
Let $\alphamax := \alpha_0 \sqrt{\log(1/\delta)}$.

Note that the initial state is $O(1/\delta)$-misaligned.
So $\normtwosm{\vx_0} = O(h_0 \eta / \delta)$.
By \Cref{lm:ful-align-1-step},
for all $t \le T_0 := \left\lceil \frac{1}{\log(1 - 1.9 \gamin)} \log \frac{1}{\delta} \right\rceil$,
we can prove by induction that
the states are $O(\alphamax)$-bounded, and
\begin{align*}
  \normtwosm{\mPHnzt(\vphi_{t}) \vx_{t}} &= \Obig{(1-1.9\gamin)^t \normtwosm{\mPHnzt(\vphi_0) \vx_0}} \\
  \normtwosm{\vx_t} &= O((1 - 1.9\gamin)^t \normtwosm{\vx_0}), \\
  \vphi_t &= \vphi_0 + O(\normtwosm{\vx_0}^2) \\
  h_{t} &= (-1)^t h_0 + O(\alphamax \normtwosm{\vx_0})
\end{align*}
At $t = T_0$, $\normtwosm{\vx_t} = O(h_0 \eta)$,
$\vphi_t = \vphi_0 + O(\normtwosm{\vx_0}^2)$,
$\abssm{h_t - (-1)^t h_0} \le O(\alphamax h_0 \eta / \delta)$,
$\normtwosm{\mPHnzt(\vphi_{t}) \vx_{t}} = O(h_0 \eta)$.
When $C_0$ is sufficiently large, it also holds that
$\abssm{h_t - (-1)^t h_0} \le \abssm{h_0} / 4$.

After that, for all $T_0 \le t \le T_1 := \left\lceil \frac{1}{\log(1 - 1.9 \gamin)} \log \frac{1}{\delta h_0 \eta} \right\rceil = O(\alphamax)$,
we can prove by induction that the states are $O(\alphamax)$-bounded and
\begin{align*}
  \normtwosm{\mPHnzt(\vphi_{t}) \vx_{t}} &= \Obig{(1-1.9\gamin)^t \normtwosm{\mPHnzt(\vphi_0) \vx_0} + h_0^2 \eta^2} \\
  \normtwosm{\vx_t} &= O(h_0 \eta), \\
  \vphi_t &= \vphi_0 + O(\normtwosm{\vx_0}^2 + h_0^2 \eta^2 t) \\
  h_{t} &= (-1)^{t - T_0} h_{T_0} + O(\alphamax h_0 \eta t)
\end{align*}
At $t = T_1$,
$\normtwosm{\mPHnzt(\vphi_{t}) \vx_{t}} = O(h_0^2 \eta^2)$,
$\normtwosm{\vx_t} = O(h_0 \eta)$,
$\vphi_t = \vphi_0 + O(\alphamax^3 \eta^2) = \vphi_0 + O(\alphamax \eta)$,
$\abssm{h_t - (-1)^{t - T_1}} \le O(\alphamax^2 h_0 \eta)$.
When $\eta$ is sufficiently small, it also holds that
$\abssm{h_t - (-1)^{t-T_1} h_{T_1}} \le \abssm{h_0} / 4$.
So $\abssm{h_{T_1} - (-1)^t h_0} \le \abssm{h_0} / 2$.
Putting all these together proves the theorem.
\end{proof}

%% file: full-space-drifting-proof.tex
\subsection{Drifting Phase} \label{sec:ful-proof-drifting}

We define $\vpsi_{\vphi}(\vx)$ as in \Cref{lm:gd-2-step}. We abuse the notation to write $\vpsi_t(\vx) = \vpsi_{\vphi_t}(\vx)$, that is,
\[
  \vpsi_t(\vx) = \frac{\mu_t}{2} \left( \mU_t \partial^3\Loss_{\vphi_t}[\vx, \vx] + \partial^3 \Loss_{\vphi_t}[\mU_t \vx, \mU_t \vx] \right).
\]
We also define the following notations for this subsection.
\begin{align*}
  \hatvx_{t+2} &:= \vtheta_{t+2} - \vphi_t, & \hatu_{t+2} &:= \frac{1}{\eta}(\mu_t^2 \tilde{v}_{t+2} - 1), &
  \hath_{t+2} &:= \frac{1}{\eta}\dotpsm{\vx_{t+2}}{\vvH_1(\vphi_t)}.
\end{align*}

\begin{lemma} \label{lm:aligned-hteta}
  If the state $(\vtheta_t, \tilv_t)$ at step $t$
  is $\eta^{-o(1)}$-bounded and
  $O(\ahte)$-misaligned,
  then
  \[
    \vx_t = h_t \eta \vvH_1(\vphi_t) + O(h_t^2 \eta^2).
  \]
\end{lemma}
\begin{proof}
Note that we have the decomposition $\vx_t = h_t \eta \vvH_1(\vphi_t) + \mPHz(\vphi_t) \vx_t + \mPHnzt(\vphi_t) \vx_t$.
By definition of $O(\ahte)$-misaligned state,
$\mPHnzt(\vphi_t) \vx_t = O(h_t^2 \eta^2)$.
By \Cref{lm:mPHz-x-small}, $\mPHz(\vphi_t) \vx_t = O(\normtwosm{\vx_t}^2)$.
\[
\normtwosm{\vx_t} = h_t \eta + O(h_t^2 \eta^2) + O(\normtwosm{\vx_t}^2).
\]
Solving this equation gives $\normtwosm{\vx_t} = O(\ahte)$.
Then $\mPHz(\vphi_t) \vx_t = O(h_t^2 \eta^2)$,
and therefore we have $\vx_t = h_t \eta \vvH_1(\vphi_t) + O(h_t^2 \eta^2)$.
\end{proof}

\begin{lemma} \label{lm:ful-drift-2-step-hat}
In the setting of \Cref{lm:ful-drift-2-step},
\begin{align}
\hatvx_{t+2} &= (\mI - \eeta_{t+1}\mH_t)(\mI - \eeta_t \mH_t) \vx_t - \eta^2 h_t^2 \vpsi_t(\vvH_1(\vphi_t)) + O(\alpha h_t^2 \eta^3), \label{eq:ful-drift-2-step-hatx} \\
\eeta_{t} &= \mu_t \cdot (1 - \tfrac{1}{2} \eta u_t) + O(\alpha^2 \eta^2),  \label{eq:ful-drift-2-step-eeta-t0}\\
\eeta_{t+1} &= \mu_t \cdot (1 - \tfrac{1}{2} \eta u_t) + O(\alpha^2 \eta^2). \label{eq:ful-drift-2-step-eeta-t1} \\
\hatu_{t+2}
&= u_t + 8 \Cb \eta h_t^2 - 2\Cb \eta + O(\alpha(1+h_t^2) \eta^2). \label{eq:ful-drift-2-step-hatu}
\end{align}
\end{lemma}
\begin{proof}
Define
$\bar{g}_t = \normtwosm{\nabla \Loss(\vtheta_t) / \eta}$
and
$\bar{g}_{t+1} = \normtwosm{\nabla \Loss(\vtheta_{t+1}) / \eta}$.
Since the state at step $t$ is $O(\ahte)$-misaligned,
$\normtwosm{\vx_t} \le O(\ahte)$.
By Taylor expansion of $\nabla \Loss$ around $\vphi_t$,
\begin{align*}
  \nabla \Loss(\vtheta_t) = \nabla \Loss(\vphi_t) + \nabla^2 \Loss(\vphi_t) \vx_t + O(h_t^2 \eta^2) = \mH_t \vx_t + O(h_t^2 \eta^2).
\end{align*}
So
$\bar{g}_t = \normtwosm{\mH_t \vx_t / \eta} + O(h_t^2 \eta) = O(h_t)$.
Then \Cref{lm:lr-1-step} implies \eqref{eq:ful-drift-2-step-eeta-t0} and
the following approximation for $\hatu_{t+1}$:
\begin{align*}
  \hatu_{t+1} &= u_t + \Cb \eta (  \mu_t^2 \bar{g}_t^2 - 1) + O(\alpha \eta^2) \\
  &= u_t + \Cb \eta (  \mu_t^2 \normtwosm{\mH_t \vx_t / \eta}^2 - 1) + O((\alpha + \ahtc) \eta^2).
\end{align*}
As \eqref{eq:ful-drift-2-step-eeta-t0} verifies $\eeta_t = \mu_t + O(\alpha \eta)$,
we can use \Cref{lm:gd-1-step} to derive the zeroth-order and first-order approximations for $\vtheta_{t+1}$:
$\vtheta_{t+1} - \vphi_t = O(\ahte)$
and
$\vtheta_{t+1} - \vphi_t = \mU_t \vx_t + O(\alpha \aht \eta^2)$.
Then by Taylor expansion of $\nabla \Loss$ around $\vphi_t$ again,
\begin{align*}
  \nabla \Loss(\vtheta_{t+1}) = \nabla \Loss(\vphi_t) + \nabla^2 \Loss(\vphi_t) (\mU_t\vx_t + O(\alpha \aht \eta^2) ) + O(h_t^2\eta^2) = \mH_t \mU_t \vx_t + O(\alpha h_t \eta^2).
\end{align*}
So $\bar{g}_{t+1} = \normtwosm{\mH_t \mU_t \vx_t / \eta} + O(\alpha \aht \eta) = O(\aht)$.
Then \Cref{lm:lr-2-step} implies \eqref{eq:ful-drift-2-step-eeta-t1}.
We can further apply \Cref{lm:gd-2-step} to obtain the following:
\begin{align}
\hatvx_{t+2} = (\mI - \eeta_{t+1}\mH_t)(\mI - \eeta_t \mH_t) \vx_t - \vpsi_t(\vx_t) + O(\alpha h_t^2 \eta^3).
\label{eq:ful-drift-2-step-hatx-prelim}
\end{align}
Note that \Cref{lm:aligned-hteta} implies that
$\vx_t = h_t \eta \vvH_1(\vphi_t) + O(h_t^2 \eta^2)$.
Then by \Cref{lm:ful-psi-formula}
\[
\vpsi_t(\vx_t) = \vpsi_t(\eta h_t \vvH_1(\vphi_t)) + O(h_t^2 \eta^2 \cdot \ahte)
  = \eta^2 h_t^2 \vpsi_t(\vvH_1(\vphi_t)) + O(\ahtc \eta^3).
\]
Combining this with \eqref{eq:ful-drift-2-step-hatx-prelim} gives \eqref{eq:ful-drift-2-step-hatx}.

Finally, we derive the approximation for $\hatu_{t+2}$.
By \Cref{lm:lr-2-step},
\begin{align}
  \hatu_{t+2} &= u_t + \Cb \eta (  \mu_t^2 \bar{g}_t^2 + \mu_t^2 \bar{g}_{t+1}^2 - 2) + O(\alpha \eta^2).
  \label{eq:ful-drift-2-step-hatu-prelim}
\end{align}
Since $\vx_t = h_t \eta \vvH_1(\vphi_t) + O(h_t^2 \eta^2)$,
for $\barg_t$ we have
\begin{align*}
\barg_t
= \normtwosm{\mH_t (\eta h_t \vvH_1(\vphi_t) + O(h_t^2 \eta^2)) / \eta} + O(h_t^2 \eta)
&= \normtwosm{h_t \lamH_1(\vphi_t) \vvH_1(\vphi_t)} + O(h_t^2 \eta) \\
&= \lamH_1(\vphi_t) \abssm{h_t} + O(h_t^2 \eta).
\end{align*}
Similarly, for $\barg_{t+1}$ we have
\begin{align*}
\barg_{t+1}
= \normtwosm{\mH_t \mU_t (\eta h_t \vvH_1(\vphi_t) + O(h_t^2 \eta^2)) / \eta}
&= \normtwo{-h_t \lamH_1(\vphi_t) \vvH_1(\vphi_t)} + O(h_t^2 \eta) \\
&= \lamH_1(\vphi_t) \abssm{h_t} + O(h_t^2 \eta).
\end{align*}
So both $\barg_t^2$ and $\barg_{t+1}^2$ can be approximated by $\lamH_1(\vphi_t)^2 h_t^2 + O(\ahtc \eta)$.
Combining this with \eqref{eq:ful-drift-2-step-hatu-prelim} gives
\begin{align*}
\hatu_{t+2}
&= u_t + \Cb \eta \mu_t^2 (\lamH_1(\vphi_t)^2 h_t^2 + O(\ahtc \eta)) + \Cb \eta \mu_t^2 (\lamH_1(\vphi_t)^2 h_t^2 + O(\ahtc \eta)) \\
& \quad - 2 \Cb \eta + \Obig{\alpha(1 + h_t^2)\eta^2} \\
&= u_t + 8 \Cb \eta h_t^2 - 2\Cb \eta + O(\alpha(1+h_t^2) \eta^2),
\end{align*}
which implies \eqref{eq:ful-drift-2-step-hatu}.
\end{proof}

\begin{lemma} \label{lm:ful-drift-2-step-hat2}
In the setting of \Cref{lm:ful-drift-2-step},
\begin{align}
\hath_{t+2} &= (1 - 2 \eta u_t) h_t  + \Osm{\alpha^2 \aht \eta^2} \label{eq:ful-drift-2-step-hath} \\
\mPHz(\vphi_t) \hatvx_{t+2}
&= \mPHz(\vphi_t) \vx_t - 2 \eta^2 h_t^2 \gradGa \log \lamH_1(\vphi_t) + \Osm{\alpha h_t^2 \eta^3} \label{eq:ful-drift-2-step-hatx-tan} \\
\normtwosm{\mPHnzt(\vphi_t) \hatvx_{t+2}}
&\le (1 - 1.9\gamma)^2 \normtwosm{\mPHnzt(\vphi_t) \vx_t}
+ O(h_t^2 \eta^2) \label{eq:ful-drift-2-step-hatx-w}
\end{align}
\end{lemma}
\begin{proof}
In the following, we derive the approximations from \eqref{eq:ful-drift-2-step-hatx},
\begin{align*}
\hatvx_{t+2} &= (\mI - \eeta_{t+1}\mH_t)(\mI - \eeta_t \mH_t) \vx_t - \eta^2 h_t^2 \vpsi_t(\vvH_1(\vphi_t)) + O(\alpha h_t^2 \eta^3).
\end{align*}

\myparagraph{Approximation for $\hath_{t+2}$.}
Note that $\vx_t = \eta h_t \vvH_1(\vphi_t) + O(h_t^2\eta^2)$
since the state at step $t$ is $O(\ahte)$-misaligned.
For $\hath_{t+2}$, we have
\begin{align*}
\hath_{t+2} &= \tfrac{1}{\eta}\left( \dotpsm{\vx_t}{(\mI - \eeta_{t+1} \mH_t)(\mI - \eeta_t \mH_t) \vvH_1(\vphi_t)}
- \eta^2 h_t^2 \dotpsm{\vpsi_t(\vvH_1(\vphi_t))}{\vvH_1(\vphi_t)} + O(\alpha h_t^2\eta^3) \right) \\
&= (1 - \eeta_{t+1} \lamH_1(\vphi_t)) (1 - \eeta_t \lamH_1(\vphi_t)) h_t + O(\alpha h_t^2 \eta^2),
\end{align*}
where we use 
the fact that $\dotpsm{\vpsi_t(\vvH_1(\vphi_t))}{\vvH_1(\vphi_t)} = 0$ by \Cref{lm:ful-psi-formula}.

By \eqref{eq:ful-drift-2-step-eeta-t0}, \eqref{eq:ful-drift-2-step-eeta-t1},
$\eeta_{\tau} = \mu_t \cdot (1 - \tfrac{1}{2} \eta u_t) + O(\alpha^2 \eta^2)$ for $\tau \in \{t, t+1\}$.
Note that $\mu_t \cdot \lamH_1(\vphi_t) = 2$.
Then for $\tau \in \{t, t+1\}$,
\begin{align*}
1 - \eeta_{\tau} \lamH_1(\vphi_t)
&= 1 - \mu_t \cdot (1 - \tfrac{1}{2} \eta u_t) \cdot \lamH_1(\vphi_t) + O(\alpha^2 \eta^2) \\
&= 1 - 2 \cdot (1 - \tfrac{1}{2} \eta u_t)) + O(\alpha^2 \eta^2) \\
&= -1 + \eta u_t + O(\alpha^2 \eta^2).
\end{align*}
Then $(1 - \eeta_{t+1} \lamH_1(\vphi_t)) (1 - \eeta_t \lamH_1(\vphi_t))$ can be approximated by
\[
  (1 - \eeta_{t+1} \lamH_1(\vphi_t)) (1 - \eeta_t \lamH_1(\vphi_t)) = (-1 + \eta u_t + O(\alpha^2 \eta^2))^2 = 1 - 2 \eta u_t + O(\alpha^2 \eta^2).
\]
Therefore, we have
$\hath_{t+2} = \left(1 - 2 \eta u_t + O(\alpha^2 \eta^2)\right) h_t + O(\alpha h_t^2 \eta^2)$,
which implies \eqref{eq:ful-drift-2-step-hath}.

\myparagraph{Approximation for $\mPHz(\vphi_t) \hatvx_{t+2}$.} 
For $\mPHz(\vphi_t) \hatvx_{t+2}$, we have
\begin{align*}
\mPHz(\vphi_t) \hatvx_{t+2}
&= \mPHz(\vphi_t) (\mI - \eeta_{t+1} \mH_t)(\mI - \eeta_t \mH_t) \vx_t
- \eta^2 h_t^2 \mPHz(\vphi_t) \vpsi_t(\vvH_1(\vphi_t)) + O(\alpha h_t^2 \eta^3) \\
&= \mPHz(\vphi_t) (\mI - \eeta_{t+1} \mH_t)(\mI - \eeta_t \mH_t) \vx_t
- 2\eta^2 h_t^2 \gradGa \log \lamH_1(\vphi_t) + O(\alpha h_t^2 \eta^3),
\end{align*}
where we use 
the fact that
$\mPHz(\vphi_t) \vpsi_t(\vvH_1(\vphi_t)) = 2 \gradGa \log \lamH_1(\vphi_t)$
by \Cref{lm:ful-psi-formula}.
To obtain \eqref{eq:ful-drift-2-step-hatx-tan}, we only need to note that 
$\mPHz(\vphi_t) (\mI - \eeta_{t+1} \mH_t)(\mI - \eeta_t \mH_t) \vx_t = \mPHz(\vphi_t)\vx_t$.

\myparagraph{Approximation for $\normtwosm{\mPHnzt(\vphi_t) \hatvx_{t+2}}$.}
To approximate $\normtwosm{\mPHnzt(\vphi_t) \hatvx_{t+2}}$,
we note that
if $\eta$ is sufficiently small,
then $\eeta_t, \eeta_{t+1}$ are sufficiently close to $\mu_t$.
In this case, by
\Cref{lm:ful-gap-2gamin}
we have
$\normtwosm{\mPHnzt(\vphi_t)(\mI - \eeta_{t+1}\mH_t)(\mI - \eeta_t \mH_t) \vx_t} \le (1 - 1.9 \gamin)^2 \normtwosm{\mPHnzt(\vphi_t)\vx_t}$.
Combining this with \eqref{eq:ful-drift-2-step-hatx} proves \eqref{eq:ful-drift-2-step-hatx-w}.
\end{proof}

\begin{lemma} \label{lm:ful-drift-2-step-phi}
  In the setting of \Cref{lm:ful-drift-2-step},
  the approximation \eqref{eq:ful-drift-2-step-phi} holds for $\vphi_{t+2}$ and
  \begin{align}
    \log \lamH_1(\vphi_{t+2}) &= \log \lamH_1(\vphi_t)
    - 2 \eta^2 h_t^2 \normtwosm{\gradGa \log \lamH_1(\vphi_t)}^2 + \Osm{\alpha h_t^2 \eta^3}.\label{eq:ful-drift-2-step-loglam}
  \end{align}
\end{lemma}
\begin{proof}
By Taylor expansion,
\begin{align}
  \vphi_t &= \Phi(\vphi_t) + \partial \Phi_{\vphi_t}[\vx_t] + \frac{1}{2} \partial^2 \Phi_{\vphi_t}[\vx_t, \vx_t] + O(\ahtc \eta^3). \label{eq:ful-drift-2-step-phi-taylor1} \\
  \vphi_{t+2} &= \Phi(\vphi_t) + \partial \Phi_{\vphi_t}[\hatvx_{t+2}] + \frac{1}{2} \partial^2 \Phi_{\vphi_t}[\hatvx_{t+2}, \hatvx_{t+2}] + O(\ahtc \eta^3). \label{eq:ful-drift-2-step-phi-taylor2}
\end{align}
By definition, $\Phi(\vphi_t) = \vphi_t$. By \eqref{eq:ful-drift-2-step-hatx-tan}
and \Cref{lm:Phi-proj-ful},
\begin{align*}
\partial\Phi_{\vphi_t}[\hatvx_{t+2}]
&= \partial \Phi_{\vphi_t}[\vx_t] - 2 \eta^2 h_t^2 \gradGa \log \lamH_1(\vphi_t) + \Osm{\alpha h_t^2 \eta^3}.
\end{align*}
Note that $\vx_t = \eta h_t \vvH_1(\vphi_t) + O(h_t^2\eta^2)$
since the state at step $t$ is $O(\ahte)$-misaligned.
Also note that the identity
$\partial^2 \Phi_{\vphi_t}[\vvH_1(\vphi_t), \vvH_1(\vphi_t)] = \vzero$
holds
by \Cref{lm:Phi-2nd-normal-ful}. Then
\begin{align*}
\partial^2 \Phi_{\vphi_t}[\vx_t, \vx_t]
&= \partial^2 \Phi_{\vphi_t}[\eta h_t \vvH_1(\vphi_t), \eta h_t \vvH_1(\vphi_t)] + O(\ahtc \eta^3) \\
&= O(\ahtc \eta^3).
\end{align*}
Similarly, we have 
$\partial^2 \Phi_{\vphi_t}[\hatvx_{t+2}, \hatvx_{t+2}] = O(\ahtc \eta^3)$
since
\eqref{eq:ful-drift-2-step-hatx-tan} implies $\normtwosm{\mPHz(\vphi_t) \hatvx_{t+2}} = O(h_t^2 \eta^2)$
and
\eqref{eq:ful-drift-2-step-hatx-w} implies $\normtwosm{\mPHnzt(\vphi_t) \hatvx_{t+2}} = O(h_t^2 \eta^2)$.

Now we can prove \eqref{eq:ful-drift-2-step-phi} by
subtracting \eqref{eq:ful-drift-2-step-phi-taylor2} with \eqref{eq:ful-drift-2-step-phi-taylor1}:
\begin{align*}
  \vphi_{t+2} - \vphi_t
  &= \left(\partial \Phi_{\vphi_t}[\hatvx_{t+2}] - \partial \Phi_{\vphi_t}[\vx_t]\right)
    + \tfrac{1}{2} \left( \partial^2 \Phi_{\vphi_t}[\hatvx_{t+2}, \hatvx_{t+2}] - \partial^2 \Phi_{\vphi_t}[\vx_t, \vx_t] \right) + O(\ahtc \eta^3) \\
  &= \left(
  - 2 \eta^2 h_t^2 \gradGa \log \lamH_1(\vphi_t) + \Osm{\alpha h_t^2 \eta^3}
  \right) + O(\ahtc \eta^3) + O(\ahtc \eta^3) \\
  &= - 2 \eta^2 h_t^2 \gradGa \log \lamH_1(\vphi_t) + \Osm{\alpha h_t^2 \eta^3}.
\end{align*}
Finally, for $\log \lamH_1(\vphi_{t+2})$ we have
\begin{align*}
  \log \lamH_1(\vphi_{t+2}) - \log \lamH_1(\vphi_t) 
  &= \dotpsm{\nabla \log \lamH_1(\vphi_t)}{\vphi_{t+2} - \vphi_t} + O((h_t^2 \eta^2)^2) \\
  &= \dotpsm{\nabla \log \lamH_1(\vphi_t)}{-2 \eta^2 h_t^2 \gradGa \log \lamH_1(\vphi_t)} + \Osm{\alpha h_t^2 \eta^3 + h_t^4 \eta^4} \\
  &= -2 \eta^2 h_t^2 \normtwosm{\gradGa \log \lamH_1(\vphi_t)}^2 + \Osm{\alpha h_t^2 \eta^3},
\end{align*}
which proves \eqref{eq:ful-drift-2-step-loglam}.
\end{proof}

\begin{proof}[Proof for \Cref{lm:ful-drift-2-step}]
We have verified \eqref{eq:ful-drift-2-step-phi} in \Cref{lm:ful-drift-2-step-phi}.
By \eqref{eq:ful-drift-2-step-phi} and definitions of $\hatvx_{t+2}$ and $\vx_{t+2}$, we have
\begin{align*}
  \vx_{t+2} - \hatvx_{t+2} = \vphi_t - \vphi_{t+2} = 2 \eta^2 h_t^2 \gradGa \log \lamH_1(\vphi_t) + \Osm{\alpha h_t^2 \eta^3}.
\end{align*}
And we can write $\vx_{t+2} - \hatvx_{t+2} = \vphi_t - \vphi_{t+2} = \Osm{h_t^2 \eta^2}$ as a loose approximation.

\myparagraph{Approximation for $h_{t+2}$.}
For $h_{t+2}$, we have
\begin{align*}
  h_{t+2} - \hath_{t+2}
  &= \dotpsm{\vx_{t+2}}{\vvH_1(\vphi_{t+2})} - \dotpsm{\hatvx_{t+2}}{\vvH_1(\vphi_t)} \\
  &= \dotpsm{\vx_{t+2}}{\vvH_1(\vphi_{t+2}) - \vvH_1(\vphi_t)} + \dotpsm{\vx_{t+2} - \hatvx_{t+2}}{\vvH_1(\vphi_t)}\\
  &= O(\ahte) \cdot O(\normtwosm{\vphi_{t+2} - \vphi_t}) + 2\eta^2 h_t^2 \dotpsm{\gradGa \log \lamH_1(\vphi_t)}{\vvH_1(\vphi_t)} + O(\alpha h_t^2 \eta^3) \\
  &= O(\ahtc \eta^3) + 0 + \Osm{\alpha h_t^2 \eta^3} \\
  &= \Osm{\alpha h_t^2 \eta^3},
\end{align*}
where the fourth equality is due to
$\dotpsm{\gradGa \log \lamH_1(\vphi_t)}{\vvH_1(\vphi_t)} = 0$ and
$\normtwosm{\vphi_{t+2} - \vphi_t} = O(h_t^2 \eta^2)$.
Combining this with \eqref{eq:ful-drift-2-step-hath} proves the claimed approximation \eqref{eq:ful-drift-2-step-h}.

\myparagraph{Approximation for $\normtwosm{\mPHnzt(\vphi_{t+2}) \vx_{t+2}}$.}
For $\mPHnzt(\vphi_{t+2}) \vx_{t+2}$, we have
\begin{align*}
  \mPHnzt(\vphi_{t+2}) \vx_{t+2} -
  \mPHnzt(\vphi_t) \hatvx_{t+2}
  &=
  (\mPHnzt(\vphi_{t+2}) - \mPHnzt(\vphi_t)) \vx_{t+2}
  + \mPHnzt(\vphi_t) (\vx_{t+2} -\hatvx_{t+2}) \\
  &=
  \Osm{\normtwosm{\vphi_{t+2} - \vphi_t}} \cdot O(\ahte)
  + O(\normtwosm{\vx_{t+2} -\hatvx_{t+2}}) \\
  &=
  O(h_t^2 \eta^2) \cdot O(\ahte)
  + O(h_t^2 \eta^2) \\
  &= O(h_t^2 \eta^2),
\end{align*}
where the third equality is due to 
$\vx_{t+2} - \hatvx_{t+2} = \vphi_t - \vphi_{t+2} = O(h_t^2 \eta^2)$.
Combining this with \eqref{eq:ful-drift-2-step-hatx-w} proves the claimed approximation \eqref{eq:ful-drift-2-step-x-w}.

\myparagraph{Approximation for $u_{t+2}$.}
Now we prove the formula for $u_{t+2}$.
Note that $\vphi_{t+2} - \vphi_t = O(h_t^2 \eta^2)$ implies that $\mu_{t+2} - \mu_t = O(h_t^2 \eta^2)$.
Then we have
\begin{align*}
  \eta\left(u_{t+2} - \hatu_{t+2}\right) &= (\mu_{t+2}^2 - \mu_t^2) \tilv_{t+2} \\
  &= (\mu_{t+2} - \mu_t) (2 \mu_t + O(h_t^2 \eta^2)) \cdot (\tfrac{1}{\mu_t^2} + O(\alpha \eta)) \\
  &= (\mu_{t+2} - \mu_t) \left( 2 \mu_t \cdot \tfrac{1}{\mu_t^2} + O(\alpha \eta) \right) \\
  &= \tfrac{2}{\mu_t}(\mu_{t+2} - \mu_t) + O(\alpha h_t^2 \eta^3).
\end{align*}
Note that $\log \mu_{t+2} - \log \mu_t = \log\left(1 + \frac{1}{\mu_t}(\mu_{t+2} - \mu_t)\right) = \tfrac{1}{\mu_t}(\mu_{t+2} - \mu_t) + O((h_t^2 \eta^2)^2)$.
By \eqref{eq:ful-drift-2-step-loglam},
\begin{align*}
  \log \mu_{t+2} - \log \mu_t = \log \lamH_1(\vphi_t) - \log \lamH_1(\vphi_{t+2})
  = 2 \eta^2 h_t^2 \normtwosm{\gradGa \log \lamH_1(\vphi_t)}^2 + O(\alpha h_t^2 \eta^3).
\end{align*}
Combining these together gives the following approximation for $\frac{1}{\mu_t}(\mu_{t+2} - \mu_t)$:
\begin{align*}
  \tfrac{1}{\mu_t}(\mu_{t+2} - \mu_t) &= \log \mu_{t+2} - \log \mu_t + O(h_t^4 \eta^4) \\
  &= 2 \eta^2 h_t^2 \normtwosm{\gradGa \log \lamH_1(\vphi_t)}^2 + \Osm{\alpha h_t^2 \eta^3}.
\end{align*}
So $\eta \left(u_{t+2} - \hatu_{t+2}\right) = 4 \eta^2 h_t^2 \normtwosm{\gradGa \log \lamH_1(\vphi_t)}^2 + O(\alpha h_t^2 \eta^3)$.
Then by \eqref{eq:ful-drift-2-step-hatu}, we have
\begin{align*}
  u_{t+2} &= \hatu_{t+2} + 4 \eta h_t^2 \normtwosm{\gradGa \log \lamH_1(\vphi_t)}^2 + \Osm{\alpha h_t^2 \eta^2} \\
  &= u_t + 8 \eta \Cb h_t^2 + 4 \eta h_t^2 \normtwosm{\gradGa \log \lamH_1(\vphi_t)}^2 - 2 \eta \Cb + \Obig{\alpha(1+h_t^2)\eta^2} \\
  &= u_t + 4 \eta h_t^2 (2 \Cb + \normtwosm{\gradGa \log \lamH_1(\vphi_t)}^2) - 2 \eta \Cb + \Osm{\alpha(1+h_t^2)\eta^2},
\end{align*}
which proves the claimed approximation \eqref{eq:ful-drift-2-step-u}.
\end{proof}

%% file: spherical-proof.tex
\section{Reduction to RMS-drift Process: The Case of Spherical Optimization} \label{sec:sph-proof}

In this section, we let $\{(\vtheta_t, v_t)\}_{t \ge 0}$ be a trajectory of projected gradient descent
with quasi-RMSprop scheduler,
and let $\eta, \beta$ be the base learning rate and decay rate.
We follow the notations and terminologies in \Cref{sec:proof-outline-main}.

As the analysis in the spherical case is nearly the same
as the full space case,
we only discuss the difference here.

\subsection{Good Initialization} \label{sec:sph-good-init}
\begin{proof}[Proof for \Cref{lm:sph-good-init}]
Using a similar argument as in \Cref{lm:ful-good-init},
we know that the lemma holds if there is no projection
in the random initialization.
But the projection only leads to an error of order $O(r^2)$,
so the lemma holds.
\end{proof}

\subsection{Alignment Phase} \label{sec:sph-proof-alignment}

\begin{proof}[Proofs for \Cref{lm:sph-align-1-step,thm:sph-alignment-main}]
The proof is essentially the same as 
\Cref{lm:ful-align-1-step,thm:ful-alignment-main}.
To see this,
we only need to note that
we only have used a linear approximation
of the update rule
with error $O(r_t^2)$,
and the linear approximation remains unchanged if we add a projection operator
(\Cref{lm:pgd-1-step}).
\end{proof}

\subsection{Drifting Phase}
We define $\vpsi_{\vphi}(\vx)$ as in \Cref{lm:pgd-2-step}. We abuse the notation to write $\vpsi_t(\vx) = \vpsi_{\vphi_t}(\vx)$, that is,
\begin{equation} \label{eq:sph-psi}
    \vpsi_t(\vx) :=
    \frac{\mu_t}{2}\left( \mU_t \partial^3 \Loss_{\vphi_t}[\vx, \vx] + \partial^3 \Loss_{\vphi}[\mU_t \vx, \mU_t \vx] \right) 
    + \frac{\mu_t^2}{2} \left( \normtwosm{\mH_t \vx}^2 + \normtwosm{\mH_t \mU_t \vx}^2 \right) \vphi_t.
\end{equation}

\begin{lemma}
In the setting of \Cref{lm:sph-drift-2-step},
the same statement as \Cref{lm:ful-drift-2-step-hat,lm:ful-drift-2-step-hat2} holds,
where $\vpsi_t$ is interpreted as \eqref{eq:sph-psi}.
\end{lemma}
\begin{proof}
  We can follow the argument in the proof for \Cref{lm:ful-drift-2-step-hat,lm:ful-drift-2-step-hat2},
  but now we are using \Cref{lm:pgd-2-step,lm:sph-psi-formula} to establish the proof.
\end{proof}

\begin{lemma}
In the setting of \Cref{lm:sph-drift-2-step},
  the same statement as \Cref{lm:ful-drift-2-step-phi} holds.
\end{lemma}
\begin{proof}
  The argument is the same as \Cref{lm:ful-drift-2-step-phi},
  but we apply \Cref{lm:Phi-proj-sph,lm:Phi-2nd-normal-sph} in doing Taylor expansion
  for $\Phi(\vtheta)$.
\end{proof}

\begin{proof}[Proof for \Cref{lm:sph-drift-2-step}]
Same as \Cref{lm:ful-drift-2-step}
but we invoke the spherical version of lemmas.
\end{proof}

%% file: rmsdrift-proof.tex
\section{Analysis of RMS-drift Process} \label{sec:rmsdrift-proof}

In this section, we provide proofs for theorems in \Cref{sec:outline-rmsdrift-analysis}.
For convenience, we define 
$R(\vtheta) = \nabla \log \lamH_1(\vtheta)$,
$K_t := \sqrt{2\Cb + \normtwosm{\gradGa R(\vphi_t)}^2}$.
Then a $C_0$-RMS-drift transition
$S_t \to S_{t+2}$
can be written as
\begin{align*}
  h'_{t+2} &:= (1 - 2\eta u_t) h_t,  &
  \abssm{h_{t+2} - h'_{t+2}} &\le C_0 \alpha^2 \abssm{h_t} \eta^2, \\
  u'_{t+2} &:= u_t + 4 \eta K_t^2 h_t^2 - 2\eta \Cb, &
  \abssm{u_{t+2} - u'_{t+2}} &\le C_0 \alpha (1 + h_t^2) \eta^2,  \\
  \vphi'_{t+2} &:= \vphi_t - 2 \eta^2 h_t^2 \gradGa R(\vphi_t), &
  \normtwosm{\vphi_{t+2} - \vphi'_{t+2}} &\le C_0 \alpha h_t^2 \eta^3.
\end{align*}

\subsection{Conservation of Energy} \label{sec:rmsdrift-proof-energy}

To establish the conservation of energy,
we first compute the change in energy after one transition.
\begin{lemma} \label{lm:e-cons}
  Given two drift states $S_0 = (h_0, u_0, \vphi_0)$ and $S_2 = (h_2,
  u_2, \vphi_2)$ in the working zone,
  for learning rate $\eta$ and hyperparameter $\Cb > 0$,
  if $S_0$ is $\alpha$-bounded for some $1 \le \alpha \le \eta^{-o(1)}$,
  and
  $S_0 \to S_2$ is a $C_0$-RMSdrift transition,
  then
  \[
  	 E(S_2) - E(S_0) = \Obig{\alpha^2(1+h_0^2) \eta^2} = \begin{cases}
       O(\alpha^2 \eta^2) & \quad \abssm{h_0} \le 2, \\
       O(\alpha^2 h_0^2 \eta^2) & \quad \abssm{h_0} > 2. \\
     \end{cases}
  \]
\end{lemma}
\begin{proof}[Proof for \Cref{lm:e-cons}]
  $h_0$ and $h_2$ have the same sign when $\eta$ is small enough.
  We can decompose $E(S_2) - E(S_0)$ as follows:
	\begin{align*}
		E(S_2) - E(S_0) &= \underbrace{\frac{1}{2} (u_2^2 - u_0^2)}_{=: \delta_1}
    + \underbrace{K_0^2(h_2^2 - h_0^2)}_{=: \delta_2} 
	+ \underbrace{(\normtwosm{\gradGa R(\vphi_2)}^2 - \normtwosm{\gradGa R(\vphi_0)}^2) h_2^2}_{=: \delta_3} + \underbrace{\Cb \log \frac{h_2}{h_0}}_{=: \delta_4}.
	\end{align*}
	Now we bound each error term. For $\delta_1$ and $\delta_2$, we use the formula $a^2 - b^2 = 2b(a-b) + (a-b)^2$:
	\begin{align*}
		\delta_1 &= u_0 (u_2 - u_0) + \tfrac{1}{2}(u_2 - u_0)^2 \\
		&= u_0 \left(4 \eta h_0^2K_0^2 - 2 \eta \Cb + O(\alpha(1+h_0^2) \eta^2)\right) + \Obig{(h_0^2 \eta)^2}. \\
		&= 4 \eta u_0 h_0^2K_0^2 - 2 \eta u_0 \Cb + O(\alpha^2(1+h_0^2) \eta^2). \\
		\delta_2 &= K_0^2\left(2h_0 (h_2 - h_0) + (h_2 - h_0)^2\right) \\
		&= K_0^2 \left(\left(-4\eta u_0 h_0^2 + O(\alpha^2 h_0^2 \eta^2) \right) + \Obig{(\alpha h_0 \eta)^2}\right) \\
		&= -4\eta u_0 h_0^2 K_0^2 + \Osm{\alpha^2 h_0^2 \eta^2}.
	\end{align*}
	For $\delta_3$, we use the Lipschitzness of $\gradGa R(\vtheta)$:
	\begin{align*}
		\delta_3 &= O(\normtwosm{\vphi_2 - \vphi_0}) \cdot h_2^2 = O(h_0^2\eta^2) \cdot O(h_0^2) = O(h_0^4 \eta^2).
	\end{align*}
	For $\delta_4$, note that $\log(1 + z) \le z + O(z^2)$ when $z = o(1)$. Then
	\begin{align*}
		\delta_4 = -\Cb \log\left(1 + \frac{h_2 - h_0}{h_0}\right) &= -\Cb \cdot \frac{h_2 - h_0}{h_0} + O(\alpha^2 \eta^2) \\
		&= 2\eta u_0 \Cb + O(\alpha^2 \eta^2).
	\end{align*}
	Adding $\delta_1, \delta_2, \delta_3, \delta_4$ together gives
	\begin{align*}
		E(S_2) - E(S_0) = & +  4 \eta u_0 h_0^2K_0^2 & &- 2 \eta u_0 \Cb & &+ O(\alpha^2(1+h_0^2) \eta^2) \\
		& - 4\eta u_0 h_0^2 K_0^2 & & & & + O(\alpha^2 h_0^2 \eta^2) \\
		& & & & & +O(h_0^4 \eta^2) \\
		& & & +2\eta \Cb u_0 & & + O(\alpha^2 \eta^2).
	\end{align*}
  So $E(S_2) - E(S_0) = O(\alpha^2(1+h_0^2) \eta^2)$.
\end{proof}

To sum up the energy change over time, we need
the following lemma.
\begin{lemma} \label{lm:sum-h2}
  For any $M = o((\alpha \eta)^{-2})$,
  if 
  $S_0, \dots, S_{2M}$ is an $O(1)$-RMS-drift process in the working zone,
  and $S_t$ is $\alpha$-bounded for all even numbers $t \le 2M$, then
  \[
    \sum_{m = 0}^{M - 1} h_{2m}^2 = \frac{\Cb}{2 K_0^2} M + O(\alpha^2 \eta^2 M^2 + \alpha / \eta).
  \]
\end{lemma}
\begin{proof}[Proof for \Cref{lm:sum-h2}]
  By the update rule of $u_t$, we have
  \begin{align*}
    u_{2M} - u_0 = \sum_{m=0}^{M-1} u_{2m+2} - u_{2m} &= \sum_{m=0}^{M-1} \left(4\eta K_{2m}^2 h_{2m}^2 - 2 \eta \Cb + O(\alpha^3 \eta^2)\right).
  \end{align*}
  Since $S_{2m}$ is $\alpha$-bounded for all $m \le M$, $\vphi_{2m} - \vphi_0 = O(\alpha^2 \eta^2 M)$.
  By smoothness of $R$, $\normtwosm{\gradGa R(\vphi_{2m})}^2 = \normtwosm{\gradGa R(\vphi_0)}^2 + O(\alpha^2 \eta^2 M)$,
  then $K_{2m} = K_0 + O(\alpha^2 \eta^2 M)$.
  Since $S_0$ and $S_{2M}$ are $\alpha$-bounded, $u_T - u_0 = O(\alpha)$.
  Combining all these together,
  \begin{align*}
    O(\alpha) &= \sum_{m=0}^{M-1} \left(4\eta \left(K_0^2 + O(\alpha^2 \eta^2 M)\right) h_{2m}^2 - 2 \eta \Cb +O(\alpha^3 \eta^2) \right)
  \end{align*}
  Let $Q := \sum_{m=0}^{M-1} h_{2m}^2$.
  Then we have
  \begin{align*}
    O(\alpha) = 4\eta \left(K_0^2 + O(\alpha^2 \eta^2 M)\right) Q - 2\eta \Cb M + O(\alpha^3 \eta^2 M).
  \end{align*}
  Rearranging the terms while noting that $\alpha^3 \eta^2 M = o(\alpha)$, we have
  \begin{align*}
    4\eta \left(K_0^2 + O(\alpha^2 \eta^2 M)\right) Q = 2\eta \Cb M + O(\alpha).
  \end{align*}
  So we can estimate $Q$ by
  \begin{align*}
    Q &= \frac{2\Cb M + O(\alpha / \eta)}{4K_0^2 + O(\alpha^2 \eta^2 M)} = \frac{\Cb}{2 K_0^2} M + O(\alpha^2 \eta^2 M^2 + \alpha / \eta),
  \end{align*}
  which completes the proof.
\end{proof}

\begin{lemma} \label{lm:rmsdrift-ec-15}
  For an $O(1)$-RMSdrift process $S_0, \dots, S_{2M}$ in the working zone,
  if $E(S_0) \le \alpha^2$ for some parameter $1 \le \alpha \le \eta^{-o(1)}$
  and $M = O(1/\eta^{1.5})$,
  then $E(S_t) = E(S_0) + O(\alpha^2 \eta^{0.5})$ for all even numbers $0 \le t \le 2M$.
\end{lemma}
\begin{proof}[Proof for \Cref{lm:rmsdrift-ec-15}]
    We do a bootstrap.
    As $E(S_0) \le O(\alpha^2)$, we can leverage \Cref{lm:e-cons} to prove by induction that $E(S_t) \le E(S_0) + O(\alpha^4 \eta^2 M)$
    for all $0 \le t \le 2M$.
    Then we apply \Cref{lm:e-cons} again. For all $N \le M$,
    \begin{align*}
        E(S_{2N}) - E(S_0) &= \sum_{m = 0}^{N-1} O(\alpha^2 (1 + h_{2m}^2) \eta^2) \\
        &= O(\alpha^2\eta^2) \left( N + \sum_{m = 0}^{N-1} h_{2m}^2\right) \\
        &\le O(\alpha^2\eta^2) \left( N + O(N + \alpha^2 \eta^2 N^2 + \alpha / \eta)\right) \\
        &\le O(\alpha^2\eta^{0.5}).
    \end{align*}
    where the third line uses \Cref{lm:sum-h2}.
\end{proof}

\begin{proof}[Proof for \Cref{thm:rmsdrift-ec}]
  We group the $M$ transitions into $O(1/\eta^{0.5})$
  segments of length $O(1/\eta^{1.5})$.
  We can do an induction with
  \Cref{lm:rmsdrift-ec-15} applied
  on each segment to show that for all $t$ in the $k$-th segment,
  $E(S_t) \le (1 + O(\eta^{0.5}))^k E(S_0)$.
  Noting that $k = O(1/\eta^{0.5})$ and $(1 + O(\eta^{0.5}))^{O(1/\eta^{0.5})} = O(1)$ finishes the proof.
\end{proof}

\subsection{Flow Approximation} \label{sec:rmsdrift-flow}

\begin{lemma} \label{lm:rmsdrift-flow-15}
  In the setting of \Cref{thm:rmsdrift-flow} but with $M = \Theta(1/\eta^{1.5})$,
  if $S_t$ is $\alpha$-bounded for all even numbers $0 \le t \le 2M$, then
  \[
    \normtwosm{\vphi_t - \vzeta(t\eta^2)} \le O(\alpha^2 \eta).
  \]
\end{lemma}
\begin{proof}[Proof for \Cref{lm:rmsdrift-flow-15}]
  When $S_t$ is $\alpha$-bounded, $\normtwosm{\vphi_{t+2} - \vphi_t} =
  O(\alpha^2 \eta^2)$, so $\normtwosm{\vphi_t - \vphi_0} = O(\alpha^2
  \eta^2 M) = O(\alpha^2 \eta^{0.5})$.
  For every $N \le M$
  we have
  \begin{align*}
    \vphi_{2N} - \vphi_0 = \sum_{m=0}^{N-1} \vphi_{2m+2} - \vphi_{2m}
    &= \sum_{m=0}^{N-1} \left(-2\eta^2h_{2m}^2 \gradGa R(\vphi_{2m}) + O(\alpha^3 \eta^3) \right) \\
    &= \sum_{m=0}^{N-1} \left(-2\eta^2h_{2m}^2 (\gradGa R(\vphi_0) + O(\alpha^2 \eta^{0.5})) + O(\alpha^3 \eta^3) \right) \\
    &= \sum_{m=0}^{N-1} \left(-2\eta^2h_{2m}^2 \gradGa R(\vphi_0) + O(\alpha^2 \eta^{2.5}) \right) \\
    &= -\underbrace{2 \eta^2 \left(\sum_{m = 0}^{N-1} h_{2m}^2\right)}_{=: \delta} \gradGa R(\vphi_0) + O(\alpha^2 \eta).
  \end{align*}
  By \Cref{lm:sum-h2}, we have
  \begin{align*}
    \delta = 2\eta^2 \cdot \left(\frac{\Cb}{2K_0^2} N + O(\alpha^2 \eta^2 N^2 + \alpha / \eta)\right)
    &= \frac{\Cb}{K_0^2} \eta^2 N + 2\eta^2 \cdot O(\alpha^2 / \eta) \\
    &= \frac{\Cb}{K_0^2} \eta^2 N + O(\alpha^2 \eta).
  \end{align*}
  Note that $\frac{\Cb^2}{2 K_0^2} \gradGa R(\vphi_0) = \frac{\dd}{\dd t} \vzeta(0) + O(\alpha^2 \eta^{1/2})$. Then we have
  \begin{align*}
    \vphi_{2N} &= \vphi_0 -\delta \gradGa R(\vphi_0) + O(\alpha^2 \eta) \\
    &= \vphi_0 - (2\eta^2 N) \left(\frac{\dd}{\dd t}\vzeta(\vphi_0) + O(\alpha^2 \eta^{1/2})\right) + O(\alpha^2 \eta) \\
    &= \vzeta(2N\eta^2) + \Obig{(\eta^2 N)^2} + O(\alpha^2 \eta) \\
    &= \vzeta(2N\eta^2) + O(\alpha^2 \eta),
  \end{align*}
  where the third equality uses the smoothness of $\vzeta(t)$.
\end{proof}

\begin{proof}[Proof for \Cref{thm:rmsdrift-flow}]
  We group the $M$ transitions into $O(1/\eta^{1/2})$
  segments of length $O(1/\eta^{1.5})$.
  Then we leverage
  \Cref{lm:rmsdrift-flow-15}
  and do an induction to show that $\normtwosm{\vphi_t - \vzeta(t\eta^2)} \le O(\alpha^2 \eta^{1/2})$ for all even numbers $0 \le t \le 2M$.
\end{proof}

%% file: app-applications.tex
\section{Proofs for Linear Regression with Batch Normalization} \label{sec:app-lin}

\begin{lemma} \label{lm:lin-bn-H}
    Assume that the regression targets $y_i$ are generated by a linear model.
    For linear regression with BN, the global minimizer manifold of $\Loss(\vw) := \frac{1}{n} \sum_{i=1}^{n}(\Phi(\vw_i; \vw, \sigmay, \muy) - y_i)^2$ on the unit sphere is
    \begin{align*}
        \manGa := \left\{ \vw \in \sphS^{d-1} : \dotp{\frac{\vw}{\normsm{\vw}_{\mSigmax}}}{\vx_i - \vmux} = \frac{y_i - \muy}{\sigmay} \right\}.
    \end{align*}
    For any global minimizer $\vw \in \manGa$,
    the Hessian matrix $\mH(\vw)$ of the loss is given by
    \begin{align*}
        \mH(\vw) = 2\normtwosm{\tilvw}^2
        \left( \mSigmax - \vz \vz^\top \right),
    \end{align*}
    where $\tilvw := \frac{\sigmay \vw}{\normsm{\vw}_{\mSigmax}}$ as defined in \eqref{eq:lin-bn-Phi}, and $\vz := \frac{1}{n} \sum_{i=1}^{n} \frac{y_i - \muy}{\sigmay} (\vx_i - \vmux)$.
\end{lemma}
\begin{proof}
The model output $\Phi(\vx; \vw, \muy, \sigmay)$
can be written as
\[
    \Phi(\vx; \vw, \muy, \sigmay) = \sigmay \dotp{\frac{\vw}{\normsm{\vw}_{\mSigmax}}}{\vx_i - \vmux} + \muy.
\]
Then it is easy to verify that the global minimizer manifold is $\manGa$.

Now we compute the Hessian.
    Let $\tilvx_i := \vx_i - \vmux$ and $q_i := \frac{y_i - \muy}{\sigmay}$.
Because of the use of squared loss,
on $\manGa$ the Hessian can be written as
the sum of outer products of gradients:
\begin{align*}
    \mH(\vw)
    &= \frac{2}{n} \sum_{i=1}^{n} \nabla_{\vw} \Phi(\vx_i; \vw, \sigmay, \muy) \nabla_{\vw} \Phi(\vx_i; \vw, \sigmay, \muy)^\top.
\end{align*}
For each gradient we have
\begin{align*}
    \nabla_{\vw} \Phi(\vx_i; \vw, \sigmay, \muy)
    = \frac{\sigmay}{\normsm{\vw}_{\mSigmax}} \left( \mI - \frac{\mSigmax \vw\vw^\top}{\normsm{\vw}_{\mSigmax}^2} \right) \tilvx_i
    = \frac{\sigmay}{\normsm{\vw}_{\mSigmax}} \left( \mI - \mSigmax \tilvw \tilvw^\top \right) \tilvx_i
\end{align*}
Note that
    $\mSigmax \tilvw = \frac{1}{n} \sum_{i=1}^{n} \tilvx_i \tilvx_i^\top \tilvw
    = \frac{1}{n} \sum_{i=1}^{n} q_i \tilvx_i =: \vz$.
Then we can simplify the gradient by
\begin{align*}
    \nabla_{\vw} \Phi(\vx_i; \vw, \sigmay, \muy) &=
    \frac{\sigmay}{\normsm{\vw}_{\mSigmax}} \left( \tilvx_i - q_i \vz \right)
\end{align*}
Now we simplify the Hessian.
\begin{align*}
    \mH(\vw)
    &= \frac{2}{n} \sum_{i=1}^{n} \nabla_{\vw} \Phi(\vx_i; \vw, \sigmay, \muy) \nabla_{\vw} \Phi(\vx_i; \vw, \sigmay, \muy)^\top \\
    &= \frac{2\sigmay^2}{n \normsm{\vw}_{\mSigmax}^2}
    \sum_{i=1}^{n}
    \left( \tilvx_i - q_i \vz \right)\left( \tilvx_i - q_i \vz \right)^\top \\
    &= \frac{2\sigmay^2}{n \normsm{\vw}_{\mSigmax}^2}
    \left(
    \sum_{i=1}^{n} \tilvx_i \tilvx_i^\top
    - \sum_{i=1}^{n} q_i \tilvx_i \vz^\top 
    - \sum_{i=1}^{n} q_i \vz \tilvx_i^\top  
    + \sum_{i=1}^{n} q_i^2 \vz\vz^\top
    \right) \\
    &= \frac{2\sigmay^2}{\normsm{\vw}_{\mSigmax}^2}
    \left(
        \mSigmax
    - \vz \vz^\top 
    - \vz \vz^\top  
    + \vz\vz^\top
    \right) \\
    &= \frac{2\sigmay^2}{\normsm{\vw}_{\mSigmax}^2}
    \left(
        \mSigmax
    - \vz \vz^\top 
    \right).
\end{align*}
We complete the proof by noting that $\frac{\sigmay^2}{\normsm{\vw}_{\mSigmax}^2} = \normtwosm{\tilvw}^2$.
\end{proof}

\begin{proof}[Proof for \Cref{thm:bn-lin-main}]
By \Cref{lm:lin-bn-H},
for $\vw \in \sphS^{D-1}$,
\begin{align*}
    \gradGa \log \lamH_1(\vtheta) = \gradGa \log \normtwosm{\tilvw}^2 = \tfrac{1}{\normtwosm{\tilvw}^2} \gradGa \normtwosm{\tilvw}^2.
\end{align*}
By simple calculation,
it can be verified that
the only point that has $\gradGa \normtwosm{\tilvw}^2 = \vzero$
is the unique point $\vwopt$ on $\manGa$
that is a linear combination of $\vx_i - \vmux$.
As the shperical sharpness is bounded from below and there is only one stationary point,
the sharpness-reduction flow must converge on $\manGa$
and the convergence point must be $\vwopt$.

Since $\vwopt$ is a linear combination of $\vx_i - \vmux$,
the associated $\tilvw^*$
should be
the least square solution (without bias) of this ``shifted'' dataset: $\{ (\vx_i - \vmux, y_i - \muy)\}$.
In other words, $\tilvw^*$ is the optimal solution of the following constrained optimization problem:
\[
    \min \quad \normtwosm{\vw}^2 \quad \text{s.t.} \quad \vw^\top (\vx_i - \vmux) = y_i - \muy, \quad \forall i \in [n].
\]
When $\vw$ is given in the above optimization problem,
there is only a unique $b$ such that $\vw^\top \vx_i + b = y_i$.
So we can introduce a bias to this problem without changing the optimal solution:
\[
    \min \quad \normtwosm{\vw}^2 \quad \text{s.t.} \quad \vw^\top \vx_i + b = y_i, \quad \forall i \in [n].
\]
One can also easily see that this $b$ must match with $\tilb$.
Therefore, we can conclude that the sharpness-reduction flow \eqref{eq:main-zeta}
finds to the optimal solution of \eqref{eq:M} at convergence.
\end{proof}

%% file: add-exp.tex
\section{Experiments} \label{sec:add-exp}
\begin{figure}[!b] 
  \includegraphics[width=\textwidth]{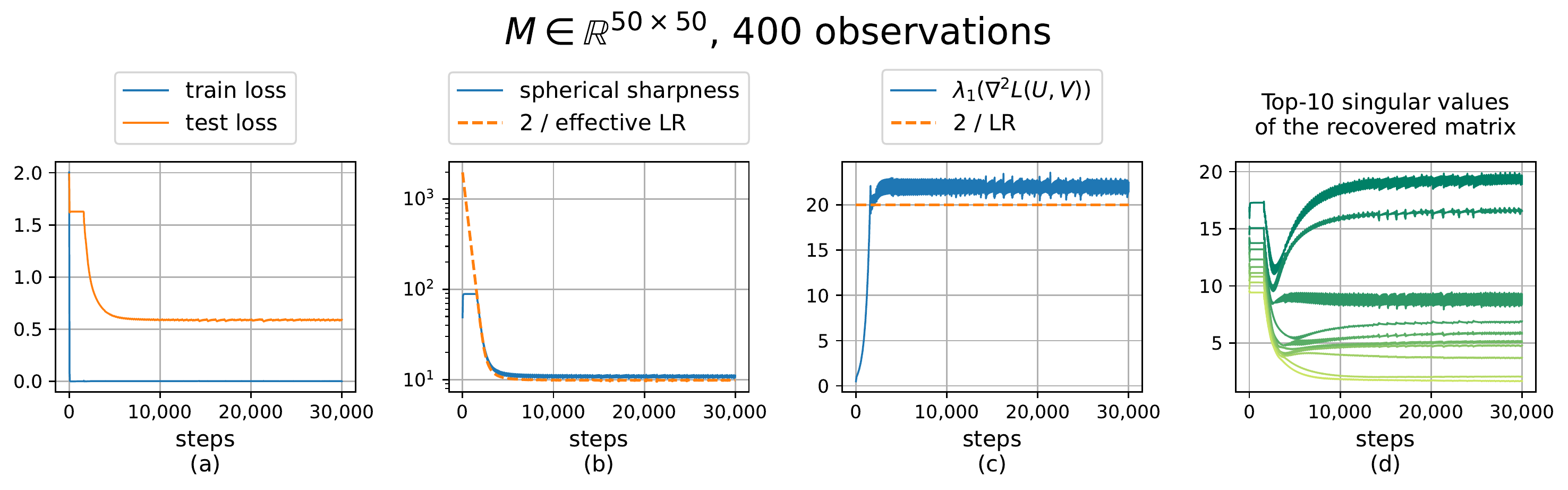}
  \includegraphics[width=\textwidth]{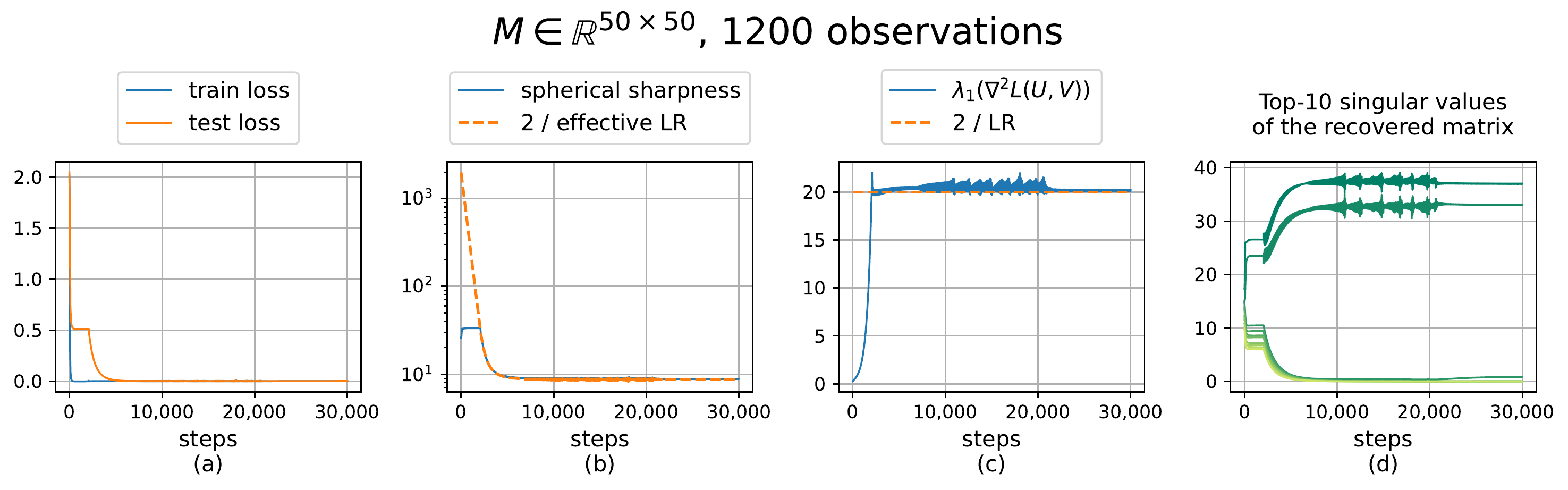}
  \includegraphics[width=\textwidth]{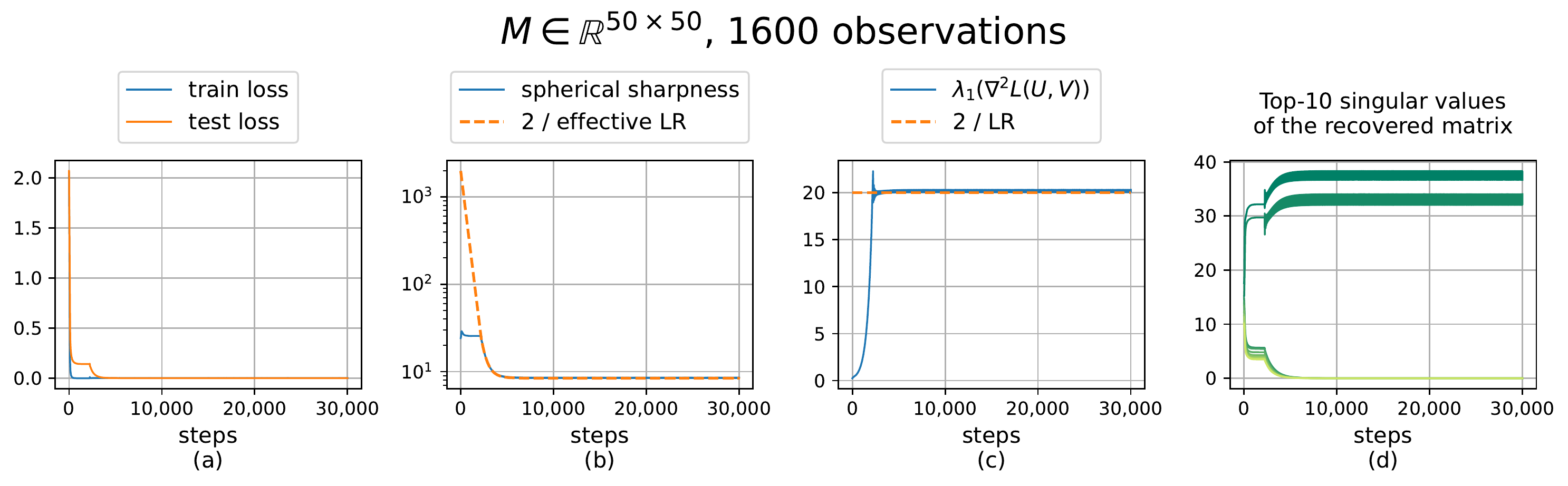}
  \caption{
    Overparameterized matrix completion with BN,
    where
    the ground-truth matrix
    $\mM \in \R^{50 \times 50}$ is of rank 2,
    and the number of observations varies from $400$ to $1600$.
    See \Cref{fig:matcom-intro} for the case of $800$ observations.
    The test loss starts to decrease significantly
    as soon as the spherical sharpness starts to decrease.
  }
  \label{fig:matcom-d50}
\end{figure}

In this section,
we provide experiments on matrix completion and CIFAR-10
to validate the main claim in our theory:
\GDWD on scale-invariant loss persistently reduces spherical sharpness in the EoS regime
(the regime where $2/\eeta_t$ roughly equals to the spherical sharpness).
In addition, we validate that the generalization performance 
continues to improve as the spherical sharpness decreases.
See \Cref{sec:exp-matlin,sec:exp-cifar10}.
Then in \Cref{sec:exp-pb}, we validate a key proof insight: the magnitude of oscillation and effective LR evolve periodically.

We also provide a series of ablation studies.
In \Cref{sec:exp-ab-norm,sec:exp-ab-wd}, we demonstrate that
the two key components in our theoretical setup, normalization and WD,
are crucial.
In \Cref{sec:exp-ab-init-eff-lr}, we show that
the initial effective LR does not affect the final performance when the intrinsic LR is fixed.

\subsection{Validation of Sharpness Reduction on Matrix Completion} \label{sec:exp-matlin}

\begin{figure}[tp] 
  \includegraphics[width=\textwidth]{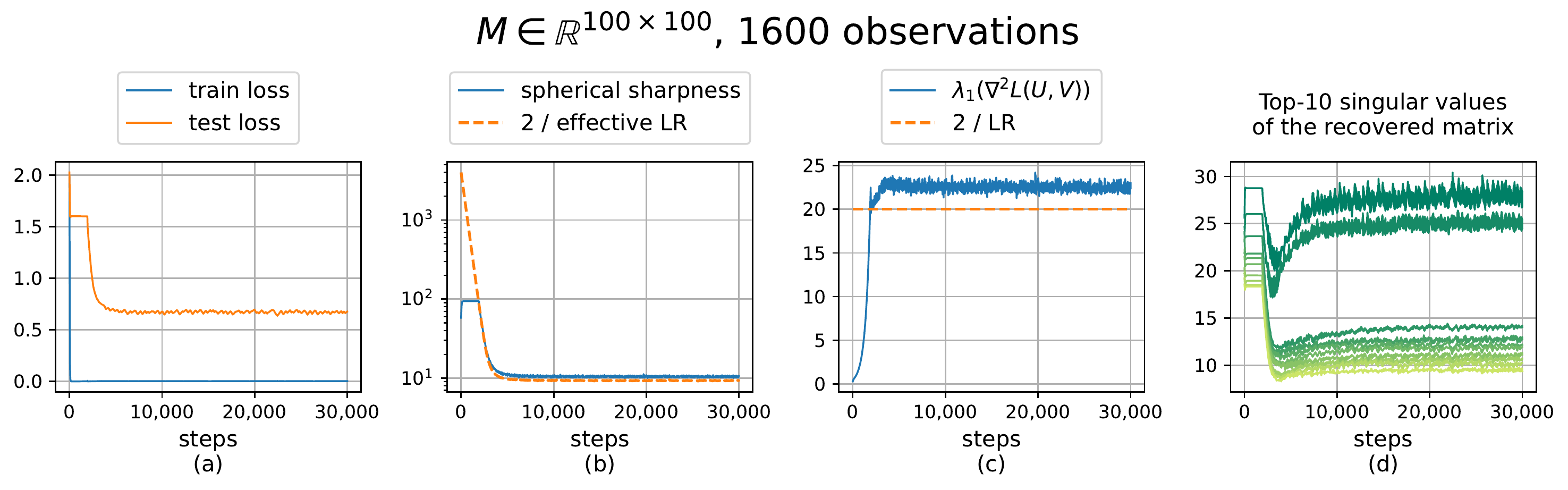}
  \includegraphics[width=\textwidth]{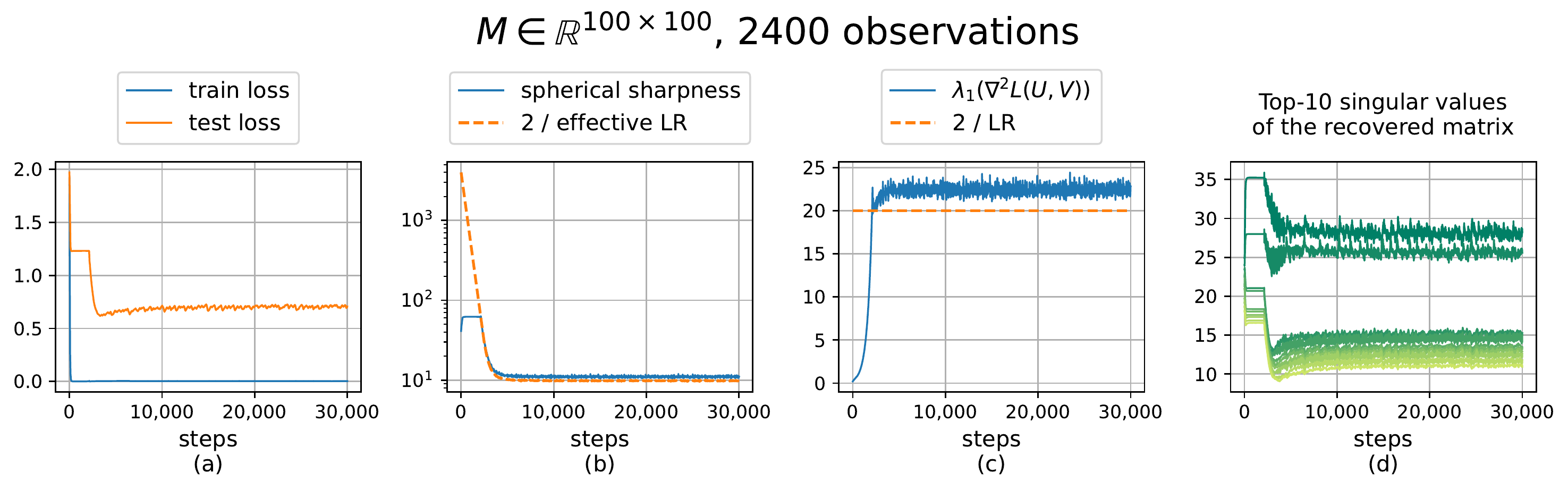}
  \includegraphics[width=\textwidth]{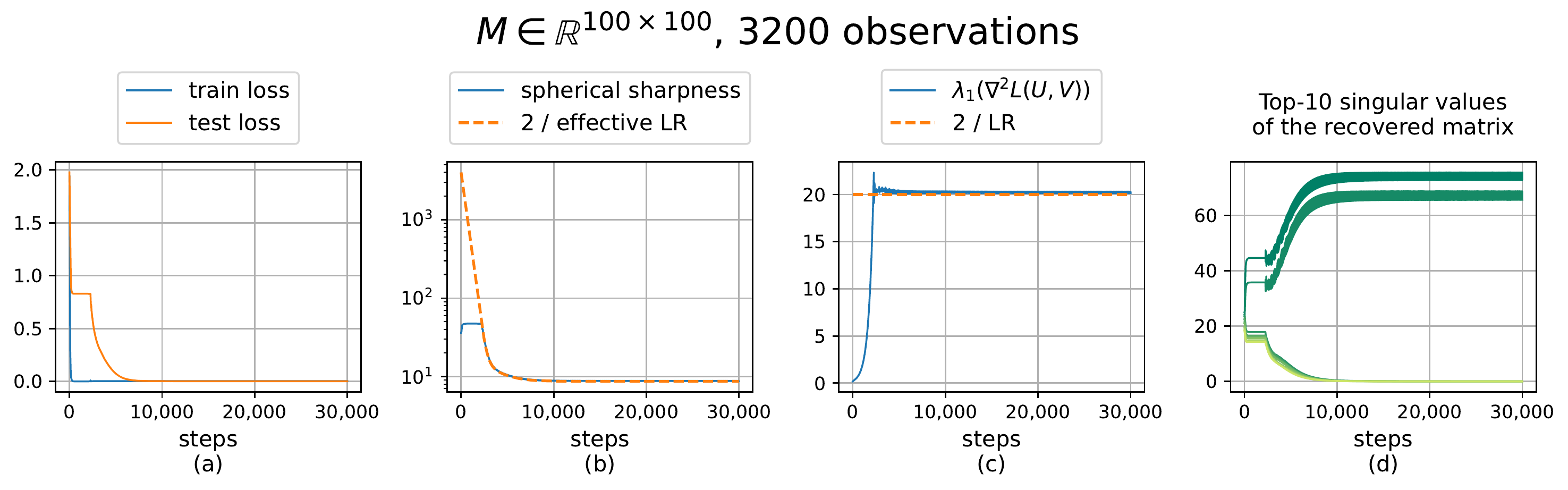}
  \caption{
    Overparameterized matrix completion with BN,
    where
    the ground-truth matrix
    $\mM \in \R^{50 \times 50}$ is of rank 2,
    and the number of observations vary from $1600$ to $3200$.
    The test loss starts to decrease significantly
    as soon as the spherical sharpness starts to decrease.
  }
  \label{fig:matcom-d100}
\end{figure}

First, we conduct experiments on matrix completion: there is an unknown low-rank ground-truth matrix $\mM \in \R^{d \times d}$,
and only $N$ entries of $\mM$ are known. The goal is to recover $\mM$ given the known entries.
Matrix completion has been studied via deep learning techniques~(see, e.g., \citep{chi2019nonconvex} for a survey).
Matrix completion can be connected to supervised learning as follows.
Each entry of the matrix can be seen as a data point,
where the $N$ observed entries
constitute the training set.
As in supervised learning,
given a subset of data points,
the goal of matrix completion is to build a model to predict
the rest of the data points.

Here we empirically study solving matrix completion using an overparameterized scale-invariant model. 
We can observe that spherical sharpness 
is indeed decreasing as soon as the process enters the EoS regime,
and the reduction of spherical sharpness encourages low-rank.

In our experiments,
we generate the ground-truth matrix
$\mM \in \R^{d \times d}$
as follows.
First, we set $\widetilde{\mM} \gets \mU_* \mV_*^\top$ for
two random matrices $\mU_*, \mV_* \in \R^{d \times 2}$,
where every entry is sampled uniformly
from $[-1, 1]$.
Then we obtain $\mM$ by normalizing $\widetilde{\mM}$
so that 
the second moment of
the entries is $1$,
i.e.,
$\mM \gets (d \cdot \normFsm{\widetilde{\mM}}^{-1})\widetilde{\mM}$.
We uniformly sample $N$
entries of $\mM$
to serve as the observations,
and use $\Omega \subseteq [d] \times [d]$
to denote the index set of observed entries.

Matrix completion
has been studied 
by a line of works~\citep{gunasekar2017implicit,li2018algorithmic,arora2019implicit,razin2020implicit,li2021towards,stoger2021small}
as a test-bed for the implicit regularization of gradient descent.
More specifically,
they parameterize the target matrix as $\mW = \mU \mV^\top$
where $\mU, \mV \in \R^{d \times d}$ are two trainable matrices,
and run GD to minimize the squared loss $\Loss(\mU, \mV) := \frac{1}{N} \sum_{(i, j) \in \Omega} (W_{i,j} - M_{i,j})^2$.
Although there is no explicit constraint on rank,
GD with small random initialization can still exhibit an implicit bias
towards low-rank solutions.

Inspired by this line of works,
we conduct matrix completion experiments to test
if \GDWD exhibits
the same low-rank bias
in training overparameterized models with BN.
More specifically,
we parameterize the target matrix as $\mW = \BN(\mU \mV^\top)$,
where $\mU, \mV \in \R^{d \times d}$ are two trainable matrices.
Given the observed positions $\Omega$, the output of the model
for a single position $(i, j) \in \Omega$
is
$\frac{\gamma}{\sigma}[\mU \mV^\top]_{i,j}$,
where $\sigma^2 := \frac{1}{N}\sum_{(i,j) \in \Omega} [\mU \mV^\top]_{i,j}^2$
is the second moment of $[\mU \mV^\top]_{i,j}$
over all observed positions,
and $\gamma$ is a rescaling factor.
Multiplying the factor $\frac{\gamma}{\sigma}$
can be seen as doing BN
over observed entries
because
it rescales the output in the same manner as BN.
But the difference is that,
for the sake of simplicity,
we do not subtract the mean.
To ensure the loss to be scale-invariant,
we also fix $\gamma$ to match the second moment of observed entries,
i.e., $\gamma^2 := \frac{1}{N} \sum_{(i, j) \in \Omega} M_{i,j}^2$.

To train this model, we use the standard squared loss,
and we run gradient descent with LR $\heta = 0.1$ and WD $\hlam = 0.01$ to optimize the loss.
It is obvious that this loss is scale-invariant due to BN.
\begin{align*}
  \Loss(\mU, \mV) := \frac{1}{N} \sum_{(i, j) \in \Omega} \left(\frac{\gamma}{\sigma} [\mU \mV^\top]_{i,j} - M_{i,j}\right)^2.
\end{align*}
For an unobserved entry $(i, j)$,
the model
uses the same batch statistics
as for observed entries,
and predicts
$\frac{\gamma}{\sigma}[\mU \mV^\top]_{i,j}$,
where $\gamma, \sigma$ are the same as above.
So we measure the test loss as
\begin{align*}
  \bar{\Loss}(\mU, \mV) := \frac{1}{d^2} \sum_{i, j \in [d]} \left(\frac{\gamma}{\sigma} [\mU \mV^\top]_{i,j} - M_{i,j}\right)^2.
\end{align*}
We note that the loss function has no explicit constraint on rank.
But surprisingly,
in our experiments,
\GDWD
tends to prefer low-rank solution
as soon as the sharpness-reduction bias starts to
occur.
See \Cref{fig:matcom-intro,fig:matcom-d50}
for experiments on reconstructing
a rank-2 matrix of size $50 \times 50$,
and \Cref{fig:matcom-d100}
for experiments on reconstructing a rank-2 matrix of size $100 \times 100$.
In all these experiments,
we can observe that the train loss first decreases to near zero in a short time,
but the test loss
remains at a high level.
Then after effective LR increases for some more steps,
the dynamics enters the EoS regime,
and as predicted by our theory,
the spherical sharpness starts to decrease.
Meanwhile,
the test loss also starts to decrease significantly.
We notice that the gap between the second and the third largest singular values
of the recovered matrix
is enlarged at the same time,
suggesting that this reduction of spherical sharpness encourages the recovered matrix to be low-rank.

\subsection{Validation of Sharpness Reduction on CIFAR-10} \label{sec:exp-cifar10}

Now we present experiments on CIFAR-10
with crossentropy loss
to validate the sharpness-reduction bias in a more realistic setting.
We run full-batch GD with accompanying WD
on three different architectures:
a scale-invariant variant of VGG-11,
a scale-invariant variant of pre-activation ResNet-20,
and the standard pre-activation ResNet-20.
We fix LR $\heta = 0.1$ and WD $\hlam = 5 \times 10^{-4}$.
See \Cref{sec:exp-details-cifar10} for more details on the architectures and training procedures.

The experiment for the scale-invariant variants of VGG-11 and ResNet-20 are presented in \Cref{fig:full-cifar,fig:full-cifar-fsi-resnet}.
Here
our ResNet-20 is non-smooth due to the use of ReLU,
but our VGG-11 is smooth because we choose to use Swish activation and mean pooling in this network.
For both the smooth VGG-11 and non-smooth ResNet-20,
it can be seen from the plots that in both cases the spherical sharpness has an overall tendency to decrease over time,
and the test accuracy is increasing accordingly.

\begin{figure}[p] 
  \includegraphics[width=\textwidth]{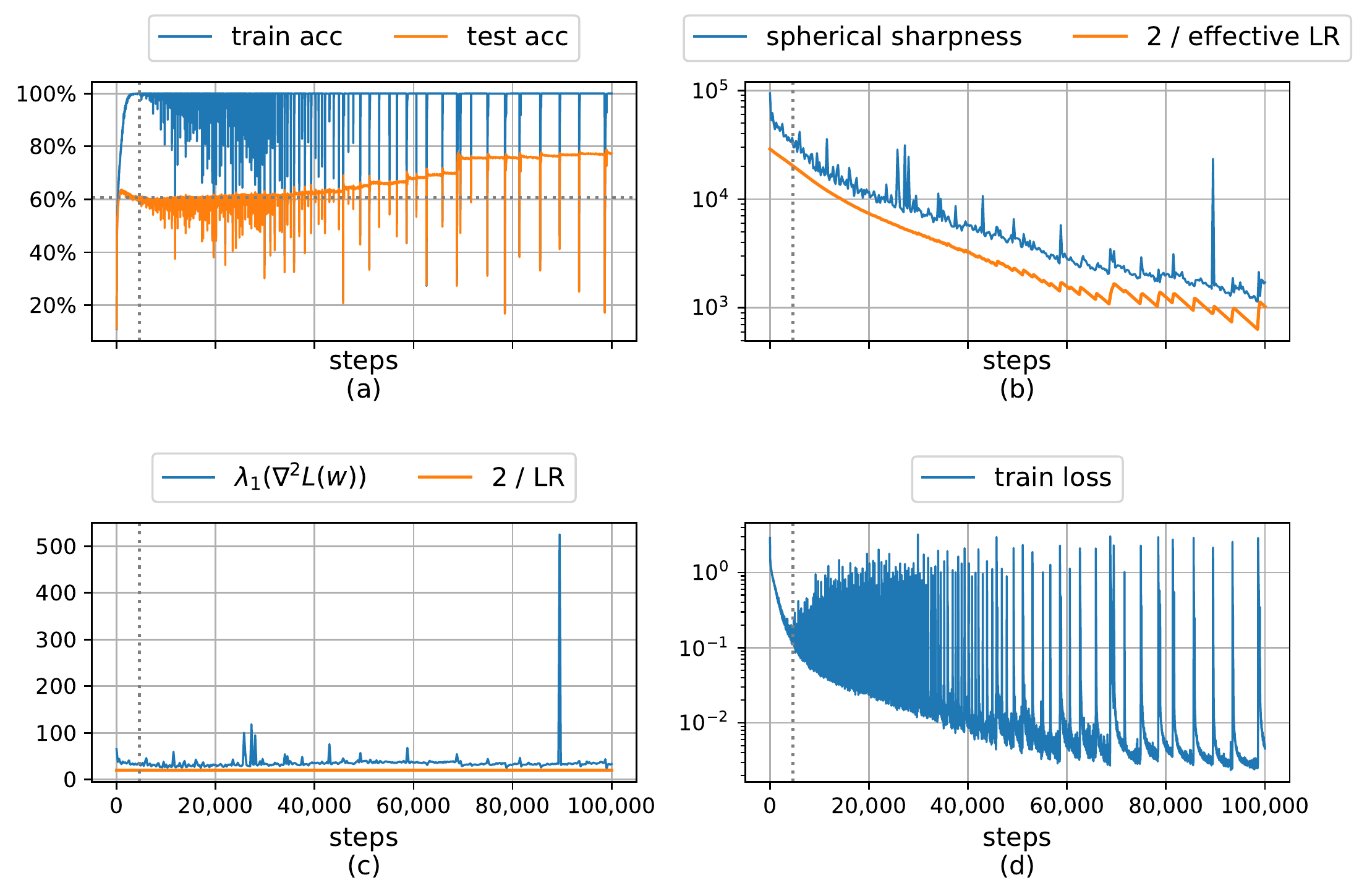}
  \vspace{-0.25in}
  \caption{
    In training a scale-invariant ResNet-20 on CIFAR-10 with (full-batch) \GDWD,
    the spherical sharpness decreases over time.
    $100\%$ training accuracy is achieved after $\sim 4700$ steps (dotted line),
    but as the training continues, 
    the test accuracy increases from $60.8\%$ to $77.4\%$.
  }
  \label{fig:full-cifar-fsi-resnet}
\end{figure}

\begin{figure}[p]
  \includegraphics[width=\textwidth]{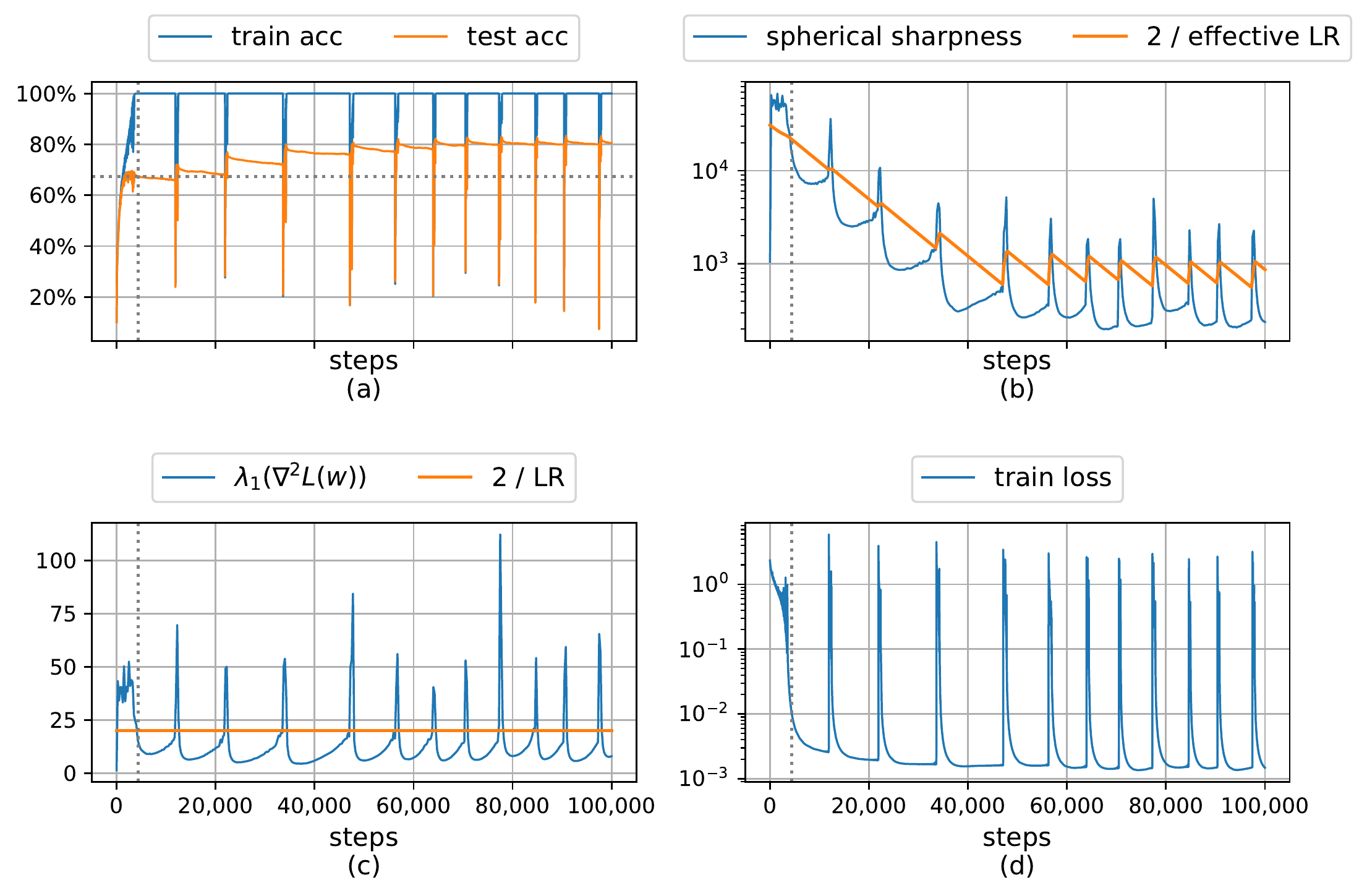}
  \vspace{-0.25in}
  \caption{
    In training the standard pre-activation ResNet-20 on CIFAR-10 with (full-batch) \GDWD,
    the spherical sharpness is decreasing in the long term run.
    $100\%$ training accuracy is achieved after $\sim 4400$ steps (dotted line),
    but as the training continues, 
    the test accuracy increases from $67.5\%$ to $80.4\%$.
    Spherical sharpness is only evaluated for a scale-invariant part of the trainable parameters.
  }
  \label{fig:full-cifar-bp-resnet}
\end{figure}

\begin{figure}[t] 
  \includegraphics[width=\textwidth]{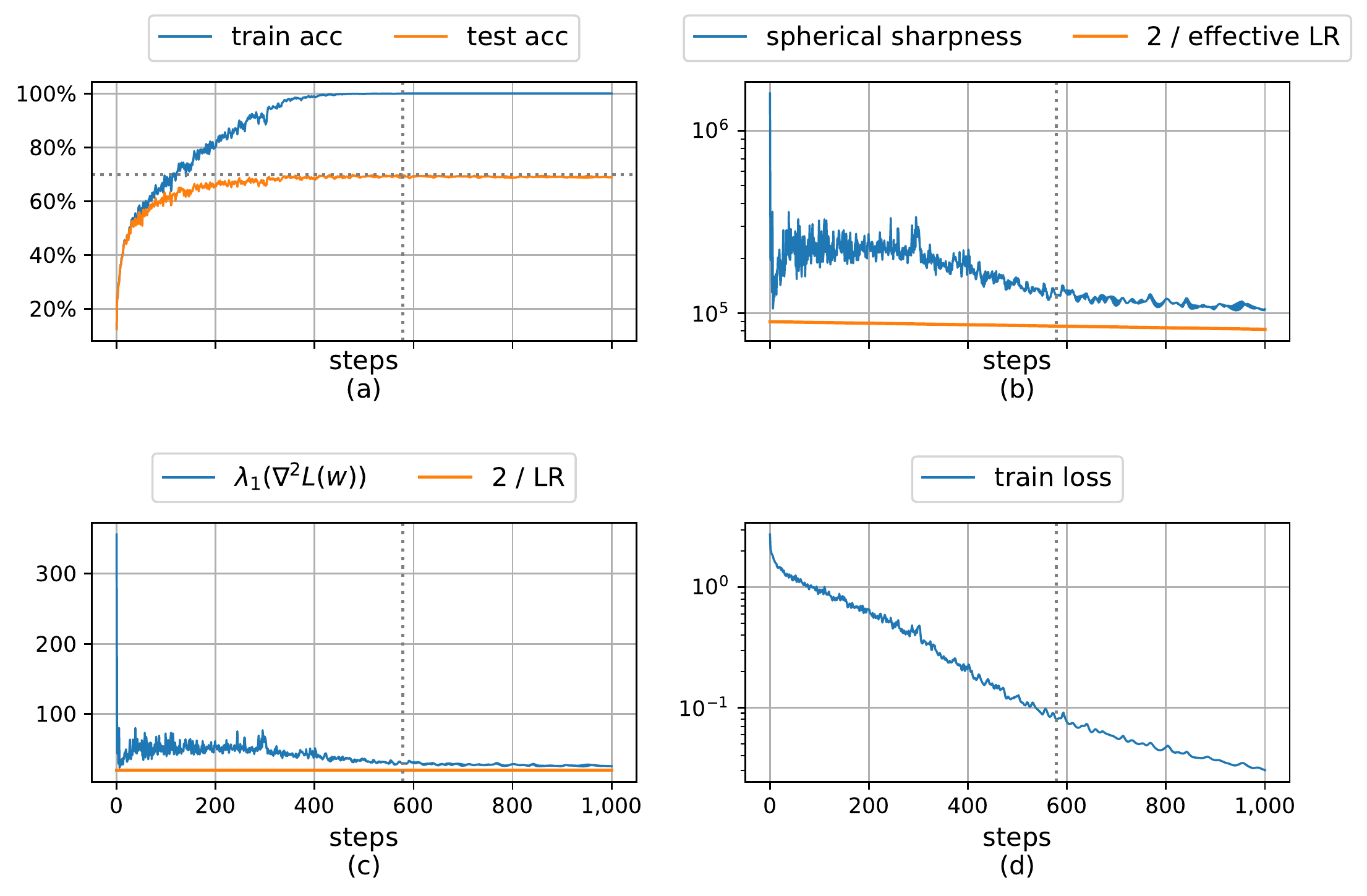}
  \vspace{-0.25in}
  \caption{
    The first 1500 steps
    in training a smooth and scale-invariant VGG-11 on CIFAR-10 with \GDWD (see also \Cref{fig:full-cifar}).
    The spherical sharpness is much larger than $2/\eeta_0$ initially
    but is then reduced to the same level as $2/\eeta_t$ after a few steps.
  }
  \label{fig:full-cifar-vgg-beginning}
\end{figure}

We note that
this sharpness-reduction bias
actually goes beyond the setting that our theory can directly apply:
the dynamic enters the EoS regime in the very beginning of training,
but our theory 
only analyzes the dynamic near a local minimizer manifold.
This is because
that the LR here is not small enough so that $\vtheta_t$ can approach to a local minimizer
before the dynamic enters the EoS regime.
In fact,
\Cref{fig:full-cifar-vgg-beginning} shows that
in the scale-invariant VGG-11 experiment,
the initial spherical sharpness is much larger than $2 / \eeta_0$,
then after a few steps, the spherical sharpness decreases to a level that is close to $2/\eeta_t$
and the dynamic enters the EoS regime.

Besides the scale-invariant models,
we also validate the sharpness-reduction bias on the standard pre-activation ResNet-20,
which is not (fully) scale-invariant but only scale-invariant to a part of its parameters.
For evaluating spherical sharpness,
we compute the partial Hessian only with respect to that part of parameters.
See \Cref{fig:full-cifar-bp-resnet} for the plot and \Cref{sec:exp-details-cifar10} for more experimental details.
Although our theory can only cover scale-invariant models,
we can still observe the sharpness-reduction bias in experiments,
and the test accuracy increases as spherical sharpness decreases.

\subsection{Periodic Behaviors} \label{sec:exp-pb}

A key insight in our proof
is that
the magnitude of oscillation and effective LR evolve periodically in the EoS regime.
Following \Cref{sec:proof-idea},
we can divide the EoS regime into two sub-regimes that occur alternatively in training:
(1)~the sub-regime where $2/\eeta_t$ is smaller than $\lamH_1(\vphi_t)$
and the magnitude of oscillation $\abssm{h_t}$ keeps increasing;
and (2)~the sub-regime where $2/\eeta_t$ is bigger than $\lamH_1(\vphi_t)$
and the magnitude of oscillation $\abssm{h_t}$ keeps decreasing.
Here $\vphi_t$ is a projection of $\vtheta_t$ onto the local minimizer manifold,
$\lamH_1(\vphi_t)$ is the spherical sharpness computed at $\vphi_t$,
and $h_t$ is the inner product between $\vtheta_t - \vphi_t$
and the top eigenvector of $\mH(\vphi_t)$.

\Cref{fig:lin-idea} provides a nice validation of this periodic behavior in linear regression with BN.
We further validate the periodic behavior
in matrix completion by visualizing
$\eeta_t$ in \Cref{fig:eos-matcom-intro} around the moment that the dynamic enters the EoS.
We do not plot $\abssm{h_t}$ because computing $\vphi_t$ is inefficient,
but we can observe that the training loss evolves periodically, which signals that $\abssm{h_t}$
evolves periodically as well.

In CIFAR-10 experiments with realistic LR and WD,
the periodic behavior still exists, but it may deviate from the regime that our theory can capture.
In most of our CIFAR-10 experiments,
only the sub-regime where $2/\eeta_t < \lamH_1(\vphi_t)$
can be observed,
and the phenomenon that $2/\eeta_t$ switches back and forth between being smaller and larger than the spherical sharpness
does not occur anymore.
However, the change of $\eeta_t$ and spherical sharpness
still cause to the loss and gradient norm to fluctuate periodically.
See \Cref{fig:full-cifar-vgg-periodic} for details.
This difference is because that
the loss function of a neural net on CIFAR-10 is much less smooth than that of linear regression and matrix completion.
To make the dynamic happen in the regime that our theory describes,
the LR and WD need to be very small
so that $\vtheta_t$ is sufficiently close to the minimizer manifold $\manGa$,
but they cannot be that small in practice due to computational inefficiency.

The periodic behavior can also happen in a macroscopic scale on CIFAR-10.
E.g.,
in \Cref{fig:full-cifar},
the training loss spikes to a high value around step 50,000,
but then recovers to a near-zero value after a few hundred steps.
Such a macroscopic periodic behavior may not always lead to sharpness reduction,
in contrast to 
the microscopic periodic behavior of our interest.
But as noted in the extensive empirical study by \citet{lobacheva2021on},
this macroscopic periodic behavior
sometimes leads to better generalization.
They also provide a theoretical analysis for the cause assuming the gradient norm is both lower and upper bounded.
However, their explanation is not completely satisfactory
because the gradient norm does not admit a lower bound when the parameter is close to the minimizer manifold $\manGa$.
We leave it a future work to look further into the cause of macroscopic periodic behavior and how it helps generalization.

\begin{figure}[t]
  \centering
  \includegraphics[width=\textwidth]{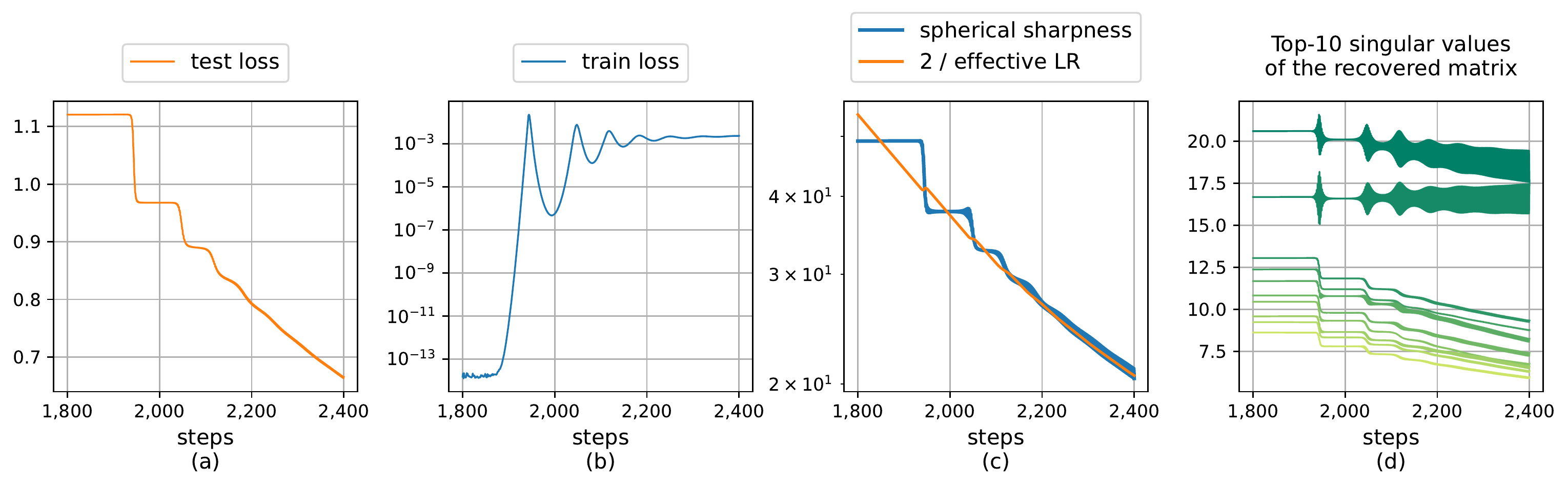}
  \caption{
    A closer look at the dynamic in 
    the matrix completion experiment (\Cref{fig:matcom-intro})
    around the moment that it enters the EoS regime.
    $2 / \eeta_t$
    switches back and forth between being smaller and larger than the spherical sharpness
    (computed as $\lamH_1(\vtheta_t)$ for efficiency),
    which causes the training loss to oscillate periodically.
  }
  \label{fig:eos-matcom-intro}
\end{figure}

\begin{figure}[t]
  \centering
  \includegraphics[width=0.85\textwidth]{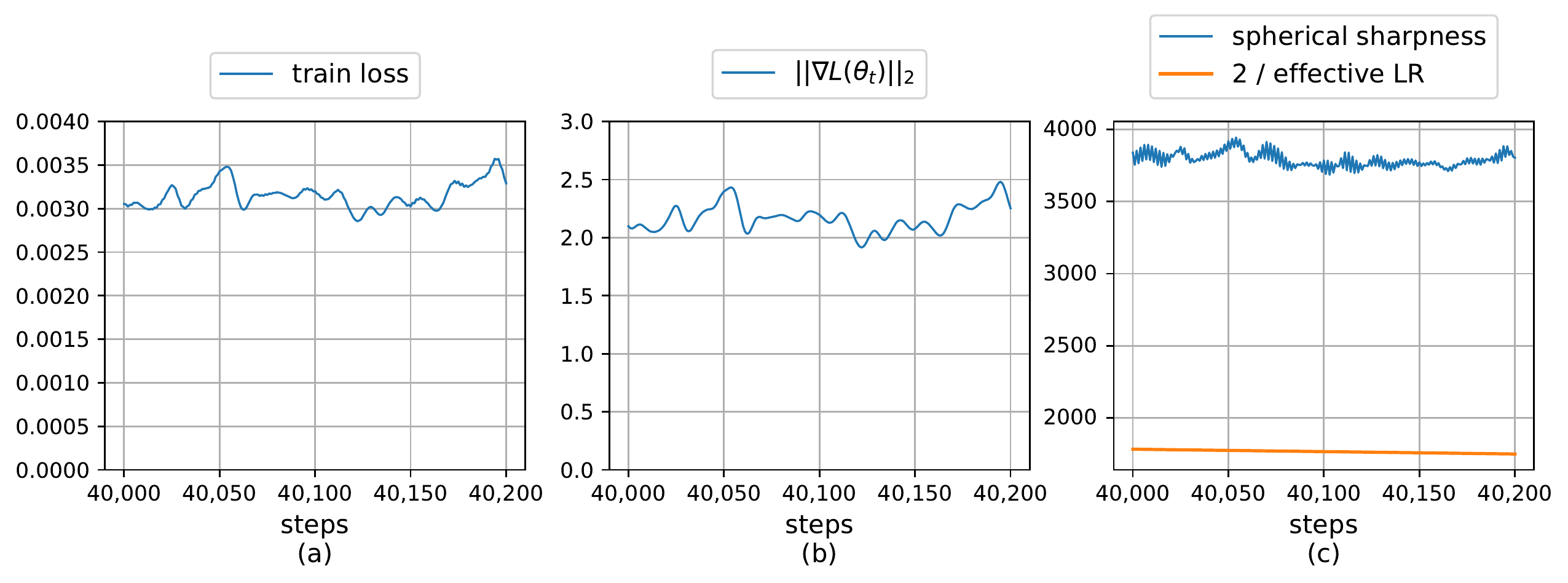}
  \caption{
    A closer look at a sample of 200 steps in the middle of the VGG-11 training on CIFAR-10 (\Cref{fig:full-cifar}).
    $2/\eeta_t$ is always smaller than the spherical sharpness,
    but the training loss and gradient norm still fluctuate periodically.
  }
  \label{fig:full-cifar-vgg-periodic}
\end{figure}
\subsection{Ablation Study: Normalization} \label{sec:exp-ab-norm}

\Cref{fig:full-cifar} and many other experiments in \Cref{sec:exp-matlin,sec:exp-cifar10}
have shown that \GDWD on normalized nets
can continue to improve test accuracy even after reaching 100\% test accuracy.
In our theoretical analysis, we connect this phenomenon with the reduction of spherical sharpness during training.

Now we empirically validate that 
this phenomenon is indeed linked the presence of normalization layers in neural nets.
We train a VGG-11 in the same setting as \Cref{fig:full-cifar}, but now
we remove all the normalization layers.
See \Cref{fig:full-cifar-vgg-ab-norm}.
The test accuracy does not increase anymore after the training accuracy reaches $100\%$.

Note that the spherical sharpness is meaningful only for normalized nets.
For unnormalized nets, even if GD is implicitly reducing a similar sharpness measure,
such a measure cannot be strongly related to generalization
because the test accuracy is not increasing accordingly.

\begin{figure}[tp]
  \centering
  \includegraphics[width=0.85\textwidth]{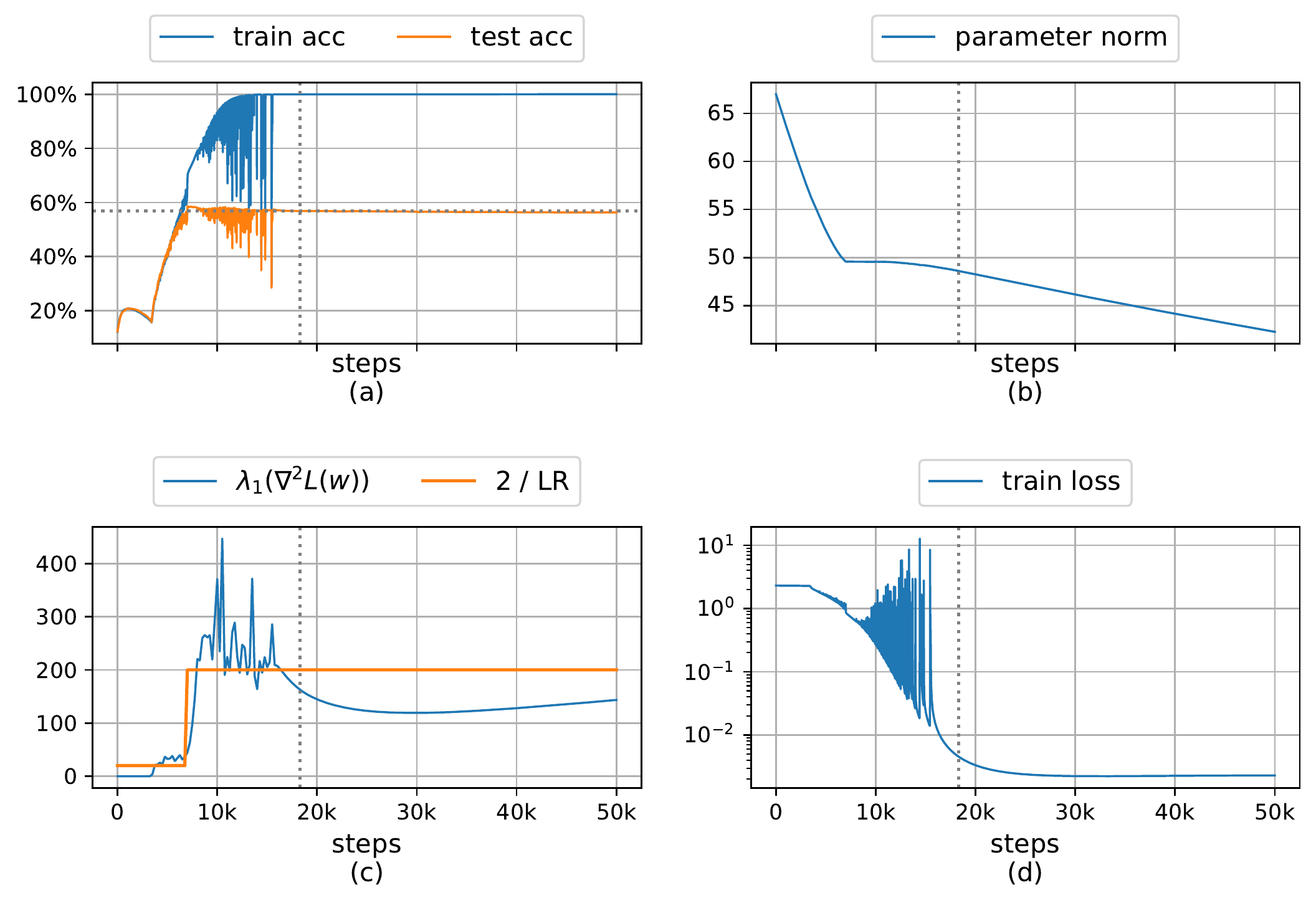}
  \vspace{-0.05in}
  \caption{
    A VGG-11 is trained on CIFAR-10 in the same setting
    as \Cref{fig:full-cifar} but now the normalization layers are all removed.
    The test accuracy slowly decreases from $56.8\%$ to $56.3\%$ after the training accuracy reaches $100\%$.
    The LR is set to $0.1$ initially and is decayed to $0.01$ at step $7k$ to avoid instability.
  }
  \label{fig:full-cifar-vgg-ab-norm}
\end{figure}
\begin{figure}[tp]
  \centering
  \includegraphics[width=0.85\textwidth]{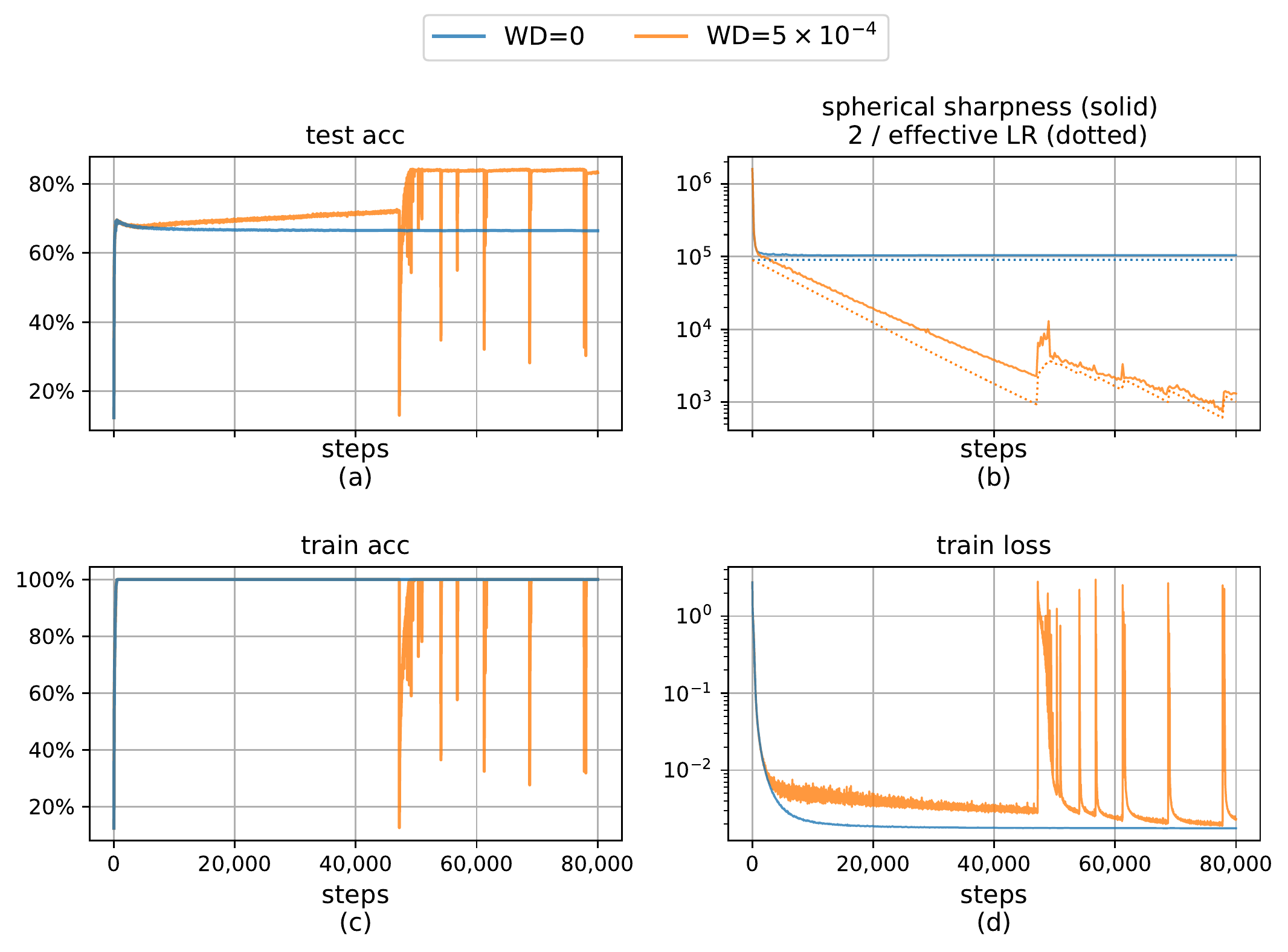}
  \vspace{-0.05in}
  \caption{
    In training scale-invariant VGG-11 on CIFAR-10 with full-batch GD (LR $\eeta = 0.1$),
    weight decay plays an important role in sharpness reduction.
    When weight decay $\hlam = 5 \times 10^{-4}$,
    the spherical sharpness persistently decreases,
    and the test accuracy increases from $69.1\%$ to $84.3\%$ (see also \Cref{fig:full-cifar}).
    But if there is no weight decay,
    the spherical sharpness does not change much after a few initial steps,
    and the test accuracy is stuck at $66.4\%$.
  }
  \label{fig:full-cifar-vgg-wd-vs-nowd}
\end{figure}

\subsection{Ablation Study: Weight Decay} \label{sec:exp-ab-wd}

Now we conduct experiments to study the effects of weight decay (WD)
on sharpness reduction.
It is crucial to have non-zero WD in our theoretical analysis.
Otherwise, it is implied by \Cref{lm:connection-prelim} that the effective LR is monotone decreasing,
but our theory only applies
to the effective LRs that can be viewed as quasi-RMSprop (see \Cref{def:qrms-sche,thm:gdwd-is-quasi-rmsprop}),
in which the effective LR can either increase or decrease as gradient norm changes.
A previous work by \citet{arora2018theoretical}
provides a detailed theoretical analysis in this case
showing that the dynamic always stays in the stable regime
after a few warm-up steps,
and the parameter eventually converges to a stationary point on the unit sphere under standard assumptions in optimization.

We can verify through experiments that
the requirement of WD to be non-zero
is not a technical artifact,
but indeed a necessity in practice to exhibit the sharpness-reduction bias.
We train a scale-invariant VGG-11 following exactly the same hyperparameters and initialization
as \Cref{fig:full-cifar}, except that we set WD $\hlam$ to zero.
The result is presented in \Cref{fig:full-cifar-vgg-wd-vs-nowd}.
The spherical sharpness is no longer decreasing with time
when WD is zero,
and the final test accuracy is much lower than the VGG-11 with WD $5 \times 10^{-4}$.

We further conduct experiments with smaller (but non-zero) WD than our default value $5 \times 10^{-4}$.
For the sake of computational efficiency,
our experiments are conducted on a subset of CIFAR-10 images consisting of 2K images,
which we call CIFAR-10-2k (see \Cref{sec:exp-details-cifar10}).
The result is presented in \Cref{fig:full-cifar2k-vgg-ab-wd},
from which we can observe that
GD exhibits a sharpness-reduction bias
as long as WD is non-zero.
But note that smaller WD leads to
a slower speed of sharpness reduction.
This is an expected phenomenon
because
the speed of tracking the sharpness-reduction flow \eqref{eq:main-zeta}
is controlled by 
the intrinsic LR $\ieta := \heta \hlam$.
The smaller the WD, the smaller the intrinsic LR (when LR $\heta$ is fixed).

\begin{figure}[tp] 
  \vspace{-0.1in}
  \centering
  \includegraphics[width=0.85\textwidth]{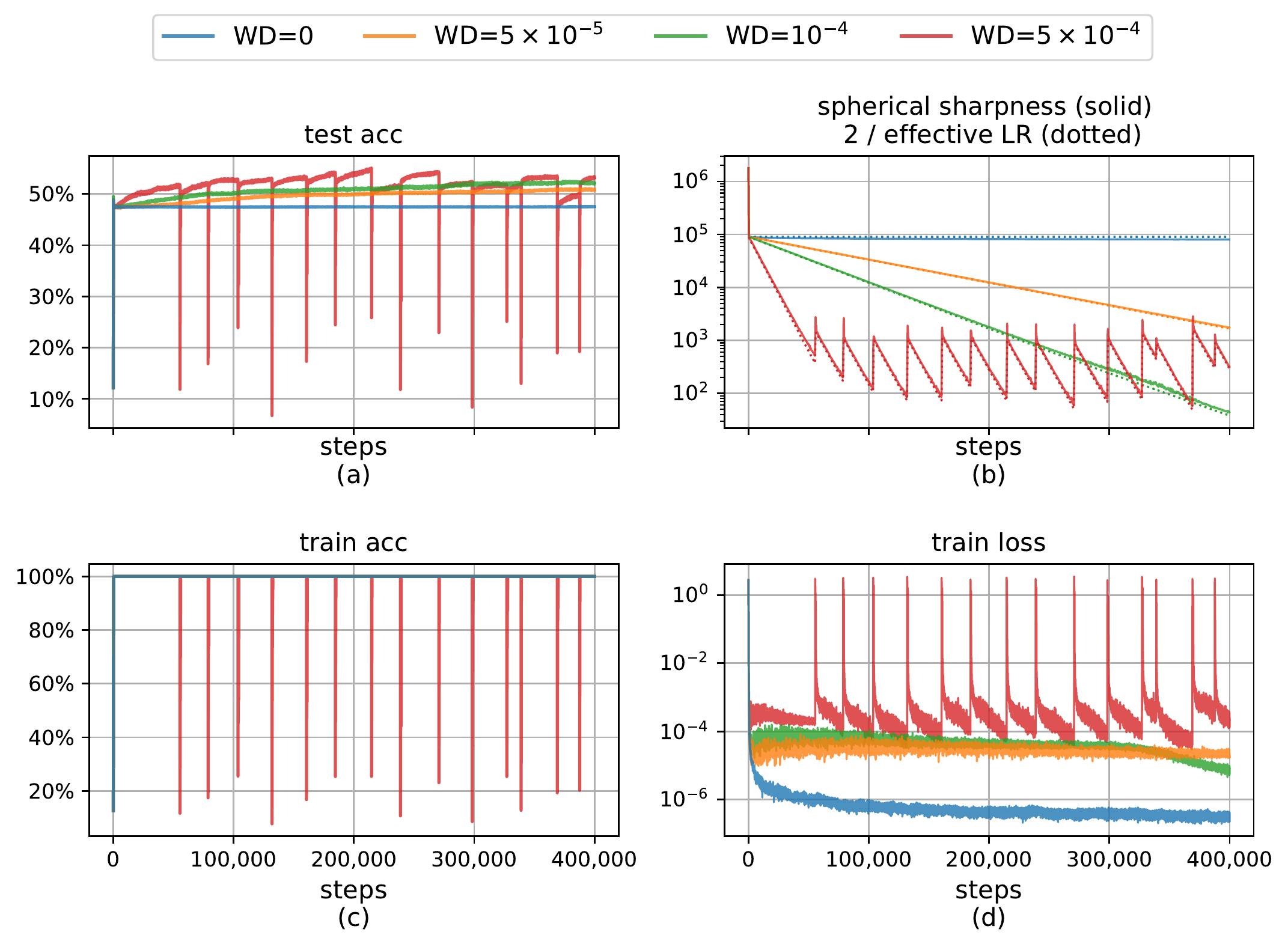}
  \vspace{-0.15in}
  \caption{
    In training scale-invariant VGG-11 on CIFAR-10-2k with full-batch GD (LR $\eeta = 0.1$),
    the sharpness-reduction bias occurs as long as the weight decay is non-zero,
    but smaller WD leads to longer training time to reduce the spherical sharpness to the same level of that with larger WD.
    The best test accuracy achieved within
    400,000 steps
    is
    $48.9\%$ when $\hlam = 0$,
    $51.1\%$ when $\hlam = 5 \times 10^{-5}$,
    $52.5\%$ when $\hlam = 10^{-4}$,
    $55.1\%$ when $\hlam = 5 \times 10^{-4}$.
  }
  \label{fig:full-cifar2k-vgg-ab-wd}
\end{figure}

\begin{figure}[tp] 
  \centering
  \includegraphics[width=\textwidth]{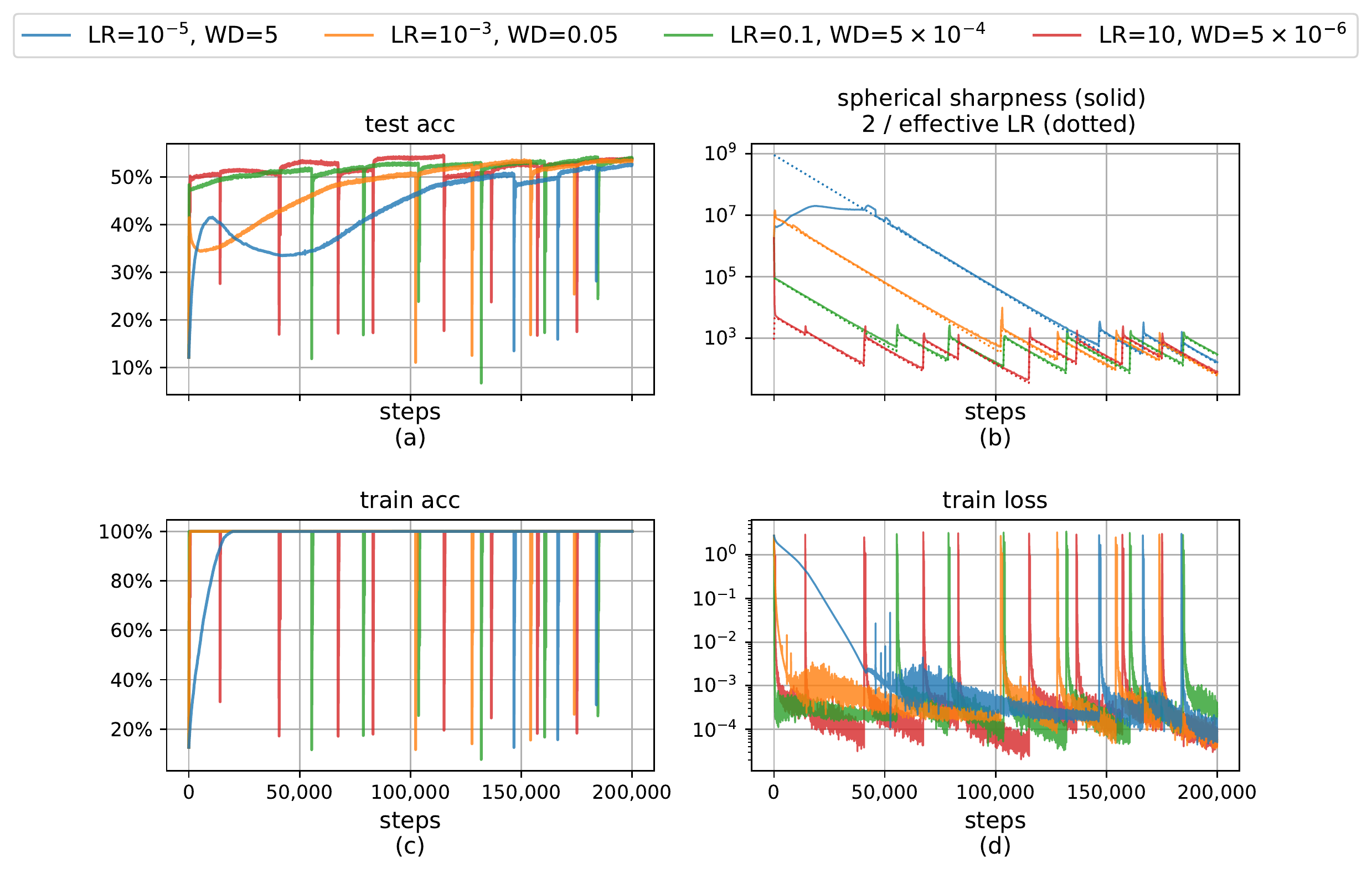}
  \vspace{-0.15in}
  \caption{
    In training scale-invariant VGG-11 on CIFAR-10-2k with full-batch GD
    and various different initial effective learning rates $\eeta_0$.
    We scale learning rate $\heta$ and weight decay $\hlam$
    inverse proportionally
    while fixing the intrinsic LR $\ieta = 5 \times 10^{-5}$.
    The best test accuracies achieved within
    200,000 steps
    are similar in all runs,
    with $52.7\%$ for $\heta = 10^{-5}$,
    $53.7\%$ for $\heta = 10^{-3}$,
    $54.3\%$ for $\heta = 0.1$,
    $54.6\%$ for $\heta = 10$.
  }
  \label{fig:full-cifar2k-vgg-ab-lrwd}
\end{figure}

\subsection{Ablation Study: Initial Effective Learning Rate} \label{sec:exp-ab-init-eff-lr}

As we have seen in \Cref{fig:full-cifar-vgg-beginning},
in training the scale-invariant VGG-11 on CIFAR-10,
if $\heta = 0.1$ and $\hlam = 5 \times 10^{-4}$, then
initially the spherical sharpness is much larger than $2/\eeta_0$,
and after a few steps, it decreases to the same level as $2/\eeta_t$,
thus bringing the dynamic into the EoS regime.

Now we conduct an ablation study 
to see how the dynamics change with different initial effective LR.
For computational efficiency, we run experiments on CIFAR-10-2k (see \Cref{sec:exp-details-cifar10})
to train our scale-invariant variant of VGG-11.
We first set LR $\heta = 0.1$ and WD $\hlam = 5 \times 10^{-4}$,
then we change the LR $\heta$ to $10, 10^{-3}, 10^{-5}$ while rescaling the WD $\hlam$ inverse proportionally.
In this way, the intrinsic LR remains unchanged.
See \Cref{fig:full-cifar2k-vgg-ab-lrwd} for the results.
For large initial effective LR, the dynamic quickly enters the EoS regime;
but when the initial effective LR is as small as $10^{-5}$,
the dynamic is in the stable regime initially,
and $2/\eeta_t$ falls so slowly
that the dynamic enters the EoS regime
only after reaching $100\%$ training accuracy.
In all experiments, the spherical sharpness persistently decreases in the EoS regime,
and the test accuracy increases accordingly.
In the end, all the experiment runs reach nearly the same spherical sharpness and test accuracy,
regardless the initial effective LR.

\section{Experiment Details} \label{sec:exp-details}

All our experiments were conducted on NVIDIA RTX A5000 GPUs.
The longest experiments are the full-batch training experiments on CIFAR-10,
each of which took 7 days to run.

\subsection{Additional Details of CIFAR-10 Experiments} \label{sec:exp-details-cifar10}

Our implementation of full-batch GD is based on the code of \citet{cohen2021gradient}\footnote{\url{https://github.com/locuslab/edge-of-stability}}.
Our implementations of VGGNets~\citep{simonyan2014vgg} and ResNets~\citep{he2016deepres,he2016identityres} on CIFAR-10 are based on a high-starred GitHub repository of 
Wei Yang (\href{https://github.com/bearpaw}{bearpaw})\footnote{\url{https://github.com/bearpaw/pytorch-classification}}.
We do not use any data augmentation.
As it is not feasible to do BN over the full dataset
under our GPU memory constraints,
we use ghost batch normalization~\citep{hoffer2017trainlonger} instead,
where we split
the dataset into $50$ ghost batches of size $1000$.

Some of our experiments are conducted on only a subset of CIFAR-10 consisting of $2000$ images,
which we call CIFAR-10-2k.
We construct the dataset by scanning the full CIFAR-10 dataset and taking the first $200$ images
for each of the $10$ classes.
We use the standard batch normalization with full batch for all CIFAR-10-2k experiments.

Three different neural network architectures
are tested in our experiments,
and we refer to them as scale-invariant VGG-11, standard ResNet-20, and scale-invariant ResNet-20 respectively.

\myparagraph{Scale-Invariant VGG-11.} Our scale-invariant VGG-11 architecture is similar to the configuration A of the original VGGNet~\citep{simonyan2014vgg},
but we make the following changes.
We add a BN layer (without affine parameters) between every convolution and activation to introduce scale-invariance,
We use mean pooling instead of max pooling, Swish instead of ReLU to make the training loss smooth.
We replace the final $3$ fully-connected layers with only one fully-connected layer
as in Yang's implementation.
All convolutional and fully-connected layers have no bias terms.
We add a BN layer
after the last fully-connect layer,
in which the affine parameters are fixed to $\gamma = \frac{3}{10} \ln 91 \approx 1.353$, $\beta = 0$ (see \Cref{sec:affine} for discussion).
A visualization of the full architecture is given in \Cref{fig:arch-vgg}, whose output function can be shown to be scale-invariant with respect to the trainable parameters.

\begin{figure}[htbp] 
  \centering
  \includegraphics[scale=0.13]{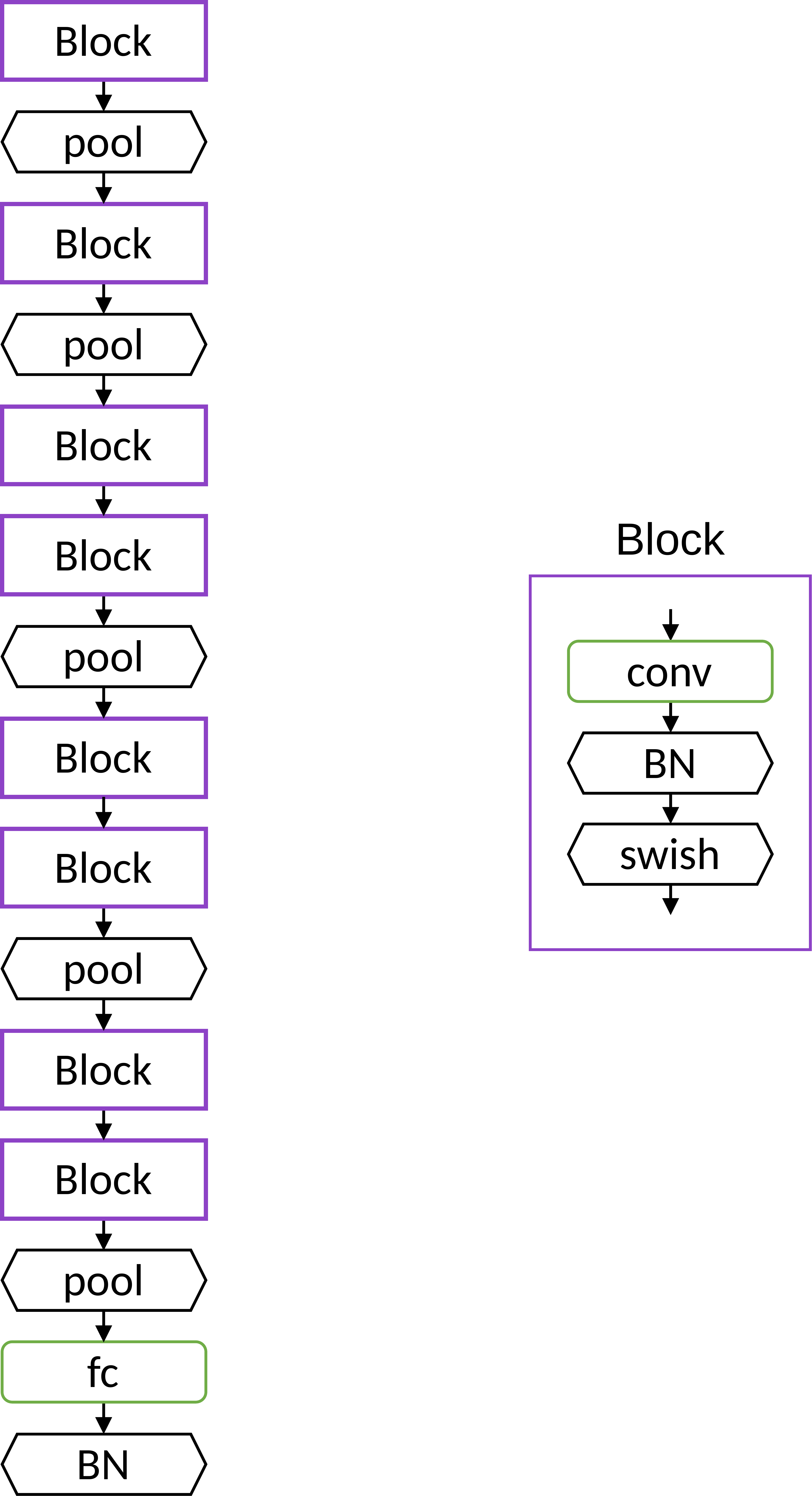}
  \caption{
    The architecture of
    a scale-invariant variant of VGG-11.
    Each rounded rectangle stands for a layer with trainable parameters,
    including convolutional and fully-connected layers.
    Each hexagon stands for a layer with no trainable parameters,
    including mean pooling, swish, BN without affine parameters, and BN with affine parameters fixed.
  }
  \label{fig:arch-vgg}
\end{figure}

\begin{figure}[htbp] 
  \centering
  \begin{subfigure}{\textwidth}
    \centering
    \includegraphics[scale=0.13]{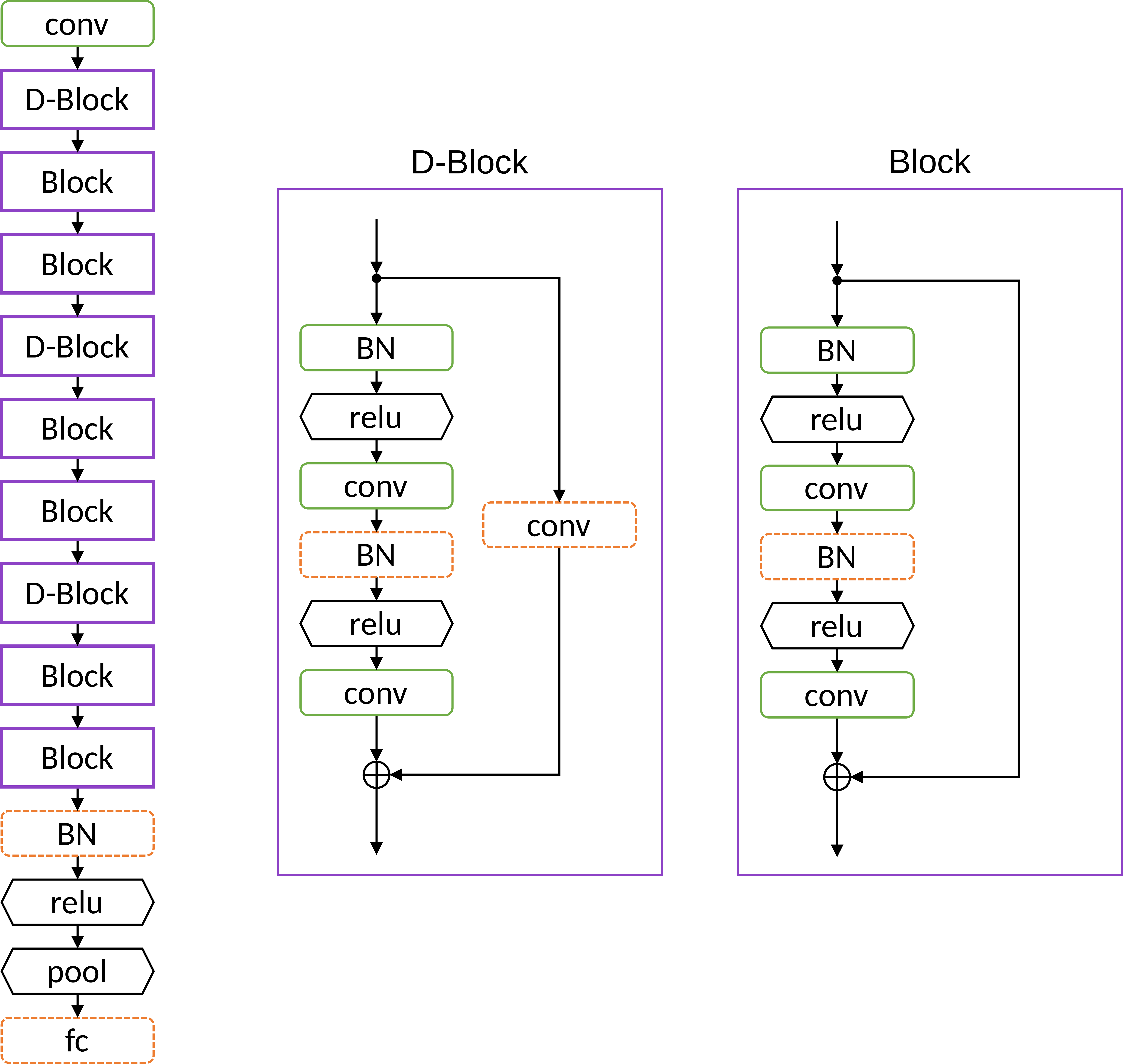}
    \caption{The architecture of ResNet-20, following Yang's implementation.}
    \label{fig:arch-std-resnet}
  \end{subfigure}
  \par\bigskip
  \begin{subfigure}{\textwidth}
    \centering
    \includegraphics[scale=0.13]{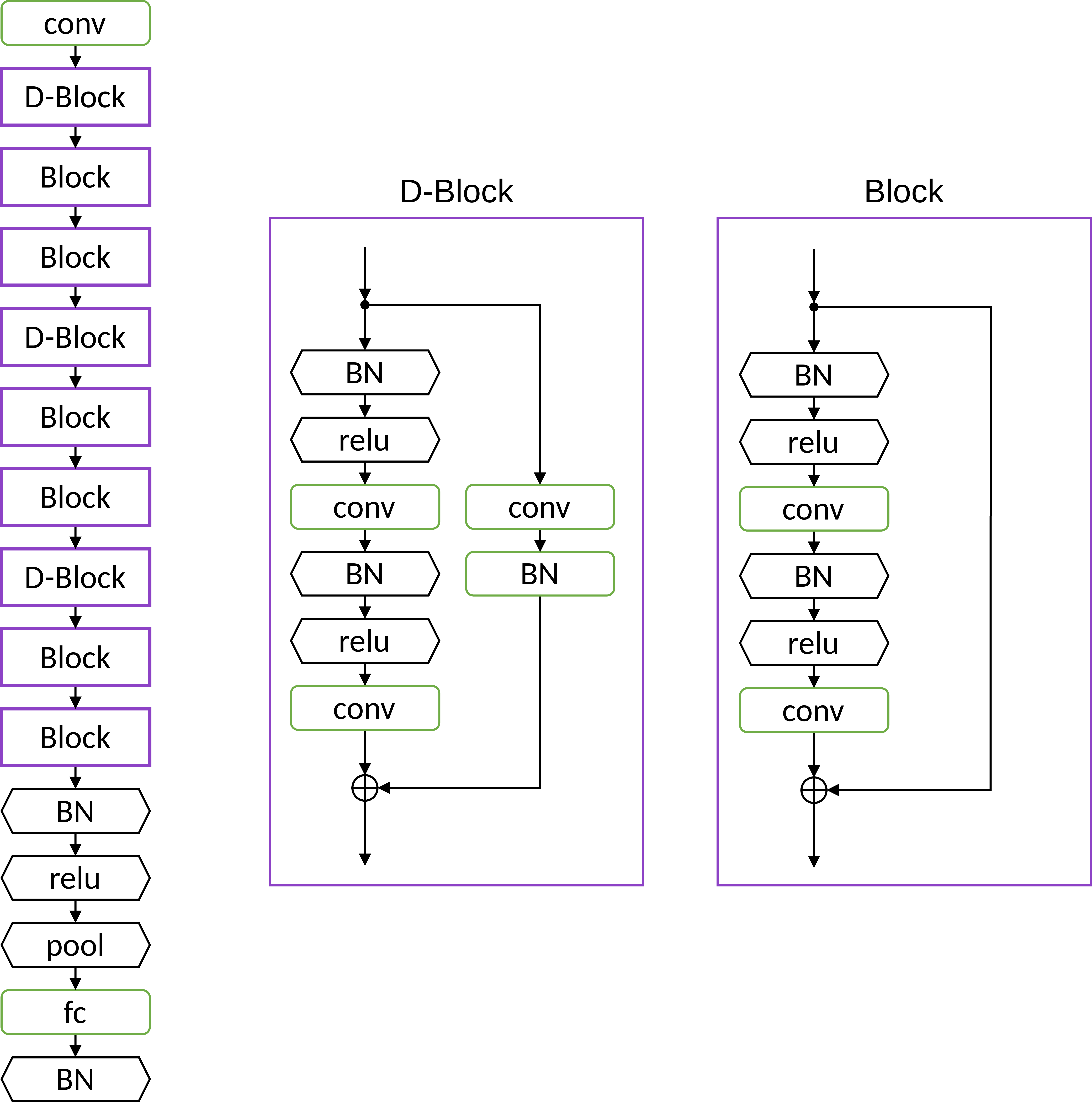}
    \caption{The architecture of a scale-invariant variant of ResNet-20.}
    \label{fig:arch-fsi-resnet}
  \end{subfigure}
  \caption{
    The ResNet architectures used in our experiments.
    Each rounded rectangle stands for a layer with trainable parameters,
    including convolutional and fully-connected layers, and BN with trainable affine parameters.
    Each hexagon stands for a layer with no trainable parameters,
    including mean pooling, ReLU, BN without affine parameters, and BN with affine parameters fixed.
    In Yang's implementation of ResNet-20,
    the output function is scale-invariant to the parameters in layers marked as rounded green rectangles;
    in the scale-invariant variant,
    the output function is scale-invariant to all trainable parameters.
  }
  \label{fig:arch-resnet}
\end{figure}

\myparagraph{Standard ResNet-20.} 
We also verify the sharpness-reduction bias on a more realistic setting with ResNet-20 and ReLU.
In some of our experiments, we use Yang's implementation for the pre-activation variant of ResNet-20.
All convolutional layers have no bias, but the final fully-connected layer does have bias terms.
We note that the output function of this architecture
is not scale-invariant to all its trainable parameters.
However, it can be shown that
the output function is indeed scale-invariant to a large part of its parameters.
More specifically,
the trainable parameters can be split into two parts, $\vw$ and $\vz$,
and the output function remains unchanged if we replace $\vw$ with $c\vw$ for all $c > 0$.
In our experiments, we measure the spherical sharpness by $\nabla^2_{\vtheta} \Loss(\vtheta, \vz)$
while letting $\vtheta := \frac{\vw}{\normtwosm{\vw}}$.
See \Cref{fig:arch-std-resnet} for a visualization of the architecture.

\myparagraph{Scale-Invariant ResNet-20.}
The standard ResNet-20
can be scale-invariant after making minor changes.
First, we remove the affine parameters in all existing BNs.
Then, we also add an extra BN (with trainable affine parameters)
to each shortcut for blocks with downsampling,
following \citet{li2020exp}.
Finally, we add a BN layer
after the last fully-connect layer,
in which the affine parameters are fixed to $\gamma = \frac{3}{10} \ln 91 \approx 1.353$, $\beta = 0$ (see \Cref{sec:affine} for discussion).
A visualization of the full architecture is given in \Cref{fig:arch-fsi-resnet}, whose output function can be shown to be scale-invariant with respect to the trainable parameters.

\subsection{Additional Details of Linear Regression Experiments}

Our experiments on linear regression (\Cref{fig:lin-idea})
follows strictly as \Cref{sec:lin}.
For generating data,
we first sample the ground-truth weight $\vwgt \in \sphS^{d-1}$ uniformly,
where $d = 40$,
and we sample the ground-truth bias as $\bgt \sim \Normal(0, 0.01)$,
Then we sample $n = 20$ points from Gaussian distribution
$\Normal(\vzero, \mSigma)$,
where $\mSigma = \diag(1/d, 2/d, 3/d, \dots, 1)$.

To generate \Cref{fig:lin-idea},
we run gradient descent with LR $\heta = 0.5$
and WD $\hlam = 2 \times 10^{-4}$.
This particular choice of hyperparameters is for better visualization of the periodic behavior
described in \Cref{sec:proof-idea}.
If we enlarge LR or WD, the period will be shortened accordingly.

\subsection{Computing Spherical Sharpness}

To compute the spherical sharpness at a point $\vw \in \R^D \setminus \{\vzero\}$ (\Cref{def:sph-sha}),
we utilize the formula
$\lambda_1(\nabla^2 \Loss(\frac{\vw}{\normtwosm{\vw}})) = \normtwosm{\vw}^2 \cdot \lambda_1(\nabla^2 \Loss(\vw))$,
which can be proved for all scale-invariant loss by simple calculus (see also \Cref{lm:scale-invariant-grad-hess}).
For computing the eigenvalues of $\nabla^2 \Loss(\vw)$,
we invoke the Lanczos algorithm from SciPy library (\texttt{scipy.sparse.linalg.eigsh}),
where a Hessian-vector product oracle $\vx \mapsto \nabla^2 \Loss(\vw) \vx$ is implemented with PyTorch
and passed to the Lanczos algorithm as a linear operator to find its eigenvalues.

In full-batch training of the full CIFAR-10 dataset,
it is time-consuming to compute even a single Hessian-vector product.
Following~\citet{cohen2021gradient},
we compute Hessian based on
only the first $5000$ training data points.

Note that in our theory
we focus on the spherical sharpness
of minimizers.
However, in practice, it is usually time-consuming
to compute the minimizers exactly,
so in matrix completion and CIFAR-10 experiments,
we compute the spherical sharpness directly 
at the current parameter $\vw_t$.
But for linear regression experiments (\Cref{fig:lin-idea}),
the spherical sharpness is indeed computed at minimizers
in each step because
our computational power is sufficient in this setting.
Specifically,
we compute the minimizer by
doing projected gradient descent on $\sphS^{D-1}$
with fixed learning rate $0.005$
until the loss decreases to $10^{-8}$.

\section{Discussion on the Affine Parameters of the Final BN} \label{sec:affine}

In our experiments for scale-invariant models,
we add a BN to the final linear layer and fix the affine parameters to be constants.
Now we analyze its effects on 
regression tasks with squared loss and
classification tasks with crossentropy loss,
and discuss how to set the affine parameters reasonably.

We consider a scale-invariant neural net in general while assuming that a BN is put as the last layer.
For input $\vx$ and parameter $\vw$,
let $(F_1(\vx; \vw), \dots, F_C(\vx; \vw)) \in \R^C$ be the output before the final BN.
After the final BN, the output is given by the following function $\Phi_k(\vw; \vw)$:
\begin{align*}
    \bar{F}_k(\vx; \vw) &:= \frac{F_k(\vx; \vw) - \mu_k}{\sigma_k}, \\
    \Phi_k(\vx; \vw) &:= \gamma_k \bar{F}_k(\vx; \vw) + \beta_k,
\end{align*}
where $\mu_k$ and $\sigma_k^2$ are mean and variance of $\{ F_k(\vx_i; \vw) \}_{i=1}^{n}$
over the training set.

\subsection{Squared Loss}
For linear regression experiments,
we use squared loss as the loss function,
and we have $C = 1$ because the output is a scalar.
Then as discussed in \Cref{sec:lin}, the mean and variance of $\Phi_1(\vx; \vw)$
are always $\beta_1$ and $\gamma_1^2$ regardless of $\vw$.
Therefore, to ensure that $\Phi$ has sufficient representation power to express the regression targets $\{y_i\}_{i=1}^{n}$,
we should fix the affine parameters in a way
that $\beta_1$ and $\gamma_1^2$ match with the mean and variance of $\{y_i\}_{i=1}^{n}$.

In the general case when $C \ge 1$, the $k$-th output unit has mean and variance $\beta_k$ and $\gamma_k^2$.
We can argue similarly that 
we should fix the affine parameters in a way
that $\beta_k$ and $\gamma_k^2$ match with the mean and variance of $\{y_{i,k}\}_{i=1}^{n}$,
where $y_{i,k}$ stands for the $k$-th coordinate of the regression target for the $i$-th data point.
In other words,
\begin{align*}
  \beta_k \gets \muy^{(k)} := \frac{1}{n} \sum_{i=1}^{n} y_{i,k},  & & \gamma_k \gets \sigmay^{(k)} := \sqrt{\frac{1}{n} \sum_{i=1}^{n} (y_{i,k} - \muy^{(k)})^2}.
\end{align*}

\subsection{Crossentropy Loss}
In our CIFAR-10 experiments, we fix all $\gamma_k$ to a constant value $\gamma$,
and fix $\beta_k$ to zero.
We note that fixing these affine parameters in this way for an overparameterized model
is equivalent to adding label smoothing to the training loss in some sense.

The basic idea is to assume the model has sufficient express any function $F$,
then $\Phi_k$ can be any function of mean zero and covariance $\gamma^2$ on the training set.
If the dataset
has $C$ classes and each class has equal numbers of samples,
with some efforts one can show that
the minimum loss is 
attained
when
the pre-softmax logit for an input
is $\gamma (C-1)^{1/2}$ for the correct class
and $-\gamma (C-1)^{-1/2}$ for every wrong class.
In this case,
the output probability $P_y(\vx_i)$ for a class $y$ is
\begin{align} \label{eq:pxy-opt}
  P_y(\vx_i) = \begin{cases}
    \frac{1}{1 + (C-1)\exp(-\gamma((C-1)^{1/2} + (C-1)^{-1/2}))} & \quad \text{if } y = y_i; \\
    \frac{1}{(C-1) + \exp(\gamma((C-1)^{1/2} + (C-1)^{-1/2}))} & \quad \text{otherwise.} \\
  \end{cases}
\end{align}
Label smoothing is a regularization technique proposed by \citet{szegedy2016rethinking}
to improve the generalization of neural nets
by replacing the one-hot hard labels
with soft labels that assign a probability of $\epsilon/C$ to each wrong class.
To see the connection to label smoothing,
we can let $\epsilon := 
    \frac{C}{(C-1) + \exp(\gamma((C-1)^{1/2} + (C-1)^{-1/2}))}$,
then $P_y(\vx_i) = \epsilon / C$ if $y \ne y_i$
and $P_y(\vx_i) = 1 - (1-1/C)\epsilon$ if $y = y_i$.
In our CIFAR-10 experiments,
$\gamma = \frac{3}{10} \ln 91 \approx 1.353$.
Then by simple calculation $\epsilon = 0.1$,
and therefore the neural net trained with $\gamma_k$ and $\beta_k$ fixed to $\gamma$ and $0$
is encouraged to produce 
the output probabilities to match with the soft labels.

The minimum loss is non-zero when the affine parameters are fixed.
We can see from \eqref{eq:pxy-opt}
that
the loss attained by any minimizer is 
\begin{align}
  \Loss_{\min} = \ln(1+ (C-1)\exp(-\gamma((C-1)^{1/2} + (C-1)^{-1/2}))).
\end{align}
In making plots for our experiments,
we always subtract the original loss with its theoretical minimum value $\Loss_{\min}$.
When $\gamma = \frac{3}{10} \ln 91$, this minimum value is $\Loss_{\min} = \ln \frac{100}{91} \approx 0.0943$.

Choosing $\gamma$ to be $\frac{3}{10} \ln 91$ is mainly because of its connection to
label smoothing with $\epsilon = 0.1$,
which is the value chosen in \citep{szegedy2016rethinking}.
But other choices of $\gamma$ can also lead to the same sharpness-reduction bias in practice;
see \Cref{fig:cifar2k-vgg-ab-final-gamma}.

\begin{figure}[htbp] 
  \centering
  \includegraphics[width=\textwidth]{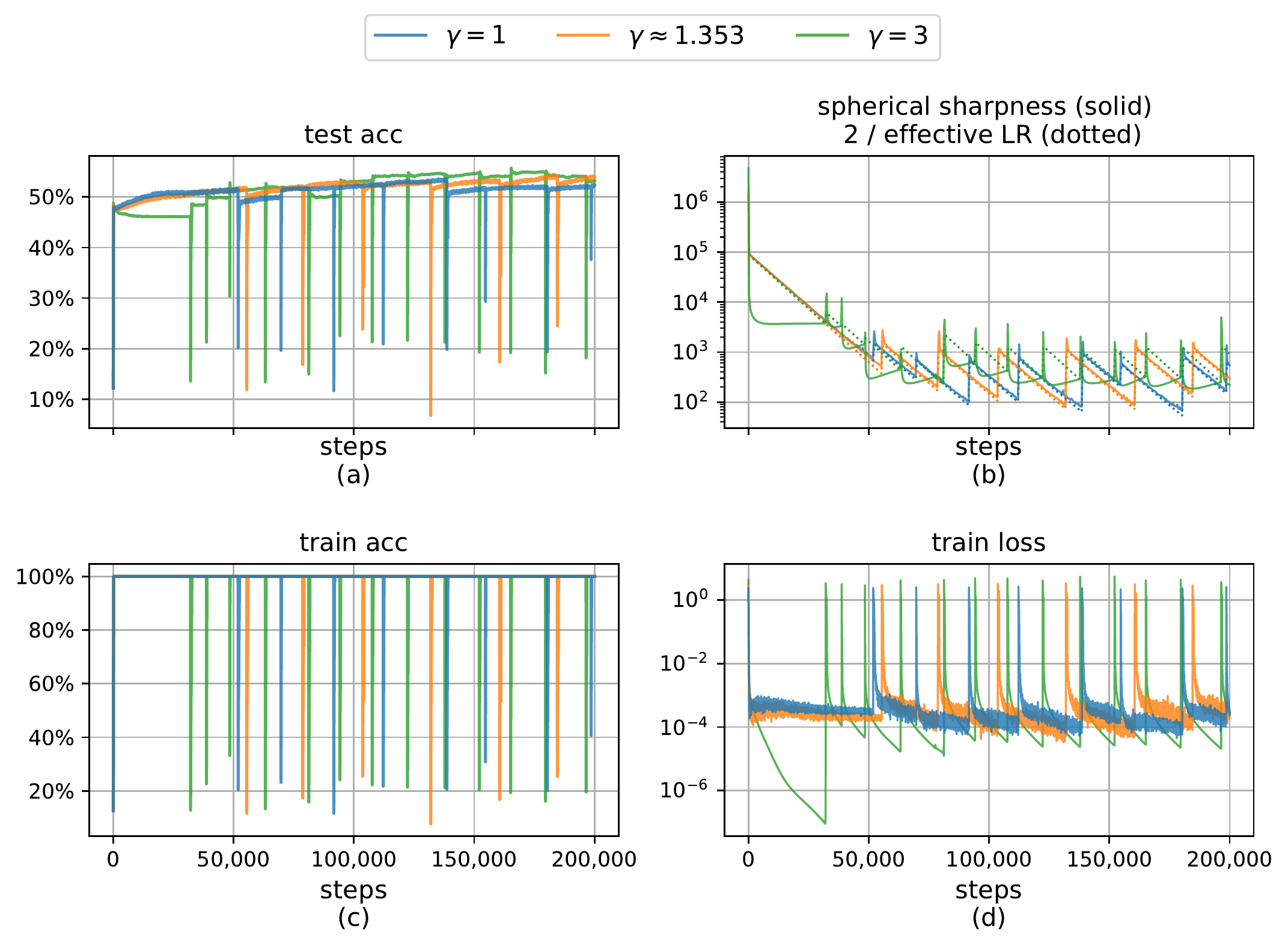}
  \vspace{-0.15in}
  \caption{
    Training scale-invariant VGG-11 on CIFAR-10-2k with full-batch GD
    (LR $\heta = 0.1$, WD $\hlam = 5 \times 10^{-4}$)
    while fixing the rescaling parameter $\gamma$ in the final BN to different values.
    The sharpness-reduction bias can be observed in all cases,
    but for larger $\gamma$ the periodic behavior is more likely to happen at a macroscopic scale.
  }
  \label{fig:cifar2k-vgg-ab-final-gamma}
\end{figure}